\documentclass[twoside,11pt]{article}
\usepackage{jair,theapa,rawfonts}  %
\ShortHeadings{Multi-Label Classification Neural Networks with Hard Logical Constraints}
{Giunchiglia \& Lukasiewicz}
\firstpageno{1}

\usepackage{subcaption}
\usepackage{amsfonts}
\usepackage{bbm}
\usepackage{rotating}
\usepackage{multirow}
\usepackage{amsmath,amsthm}
\usepackage{booktabs} 
\usepackage{tikz}
\usetikzlibrary{arrows.meta}
\usetikzlibrary{shapes}
\usepackage{wrapfig}
\usepackage{multicol}

\newcommand{\system}[1]{{\rm CCN}(#1)}

\newcommand{\problem}{$\mathcal{P}$}
\newcommand{\classes}{\mathcal{A}}
\newcommand{\tv}{\theta}
\newcommand{\module}{\text{\rm CM}}
\newcommand{\loss}{\text{\rm CLoss}}
\newcommand{\lss}{\mathcal{L}}

\newcommand{\auprc}{$AU(\ov{PRC}) $}
\newcommand{\lmlp}{{HMC-LMLP}}
\newcommand{\ens}{{Clus-Ens}}
\newcommand{\hmcnf}{{HMCN-F}}
\newcommand{\hmcnr}{{HMCN-R}}

\newcommand{\br}{BR}
\newcommand{\ecc}{ECC}
\newcommand{\camel}{CAMEL}
\newcommand{\rakel}{RAKEL}
\newcommand{\clus}{{Clus-HMC}}

\newcommand{\hmcsys}[1]{{\rm C-HMCNN}(#1)}

\newcommand{\ov}[1]{\overline{#1}}
\newcommand{\bck}{\backslash}
\newcommand{\cmc}{LCMC}

\definecolor{orange}{rgb}{1.0, 0.6, 0.2}
\definecolor{darkblue}{rgb}{0.0, 0.0, 1.0}
\definecolor{darkgreen}{rgb}{0.0, 0.42, 0.24}
\definecolor{darkred}{rgb}{0.8, 0.0, 0.0}

\newtheorem{theorem}{Theorem}[section]

\newtheorem{corollary}[theorem]{Corollary}

\newtheorem{example}[theorem]{Example}
\theoremstyle{definition} 
\newtheorem{definition}[theorem]{Definition}

\begin{document}

\title{Multi-Label Classification Neural Networks with\\  Hard Logical Constraints}
\author{\name Eleonora Giunchiglia \email eleonora.giunchiglia@cs.ox.ac.uk \\
       \name Thomas Lukasiewicz \email thomas.lukasiewicz@cs.ox.ac.uk \\
       \addr University of Oxford, Department of Computer Science, \\
       Wolfson Building, Oxford, Oxfordshire, UK
}

\maketitle

\begin{abstract}
Multi-label classification 
(MC) is a standard machine learning problem in which a data point can be associated with a set of classes. A more challenging scenario is given by hierarchical multi-label classification (HMC) problems, in which every prediction must satisfy a given set of hard constraints expressing 
subclass relationships between classes. In this paper, we propose \hmcsys{$h$}, a novel approach for solving HMC problems, which, given a network $h$ for the underlying MC problem, exploits the hierarchy information in order to produce predictions coherent with the
constraints and to improve performance. 
Furthermore, we extend 
the logic used to express 
HMC constraints in order to be 
able to specify more complex relations among the classes and
propose a new model \system{$h$}, which extends  \hmcsys{$h$} and is 
again able to 
satisfy and exploit the 
constraints to improve performance.
We conduct an extensive experimental analysis showing the superior performance of both \hmcsys{$h$} and \system{$h$} when compared to state-of-the-art models in both 
the HMC and the general MC setting
with hard logical constraints.
\end{abstract}

\section{Introduction}

Multi-label classification (MC) is a standard machine learning problem in which a data point can be associated with 
a set of 
classes. A more challenging scenario is given by hierarchical multi-label classification (HMC) problems, in which every prediction must satisfy a given set of hard
hierarchy constraints of the form 
\begin{equation}\label{eq:basic_constr}
A_1 \to A,
\end{equation}
expressing that 
$A_1$ is a subclass of $A$, i.e., that 
if a data point is associated with the class~$A_1$, then it is also associated with the class $A$.
HMC problems naturally arise in many domains, such as image classification~\cite{imagenet,dimitrovski2008,dimitrovski2011}, text categorization~\cite{klimt2004,lewis04,rouso2006}, and functional genomics~\cite{barutcuoglu2006,clare2003,vens2008}. They are very challenging for two main reasons: (i) they are normally characterized by a great class imbalance, because the number of data points per class is usually much smaller at deeper levels of the hierarchy, and (ii) the predictions must be coherent
with (i.e., satisfy) the hierarchy constraints. 
Consider, e.g., the task proposed in~\cite{dimitrovski2008}, where a radiological image has to be annotated with an IRMA code specifying,  among others, the biological system examined. In this setting, we expect to have many more {\sl abdomen} images  
 than {\sl stomach} images, making the class {\sl stomach} harder to predict. Furthermore, the  prediction \{{\sl stomach}\} alone should not be possible given the constraint
 \begin{equation}\label{eq:exhm}
 \text{\sl stomach} \to \text{\sl gastrointestinalSystem}, 
 \end{equation}
  stating that the stomach is part of the gastrointestinal system, i.e., that whenever {\sl stomach} is predicted, also {\sl gastrointestinalSystem} should be. 
  Many models have been specifically developed for HMC problems, and we can
 distinguish those that 
 directly output predictions that are coherent with the hierarchy constraints (see, e.g., \cite{kwok2011,masera2018}) from those that allow incoherent predictions and, at inference time,  require an additional post-processing step to ensure their satisfaction (see, e.g., \cite{cerri2014,obozinski2008,valentini2011}). Most of the state-of-the-art HMC models based on neural networks belong to the second category (see, e.g.,~\cite{cerri2014,cerri2016,cerri2018}),
and different post-processing techniques can be applied in order to guarantee the coherency of their outputs with the constraints (see, e.g., \cite{obozinski2008}).

In this paper, we first focus on 
HMC problems, and we propose a novel approach for solving them, called {\sl coherent hierarchical multi-label classification neural network} (\hmcsys{$h$}), which, given a network $h$ for the underlying MC problem, exploits the hierarchy information to produce  predictions coherent with the hierarchy 
constraints and improve performance. 
\hmcsys{$h$} is
based on two basic elements: 
\begin{enumerate}
    \item a constraint layer built on top of $h$,  which extends 
    to the upper classes
    the predictions made by $h$ 
    on the lower classes in the hierarchy,
    in order to ensure that the final outputs are coherent by construction with the 
    hierarchy constraints,
 and 
 \item a loss function teaching \hmcsys{$h$}
 when to exploit the hierarchy constraints, i.e., when the prediction on the lower classes in the hierarchy can be exploited to make predictions also for the upper ones.
\end{enumerate}
\hmcsys{$h$} significantly differs from previous approaches for HMC problems based on neural networks. Indeed, the constraint layer is not a simple post-processing meant to guarantee the satisfaction of the hierarchy constraints, decoupled from the rest of the system. In \hmcsys{$h$}, the constraint layer and the underlying neural network $h$ are tightly integrated, and it does not make sense to modify the constraint layer without 
modifying the way in which $h$ is trained.
\hmcsys{$h$} has the following four features: (i) its predictions are coherent without any post-processing, (ii) differently from other state-of-the-art models (see, e.g.,~\cite{cerri2018}), its number of parameters is independent from the number of hierarchical levels, (iii)  it can be easily implemented on GPUs using standard libraries, and (iv) it outperforms the state-of-the-art models {\ens}~\cite{schietgat2010}, {\lmlp}~\cite{cerri2016}, {\hmcnf}, and {\hmcnr} \cite{cerri2018} on 20 commonly used real-world HMC benchmarks. 

Secondly, we extend the language used to express the hierarchy constraints (\ref{eq:basic_constr}) to allow for the specification of more complex 
logical
relations among classes. 
Indeed, the language for expressing hierarchy constraints is very limited, and it is not expressive enough to model, e.g.,  the fact that if a medical image contains the abdomen but neither the middle nor the upper abdomen, then it contains the lower abdomen. 
Thus, borrowing concepts from the area of logic programming, we consider general constraints expressed as normal rules \cite{lloyd}, i.e., expressions of the form: 
\begin{equation} \label{eq:rule-}
A_1, \ldots, A_k, \neg A_{k+1}, \ldots, \neg A_n \to A, \qquad (0 \leq k \leq n),
\end{equation}
which imposes that whenever the classes  $A_1, \ldots, A_k$ are predicted, while $A_{k+1}, \ldots, A_n$ are not, then also the class $A$ should be predicted.
With such an extension, we can now write: 
$$
\text{\it abdomen}, \neg \text{\it middleAbdomen}, \neg \text{\it upperAbdomen} \to \text{\it lowerAbdomen},
$$
capturing the above informally stated constraint.
We call MC problems with a set of constraints in such an extended syntax {\sl logically constrained multi-label classification} ({\cmc}) problems.
By restricting to constraints with stratified negation \cite{aptBW88}, given a set $\mathcal{H}$  of initial predictions made by an underlying model $h$, we show how at inference time 
it is possible to compute in linear time in the number of constraints the unique minimal set of classes $\mathcal{M}$ that
\begin{enumerate}
    \item extends $\mathcal{H}$, i.e., such that $\mathcal{H} \subseteq \mathcal{M}$, and
    \item is coherent with (satisfies) the constraints, i.e., such that, given (\ref{eq:rule-}), $A \in \mathcal{M}$ whenever 
    $\{A_1, \ldots, A_k\} \subseteq \mathcal{M}$ and $\{A_{k+1}, \ldots, A_n\} \ \cap \mathcal{M} = \emptyset$.
\end{enumerate}
Indeed, for a non-stratified set of constraints expressed as normal rules, there can be no or more than one 
minimal
set of classes having the above two properties, and determining the 
non-existence or computing one of them can take exponential time. 
We thus propose a novel model called {\sl coherent-by-construction network} {\system{$h$}}, which is the first model able to deal with MC problems with such expressive constraints on the classes. {\system{$h$}} has the same two basic ingredients of {\hmcsys{$h$}}: 
\begin{enumerate}
    \item a constraint layer built on top of $h$, which extends the predictions made by $h$ in order to ensure that the predictions are coherent by construction with the constraints, and 
 \item a loss function, teaching \system{$h$} 
 when to exploit the constraints, i.e., in the presence of (\ref{eq:rule}), when to exploit the prediction on $\{A_1,\ldots,A_n\}$ to make predictions on $A$.
\end{enumerate}
In {\system{$h$}}, like in \hmcsys{$h$}, the constraint layer and $h$ are tightly integrated, and the result is a system that significantly differs from what we consider the standard approach to {\cmc} problems, consisting in applying the constraint layer as a simple post-processor to a state-of-the-art MC system.
{\system{$h$}} has four distinguishing features:
(i) its predictions are always coherent with the constraints,
(ii) it can be implemented on GPUs using standard libraries,
(iii) it extends {\hmcsys{$h$}}, and thus outperforms the state-of-the-art HMC models on HMC problems, and 
(iv) it outperforms standard approaches based on the state-of-the-art MC systems BR~\cite{boutell2004}, ECC~\cite{read2009},  RAKEL~\cite{tsoumakas2009CorrelationBasedPO}, and CAMEL~\cite{feng2019}
on 16 {\cmc} problems, each corresponding to a commonly used MC benchmark.

From a higher perspective, the core idea behind 
our approach is  (i) to build models based on neural networks 
in order to leverage
their learning abilities, (ii) to incorporate the constraints in the models themselves in order to guarantee their coherency with the constraints by construction, and (iii)  to exploit the background knowledge expressed by the constraints by suitably modifying the loss function in order to improve performance. 
As such, our approach represents a valid alternative to the currently deployed techniques for certifying that a neural network model behaves correctly with respect to a given set of requirements. Such certification process --- see the survey by \citeauthor{huang2020survey} \citeyear{huang2020survey} --- is mandatory especially in safety-critical applications, and is currently based on (i) verification techniques (see, e.g.,~\citeauthor{tacchella2010,lomuscio2017}), which suffer from a scalability problem, or (ii) testing techniques (see, e.g.,~\cite{pei2019,ma2018}), which cannot give any guarantee that the model does always satisfy the constraints. Our approach, on the contrary, presents neither of the above limitations.

The main contributions of this paper can thus be briefly summarized as follows:
\begin{itemize}
\item We propose a 
novel model for hierarchical multi-label classification (HMC) problems, denoted \hmcsys{$h$}, which is built upon two tightly integrated components: a constraint layer  ensuring the coherency with the hierarchy constraints and a loss function teaching \hmcsys{$h$} when to exploit the hierarchy constraints.
\item We prove that \hmcsys{$h$}'s predictions are   guaranteed to be coherent with the hierarchy constraints, and that its number of parameters is independent from the number of hierarchical levels.
\item We show that \hmcsys{$h$} can be  implemented on GPUs using standard libraries, and, through an extensive experimental analysis, that it outperforms the state-of-the-art models on HMC problems on 20 commonly used real-world HMC benchmarks.
\item We extend HMC problems by allowing for  constraints written as normal rules, which are able to capture complex relations among labels. We call such problems  logically constrained multi-label classification ({\cmc}) problems.
\item We propose a novel model for {\cmc} problems, denoted \system{$h$}, whose predictions are always guaranteed to be coherent with the constraints.
\item We demonstrate that \system{$h$}
is an extension of \hmcsys{$h$}:  \system{$h$} is thus based on the same two tightly integrated components (a constraint layer and a loss function) and, given an HMC problem, \system{$h$}  outperforms the state-of-the-art models on HMC problems as well.

\item We show that \system{$h$} can be implemented on GPUs using standard libraries,  and, through an extensive experimental analysis, that it outperforms state-of-the-art 
multi-label classification (MC) models with post-processing on {\cmc} problems on 16 commonly used real-world MC benchmarks.
\end{itemize}

The rest of this paper is organized as follows.  
In Section~\ref{sec:hmc}, we first focus on HMC problems, and we propose our  model \hmcsys{$h$}. In Section~\ref{sec:mc_with_constr}, we consider more expressive constraints and present our model \system{$h$}, which extends \hmcsys{$h$} 
to handle {\cmc} problems. The implementation of both \hmcsys{$h$} and \system{$h$}  on GPUs is presented in Section~\ref{sec:gpu}. The experimental analysis, demonstrating the superiority of our 
approach, is reported in Section~\ref{sec:exp}. We end the paper with the relevant related work in Section~\ref{sec:rel_work} and the conclusion %
in Section~\ref{sec:concl}.

\section{Hierarchical Multi-Label Classification}\label{sec:hmc}

In this section, we first introduce some basic definitions in 
hierarchical multi-label classification (HMC). 
We then describe the main intuitions underlying 
our model \hmcsys{$h$} to solve HMC problems along a simple HMC problem with just two classes, %
and we finally present our general approach to 
solve HMC problems. %

\subsection{Preliminaries}\label{sec:prelim}

We assume that every {\sl multi-label classification} ({\sl MC}) problem {\problem}$\,=(\classes,{\mathcal X})$ consists of a finite set {$\classes$} of 
{\sl classes} (also called {\sl class labels} or simply {\sl labels}), denoted by $A, A_1, A_2, \ldots$, and a finite set ${\mathcal X}$ of %
pairs $(x,y)$ where $x \in \mathbb{R}^D (D \geq 1)$ is a {\sl data point}, and $y \subseteq \classes$ is the {\sl ground truth} of $x$, i.e., the set of classes associated with $x$.
A {\sl model} $m$ for {\problem} is a function $m(\,\cdot\,,\,\cdot\,)$ mapping every class~$A$ and every data point $x\in \mathbb{R}^D$ to $[0,1]$. For every class $A$, the function $m_A\colon \mathbb{R}^D \to [0,1]$ is defined by $x\mapsto m(A,x)$, for every data point $x\in \mathbb{R}^D$.
A~data point $x \in \mathbb{R}^D$ is 
{\sl predicted} by $m$ to belong to class $A$ whenever 
$m_A(x)$ is greater than a user-defined {\sl threshold} $\tv \in [0,1]$.

A {\sl hierarchical multi-label classification} ({\sl HMC}) problem $($\problem$,\Pi)$ consists of an MC problem {\problem} and a finite set $\Pi$ of {\sl (hierarchy) constraints} of the form %
\begin{equation}\label{eq:basic_constr-2}
A_1 \to A,
\end{equation}
where $A_1$ and $A$ are classes, such that the graph with an edge from $A_1$ to $A$ for each such constraint 
in $\Pi$ is acyclic. 
Informally, given an HMC problem $($\problem$,\Pi)$, a model $m$ for $($\problem$,\Pi)$ has to be coherent with the hierarchy constraints $\Pi$ in {\problem}, i.e., $m$ has to predict $A$ whenever it predicts~$A_1$, for each constraint (\ref{eq:basic_constr-2}) in $\Pi$. This is formally defined as follows. 

\begin{definition}
 Let  $($\problem$,\Pi)$ be an HMC problem. Let $m$ be a model for {\problem}. 
If for a data point and for a constraint $A_1 \to A$ in $\Pi$, $m$ predicts $A_1$ but not $A$, then $m$ commits a {\sl logical violation}. If $m$ commits no logical violations, then $m$ is  {\sl coherent with respect to $\Pi$}.
\end{definition}

Given the above, whenever a model $m$ is not guaranteed to satisfy a constraint (\ref{eq:basic_constr-2}), $m$ is extended with 
a post-processing step  
to enforce $m_{A}(x) > \tv$ whenever $m_{A_1}(x) > \tv$  \cite{cerri2014,obozinski2008,valentini2011}. However,  it is often common practice to require the stronger condition $m_{A_1}(x) \leq m_{A}(x)$, and the falsification of this condition is referred to as hierarchy violation  \cite{vens2008,cerri2018}. 

\begin{definition}
Let  $($\problem$,\Pi)$ be an HMC problem.  Let $m$ be a model for {\problem}.
If for a data point $x$ and a constraint $A_1 \to A$ in $\Pi$, $m_{A_1}(x) > {m_{A}}(x)$, then $m$ commits a {\sl hierarchy violation}.
\end{definition}

If a model 
commits no hierarchy violations, then it also  commits no logical violations (and so is coherent relative to the constraints), while the converse does not necessarily~hold.

For ease of presentation, we often omit the dependency from %
data points, %
and simply write, e.g., %
$m_A$ instead of $m_A(x)$. %

\begin{figure*}[t]
\centering
\begin{tabular}{c@{\ \ \,}c@{\ \ \ \ }|@{\ \ \ \ }c@{\ \ \,}c@{\ \ \ \ }|@{\ \ \ \ }c@{\ \ \,}c}
\multicolumn{2}{c}{Neural Network $f^+$} & \multicolumn{2}{c}{Neural Network $g^+$} &
\multicolumn{2}{c}{\hmcsys{$h$}} \\
Class $A_1$ & Class $A$ & Class $A_1$ & Class $A$ & Class $A_1$ & Class $A$ \\
\begin{minipage}{.11\textwidth}
    \includegraphics[width=\textwidth,trim={1.1cm 0 3.7cm 1cm},clip]{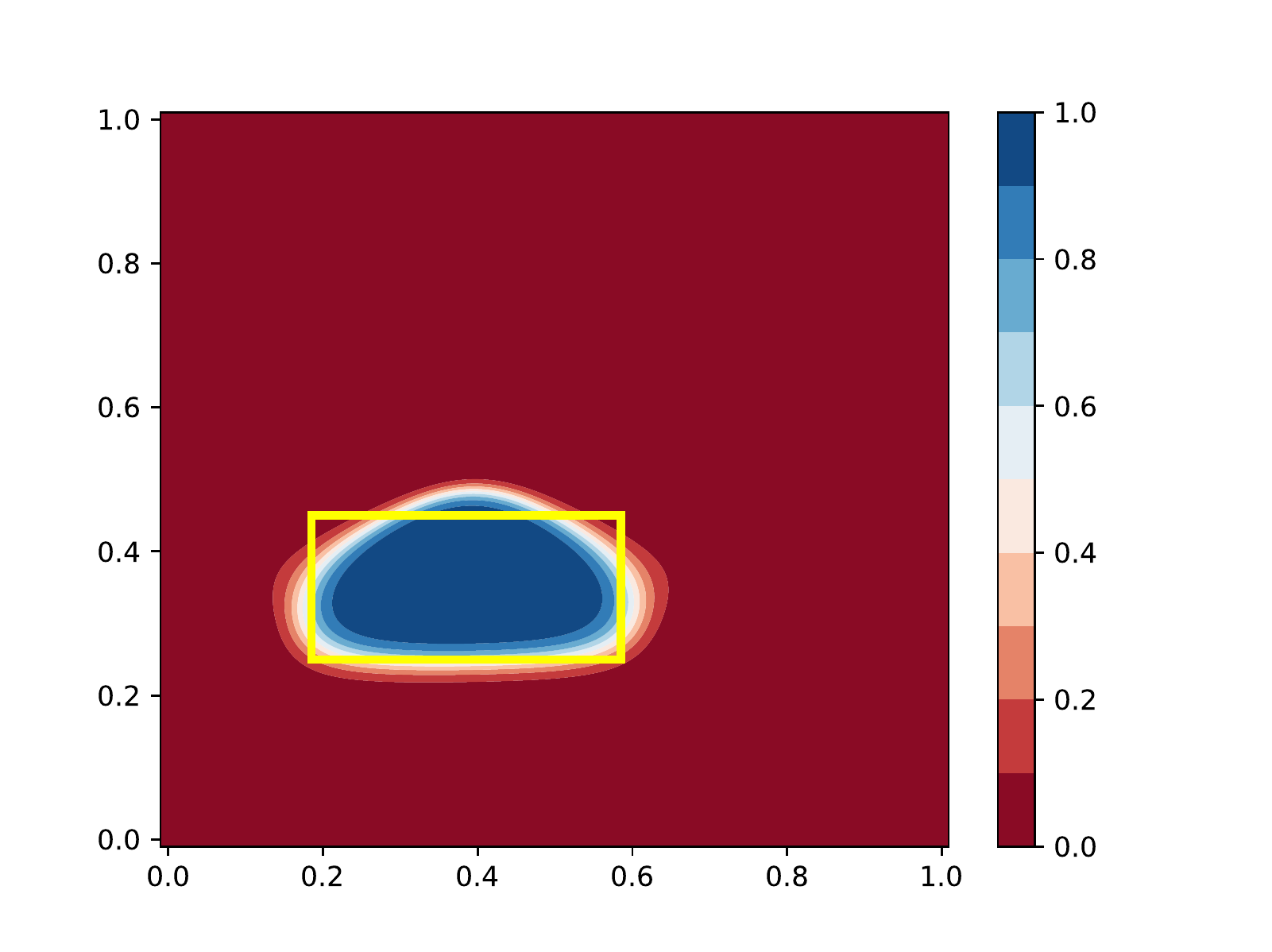}
\end{minipage} &
\begin{minipage}{.11\textwidth}
    \includegraphics[width=\textwidth,trim={1.1cm 0 3.7cm 1cm},clip]{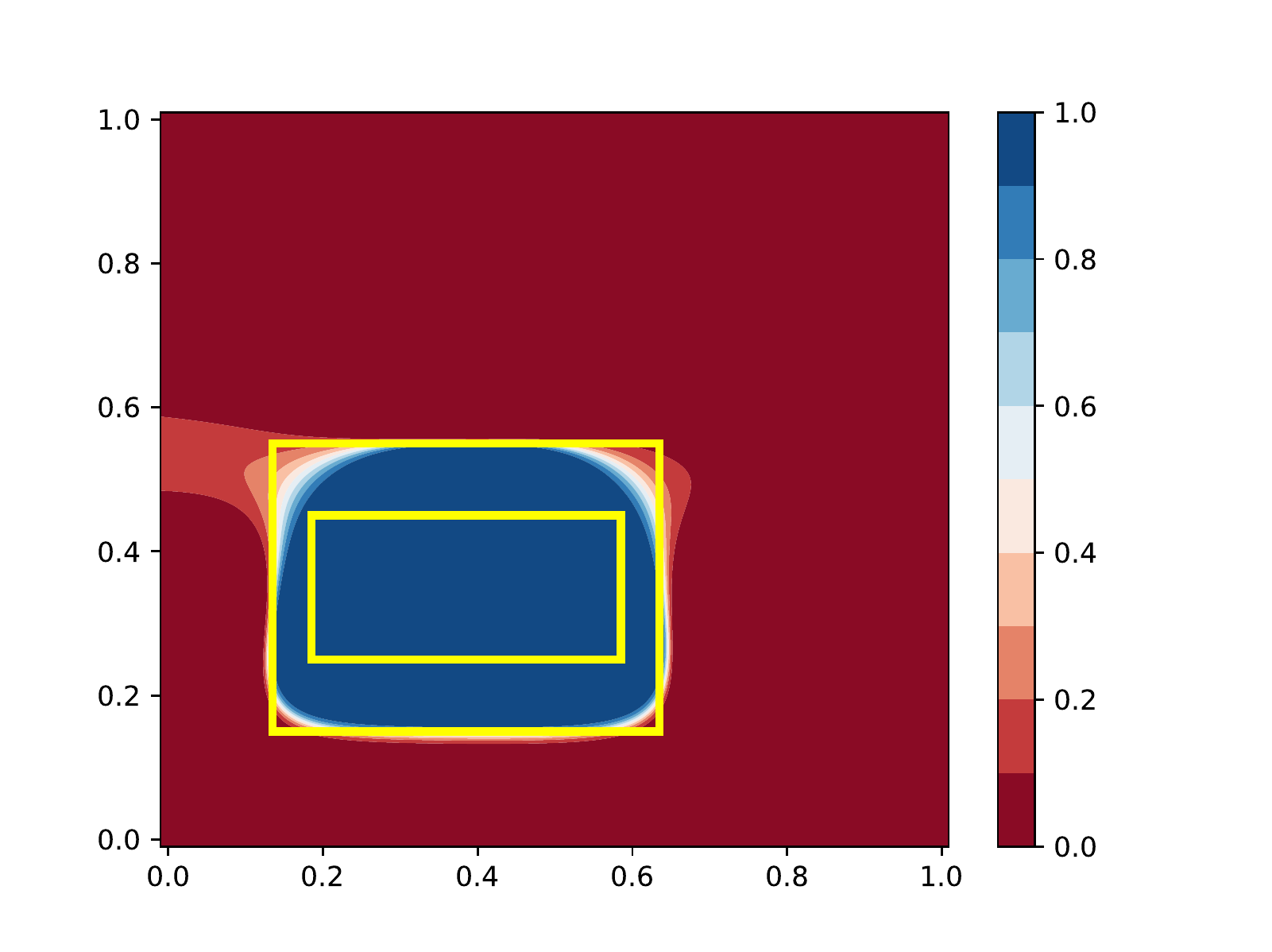} 
\end{minipage} &
\begin{minipage}{.11\textwidth}
    \includegraphics[width=\textwidth,trim={1.1cm 0 3.7cm 1cm},clip]{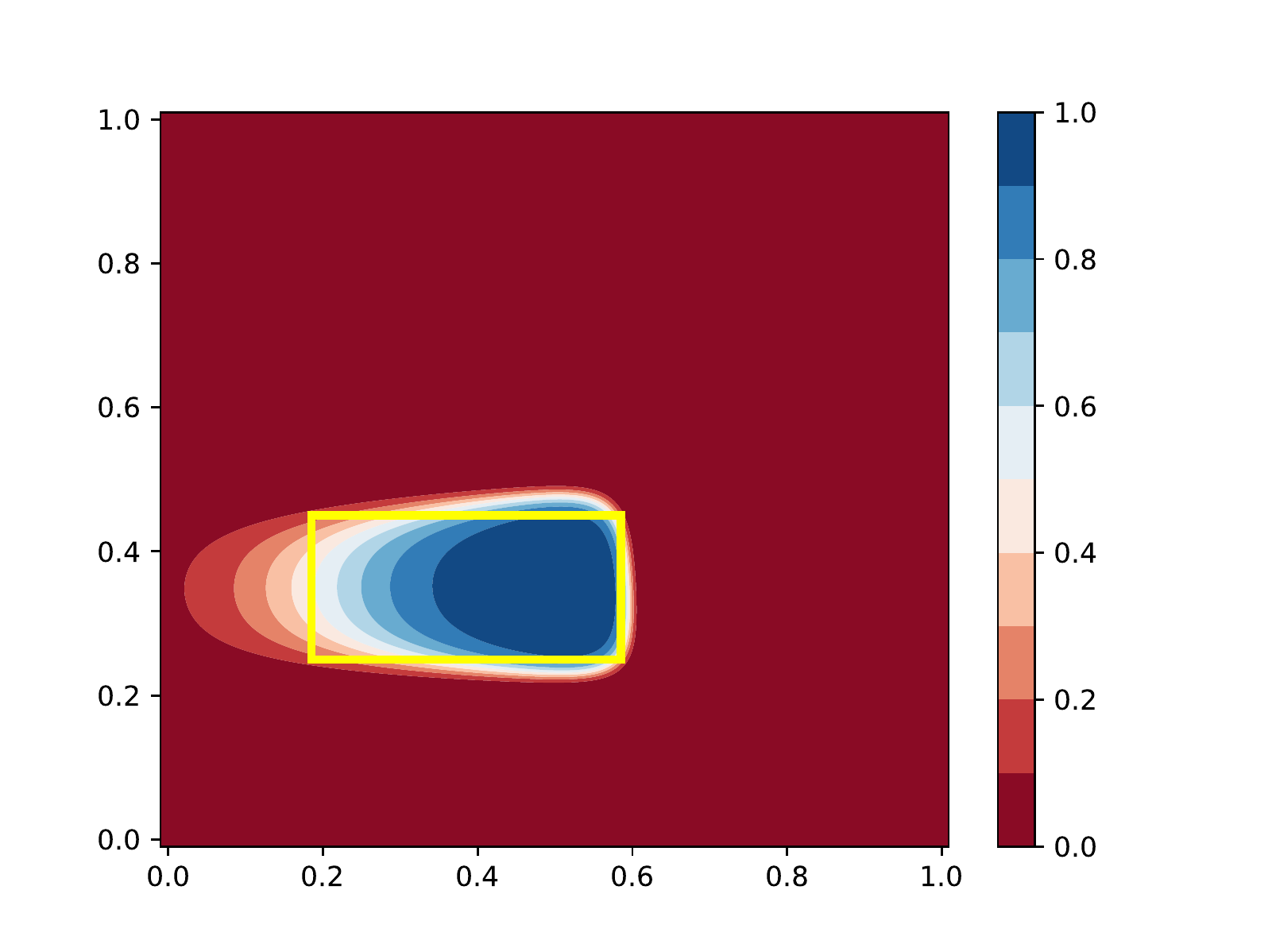} 
\end{minipage} &
\begin{minipage}{.11\textwidth}
    \includegraphics[width=\textwidth,trim={1.1cm 0 3.7cm 1cm},clip]{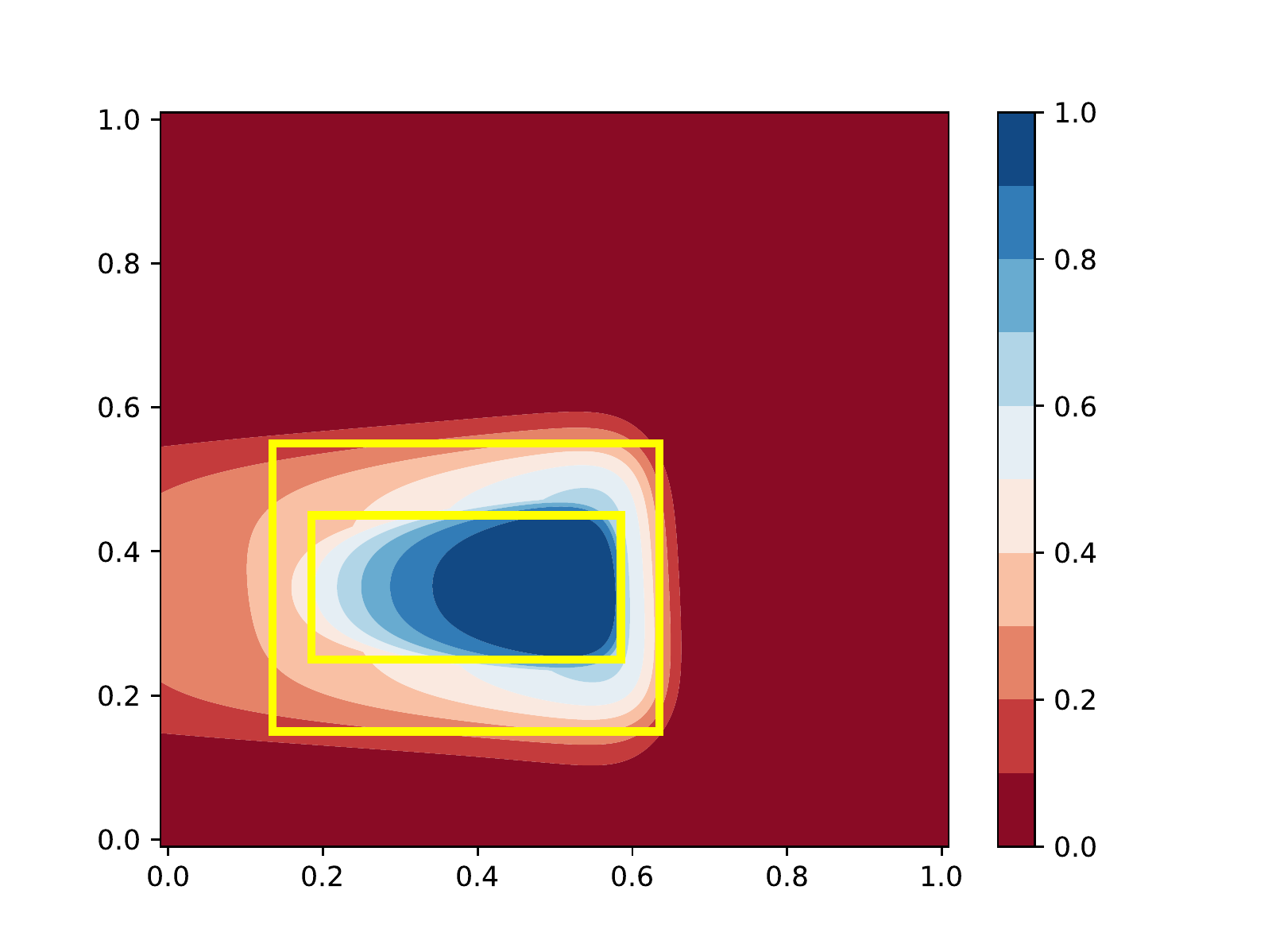} 
    \end{minipage} &
\begin{minipage}{.11\textwidth}
    \includegraphics[width=\linewidth,trim={1.1cm 0 3.7cm 1cm},clip]{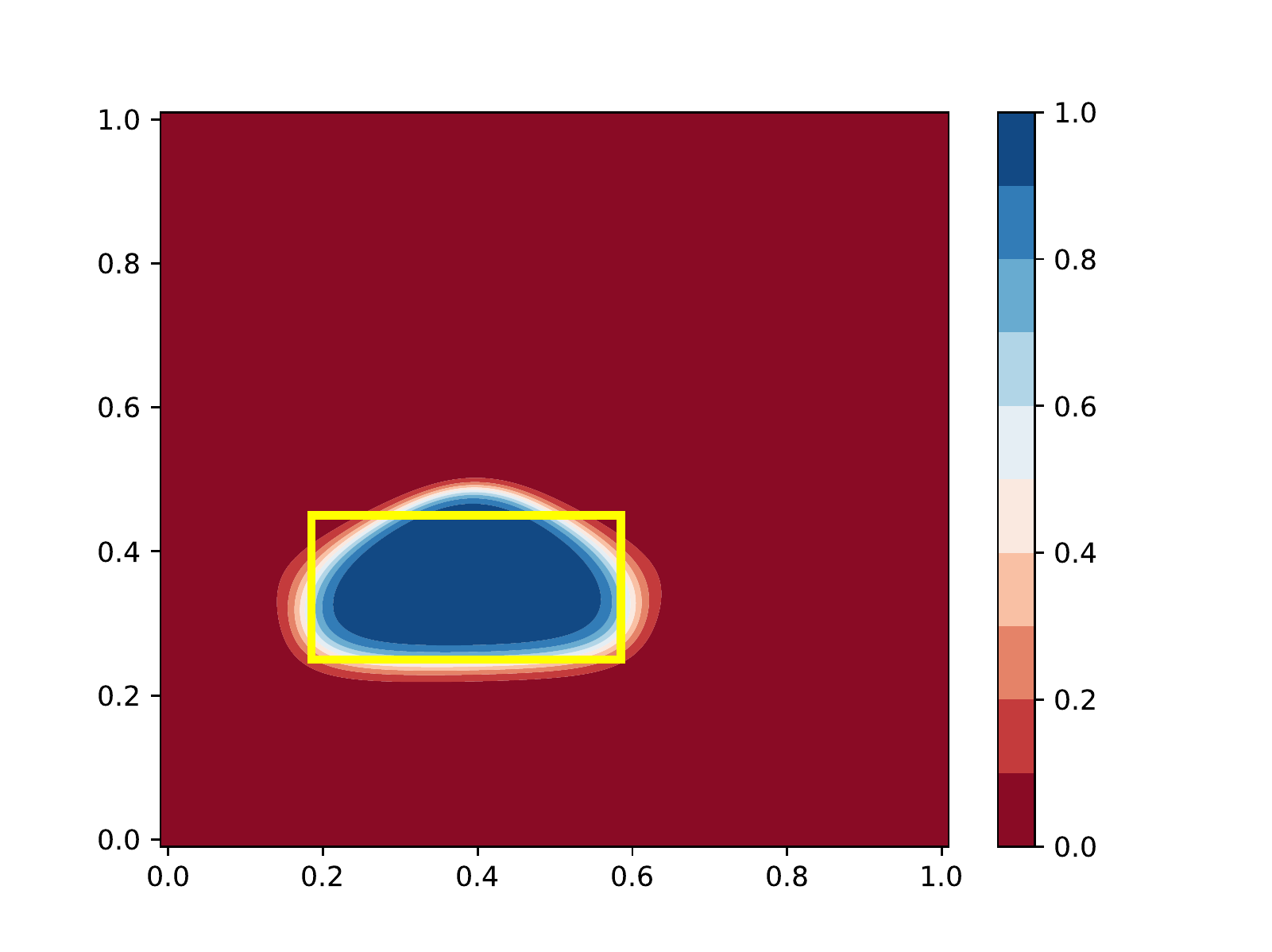}
\end{minipage} &
\begin{minipage}{.135\textwidth} 
    \includegraphics[width=\linewidth,trim={1.1cm 0 1cm 1cm},clip]{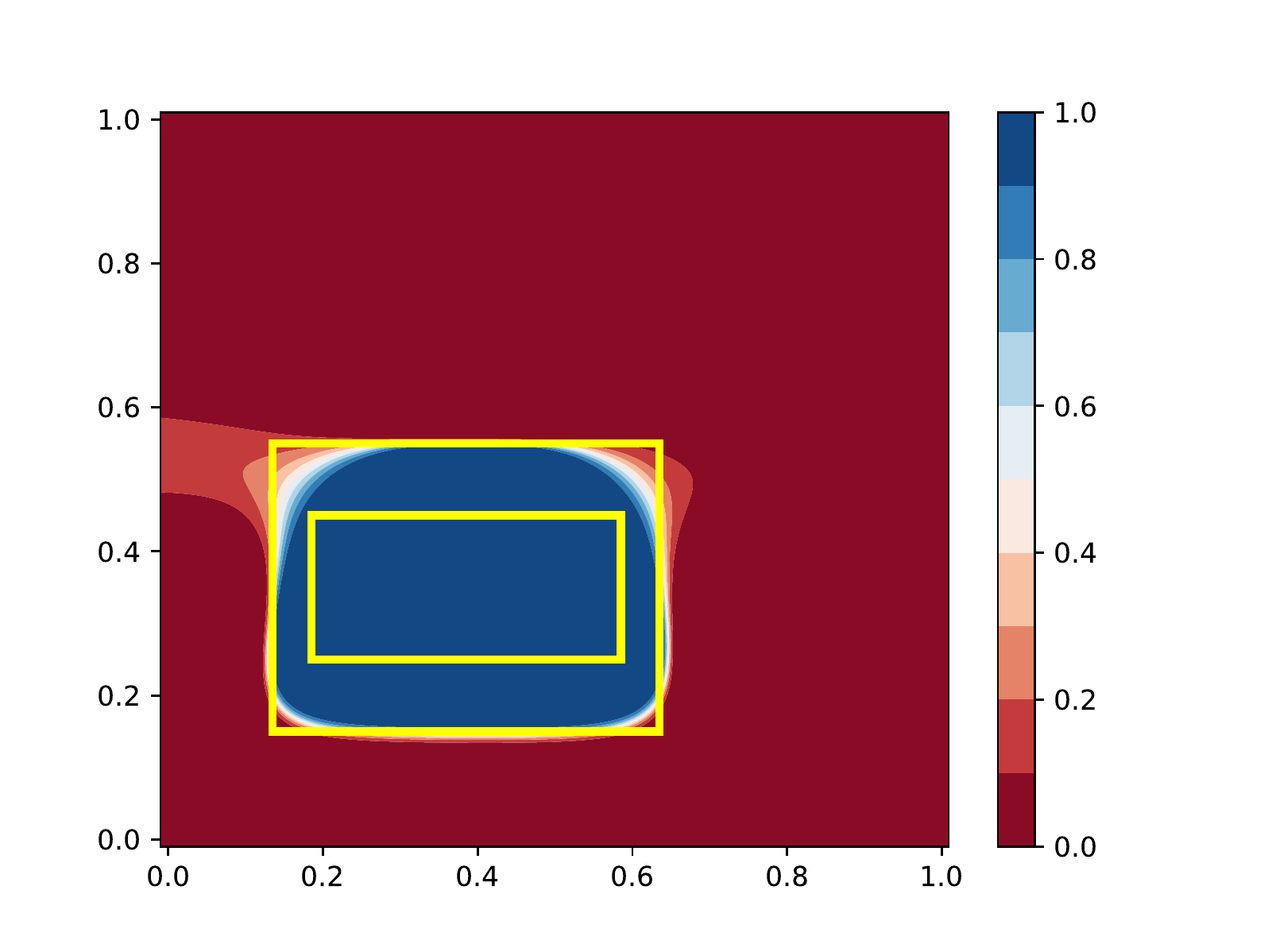}
\end{minipage} \\
\begin{minipage}{.11\textwidth}
    \includegraphics[width=\textwidth,trim={1.1cm 0 3.7cm 1cm},clip]{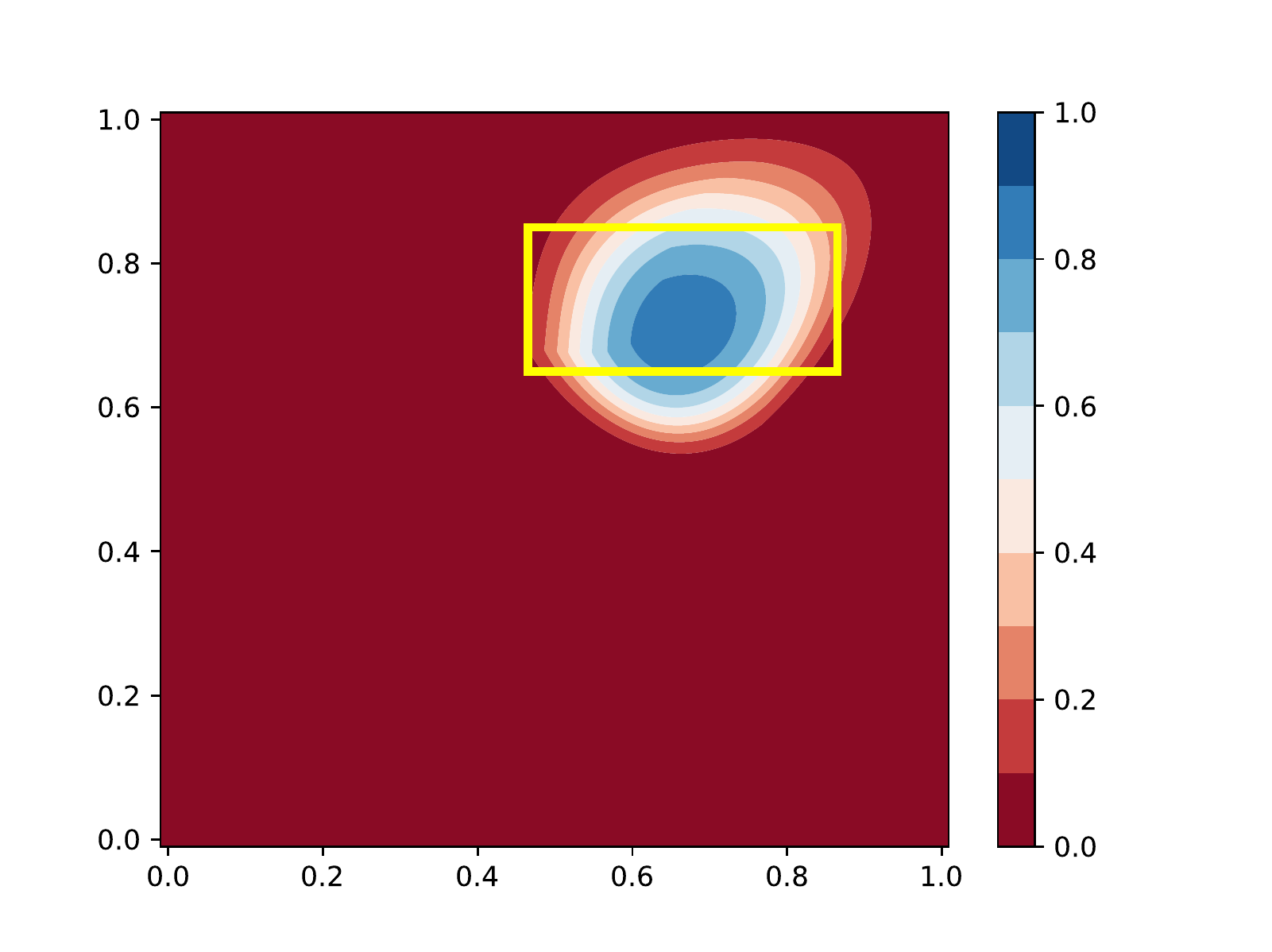}
\end{minipage} &
\begin{minipage}{.11\textwidth}
    \includegraphics[width=\textwidth,trim={1.1cm 0 3.7cm 1cm},clip]{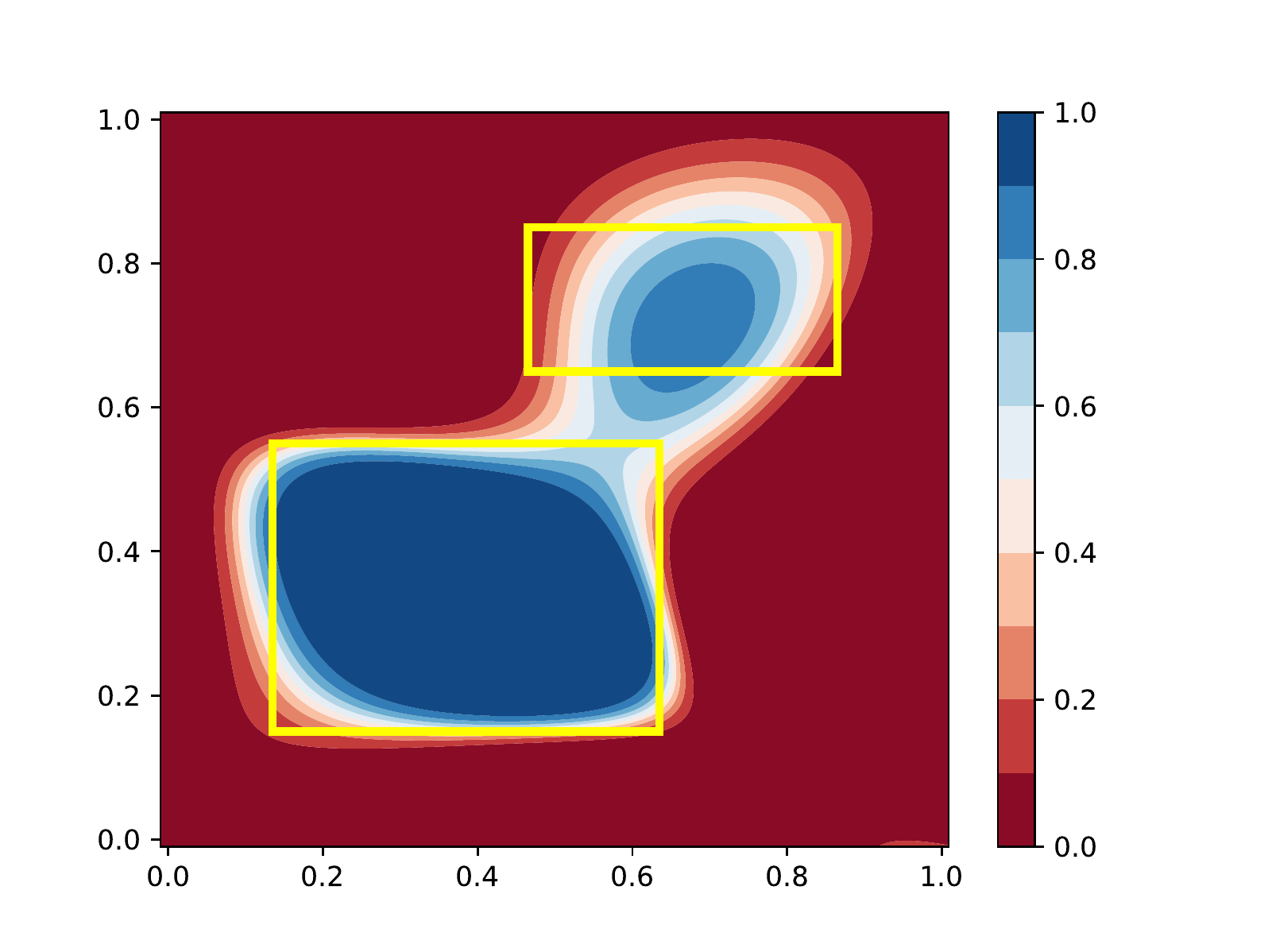}
\end{minipage} &
\begin{minipage}{.11\textwidth}
    \includegraphics[width=\textwidth,trim={1.1cm 0 3.7cm 1cm},clip]{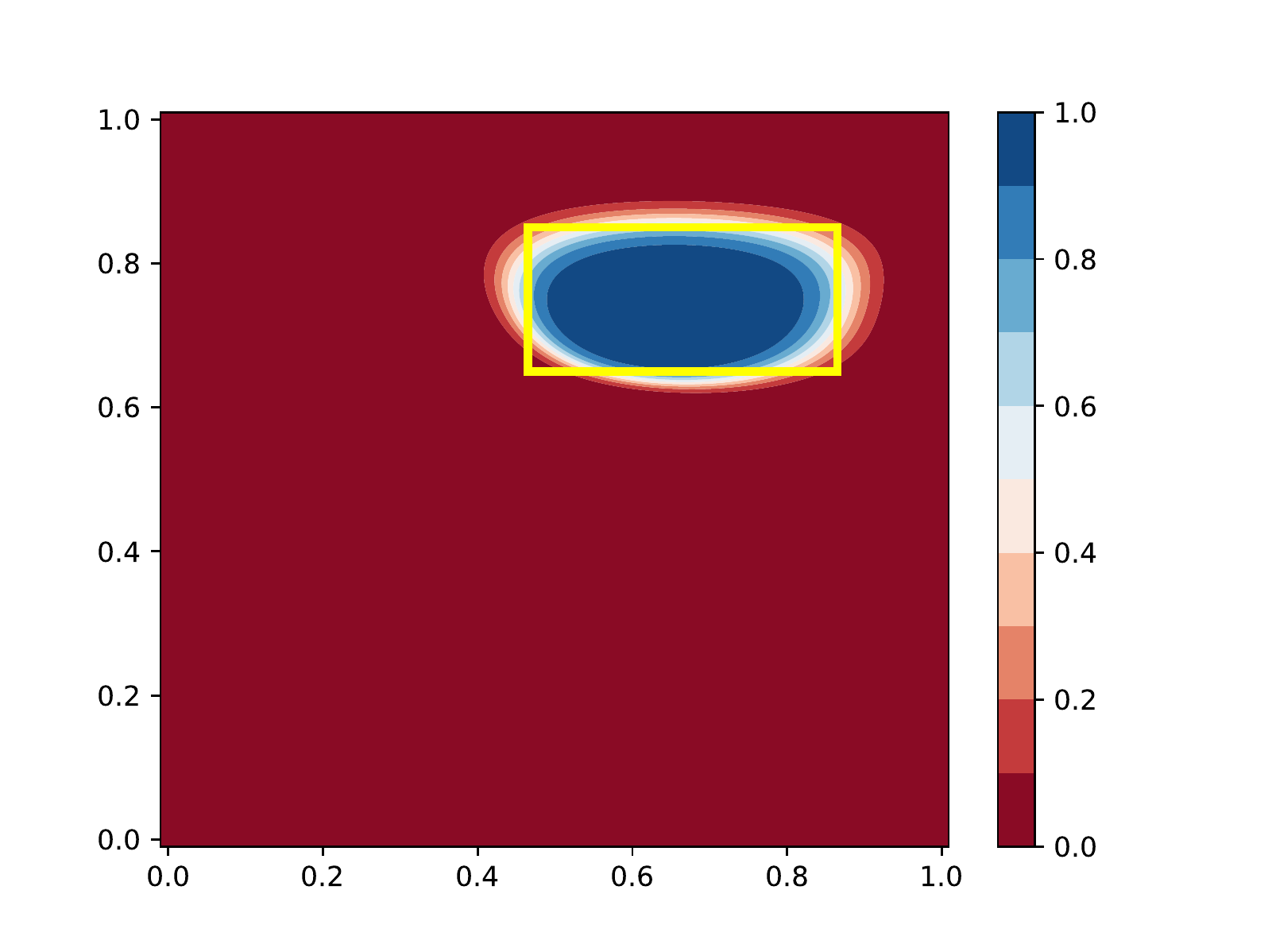}
\end{minipage} &
\begin{minipage}{.11\textwidth}
    \includegraphics[width=\textwidth,trim={1.1cm 0 3.7cm 1cm},clip]{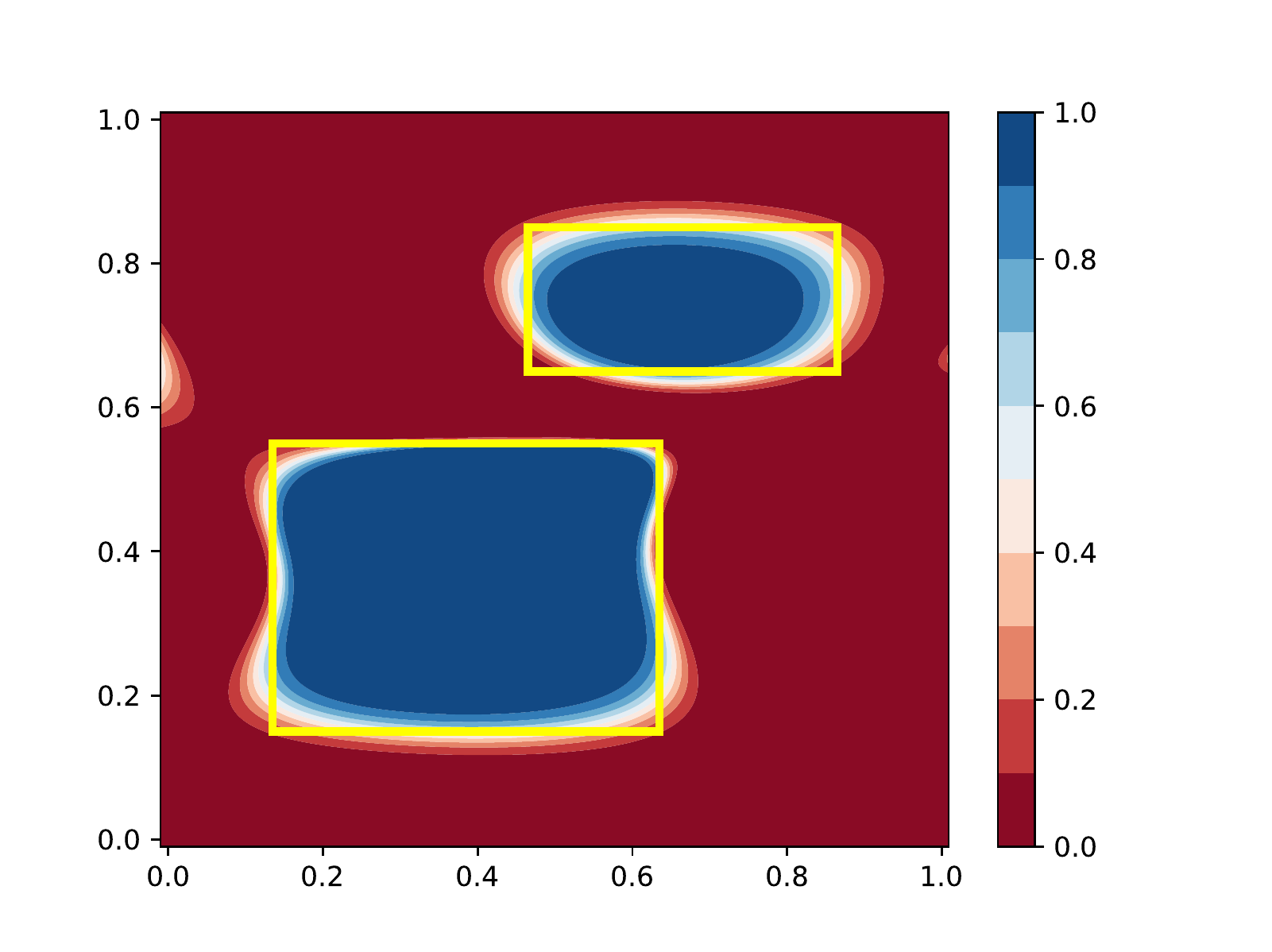}
\end{minipage} &
\begin{minipage}{.11\textwidth}
    \includegraphics[width=\linewidth,trim={1.1cm 0 3.7cm 1cm},clip]{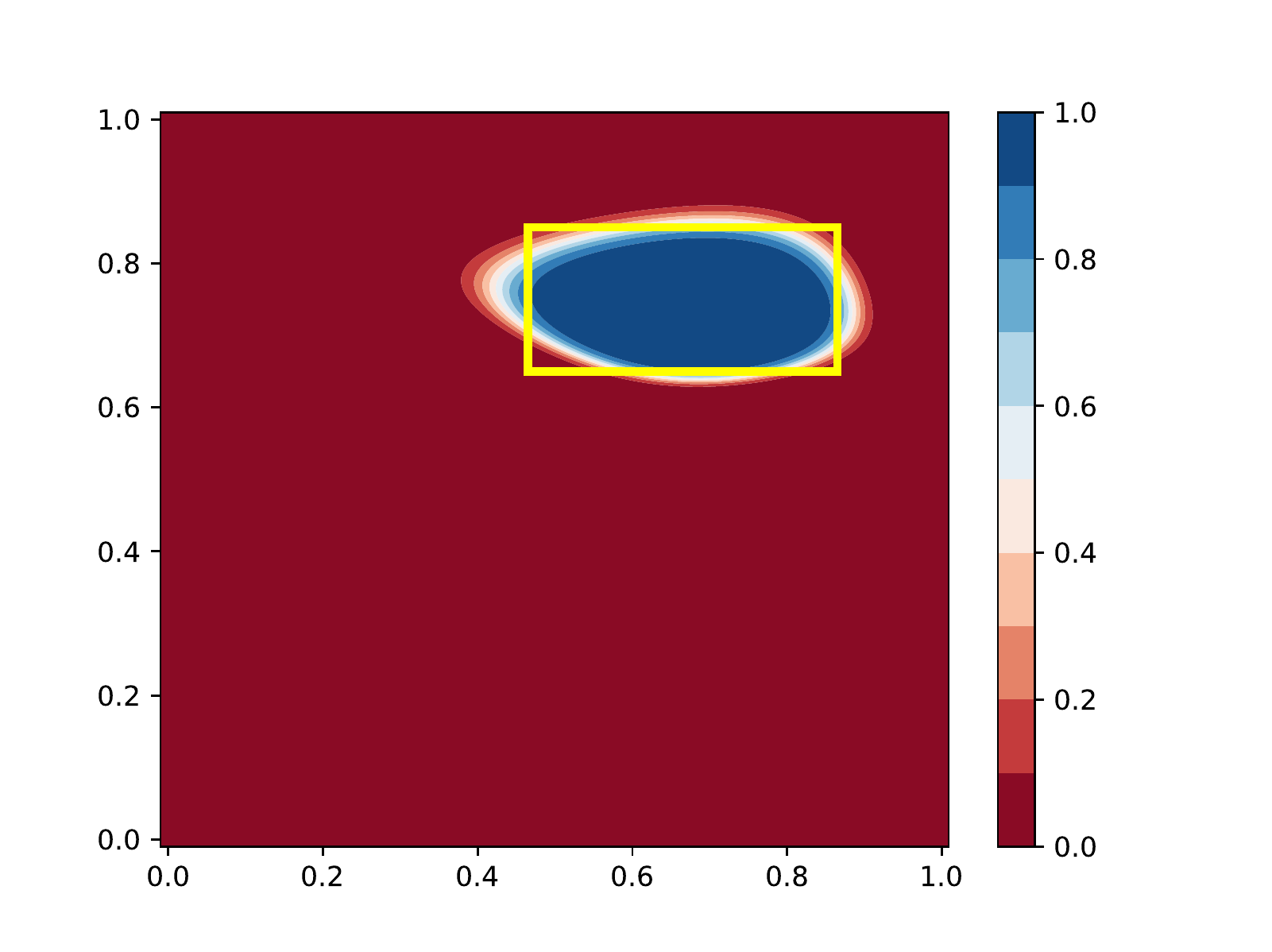}
\end{minipage} &
\begin{minipage}{.135\textwidth}
    \includegraphics[width=\linewidth,trim={1.1cm 0 1cm 1cm},clip]{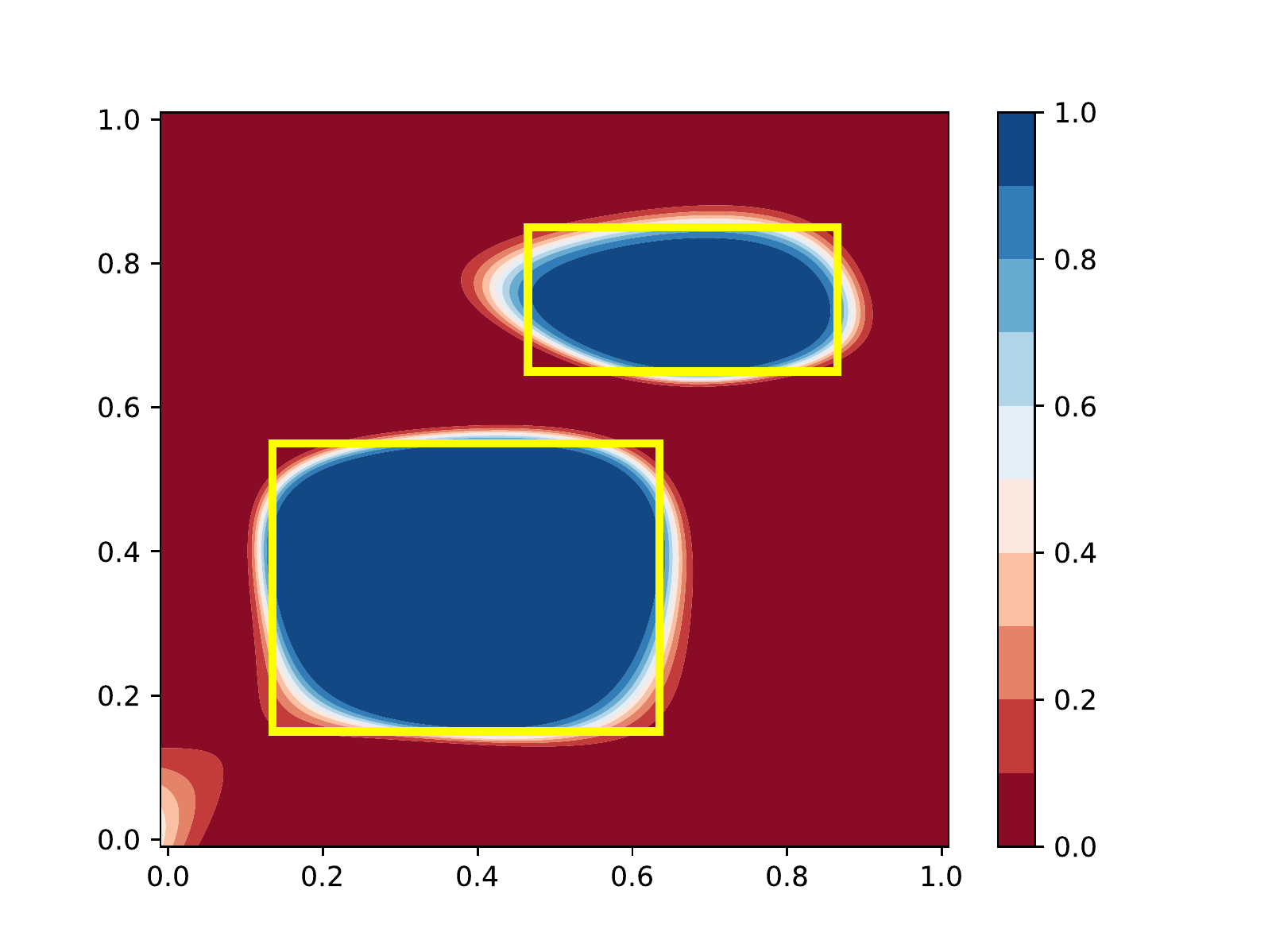}
\end{minipage} \\
\begin{minipage}{.11\textwidth}
    \includegraphics[width=\textwidth,trim={1.1cm 0 3.7cm 1cm},clip]{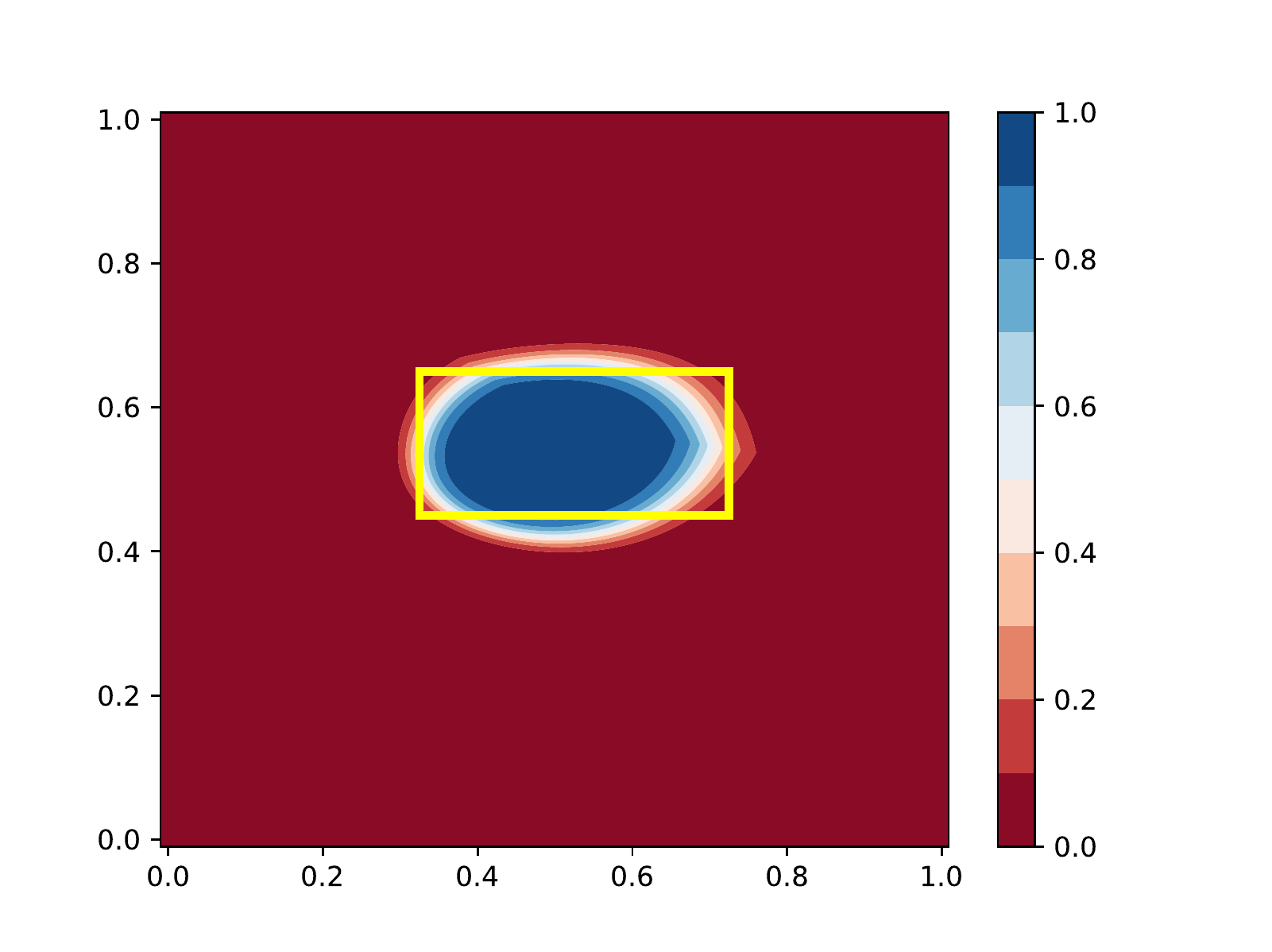}
\end{minipage} &
\begin{minipage}{.11\textwidth}
    \includegraphics[width=\textwidth,trim={1.1cm 0 3.7cm 1cm},clip]{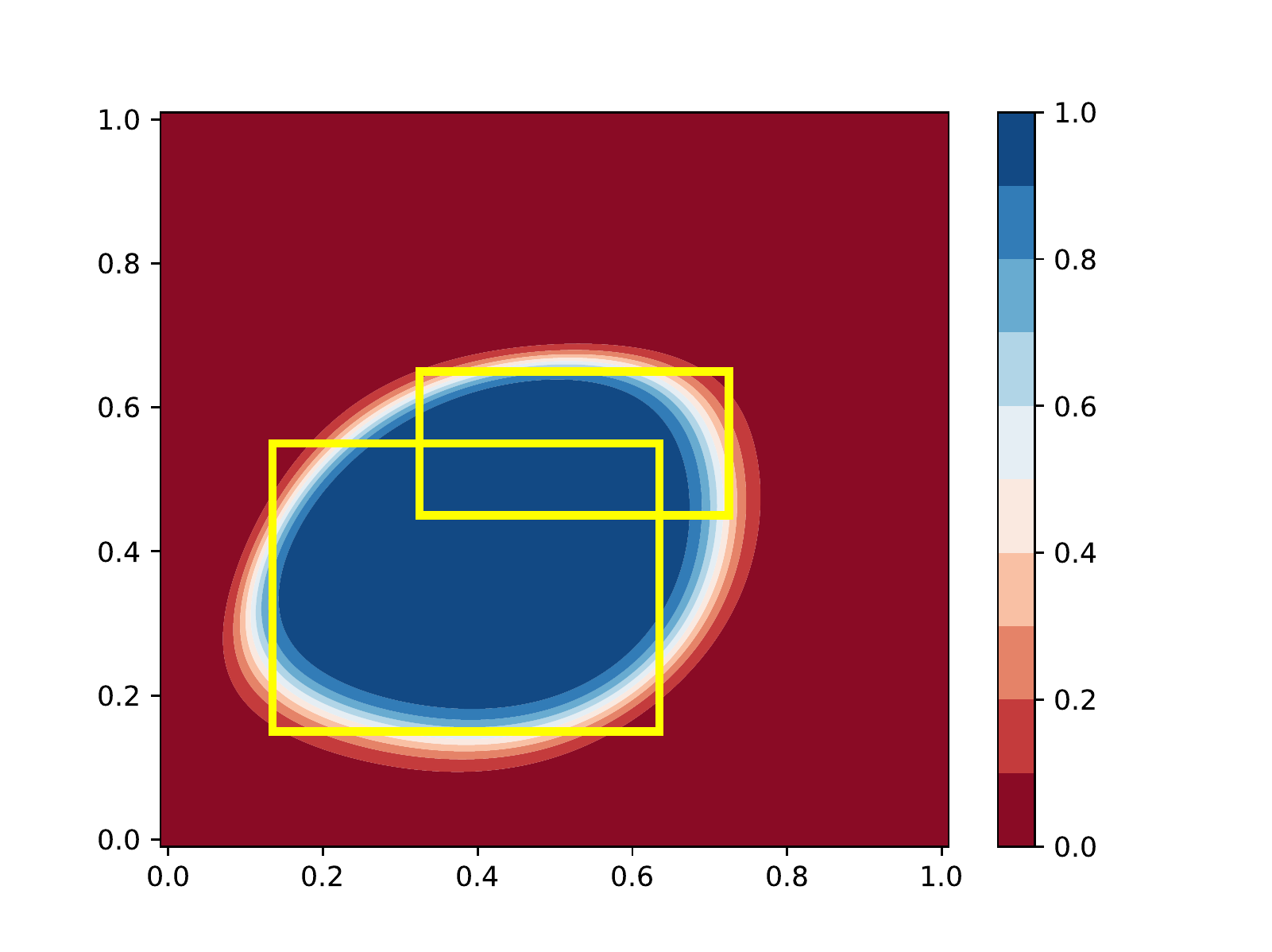}
\end{minipage} &
\begin{minipage}{.11\textwidth}
    \includegraphics[width=\textwidth,trim={1.1cm 0 3.7cm 1cm},clip]{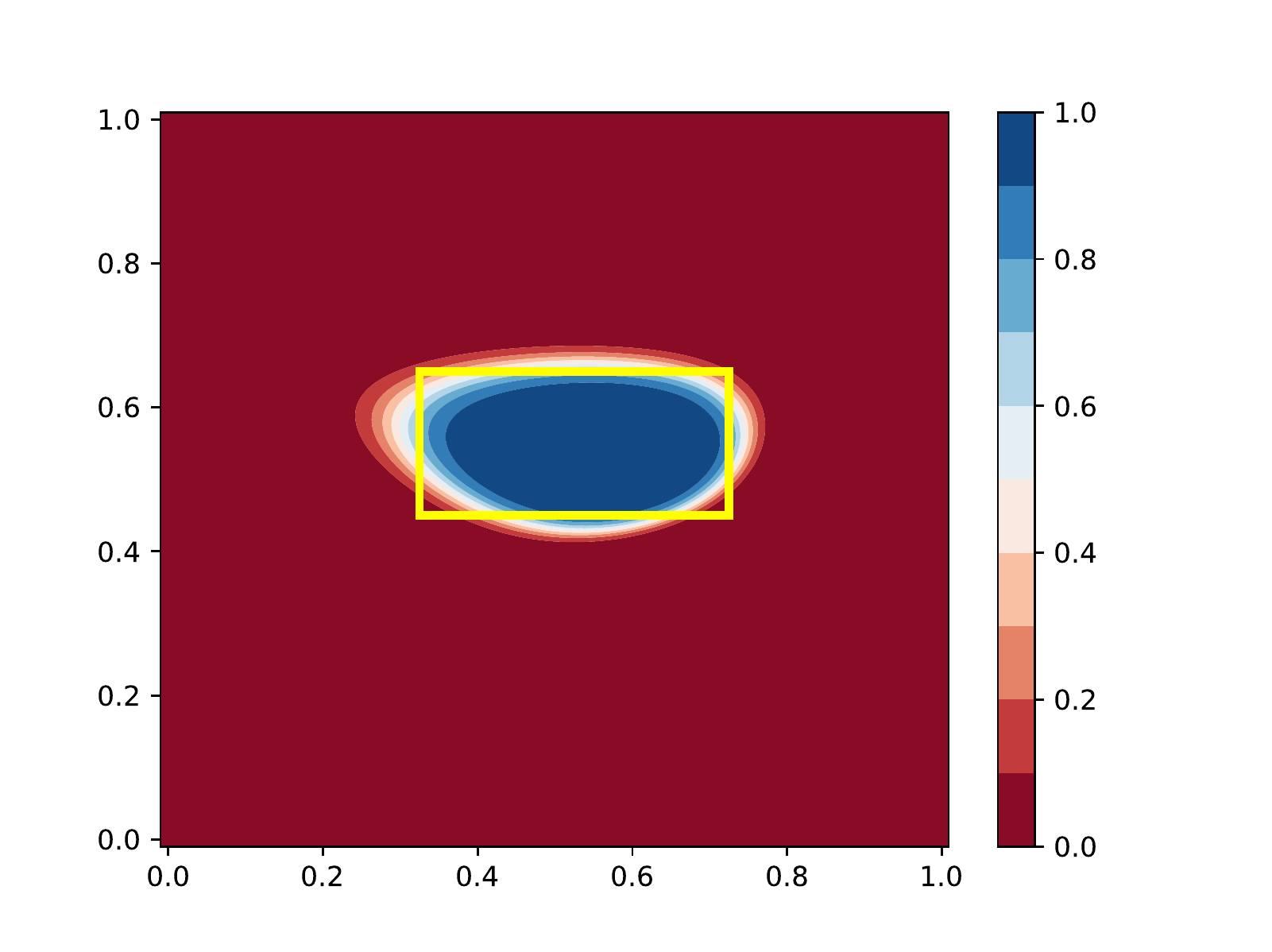}
\end{minipage} &
\begin{minipage}{.11\textwidth}
    \includegraphics[width=\textwidth,trim={1.1cm 0 3.7cm 1cm},clip]{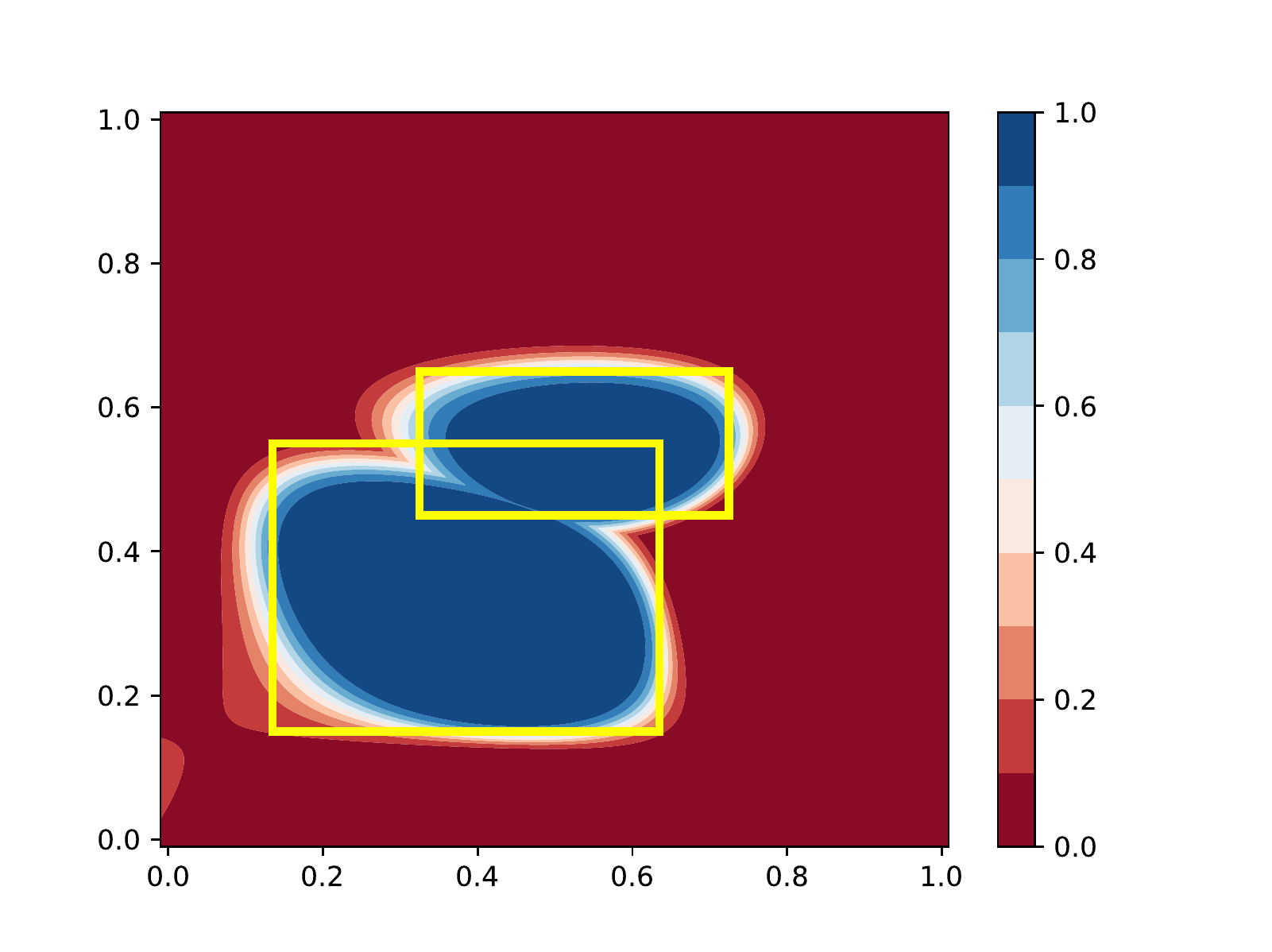}
\end{minipage} &
\begin{minipage}{.11\textwidth}
    \includegraphics[width=\linewidth,trim={1.1cm 0 3.7cm 1cm},clip]{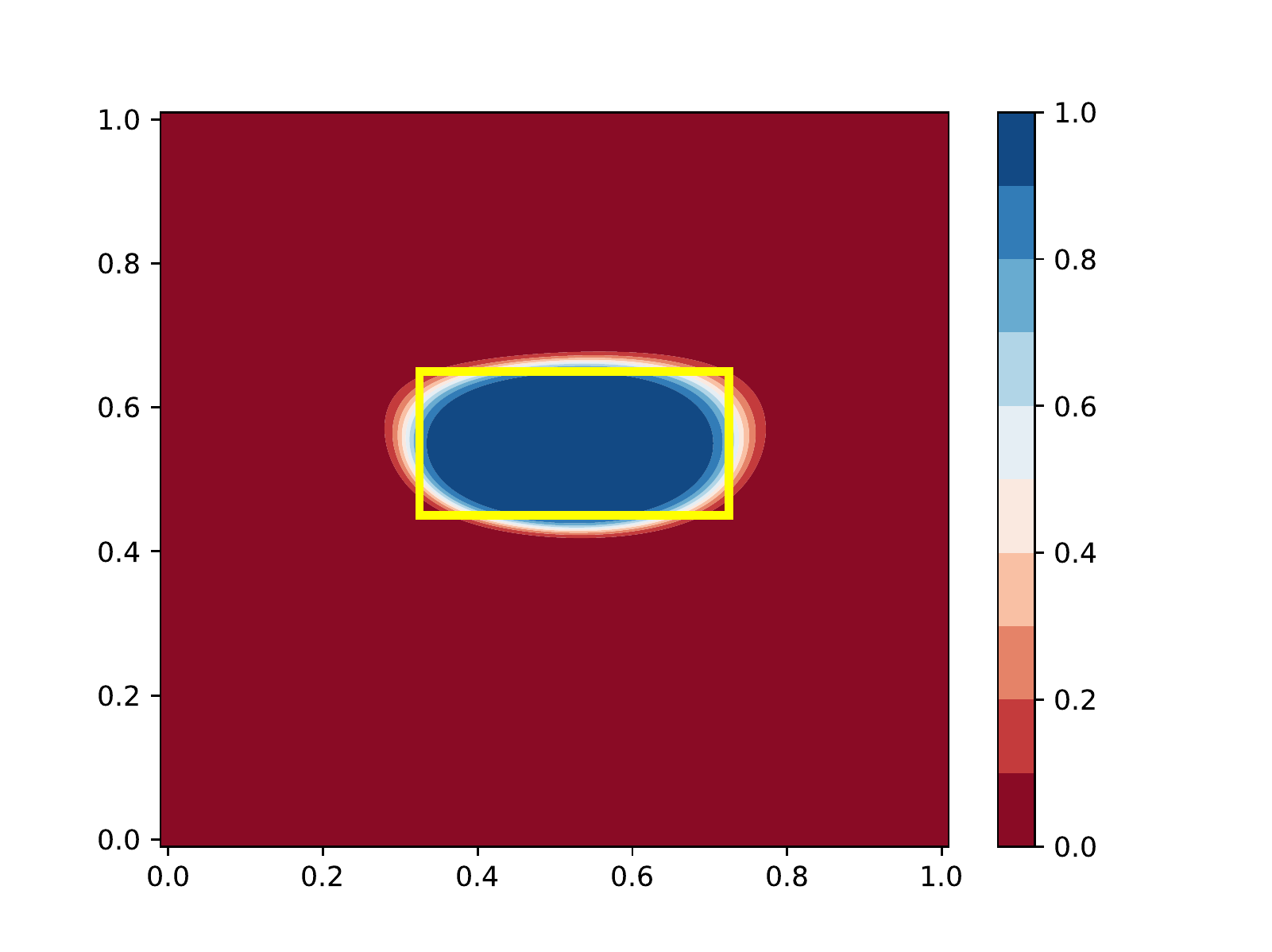}
\end{minipage} &
\begin{minipage}{.145\textwidth}
    \includegraphics[width=\linewidth,trim={0.4cm 0 1cm 1cm},clip]{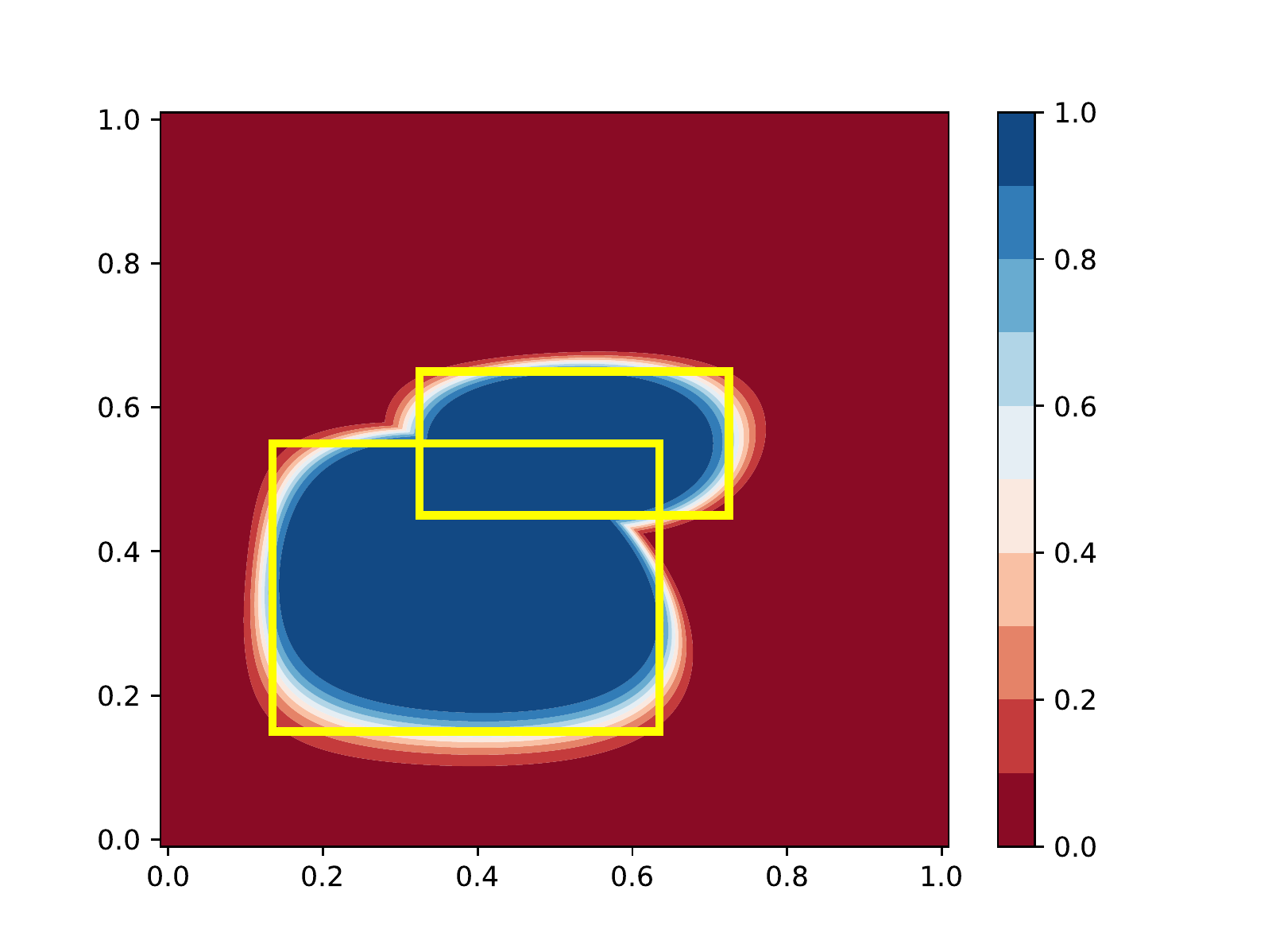}
\end{minipage} \vspace*{-1ex}
\end{tabular}
\caption{In all figures, the smaller yellow rectangle corresponds to $R_1$, while the bigger yellow one corresponds to~$R_2$. The first row of figures corresponds to $R_1 \cap R_2 = R_1$, the second corresponds to $R_1 \cap R_2 =\emptyset$, and the third corresponds to $R_1 \cap R_2 \not\in\{R_1,\emptyset\}$. 
 First four columns: decision boundaries of $f^+$ and $g^+$ for the classes $A_1$ and $A$. Last two columns: decision boundaries of \hmcsys{$h$} for the classes $A_1$ and $A$. In each figure, the darker the blue (resp., red), the more confident a model is that the data points in the region belong (do not belong) to the class (see the scale at the end of each row).
}
\label{fig:dec_bound_figs}
\end{figure*}

\begin{figure*}[t]
\centering
\begin{tabular}{c@{\ \ \,}c@{\ \ \ \ }|@{\ \ \ \ }c@{\ \ \,}c@{\ \ \ \ }|@{\ \ \ \ }c@{\ \ \,}c}
\multicolumn{2}{c}{Neural Network $f$} & \multicolumn{2}{c}{Neural Network $g$} &
\multicolumn{2}{c}{Neural Network $h$} \\
Class $A_1$ & Class $A$ & Class $A_1$ & Class $A\setminus A_1$ & Class $A_1$ & Class $A$ \\
\begin{minipage}{.11\textwidth}
    \includegraphics[width=\textwidth,trim={1.1cm 0 3.7cm 1cm},clip]{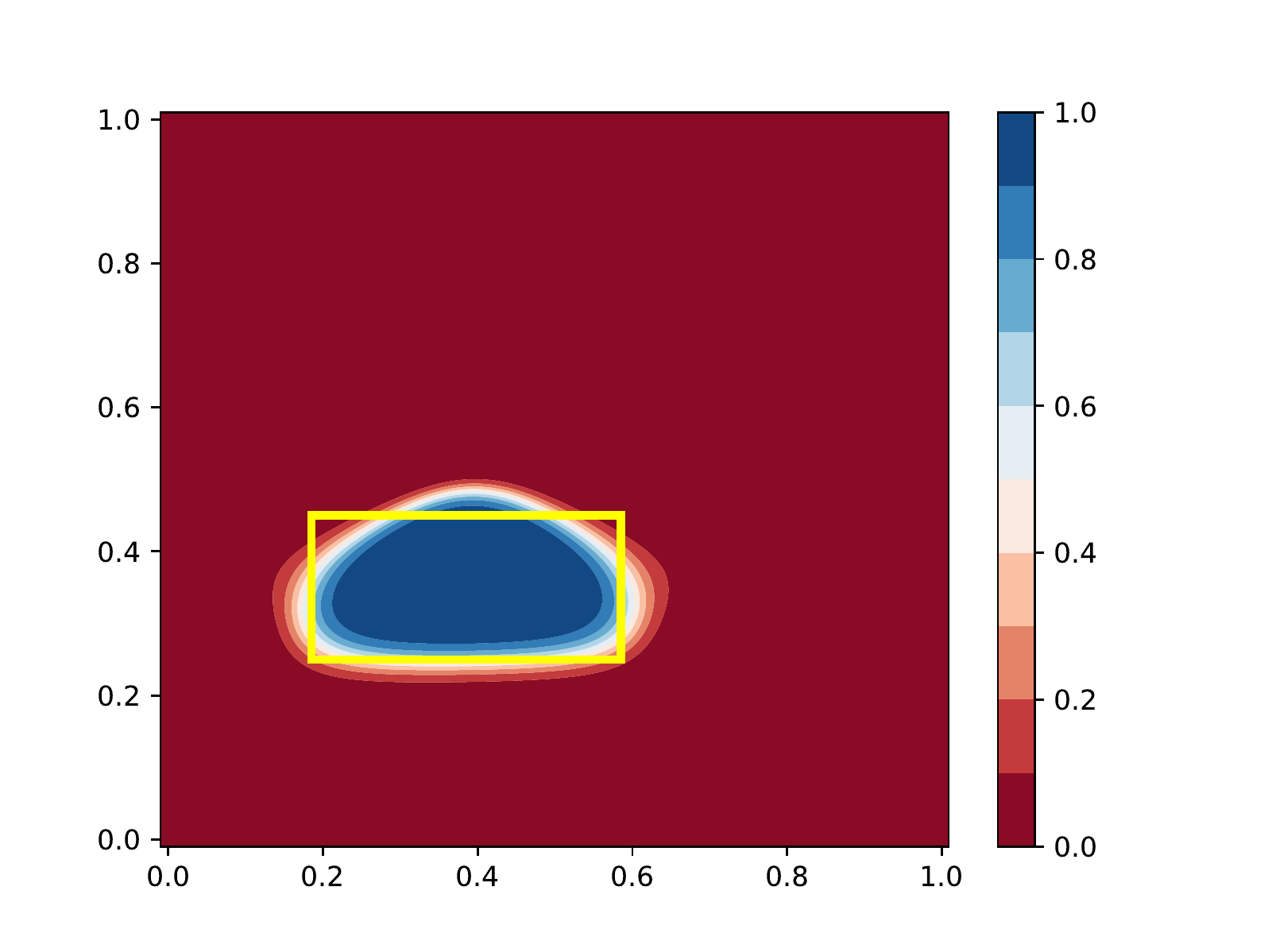}
\end{minipage} &
\begin{minipage}{.11\textwidth}
    \includegraphics[width=\textwidth,trim={1.1cm 0 3.7cm 1cm},clip]{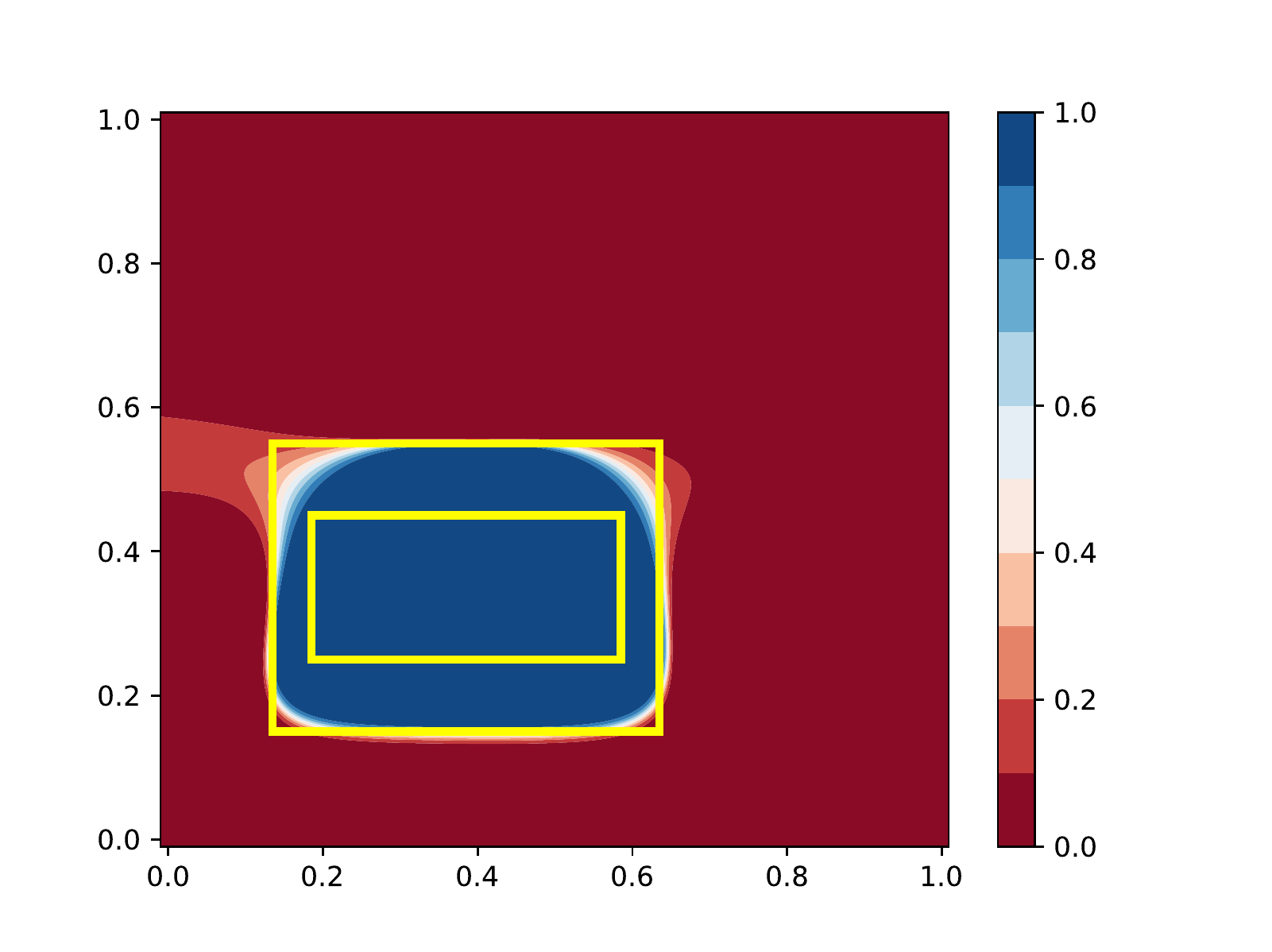} 
\end{minipage} &
\begin{minipage}{.11\textwidth}
    \includegraphics[width=\textwidth,trim={1.1cm 0 3.7cm 1cm},clip]{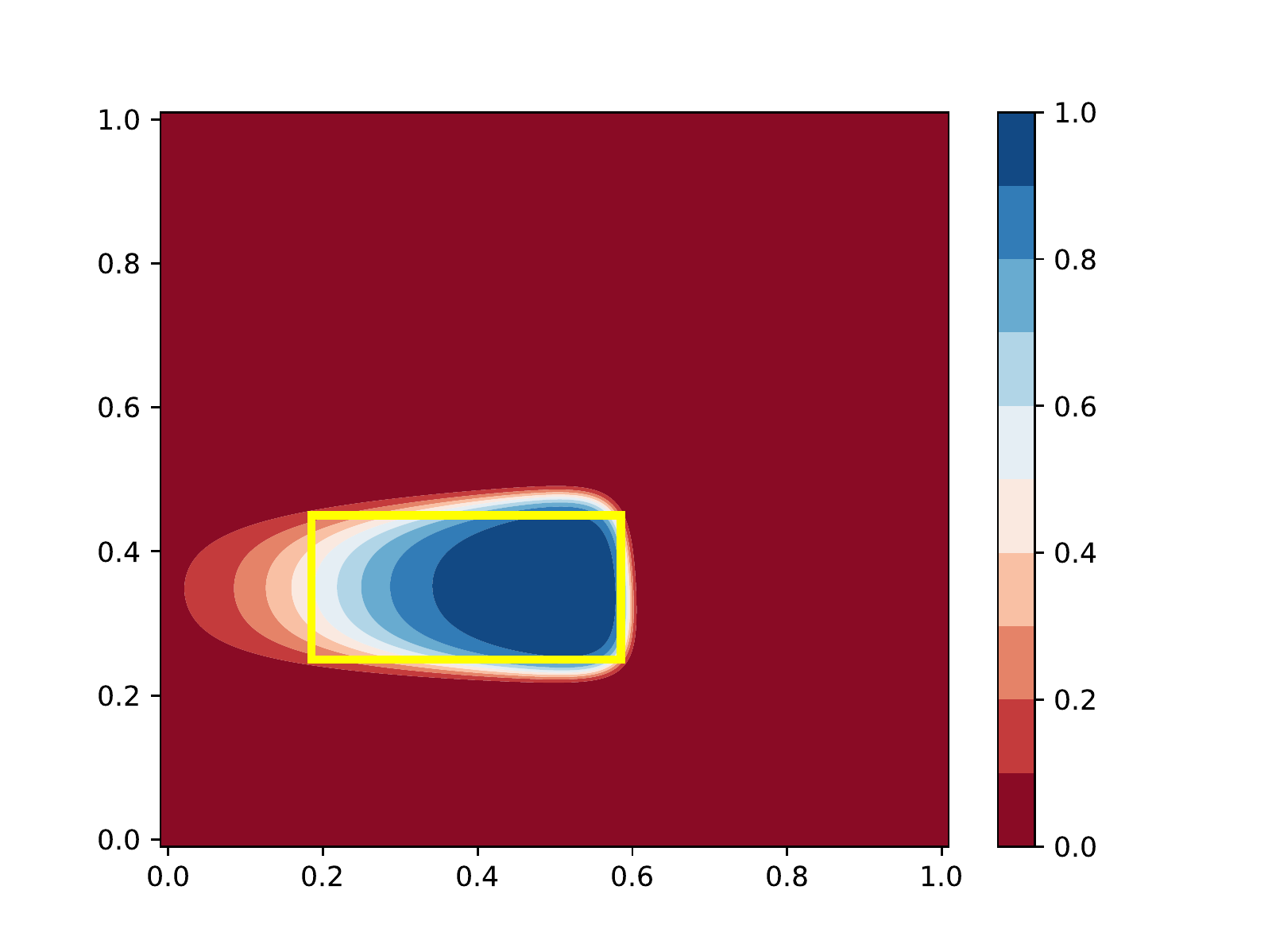} 
\end{minipage} &
\begin{minipage}{.11\textwidth}
    \includegraphics[width=\textwidth,trim={1.1cm 0 3.7cm 1cm},clip]{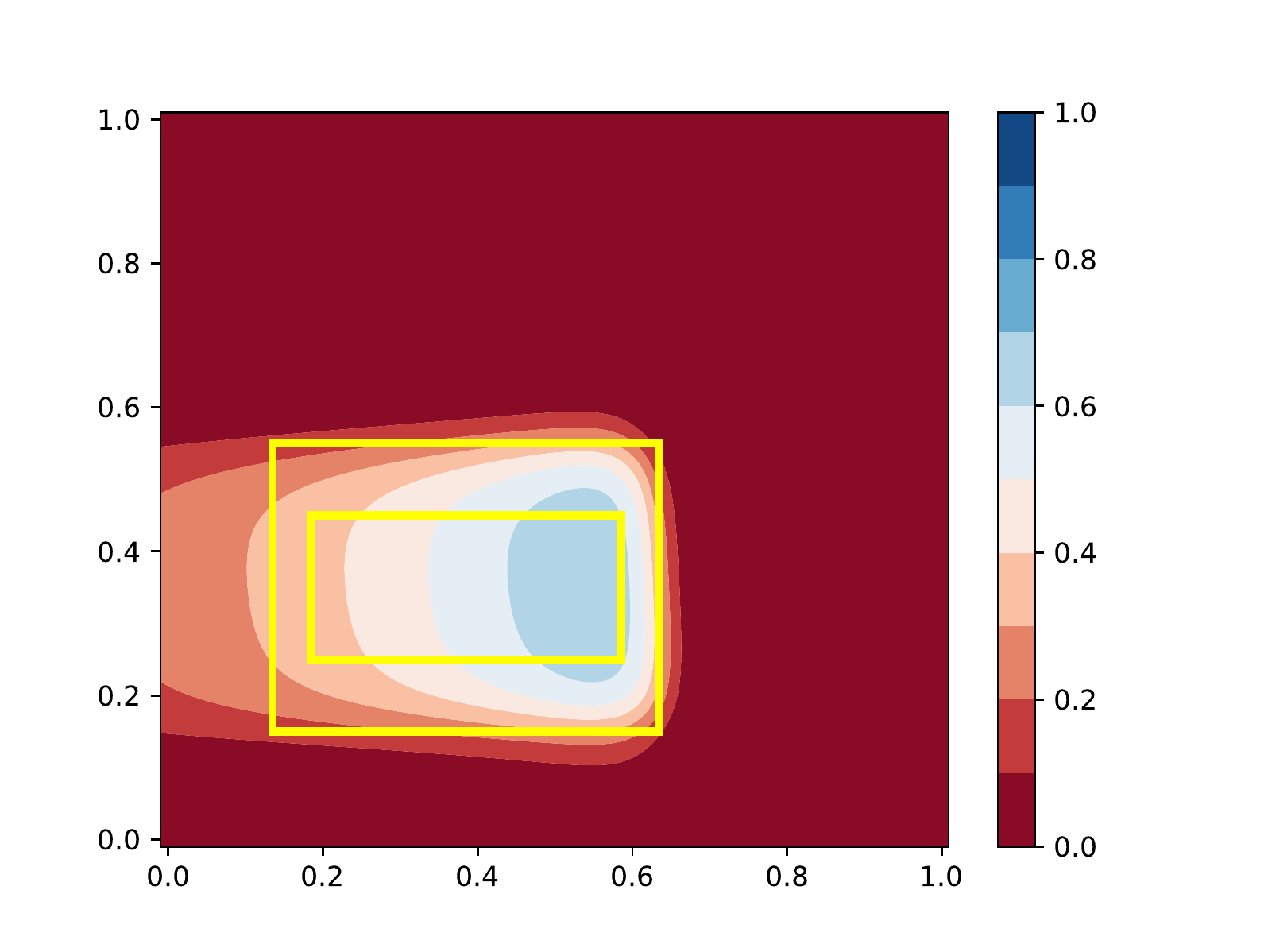} 
    \end{minipage} &
\begin{minipage}{.11\textwidth}
    \includegraphics[width=\linewidth,trim={1.1cm 0 3.7cm 1cm},clip]{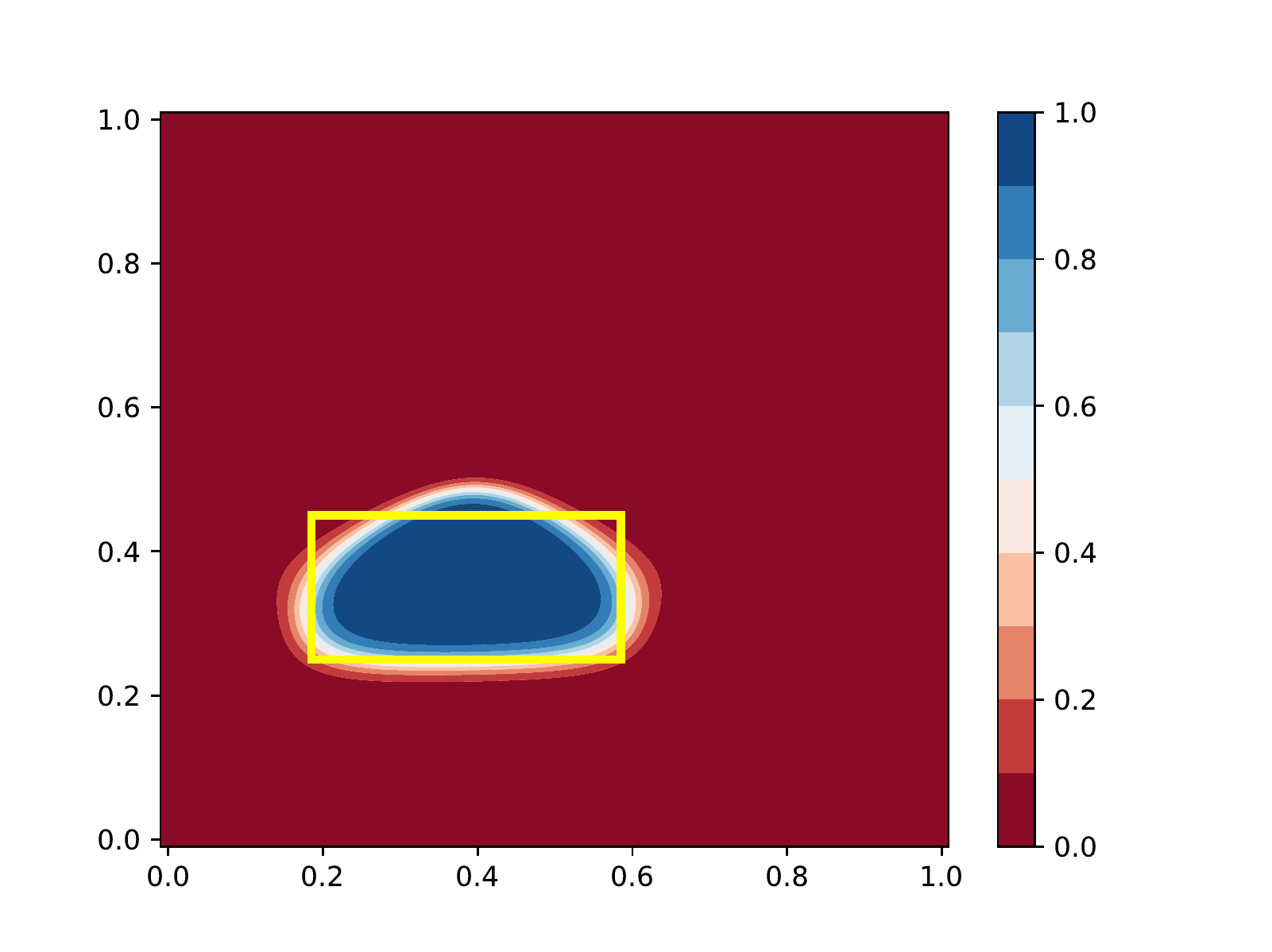}
\end{minipage} &
\begin{minipage}{.135\textwidth} 
    \includegraphics[width=\linewidth,trim={1.1cm 0 1cm 1cm},clip]{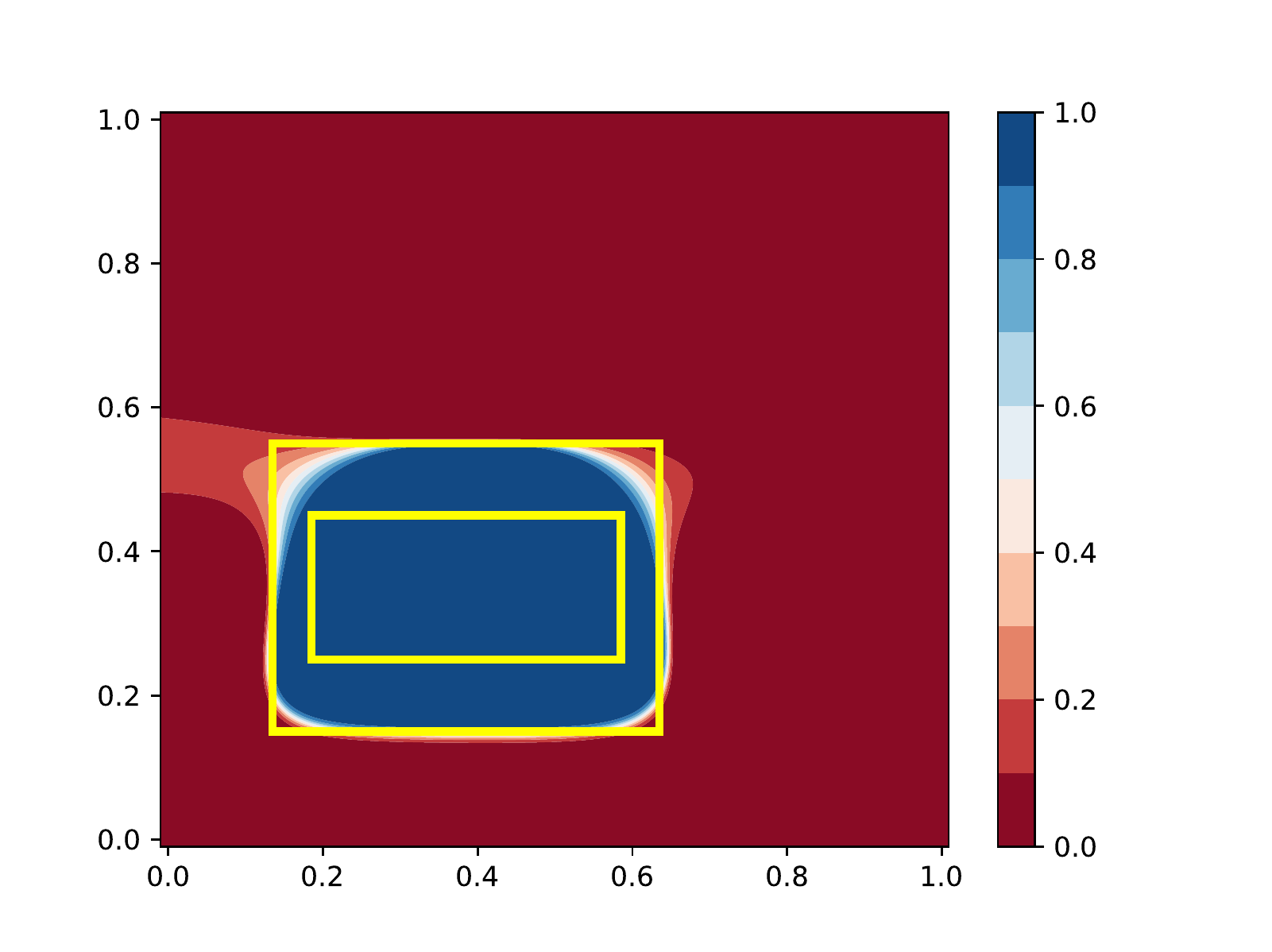}
\end{minipage} \\
\begin{minipage}{.11\textwidth}
    \includegraphics[width=\textwidth,trim={1.1cm 0 3.7cm 1cm},clip]{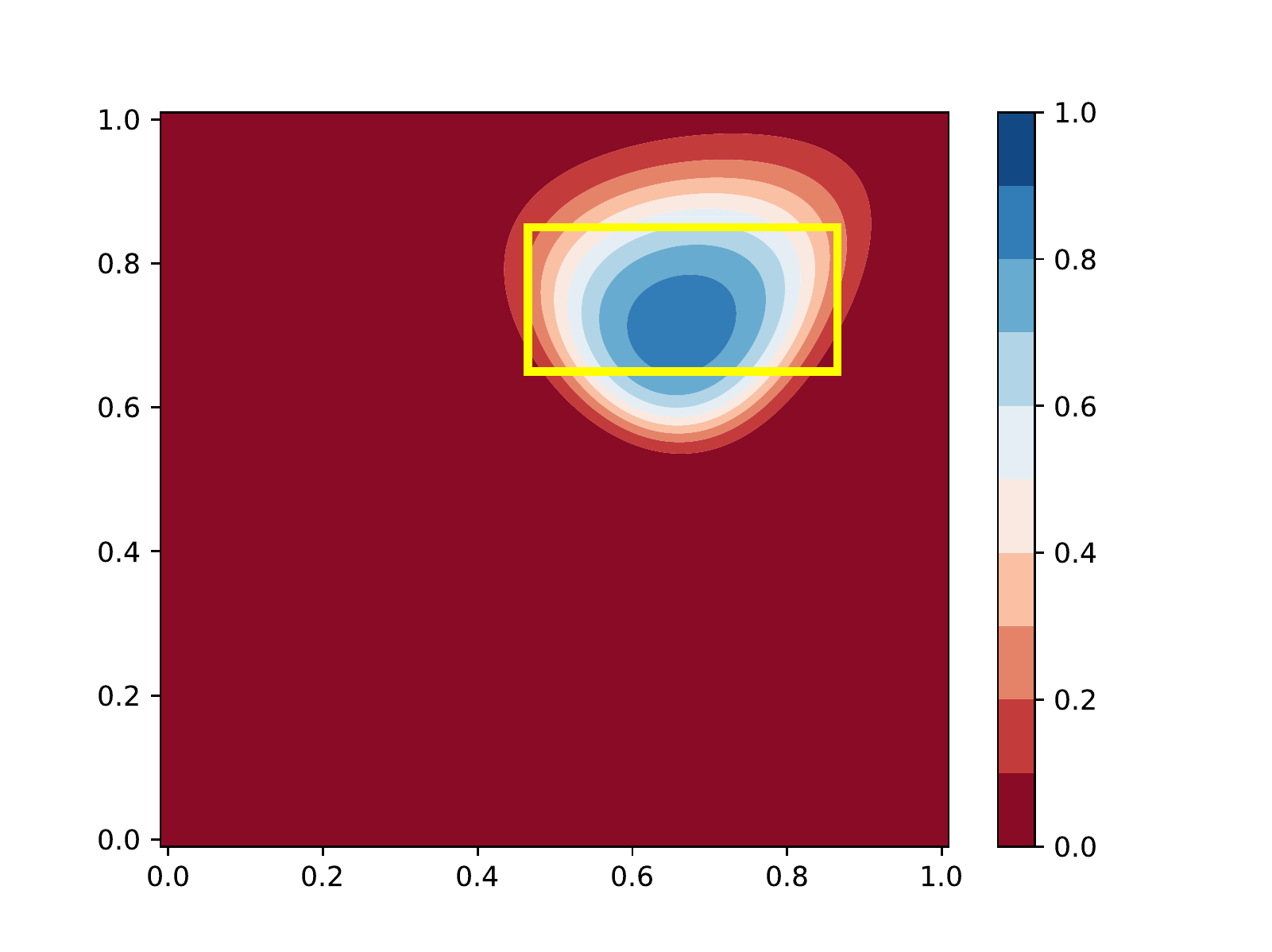}
\end{minipage} &
\begin{minipage}{.11\textwidth}
    \includegraphics[width=\textwidth,trim={1.1cm 0 3.7cm 1cm},clip]{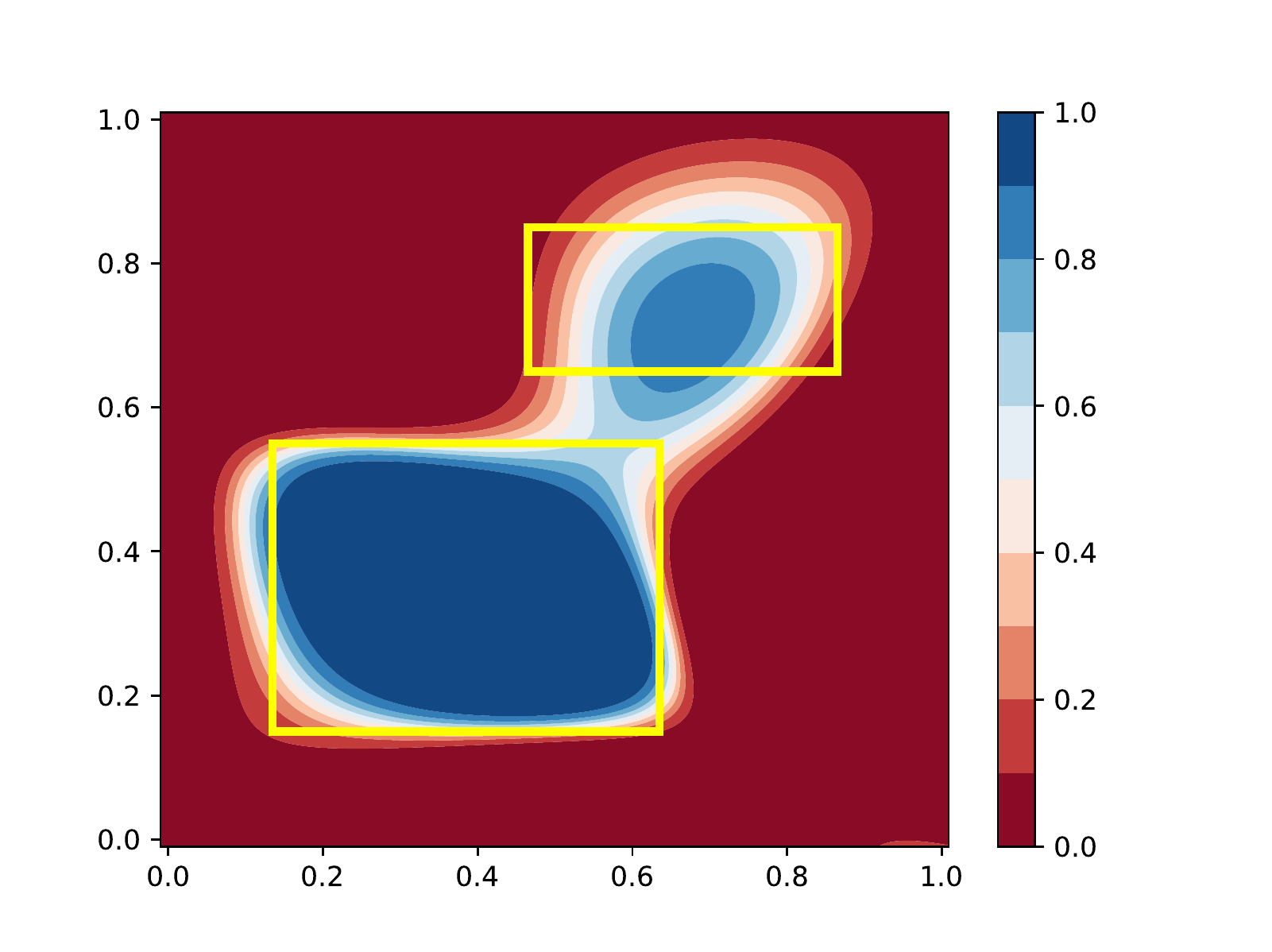}
\end{minipage} &
\begin{minipage}{.11\textwidth}
    \includegraphics[width=\textwidth,trim={1.1cm 0 3.7cm 1cm},clip]{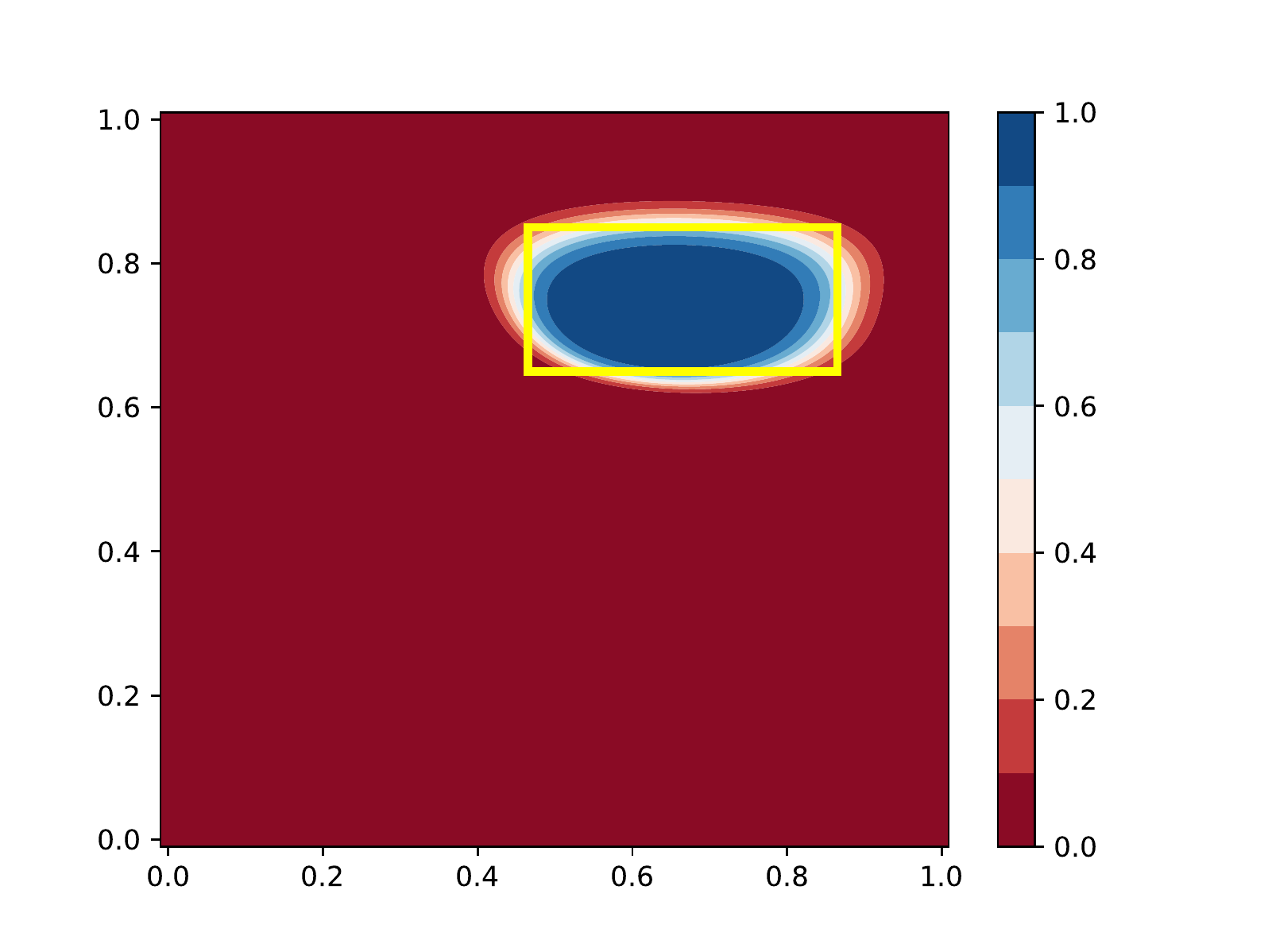}
\end{minipage} &
\begin{minipage}{.11\textwidth}
    \includegraphics[width=\textwidth,trim={1.1cm 0 3.7cm 1cm},clip]{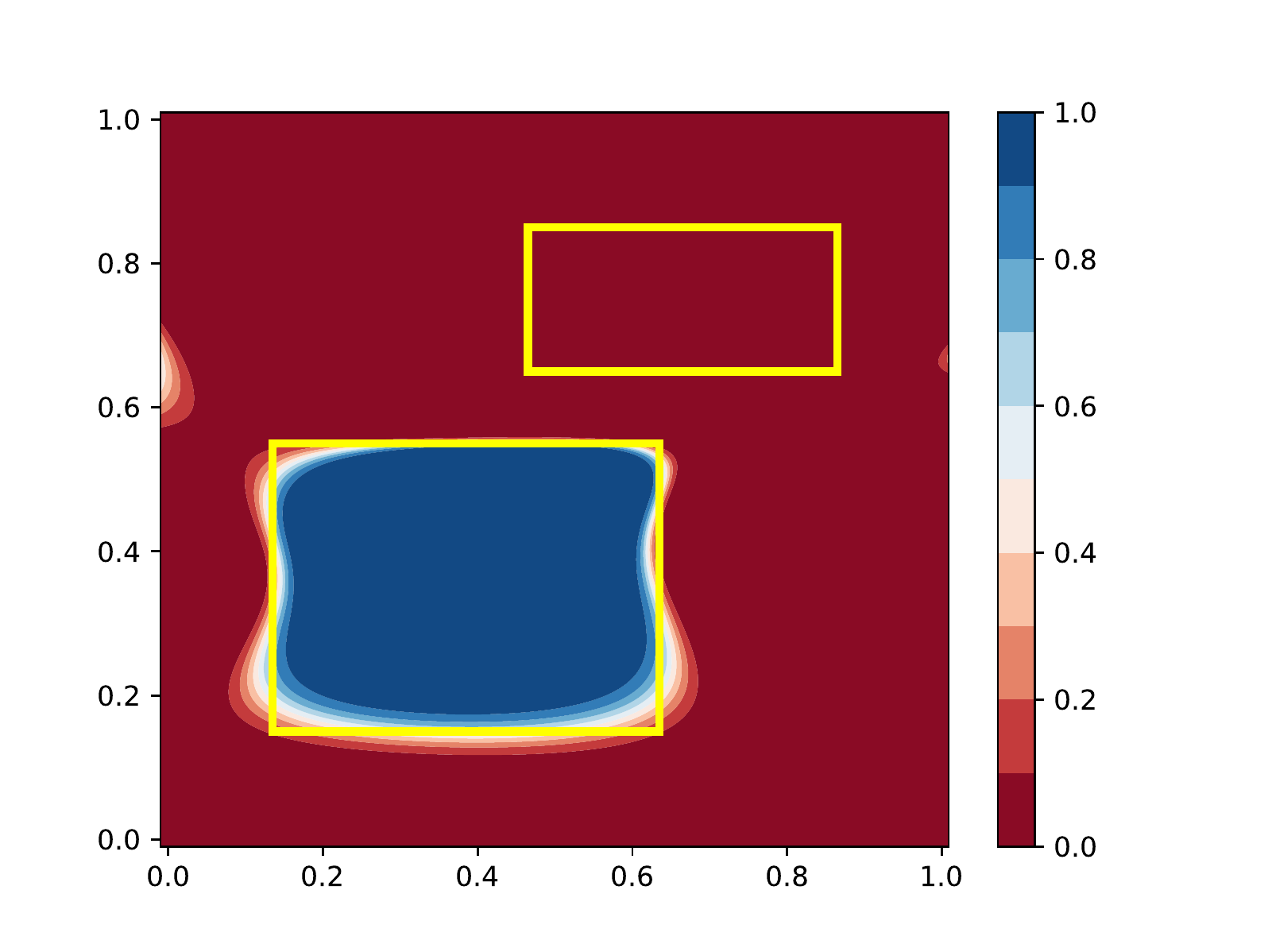}
\end{minipage} &
\begin{minipage}{.11\textwidth}
    \includegraphics[width=\linewidth,trim={1.1cm 0 3.7cm 1cm},clip]{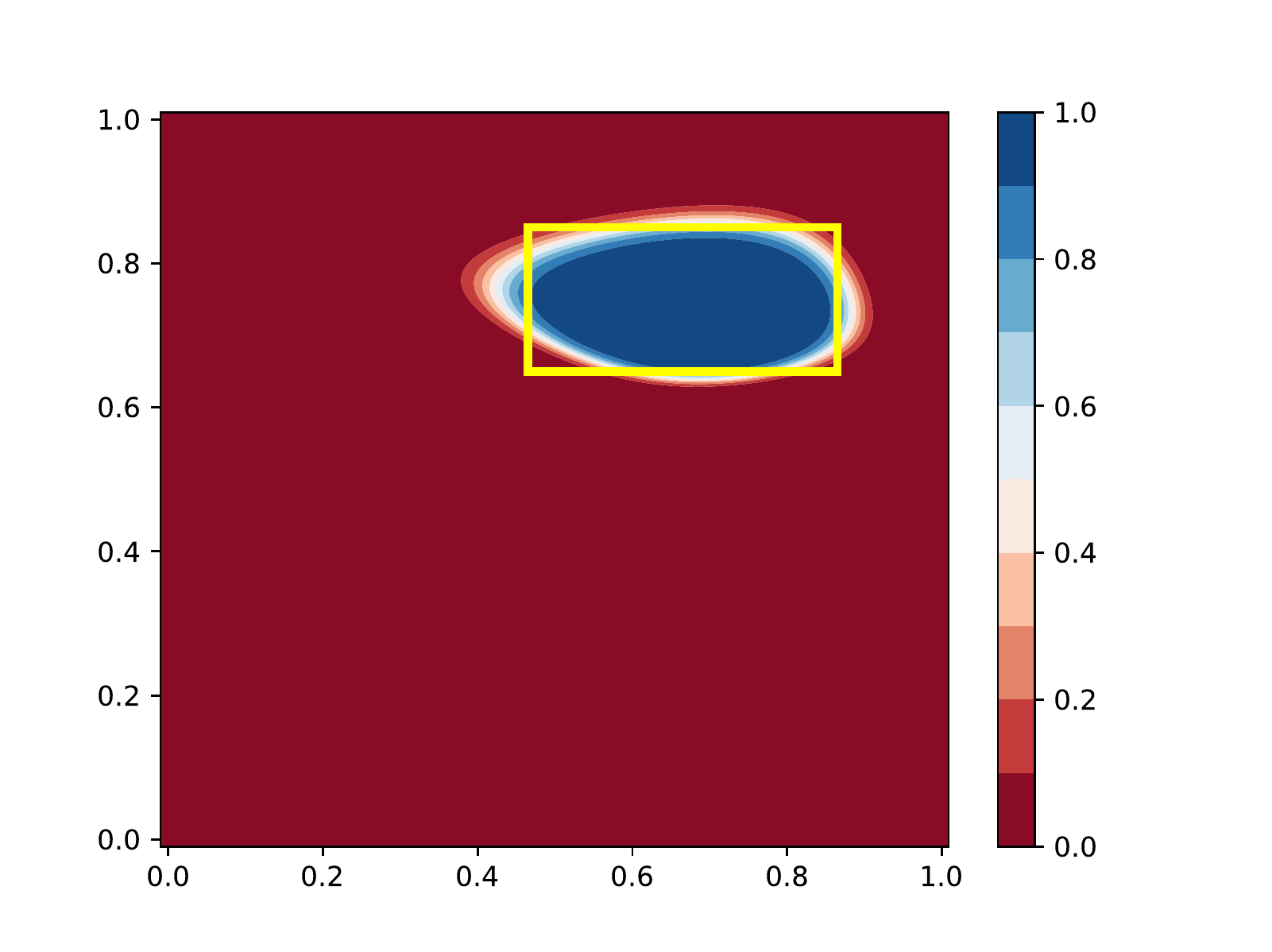}
\end{minipage} &
\begin{minipage}{.135\textwidth}
    \includegraphics[width=\linewidth,trim={1.1cm 0 1cm 1cm},clip]{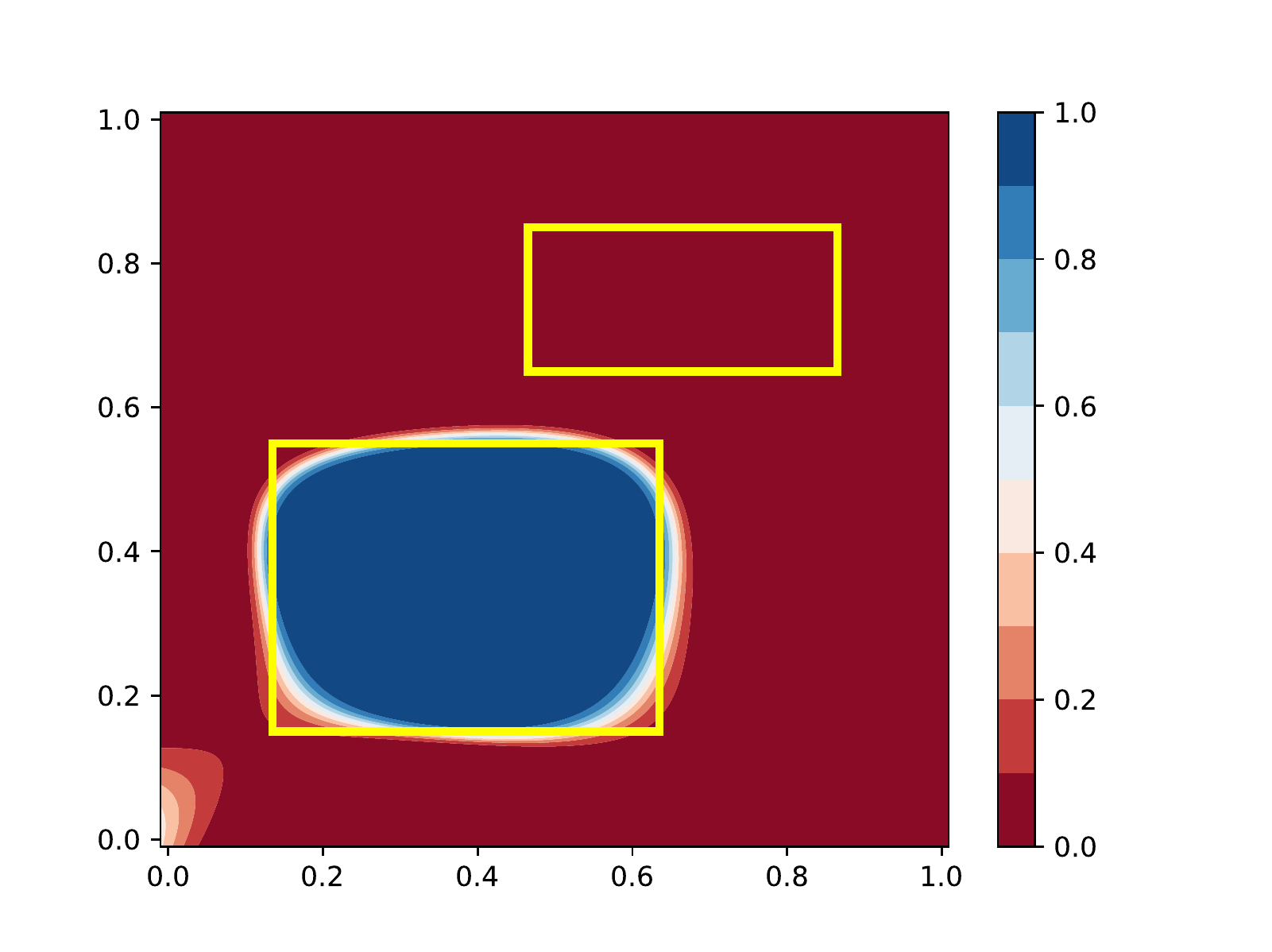}
\end{minipage} \\
\begin{minipage}{.11\textwidth}
    \includegraphics[width=\textwidth,trim={1.1cm 0 3.7cm 1cm},clip]{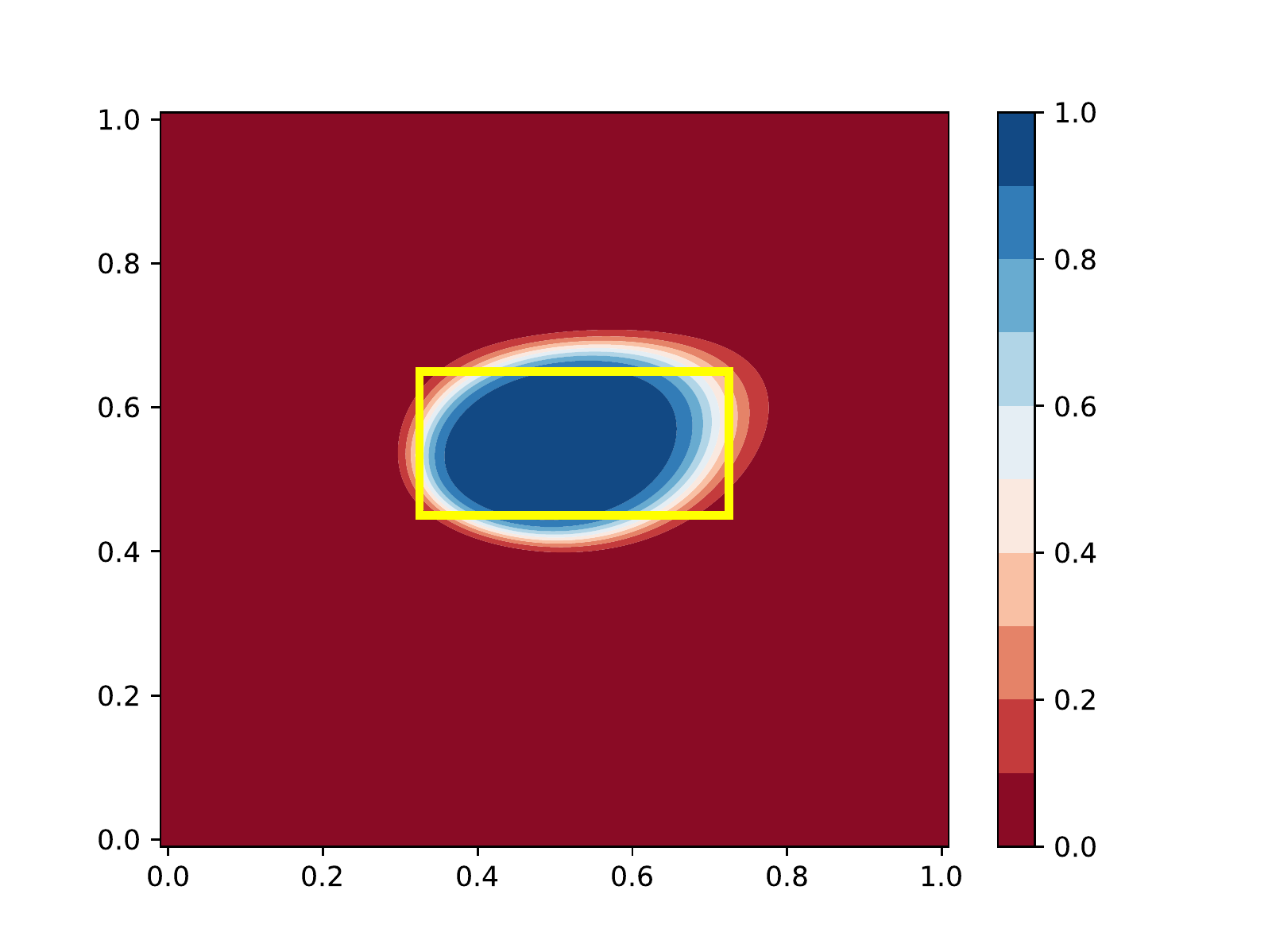}
\end{minipage} &
\begin{minipage}{.11\textwidth}
    \includegraphics[width=\textwidth,trim={1.1cm 0 3.7cm 1cm},clip]{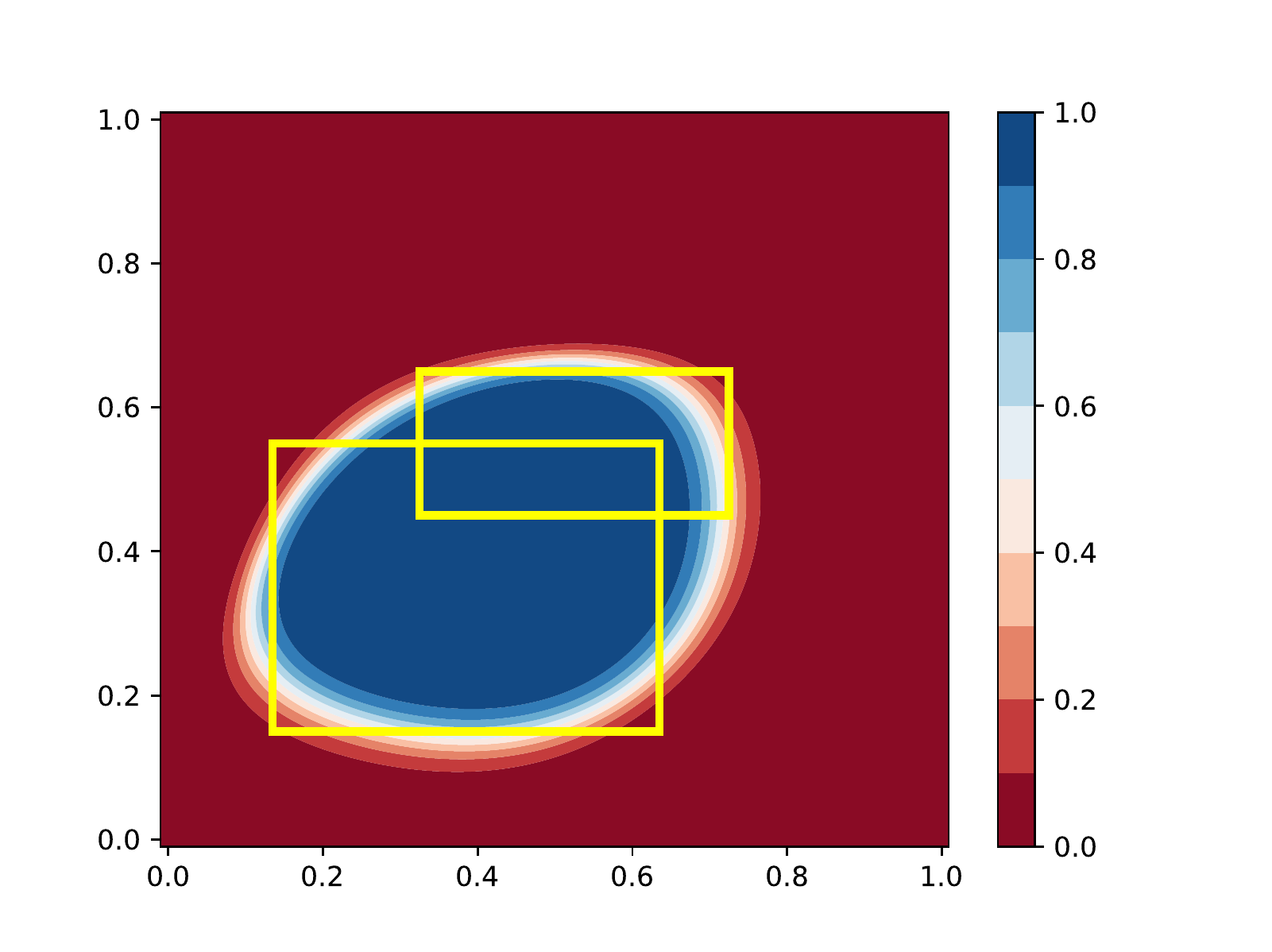}
\end{minipage} &
\begin{minipage}{.11\textwidth}
    \includegraphics[width=\textwidth,trim={1.1cm 0 3.7cm 1cm},clip]{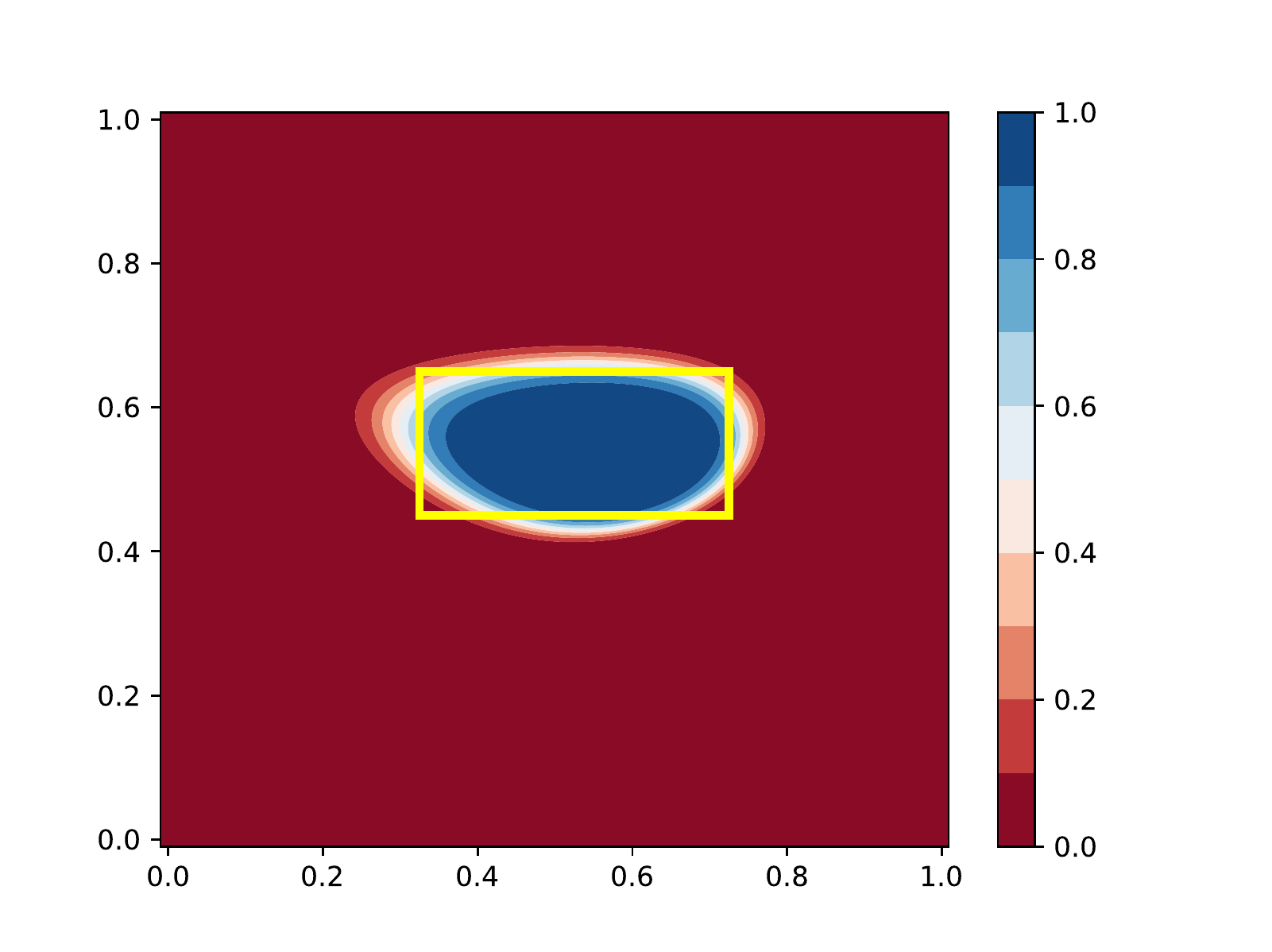}
\end{minipage} &
\begin{minipage}{.11\textwidth}
    \includegraphics[width=\textwidth,trim={1.1cm 0 3.7cm 1cm},clip]{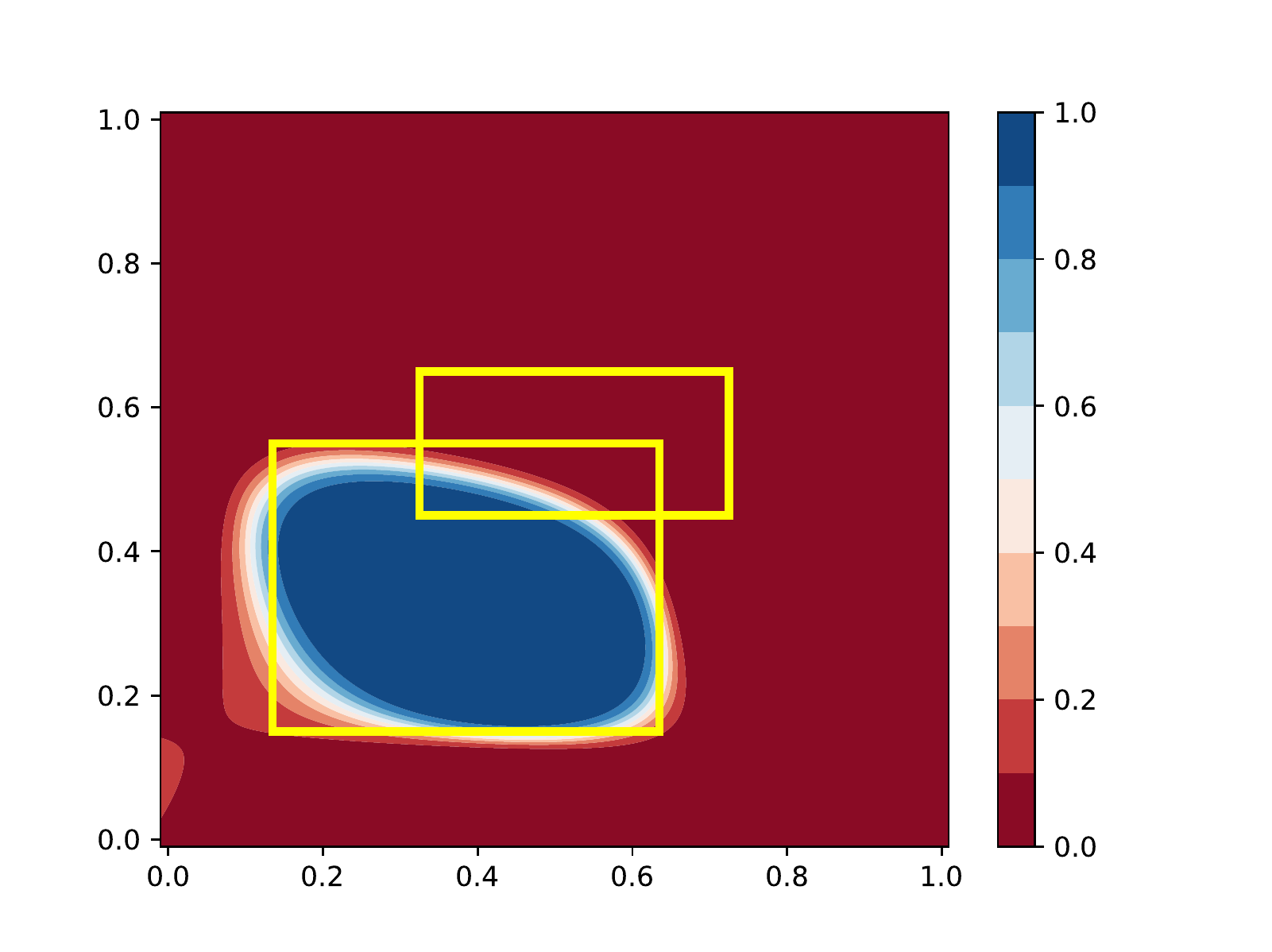}
\end{minipage} &
\begin{minipage}{.11\textwidth}
    \includegraphics[width=\linewidth,trim={1.1cm 0 3.7cm 1cm},clip]{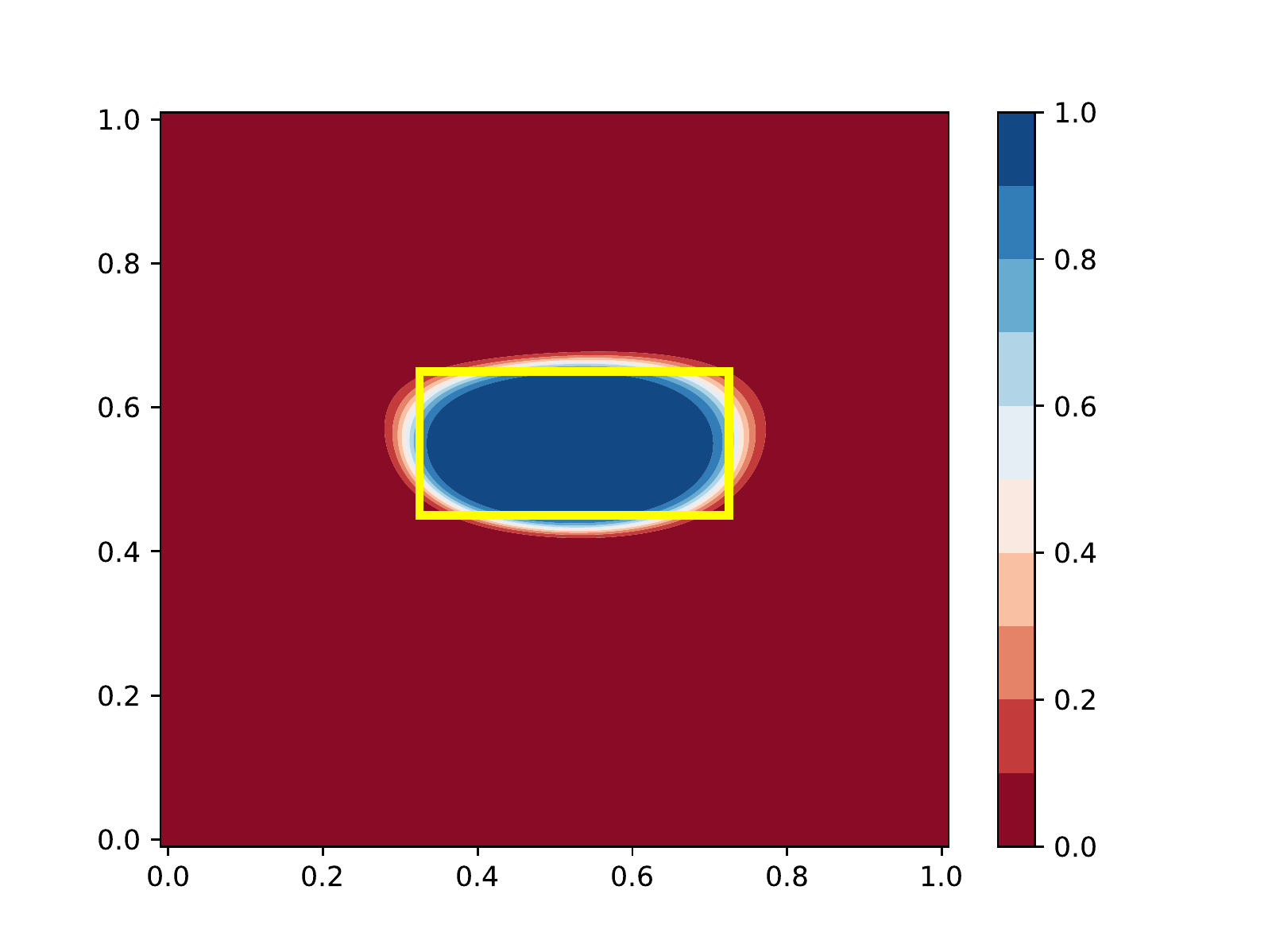}
\end{minipage} &
\begin{minipage}{.145\textwidth}
    \includegraphics[width=\linewidth,trim={0.4cm 0 1cm 1cm},clip]{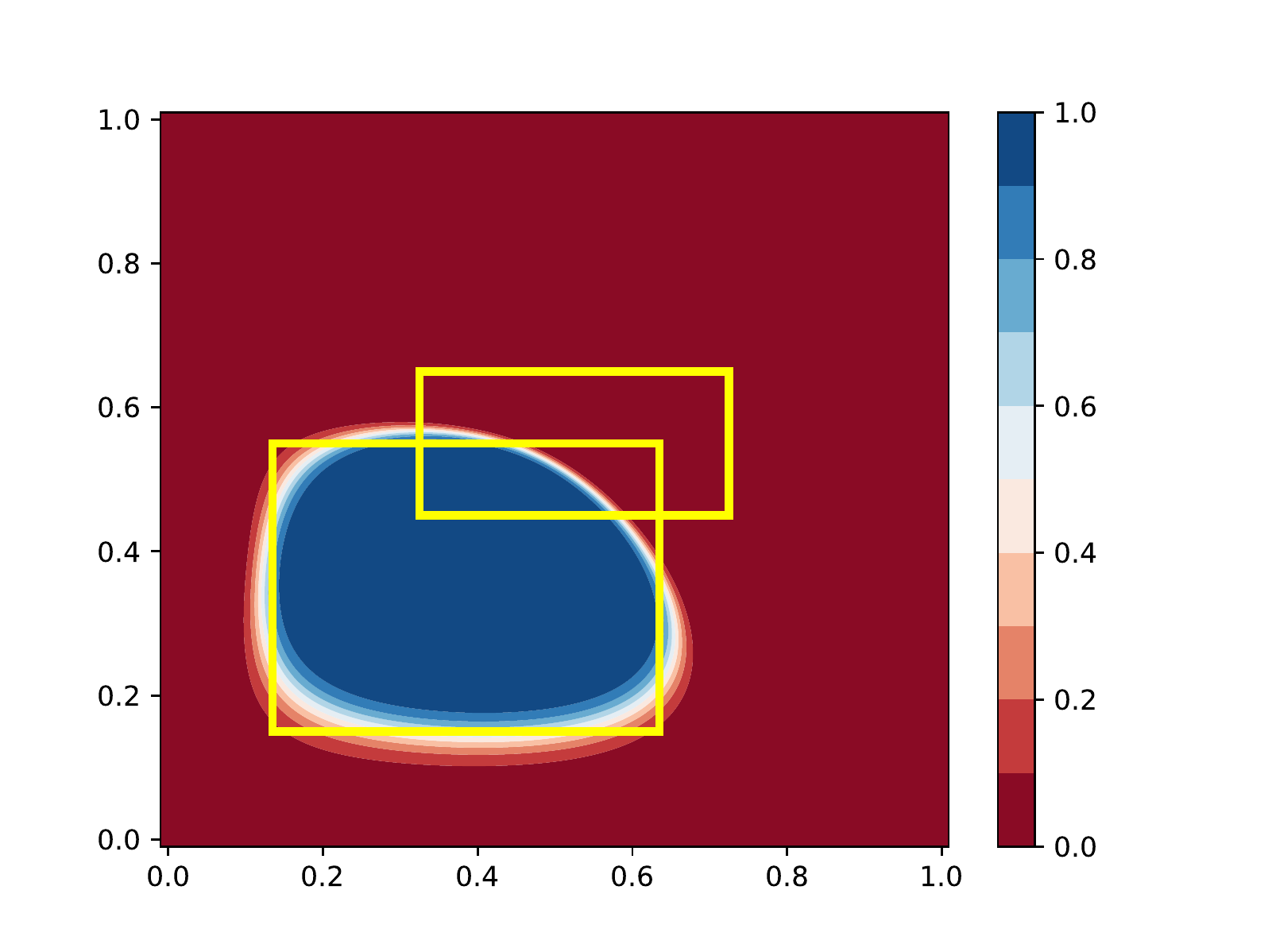}
\end{minipage} \vspace*{-1ex}
\end{tabular}
\caption{First four columns: decision boundaries of $f$ (resp., $g$) for the classes $A_1$ and $A$ (resp., $A_1$ and $A \setminus A_1$). Last two columns: decision boundaries of $h$ for the classes $A_1$ and~$A$. 
}
\label{fig:dec_bound_figs_before}
\end{figure*}

\subsection{Basic Case}\label{sec:hmc_basic}

Our goal is to leverage standard neural network approaches for MC problems and then exploit the hierarchy constraints in order to produce coherent predictions and improve performance.
Given our goal,
we first present two basic approaches, 
exemplifying their respective strengths and weaknesses. These are useful to then introduce our solution, which
is shown to present their advantages without exhibiting their weaknesses. In this section, we assume to have just two classes~$A_1 , A$ and the constraint~(\ref{eq:basic_constr-2}).

In the first approach, we treat the problem as a standard multi-label classification problem and simply set up a neural network $f$ with one output per class to be learned: to ensure that no hierarchy violation happens, we need an additional post-processing step. In this simple case, the post-processing could set the output for $A_1$ to be $\min(f_{A_1}, f_{A})$ or the output for $A$ to be $\max(f_{A}, f_{A_1})$. In this way, all predictions are always coherent with the hierarchy constraint.
A second approach is to build a network $g$ with two outputs, one for $A_1$ and one for $A\setminus A_1$. To meaningfully ensure that no hierarchy violation happens, we need an additional post-processing step in which 
each prediction for the class $A$ is given by $\max(g_{A\setminus A_1}, g_{A_1})$.
Considering the two above approaches, depending on the specific distribution of the 
data points, one solution may be significantly better than the other, and a priori we may not know which one it is.

To visualize the problem, assume that $D\,{=}\,2$, and consider two rectangles $R_1$ and $R_2$ with $R_1$~smaller than $R_2$, like the two yellow rectangles in the subfigures of Figure~\ref{fig:dec_bound_figs}. Assume $A_1 \,{=}\, R_1$ and $A = R_1 \cup R_2$.
Let $f^+$ be the model obtained by adding a post-processing step to $f$ setting $f^+_{A_1}=\min(f_{A_1}, f_A)$ and $f^+_A \,{=}\, f_A$, as in~\cite{cerri2014,cerri2016,feng2018} (analogous considerations hold, if we set $f^+_{A_1}\,{ =}\, f_{A_1}$ and $f^+_A = \max(f_A, f_{A_1})$ instead). Intuitively, we expect $f^+$ to perform well even with a very limited number of neurons when $R_1 \cap R_2\,{ =}\, R_1$, as in the first row of Figure~\ref{fig:dec_bound_figs}.
However, if $R_1 \cap R_2 \,{=}\, \emptyset$, as in the second row of Figure~\ref{fig:dec_bound_figs}, we expect $f^+$ to need more neurons to obtain 
a similar
performance.
Consider the alternative network $g$, and
let $g^+$ be the system obtained by setting $g^+_{A_1} \,{=}\, g_{A_1}$ and $g^+_A \,{=}\, \max(g_{A\setminus A_1}, g_{A_1})$.
Then, we expect $g^+$ to perform well when $R_1 \cap R_2 \,{=}\, \emptyset$.
However, if $R_1 \cap R_2\,{=}\, R_1$, we expect $g^+$ to need  more neurons to obtain 
a similar performance.
(We do not consider the model with one output for $A\setminus A_1$ and one for~$A$, since it performs poorly in both~cases.)
To test our hypothesis, we implemented $f$ and $g$ as feedforward neural networks with one hidden layer with four neurons and {\it tanh} nonlinearity. We used the {\it sigmoid} non-linearity for the output layer (from here on, we always assume that the last layer of each neural network presents {\it sigmoid} non-linearity). 
$f$ and $g$ were trained with binary cross-entropy loss using Adam optimization~\cite{kingma2014} for 20k epochs with learning rate $10^{-2}$ ($\beta_1 = 0.9, \beta_2 = 0.999$). The datasets consisted of $5000$ (50/50 train/test split) data points sampled from a uniform distribution over $[0,1]^2$. 
The first four columns of Figure~\ref{fig:dec_bound_figs} show the decision boundaries of $f^+$ and $g^+$, while the decision boundaries of $f$ and $g$ are reported in Figure~\ref{fig:dec_bound_figs_before}. These figures highlight that $f^+$ (resp., $g^+$) approximates the two rectangles better than $g^+$ (resp., $f^+$)  when $R_1 \cap R_2 = R_1$ (resp., $R_1 \cap R_2 = \emptyset$).
In general, when $R_1 \cap R_2 \not\in \{R_1,\emptyset\}$, we expect that the behavior of $f^+$ and $g^+$ depends on the relative position of $R_1$ and $R_2$. 

Ideally, we would like to build a neural network that is able to have roughly the same performance of $f^+$ when $R_1 \cap R_2 = R_1$, of $g^+$ when $R_1 \cap R_2 = \emptyset$, and better than both in any other case. We can achieve this behavior in two steps. 

In the first step, we build a new neural network consisting of two modules: (i) a bottom module $h$ with two outputs in $[0,1]$ for $A_1$ and $A$, and (ii) an upper module, called {\sl max constraint module} ($\module$),  consisting of a single layer that takes as input the output of the bottom module and imposes the hierarchy constraint. 
We call the obtained neural network the \emph{coherent hierarchical multi-label classification neural network of $h$}, denoted \hmcsys{$h$}. 
Consider a data point $x$. Let $h_{A_1}$ and $h_A$ be the outputs of $h$ for the classes $A_1$ and $A$, respectively, and let $y_{A_1}$ and $y_A$ be the ground truth for the classes $A_1$ and $A$, respectively. 
The outputs of {\module} (which are also the output of \hmcsys{$h$}) are: 
\begin{equation}
    \begin{aligned}
        & \module_{A_1}  =  h_{A_1}, \\
        & \module_{A}  = \max(h_A, h_{A_1}).
    \end{aligned}
\end{equation}
Notice that the output of \hmcsys{$h$} ensures that no hierarchy violation happens, i.e., that for any threshold, it cannot be the case that $\module$ predicts that a data point belongs to~$A_1$ but  not to~$A$. 

In the second step, to exploit the hierarchy constraint during training, 
\hmcsys{$h$} is trained with a novel loss function, called {\sl max constraint loss ($\loss$)},
defined as $\loss = \loss_{A_1} + \loss_A$, where:
\begin{equation}
    \begin{aligned}
    & \loss_{A_1}  = -y_{A_1} \ln(\module_{A_1}) - (1-y_{A_1})\ln(1-\module_{A_1}), \\
    & \loss_A  = -y_A\ln(\max(h_A, h_{A_1}y_{A_1})) - (1-y_A)\ln(1-\module_A)). \\
    \end{aligned}
\end{equation}
{\loss} differs from the standard binary cross-entropy loss $\lss$:
$$
\lss = - y_{A_1} \ln{(\module_{A_1})} - (1- y_{A_1}) \ln{(1- \module_{A_1})} - y_A \ln{(\module_A)} - (1- y_A) \ln{(1- \module_A)},
$$
iff $x\not\in A_1$ ($y_{A_1} = 0$), $x \in A$ ($y_A = 1$), and $h_{A_1} > h_A$.

The following example highlights the different behavior of {\loss} compared to $\lss$.

\begin{example}\label{ex:bgrad}{\rm 
Assume $h_{A_1} = 0.3$, $h_A = 0.1$, $y_{A_1} = 0$, and $y_A = 1$. 
Then, 
\begin{align*}
\loss = \loss_{A_1} + \loss_A = - \ln(1- h_{A_1}) - ln(h_A)\,,
\end{align*}
and the partial derivatives of  $\loss$ with respect to $h_{A_1}$ and $h_A$ are
\begin{align*}
\frac{\partial \loss}{\partial h_{A_1}} = -\frac{1}{h_{A_1}-1} \sim 1.4\,  \qquad\mbox{ and } \qquad \frac{\partial \loss}{\partial h_A} = -\frac{1}{h_A} = - 10\,,
\end{align*}
and \hmcsys{$h$} rightly learns that it needs to decrease $h_{A_1}$ and increase $h_A$.

On the other hand, if we use the standard binary cross-entropy $\lss$ after $\module$, we obtain:
\begin{align*}
\lss & = - \ln(1-\module_{A_1}) - \ln(\module_A)  = -\ln(1-h_{A_1}) - \ln(h_{A_1})\,, 
\end{align*}
and then
$$
\frac{\partial \lss}{\partial h_{A_1}} = -\frac{1}{h_{A_1}-1} - \frac{1}{h_{A_1}} \sim -1.9\qquad \mbox{ and } \qquad  \frac{\partial \lss}{\partial h_A} = 0\,.
$$
Hence, if \hmcsys{$h$} is trained with $\lss$, then it wrongly learns that it needs to increase~$h_{A_1}$ and keep $h_A$. \hfill$\lhd$
}\end{example}

Consider the example in Figure~\ref{fig:dec_bound_figs}. To check that our model behaves as expected, we implemented $h$ as $f$, and trained \hmcsys{$h$} with $\loss$ on the same datasets and in the same way as $f$ and~$g$. 
The last two columns of Figure~\ref{fig:dec_bound_figs} show the decision boundaries of \hmcsys{$h$}, while those of $h$ can be seen in Figure~\ref{fig:dec_bound_figs_before}. \hmcsys{$h$}'s decision boundaries mirror those of $f^+$ (resp., $g^+$) when 
$R_1 \cap R_2 = R_1$ (resp., $R_1 \cap R_2 = \emptyset$). Intuitively, as highlighted by Figure~\ref{fig:dec_bound_figs_before},
\hmcsys{$h$} is able to decide whether to learn $A$: 
\begin{enumerate}
    \item as a whole (top figure), 
    \item as the union of $A\setminus A_1$ and $A_1$ (middle figure), and 
    \item as the union of a subset of $A$ and a subset of $A_1$ (bottom~figure). 
\end{enumerate}
\hmcsys{$h$} has thus learned when to exploit the prediction on the lower class $A_1$ to make predictions on the upper class $A$.
    
\subsection{General Case}\label{sec:hmc_gen_case}

We now consider an arbitrary HMC problem $(${\problem}$,\Pi)$ 
with {\problem}$\,=(\classes,{\mathcal X})$. Given a class $A \in {\classes}$, we denote by $\mathcal{D}_A$ the set of subclasses of $A$ as given by $\Pi$, i.e., the set of classes $B$ such that there is a path of length $\geq 0$ from $B$ to $A$ in the  graph with an edge from $A_1$ to $A$ for each constraint~(\ref{eq:basic_constr-2}) 
in $\Pi$. 

Consider a data point $x \in \mathbb{R}^D$
and a model $h$ for {\problem}.
The output $\module_A$ of \hmcsys{$h$} for a class $A$ is: 
\begin{equation}\label{eq:module_class}
   \module_A = \max_{B \in \mathcal{D}_A}(h_{B}).
\end{equation}
For each class $A$, the number of operations performed by $\module_A$ is independent from the depth of the hierarchy, making \hmcsys{$h$} a scalable model. Thanks to $\module$, \hmcsys{$h$} is guaranteed to always output predictions satisfying the hierarchy constraints, as stated by the following theorem, which follows immediately from Eq.~(\ref{eq:module_class}). 

\begin{theorem}\label{th:hmc}
Let $(${\problem}$,\Pi)$ be an HMC problem. For any model $h$ for {\problem}, 
\hmcsys{$h$} does not 
commit any hierarchy violations.
\end{theorem}

As an immediate consequence, \hmcsys{$h$} also does not commit any logical violations and is coherent relative to the hierarchy constraints. 
\begin{corollary}
Let $(${\problem}$,\Pi)$ be an HMC problem. For any model $h$ for {\problem}, 
\hmcsys{$h$} does not commit any logical violations and is coherent with respect to $\Pi$.
\end{corollary}

The next step is to improve performance by modifying the loss function in order to exploit the constraints. For each class $A$, $\loss_A$ is defined as:
\begin{align*}\label{eq:loss_class}
    \loss_A & = -y_A\ln(\max_{B \in \mathcal{D}_A}(y_{B} h_{B})) - (1-y_A)\ln(1-\module_A)\, ,%
\end{align*}
where $y_A$ is the ground truth class for $A$. 
The final $\loss$ is then given by:
\begin{equation}
    \loss = \sum_{A \in {\classes}} \loss_A.
\end{equation}
$\loss$ has the fundamental property that the negative gradient descent algorithm behaves as expected, i.e., that for each class, it moves in the ``right'' direction as given by the ground truth. This is formally expressed by the following theorem. 

\begin{theorem}
Let $(${\problem}$,\Pi)$ be an HMC problem. For any model $h$ for {\problem} and class $A$, let $\frac{\partial \text{\rm\loss}}{\partial h_A}$ be the partial derivative of  $\text{\rm\loss}$ with respect to $h_A$. For each data point, 
if $y_A = 0$, then $\frac{\partial \text{\rm\loss}}{\partial h_A} \geq 0$, and 
if $y_A = 1$, then $\frac{\partial \text{\rm\loss}}{\partial h_A} \leq 0$.
\end{theorem}

\begin{proof}
Consider a data point and a class $A$. 
$$
\frac{\partial \loss}{\partial h_A} = \sum_{B \in {\classes}} \frac{\partial \loss_{B}}{\partial h_A}.
$$

Assume $y_A = 1$. For each class $B$ such that $y_{B} = 0$:
$$
\loss_{B} = -\ln{(1-\module_{B})}, \qquad \qquad \frac{\partial \loss_{B}}{\partial h_A} = \frac{1}{1-\module_{B}}\frac{\partial \module_{B}}{\partial h_A} = 0, 
$$
because $\module_B$ is not a function of $h_A$
(since $y_A =1$ and $y_B = 0$, $A \not \in \mathcal{D}_B$), and hence $\frac{\partial \module_{B}}{\partial h_A} = 0$.
For each class $B$ such that $y_{B} = 1$,
$$
\loss_{B} = - \ln(\max_{C \in \mathcal{D}_B}(y_{C} h_{C})), \qquad  \frac{\partial \loss_{B}}{\partial h_A} =
- \frac{1}{\max_{C \in \mathcal{D}_B}(y_{C} h_{C})}\frac{\partial \max_{C \in \mathcal{D}_B}(y_{C} h_{C})}{\partial h_A} \le 0, 
$$
because if $\max_{C \in \mathcal{D}_B}(y_{C} h_{C}) = h_A$, then $\frac{\partial \loss_{B}}{\partial h_A} = -\frac{1}{h_A} \le 0$, otherwise $\frac{\partial \loss_{B}}{\partial h_A} = 0$.
Since $\frac{\partial \loss}{\partial h_A}$ is given by the sum of quantities that are smaller or equal zero, then $\frac{\partial \loss}{\partial h_A} \le 0$.

Assume $y_A = 0$. For each class $B$ such that $y_{B} = 0$:
$$
\loss_{B} = -\ln{(1-\module_{B})}, \qquad \qquad \frac{\partial \loss_{B}}{\partial h_A} = \frac{1}{1-\module_{B}}\frac{\partial \module_{B}}{\partial h_A} \ge 0, 
$$
because if $\module_B = h_A$, then $\frac{\partial 
\module_{B}}{\partial h_A} = 1$, otherwise $\frac{\partial \module_{B}}{\partial h_A} = 0$.
For each class $B$ such that $y_{B} = 1$:
$$
\loss_{B} = - \ln(\max_{C \in \mathcal{D}_B}(y_{C} h_{C})), \qquad  \frac{\partial \loss_{B}}{\partial h_A} =
- \frac{1}{\max_{C \in \mathcal{D}_B}(y_{C} h_{C})}\frac{\partial \max_{C \in \mathcal{D}_B}(y_{C} h_{C})}{\partial h_A} = 0, 
$$
because $\max_{C \in \mathcal{D}_B}(y_{C} h_{C}) \not = h_A$, since $y_A = 0$.
Since $\frac{\partial \loss}{\partial h_A}$ is given by the sum of quantities that are 
greater than or equal to zero, then $\frac{\partial \loss}{\partial h_A} \ge 0$.
\end{proof}
Example \ref{ex:bgrad} already pointed out that the standard loss function may not behave as expected. This
becomes even more apparent in the general case. Indeed, as highlighted by the following example, the more superclasses a class has, the more likely it is that \hmcsys{$h$} trained with the standard binary cross-entropy loss $\lss$ will 
not behave correctly.
 \begin{example}\label{ex:grad}{\rm 
 Consider an  HMC problem with $n+1$ classes $A, A_1, \ldots, A_n$. Assume
 \begin{enumerate}
     \item $A \in  \cap_{i=1}^n \mathcal{D}_{A_i}$,
     \item $h_A > \max(h_{A_1}, \ldots, h_{A_n})$, and
     \item $y_A=0$ while $y_{A_1} = \cdots = y_{A_n}=1$.
 \end{enumerate}
Then, for the standard binary cross-entropy loss $\lss$, we obtain:
$$
\lss = \lss_{A} + \sum_{i=1}^n \lss_{A_i}, \qquad 
\lss = - \ln(1-h_A) - n\ln(h_A), \qquad
\frac{\partial \lss}{\partial h_A} = \frac{1}{1-h_A} - \frac{n}{h_A}.
$$
Since $y_A = 0$, we would like to get  $\frac{\partial \lss_A}{\partial h_A} \geq 0$. 
However, this is possible only if $h_A \geq \frac{n}{n+1}$: if $n = 1$, then we need $h_A \geq 0.5$, while if $n = 10$, then we need $h_A \geq 10/11 \sim 0.91$. 
On the other hand, for {\rm \loss{}}, we obtain:
$$
{
\begin{aligned}
{\text{\rm \loss}} = {\text{\rm \loss}}_A +\! \sum_{i=1}^n\!{\text{\rm \loss}}_{A_i}, \quad
{\text{\rm \loss}} = - \ln(1 - h_A)  + \! \sum_{i=1}^n \!{\text{\rm \loss}}_{A_i}, \quad
\frac{\partial {\text{\rm \loss}}}{\partial h_A} = \frac{1}{1- h_A}.
\end{aligned}}
$$
No matter the value of $h_A$, we get $\frac{\partial \text{\rm\loss}_A}{\partial h_A} > 0$.\hfill$\lhd$
}\end{example}
 
Finally, due to  both $\module$ and $\loss$, \hmcsys{$h$} has the ability of delegating the prediction on a class~$A$ to one of its subclasses in $\mathcal{D}_A$.

\begin{definition}[Delegate]
Let $(${\problem}$,\Pi)$ be an HMC problem. Let $h$ be a model for {\problem}. Let $A$ and $A_1$ be two classes with $A_1 \in {\mathcal D}_{A}$. \hmcsys{$h$} {\sl delegates} the prediction on  $A$ to $A_1$ for a data point, if~\hmcsys{$h$}$_{A} = h_{A_1}$ and $h_{A_1} >  h_{A}$.
\end{definition}

Consider the basic case in Section~\ref{sec:basic_case} and the figures in the last column of Figure~\ref{fig:dec_bound_figs_before}. \hmcsys{$h$} 
delegates the prediction on $A$ to $A_1$ for 
\begin{enumerate}
    \item 0\% of the points in $A_1$ when $R_1 \cap R_2 = R_1$ as in the top figure, 
    \item 100\% of the points in $A_1$ when $R_1 \cap R_2 = \emptyset$ as in the middle figure, and 
    \item 85\% of the points in $A_1$ when $R_1$ and $R_2$ are as in the bottom figure.
\end{enumerate}

\section{Multi-Label Classification with Hard Logical  Constraints}\label{sec:mc_with_constr}

In this section, we first introduce  logically constrained  multi-label classification ({\cmc}) problems, and then (analogously to what we did in the HMC case)
we  present the intuitions at the basis
of our model \system{$h$} through a simple {\cmc} problem.  %
Thereafter, we finally provide the general solution. %
We keep the same notation and terminology as introduced in the HMC case.

\subsection{Preliminaries}

Borrowing notation and concepts from the area of logic programming, we consider {\sl logically constrained  multi-label classification} (\cmc) problems, defined as MC problems with a finite set $\Pi$ of {\sl constraints} or {\sl (normal) rules}   
$r$ having the form~(\ref{eq:rule}): 
\begin{equation} \label{eq:rule}
A_1, \ldots, A_k, \neg A_{k+1}, \ldots, \neg A_n \to A, \qquad (0 \leq k \leq n),
\end{equation}
where $A,A_1,\ldots,A_n$ are classes. We also assume, w.l.o.g., that $A_i \neq A_j$ for $1 \leq i < j \leq k$ and for $k+1 \leq i < j \leq n$. 
We call $head(r)\,{ = }\,A$ the {\sl head} of $r$,
and $body(r)$ $=$ $body^+(r) \, \cup \, body^-(r)$ the {\sl body} of $r$, 
where 
 $body^+(r) =  \{A_1, \ldots,A_k\}$ and 
 $body^-(r)= \{\neg A_{k+1}, \ldots,\neg A_n \}$. 
 We say that $r$ is {\sl definite} if $n=k$.

Constraint (\ref{eq:rule}) imposes that for each data point $x$ and model $m$, 
if $m$ predicts the classes $A_1, \ldots, A_k$ and not $A_{k+1},\ldots,A_n$, then $m$ must also predict $A$. Given this logical interpretation, we can thus define the concepts of logical violation and coherency, 
which generalize the corresponding definitions given for the hierarchical case.
\begin{definition} %
Let $(${\problem}$,\Pi)$ be an {\cmc} problem. Let $m$ be a model for {\problem}. If for a data point and a constraint $r \in \Pi$ of the form (\ref{eq:rule}), $m$ predicts $A_1, \ldots, A_k$ and not $A_{k+1}, \ldots, A_n, A$, then $m$ commits a {\sl logical violation} with respect to $r$.  If $m$ commits no logical violations, then $m$ is  {\sl coherent} with respect to $\Pi$.
\end{definition}

The above definition allows us to determine whether any model $m$ is coherent with respect to the given constraints. However, 
we want to go beyond coherency and generalize what we did in the HMC setting: whenever convenient, exploit the constraints 
to compute a value for the classes in the head to ensure coherency and improve performance.

For ease of presentation, assume that we have a single constraint $r$ of the form (\ref{eq:rule}).

In the special case where $r$ is definite ($n=k$), we can associate with the head $A$ a value that is at least  the  smallest value associated with the classes in the body, i.e., we can
set
$$
m_A = \min(m_{A_1}, \ldots, m_{A_k}).
$$
In this case, the constraint $r$ is always satisfied for any threshold $\tv$. 
This corresponds to interpreting (\ref{eq:rule}) according to the G\"odel t-norm  $T_G$~\cite{metcalfe}, which is the only function~$T: [0, 1]^2 \to [0, 1]$ that, for every $a,b,c \in [0,1]$, satisfies the following properties (common to all t-norms):
$$
\begin{aligned}
T(a,b) &= T(b,a), &T(a,1) &= a, \\
T(a,T(b,c)) &= T(T(a,b),c), &T(a,0) &= 0, \\
& a \le b \to  T(a,c) \le T(b,c), 
\end{aligned}
$$
and also the following ({\sl idempotency}, characterizing $T_G$):
$$
T(a,a) = a.
$$

If $r$ is not definite ($n > k$), given $m_{A_{k+1}}, \ldots, m_{A_n}$, we need to 
compute values $\ov{m}_{A_{k+1}}, \ldots,$ $\ov{m}_{A_n}$ such that, for each class $A_i \in \{A_{k+1}, \ldots, A_n\}$ and threshold $\theta$,
\begin{enumerate}
    \item $\ov{m}_{A_i} = 1$ when $m_{A_i}= 0$, and  $\ov{m}_{A_i} = 0$ when $m_{A_i} = 1$,
    \item $\ov{m}_{A_i}$ is strictly decreasing and continuous (small changes to the value of $m_{A_i}$ should correspond to small changes in the value of $\ov{m}_{A_i}$), and
    \item
    $\ov{m}_{A_i} = \theta$ when $m_{A_i} = \theta$.
\end{enumerate}
The first two conditions say that the function $\ov{v}$ of $v \in [0,1]$ is a
 {\sl strict negation} \cite{metcalfe},%
\footnote{A negation is non-strict if it is either non-strictly decreasing or non-continuous. An example of a non-strict negation is the residual negation in the G\"odel t-norm according to which we would have $\ov{m}_A = 1$ if $m_A = 0$, and $\ov{m}_{A} = 0$, otherwise.} and, together with the third entail
\begin{enumerate}
    \item if $m_{A_i} > \theta$, then $\ov{m}_{A_i} < \theta$,
    \item if $m_{A_i} < \theta$, then $\ov{m}_{A_i} > \theta$.
\end{enumerate}
For any threshold $\theta$ there are infinitely many functions $\ov{v}$ satisfying such requirements. A simple solution is to require $\ov{v}$ to be piecewise linear with two segments joining when~$\ov{v} = v = \theta$, in which case,
\begin{enumerate}
    \item $\ov{v}$ is a {\sl strong negation} \cite{metcalfe}, since $\ov{\ov{v}} = v$, and
    \item if $\theta = 0.5$, we obtain $\ov{v} = 1 - v$, i.e., the {\sl standard negation} in fuzzy logics.
\end{enumerate}
For simplicity, from here on, we assume to have the standard negation, i.e., to fix the threshold $\theta$ to $0.5$ and $\ov{v} = 1-v$, for each $v \in [0,1]$. All the definitions and results generalize to the case in which we have an arbitrary strict negation with $\ov{\theta} = \theta$.

Given the above,
we can now introduce the concept of constraint violation, generalizing the corresponding definition of hierarchy violation.
\begin{definition}
Let $(${\problem}$,\Pi)$  be an {\cmc} problem.  Let $m$ be a model for {\problem}. 
If for a data point and for a constraint (\ref{eq:rule}) in $\Pi$, $m$ does not satisfy
\begin{equation}\label{eq:safer}
     \min(m_{A_1}, \ldots, m_{A_k}, \ov{m}_{A_{k+1}}, \ldots, \ov{m}_{A_n}) \le m_A, 
\end{equation}
then $m$ commits a {\sl constraint violation}. 
\end{definition}

The following theorem easily follows from the previous two definitions.

\begin{theorem}
Let $(${\problem}$,\Pi)$ be an {\cmc} problem.  Let $m$ be a model for {\problem}. If $m$ does not commit constraint violations, then $m$ is coherent with respect to $\Pi$.
\end{theorem}

\subsection{Basic Case}\label{sec:basic_case}
We now present the main ideas behind our model \system{$h$} through a simple \cmc{} problem.
Assume that we have an MC problem with three classes $A$, $A_1$, and $A_2$, and we know that $A_1$ and $A_2$ are subsets of $A$, and that $A_2$ 
includes the set of data points belonging to $A$ and not to $A_1$. Then, 
$A_1 \cup A_2 \subseteq A$ can be imposed with the constraints
\begin{equation}\label{eq:constr_a}
    A_1 \to A; \qquad A_2 \to A, 
\end{equation}
having the form  %
(\ref{eq:basic_constr-2}),
while $A \setminus A_1 \subseteq A_2$ can be expressed as
\begin{equation}\label{eq:constr_a2}
A, \neg A_1 \to A_2,
\end{equation}
which imposes, to any model $m$ that, for each $x \in \mathbb{R}^D$,  if $m$ predicts $A$ and not $A_1$, then $m$ must also predict $A_2$.

Our goal is to develop a method that is able to leverage standard neural network approaches for MC problems, while exploiting all the above constraints in order to produce predictions that are guaranteed to satisfy the constraints while improving performance and extending the method presented for HMC problems.

To understand how the three constraints can be exploited to improve performance, assume that $D\,{=}\,2$, and consider the yellow ($R_1$) and green ($R_2$) rectangles in Figure~\ref{fig:dec_bound_figs_gen}.
Assume that $A = R_1 \cup R_2$, $A_1 = R_1$, and $A_2 = R_2 \setminus R_1$. 
Let $f$ be a neural network with one output for each class to be learned. Intuitively, when $R_1$ and $R_2$ are  as in the first row of Figure~\ref{fig:dec_bound_figs_gen}, we can expect to be more difficult for $f$ to learn $A$ than to learn $A_1$ and~$A_2$. Hence, we would like to exploit the information coming from (\ref{eq:constr_a}) to learn $A$, given~$A_1$ and~$A_2$. 
On the other hand, when the two rectangles are arranged as in the second row, we can expect to be more difficult for $f$ to learn $A_2$ than $A$ and $A_1$. In this case, we would like to exploit the information coming from (\ref{eq:constr_a2}) to learn $A_2$, given $A$ and $A_1$. 
Finally, when~$R_1$ and $R_2$ are arranged as in the third row of the figure, learning both $A$ and $A_2$ will be difficult, and hence we would like to be able to exploit all the constraints in~(\ref{eq:constr_a}) and~(\ref{eq:constr_a2}) to improve performance.

As for \hmcsys{$h$}, we can achieve our goal in two steps. In the first step, we build a new neural network consisting of two modules: (i)~a bottom module $h$, which can be any neural network with one output 
for $A$, $A_1$, and $A_2$, respectively, and (ii) an upper {\sl constraint module} (\module), 
that takes as input the output of the bottom module and imposes the constraints. We call the obtained neural network  \emph{coherent-by-construction network}~({\system{$h$}}). 

Consider a data point $x$. Let $h_A$, $h_{A_1}$, and $h_{A_2}$ be the outputs of $h$ for the classes $A$, $A_1$, and $A_2$ respectively. Let $y_A$, $y_{A_1}$, and $y_{A_2}$ be the ground truth for the classes $A$, $A_1$, and $A_2$, respectively. Let $\module_A$, $\module_{A_1}$, and $\module_{A_2}$ be the outputs of {\module} (which are the outputs of \system{$h$}).

We want \system{$h$} to extend the set of classes associated with $x$ by the bottom module $h$, exploiting, and thus satisfying, the constraints. This is obtained by defining $\module_A$, $\module_{A_1}$, and $\module_{A_2}$ to be the smallest values such that
\begin{equation}\label{eq:basic_cm}
    \begin{aligned}
    &\module_A  \,\, =  \max(h_A, \module_{A_1}, \module_{A_2}), \\
    &\module_{A_1}  =  h_{A_1},\\
    &\module_{A_2}  =  \max(h_{A_2}, \min(\module_{A}, \ov{\module}_{A_1})).
    \end{aligned}
\end{equation}
Indeed, the first equation ensures that (i) $x$ will be associated with the class $A$ whenever $h$ already predicts it, and that (ii) $\module_{A_1}$, $\module_{A_2} \leq \module_A$; thus guaranteeing that (\ref{eq:constr_a}) is satisfied. The other equations have a similar reading. Depending on the values of $h_A$, $h_{A_1}$, and $h_{A_2}$, (\ref{eq:basic_cm}) may admit more than one solution,
but we will show (see Example \ref{ex:stratum7} and Theorem \ref{thm:8}) that none of them has a value for $\module_A$, $\module_{A_1}$, and $\module_{A_2}$ smaller than that defined by
$$
\begin{aligned}
    &\module_A  \,\, =  \max(h_A, h_{A_1}, h_{A_2}), \\
    &\module_{A_1}  =  h_{A_1},\\
    &\module_{A_2}  =  \max(h_{A_2}, \min(h_{A}, \ov{h}_{A_1}), \min(h_{A_1},\ov{h}_{A_1})),
    \end{aligned}
$$
which we define to be the outputs of $\module$.

\begin{figure*}[t]
\centering
\begin{tabular}{c@{\ \ \,}c@{\ \ \ \,}c@{\ \ \ \ }|@{\ \ \,}c@{\ \ \ \,}c@{\ \ \ \,}c@{\ \ \ }}
\multicolumn{3}{c}{Neural Network $f^+$} & 
\multicolumn{3}{c}{\system{$h$}} \\
Class $A$ & Class $A_1$ & Class $A_2$ & Class $A$ & Class $A_1$ & Class $A_2$ \\
\begin{minipage}{.135\textwidth}
    \includegraphics[width=\textwidth,trim={1.1cm 0 3.7cm 1cm},clip]{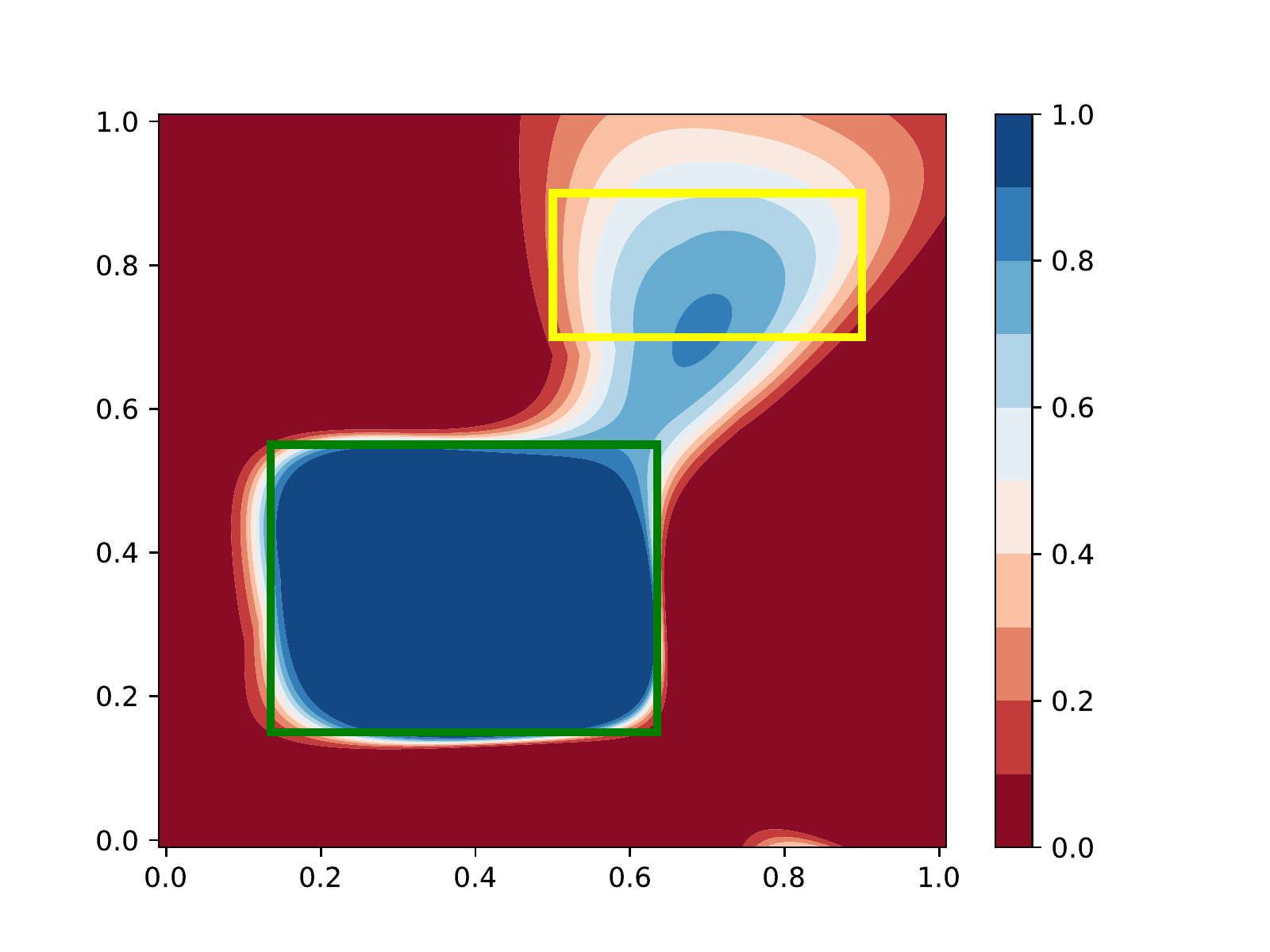}
\end{minipage} &
\begin{minipage}{.135\textwidth}
    \includegraphics[width=\textwidth,trim={1.1cm 0 3.7cm 1cm},clip]{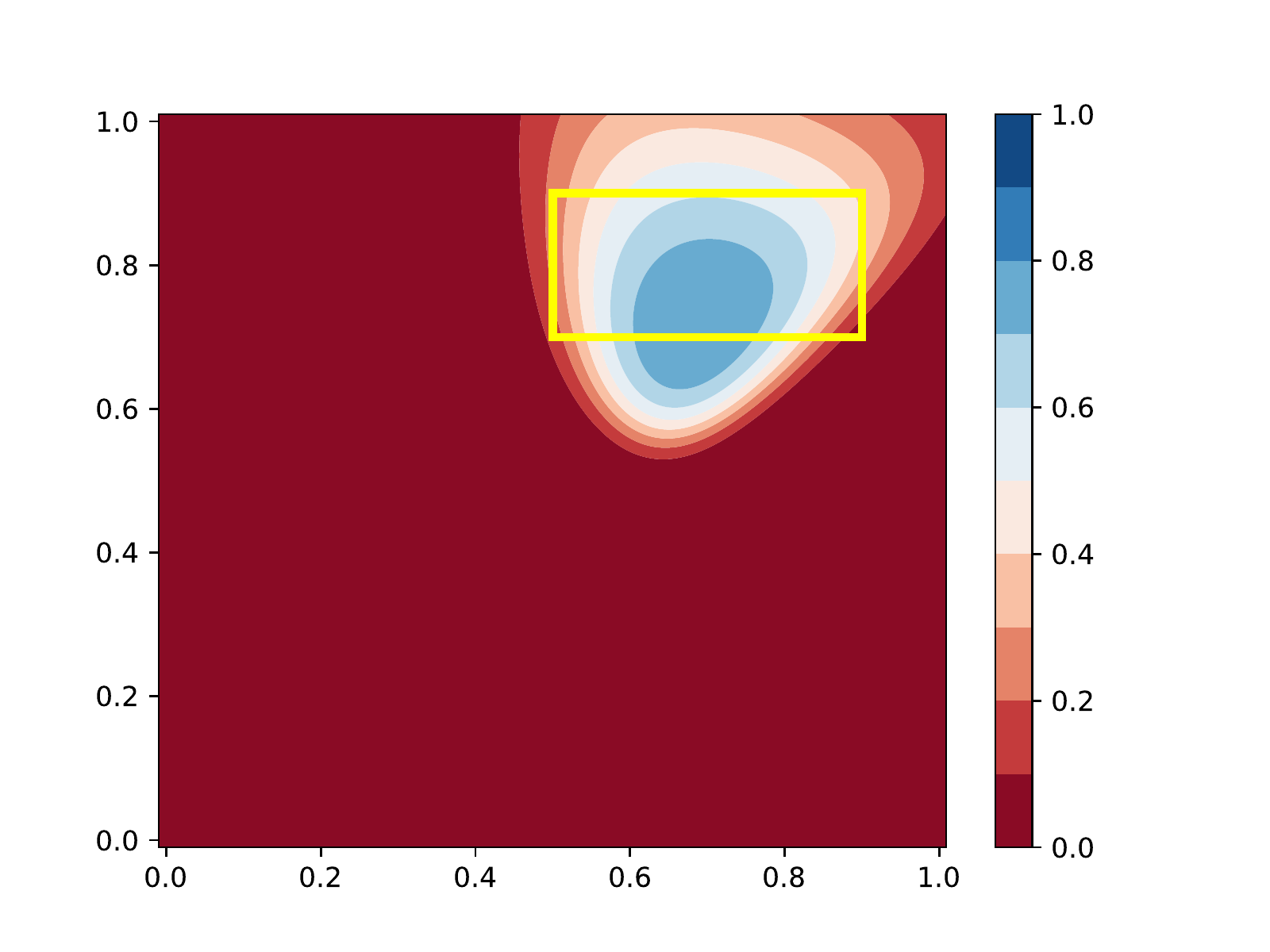}
\end{minipage} &
\begin{minipage}{.135\textwidth}
    \includegraphics[width=\textwidth,trim={1.1cm 0 3.7cm 1cm},clip]{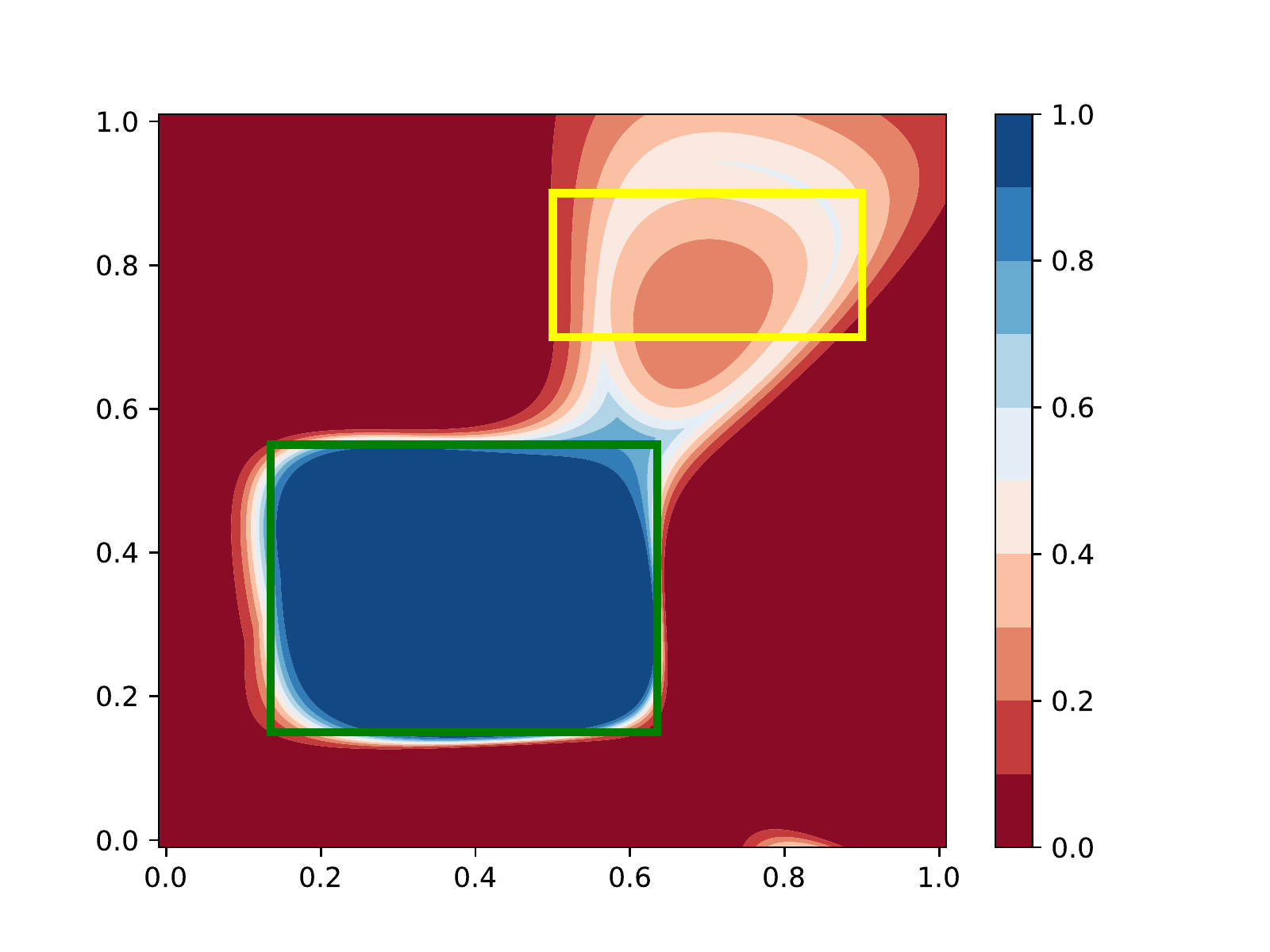}
\end{minipage} &
\begin{minipage}{.135\textwidth}
    \includegraphics[width=\textwidth,trim={1.1cm 0 3.7cm 1cm},clip]{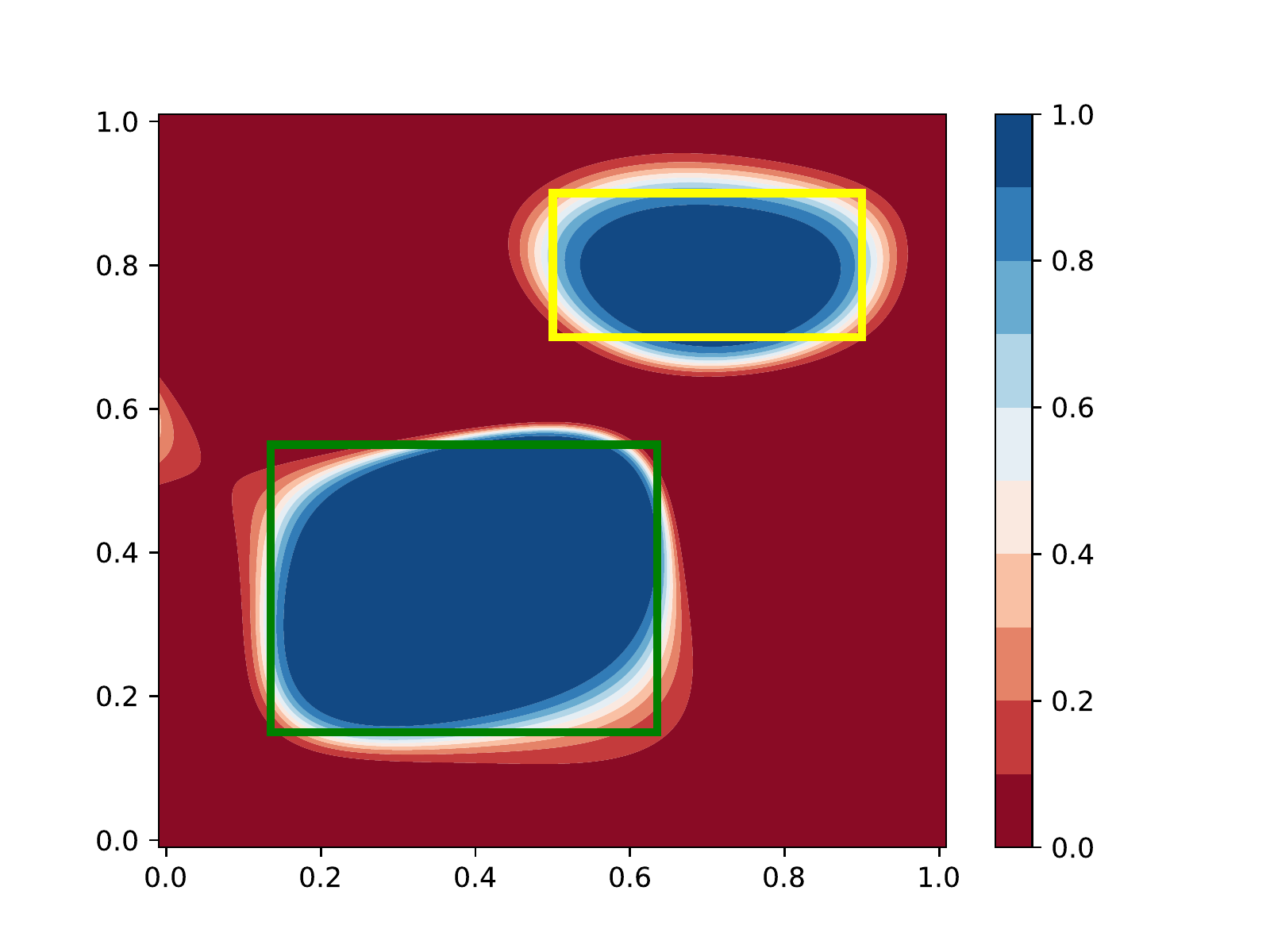}
\end{minipage} &
\begin{minipage}{.135\textwidth}
    \includegraphics[width=\linewidth,trim={1.1cm 0 3.7cm 1cm},clip]{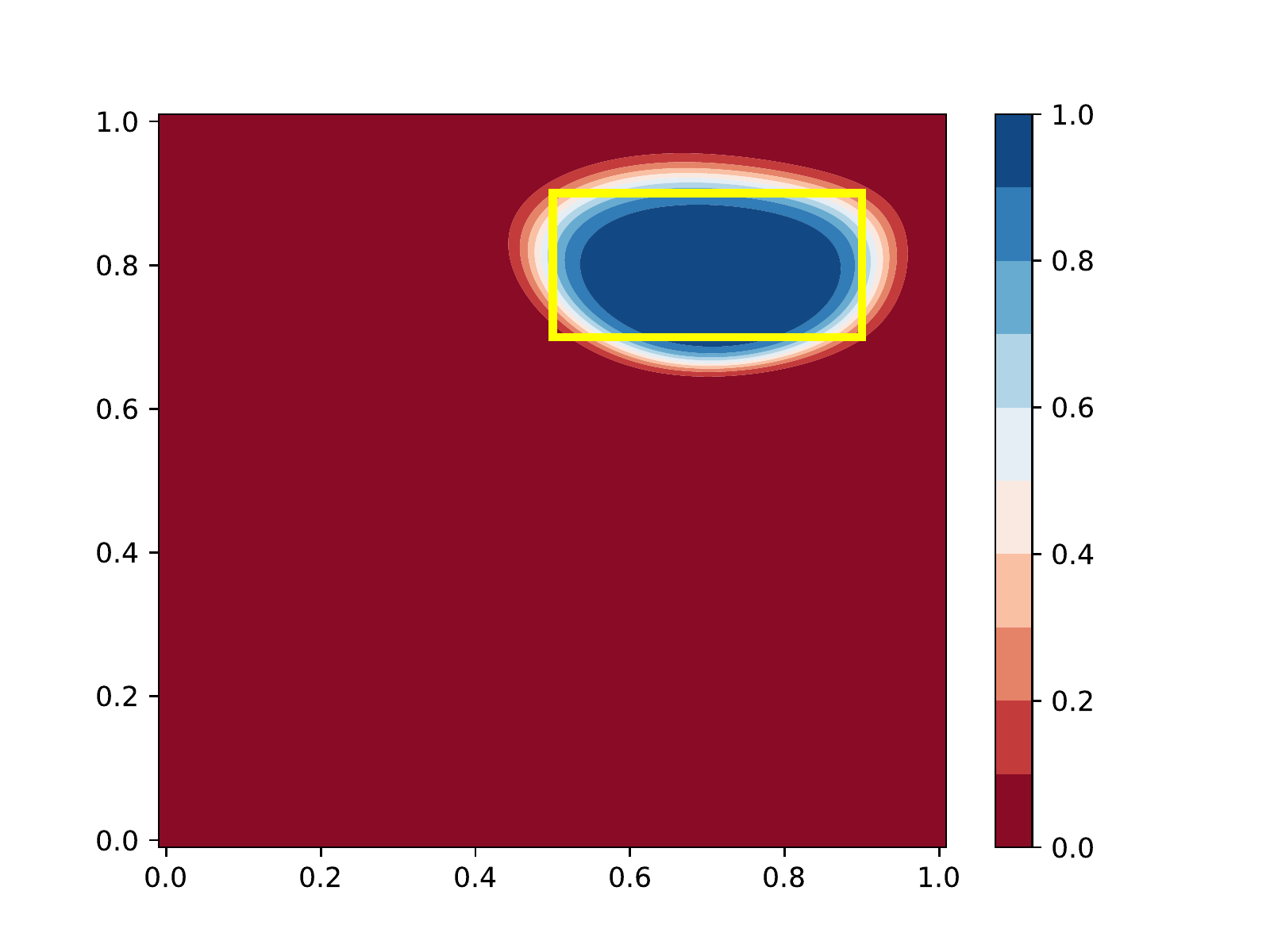}
\end{minipage} &
\begin{minipage}{.165\textwidth}
    \includegraphics[width=\linewidth,trim={1.1cm 0 1cm 1cm},clip]{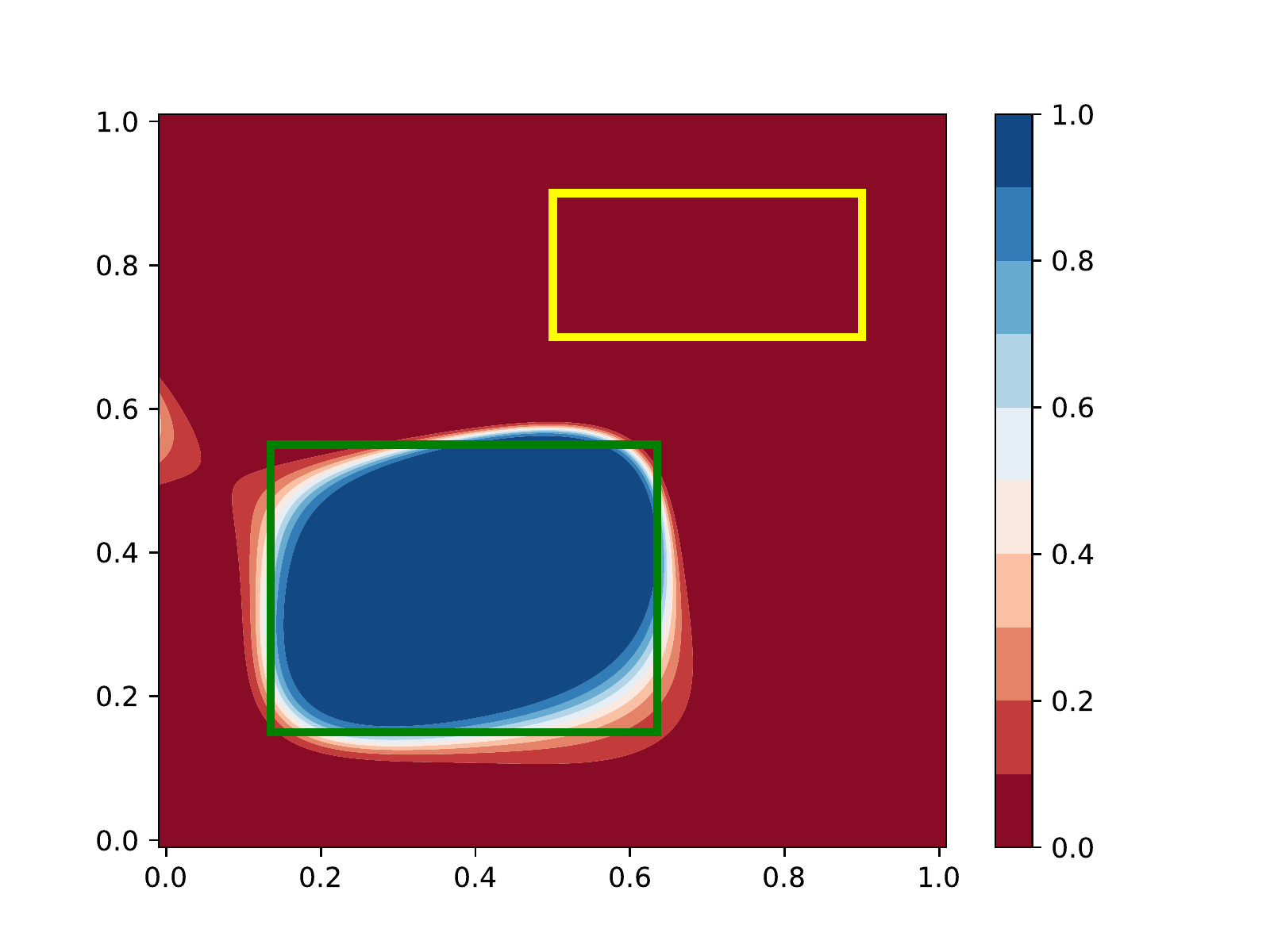}
\end{minipage} \\
\begin{minipage}{.135\textwidth}
    \includegraphics[width=\textwidth,trim={1.1cm 0 3.7cm 1cm},clip]{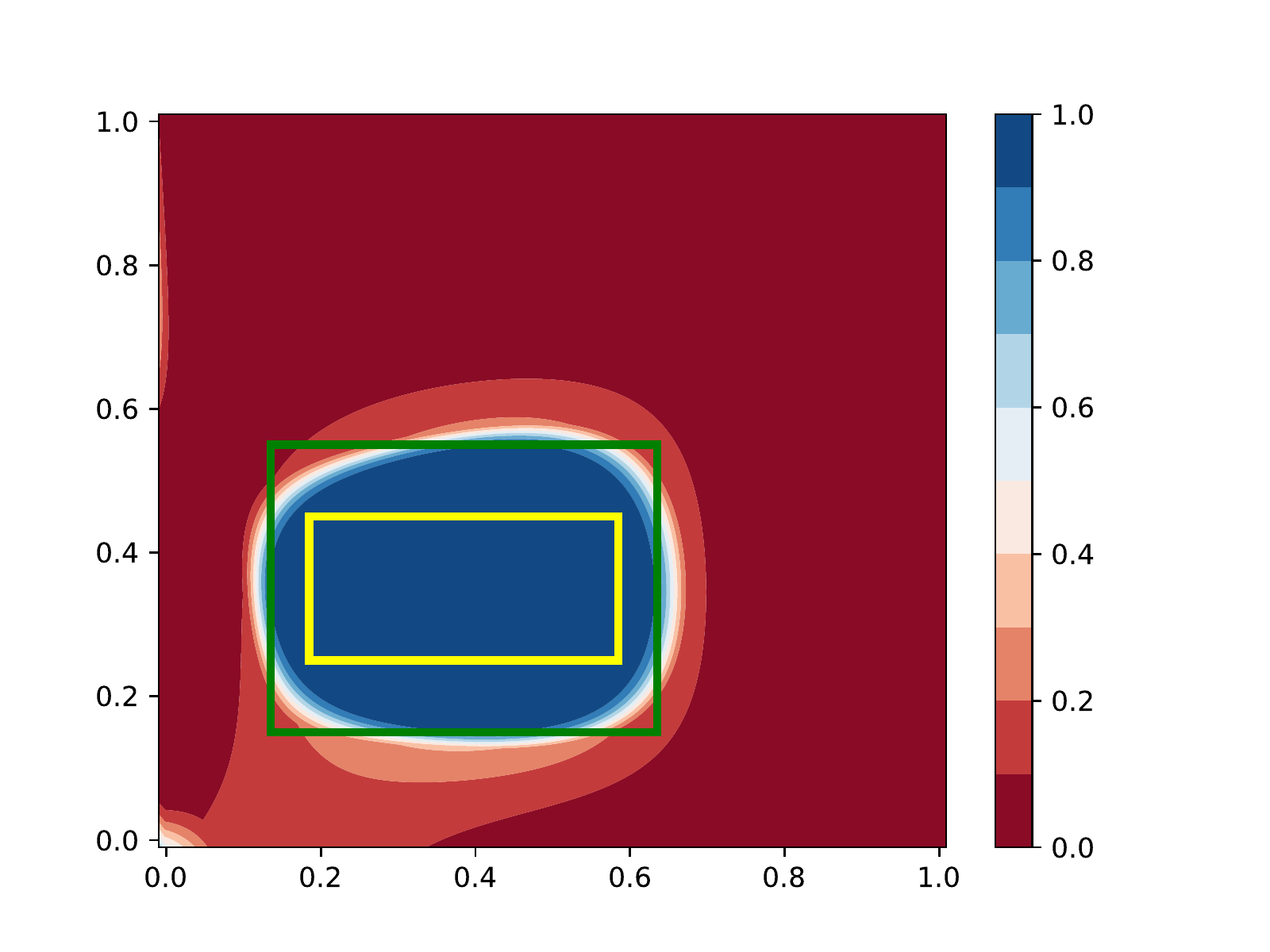}
\end{minipage} &
\begin{minipage}{.135\textwidth}
    \includegraphics[width=\textwidth,trim={1.1cm 0 3.7cm 1cm},clip]{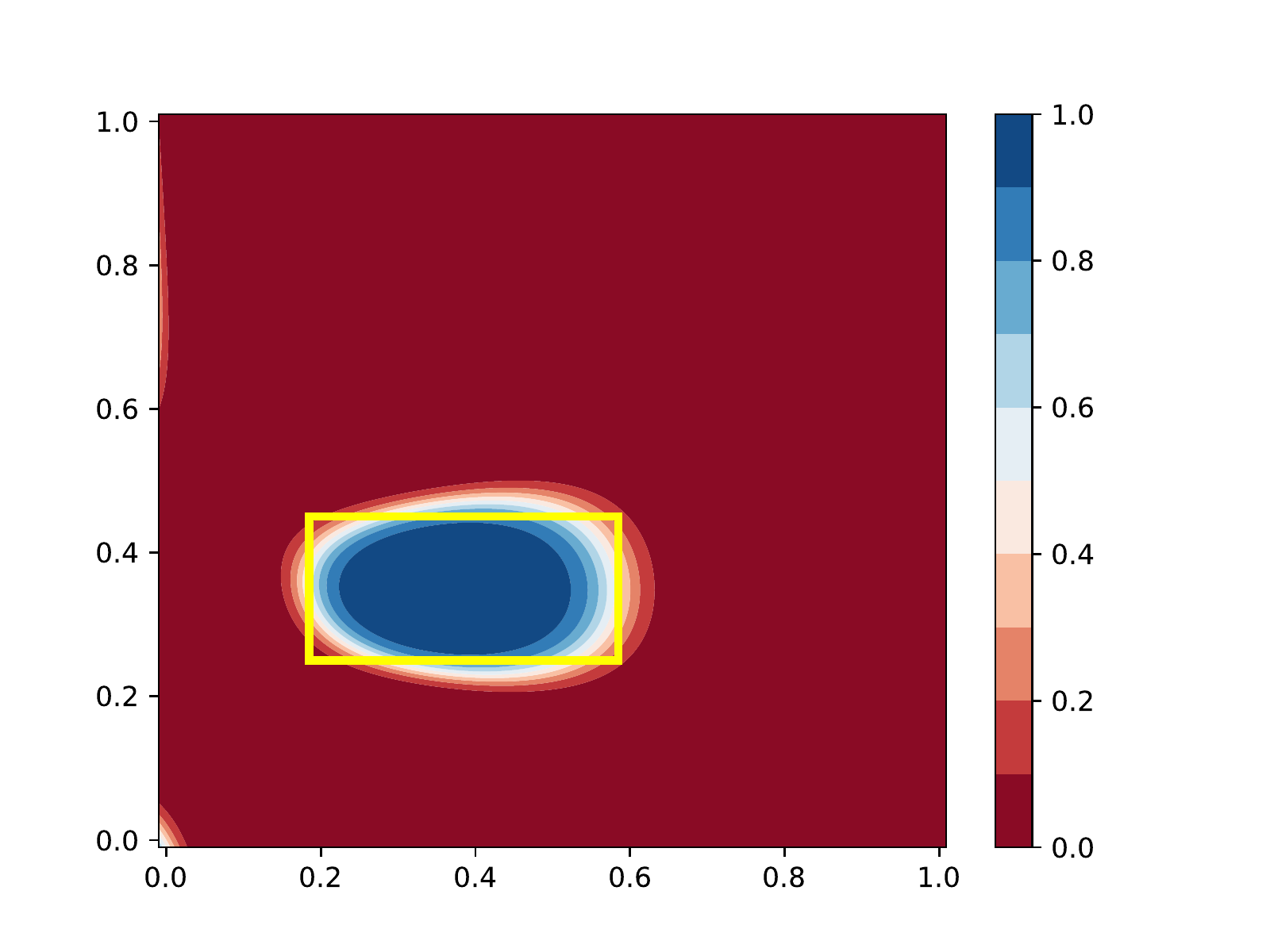} 
\end{minipage} &
\begin{minipage}{.135\textwidth}
    \includegraphics[width=\textwidth,trim={1.1cm 0 3.7cm 1cm},clip]{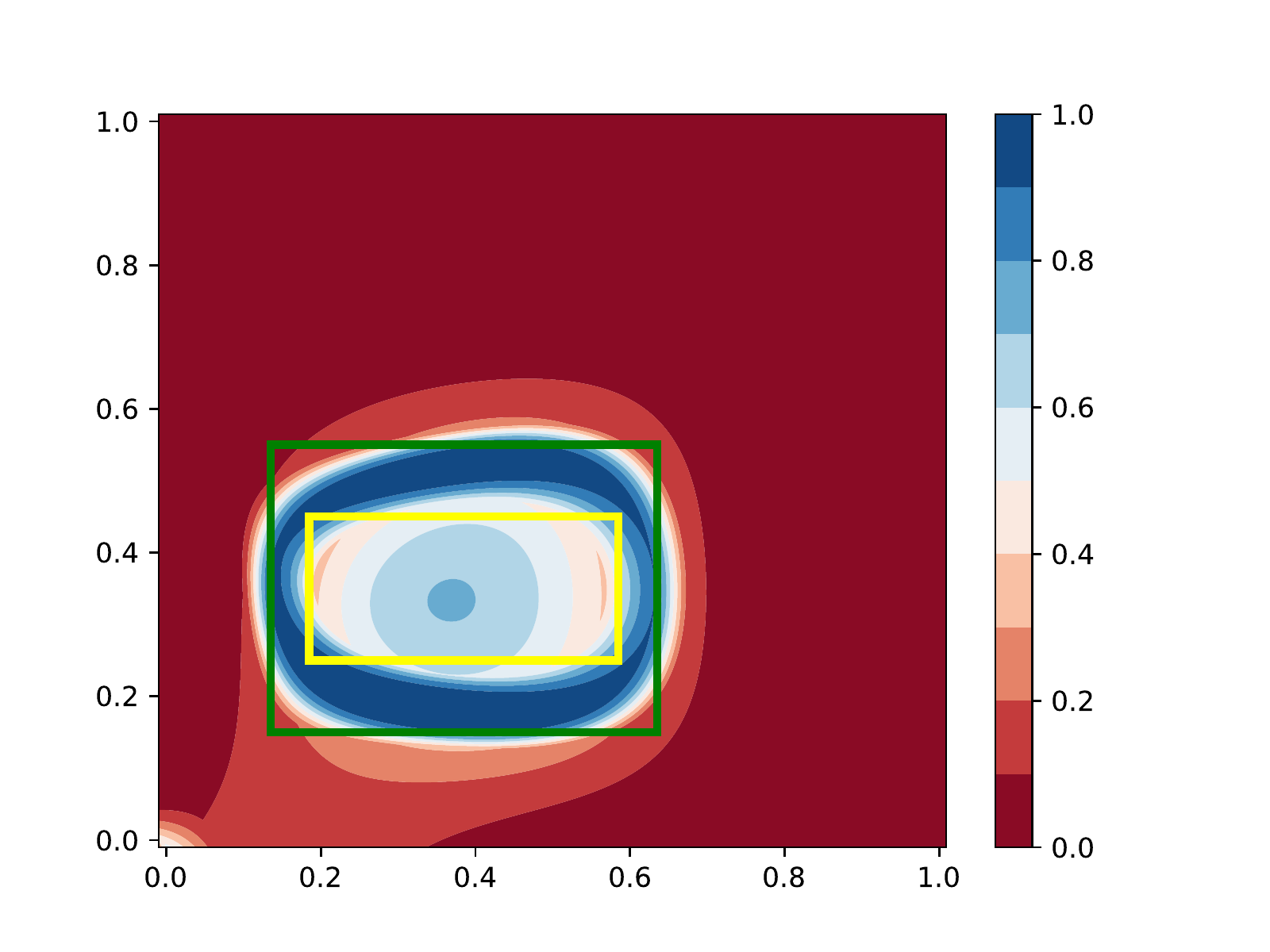} 
\end{minipage} &
\begin{minipage}{.135\textwidth}
    \includegraphics[width=\textwidth,trim={1.1cm 0 3.7cm 1cm},clip]{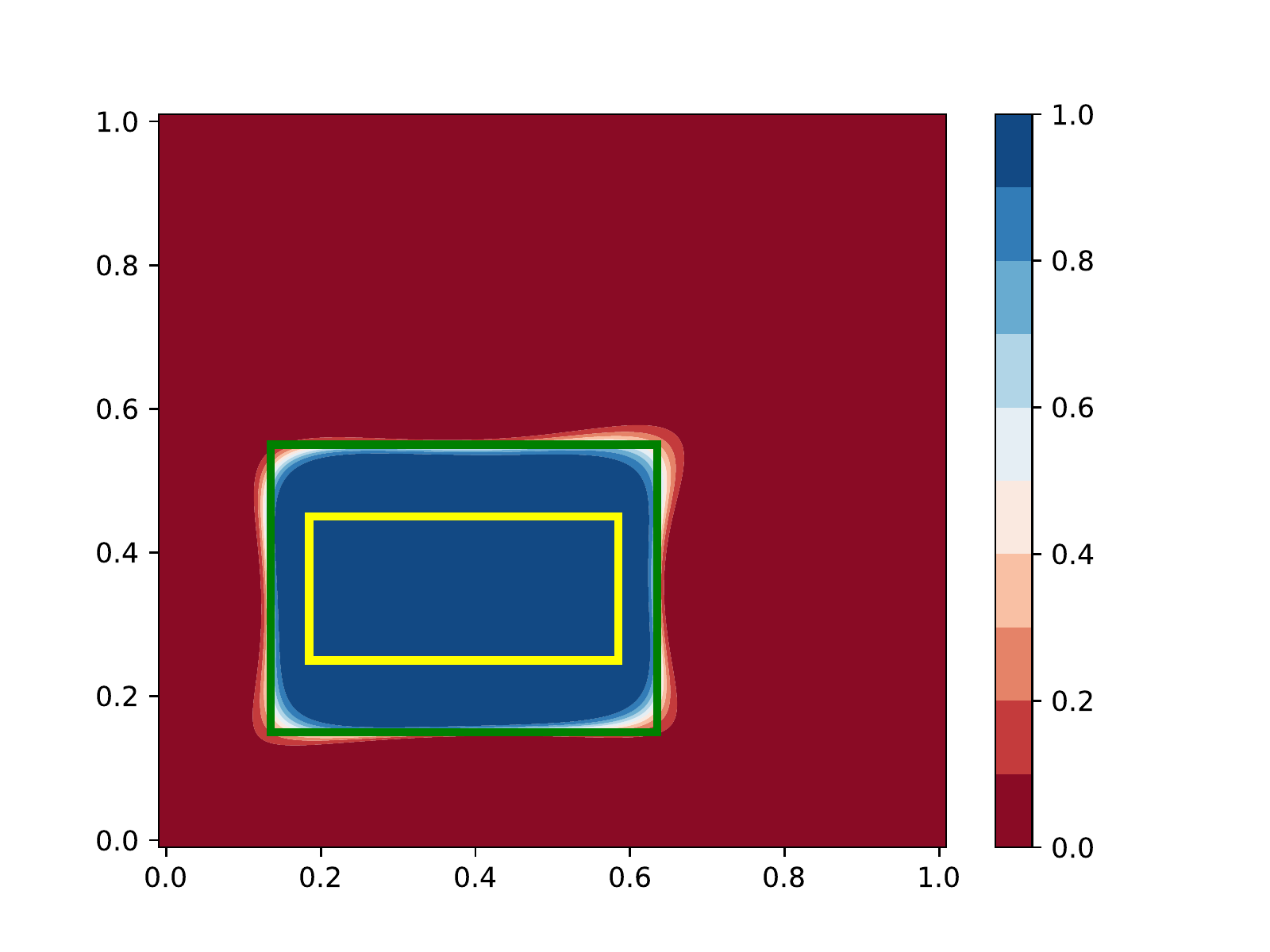} 
    \end{minipage} &
\begin{minipage}{.135\textwidth}
    \includegraphics[width=\linewidth,trim={1.1cm 0 3.7cm 1cm},clip]{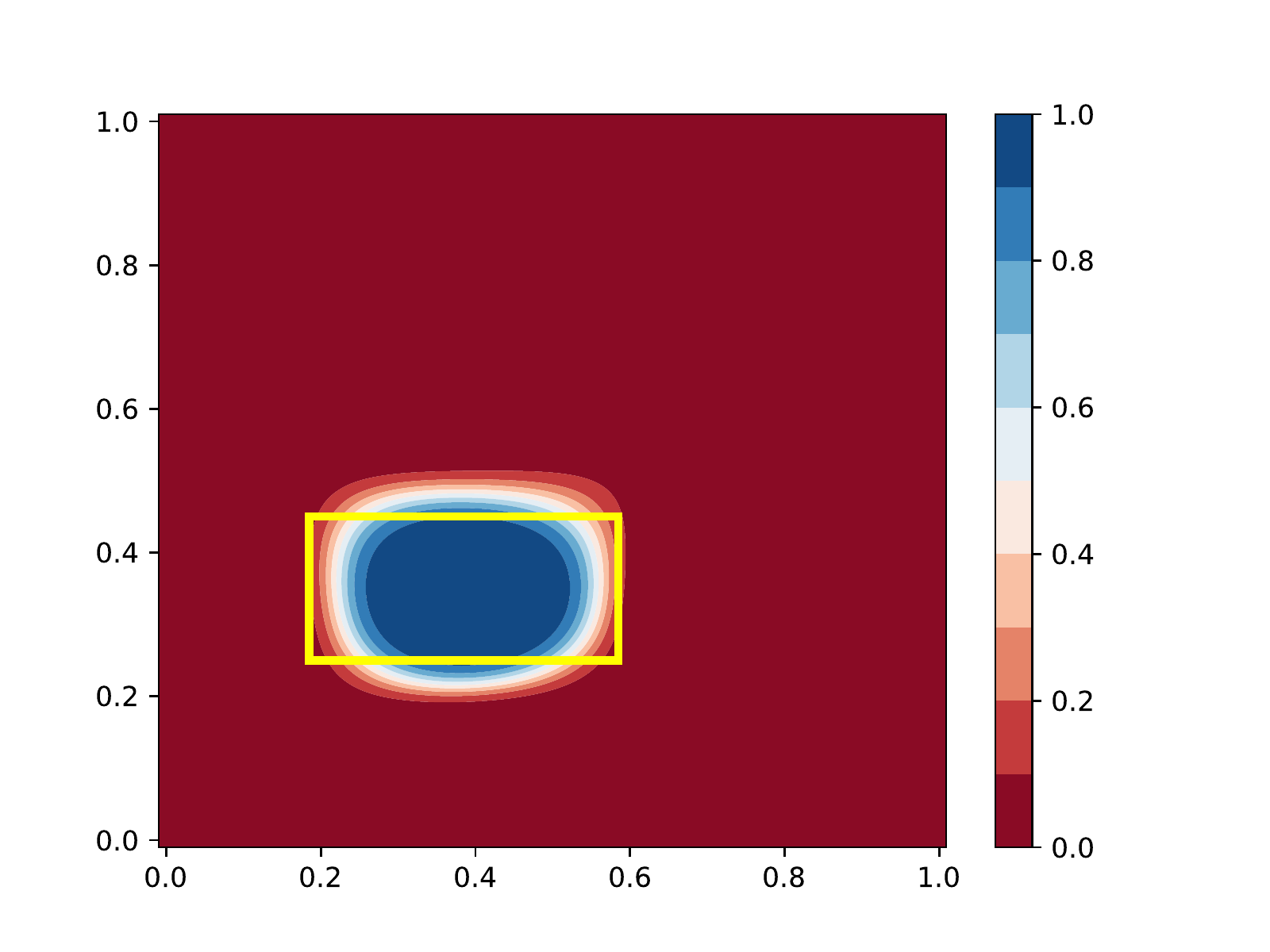}
\end{minipage} &
\begin{minipage}{.165\textwidth} 
    \includegraphics[width=\linewidth,trim={1.1cm 0 1cm 1cm},clip]{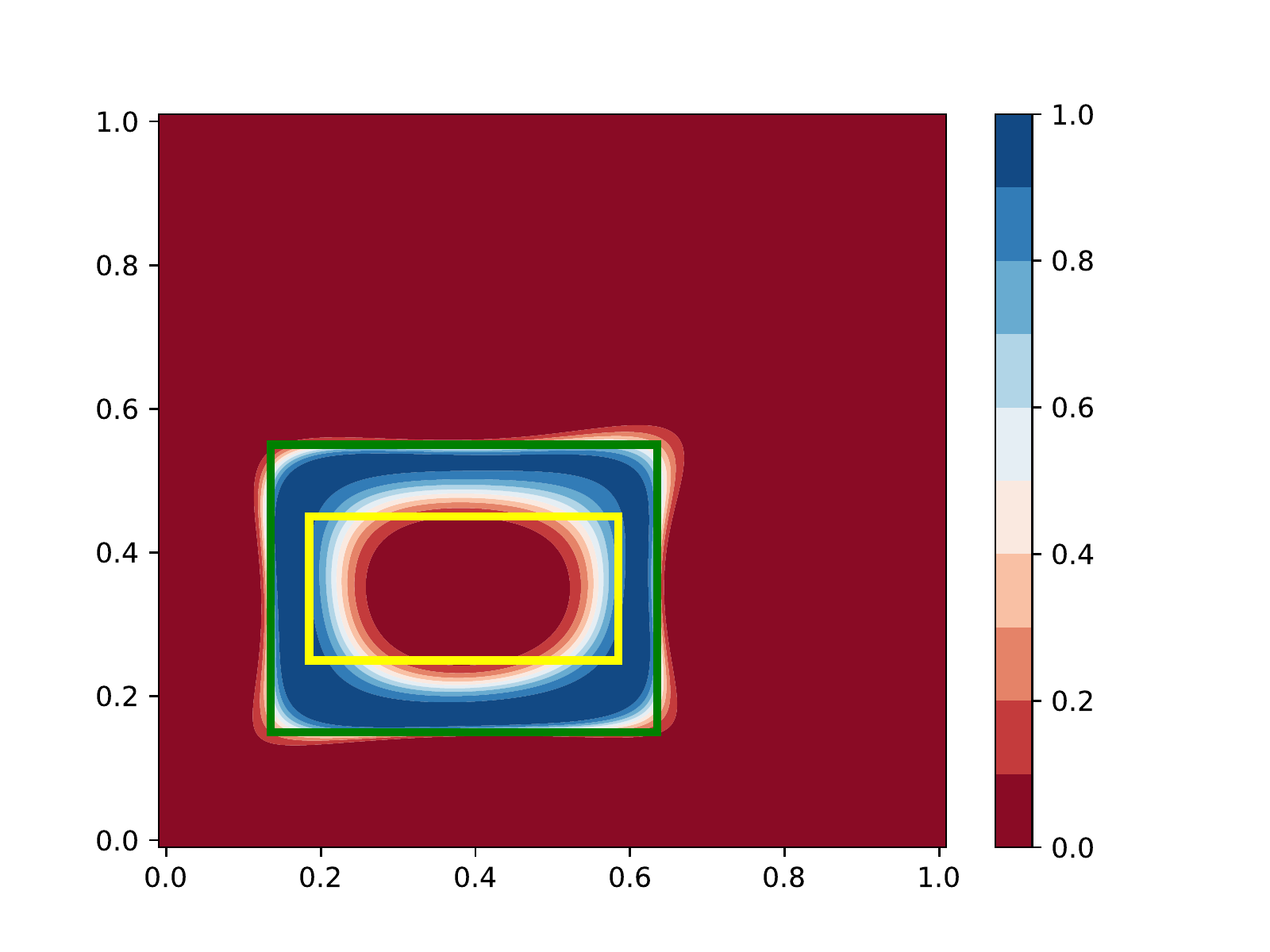}
\end{minipage} \\
\begin{minipage}{.135\textwidth}
    \includegraphics[width=\textwidth,trim={1.1cm 0 3.7cm 1cm},clip]{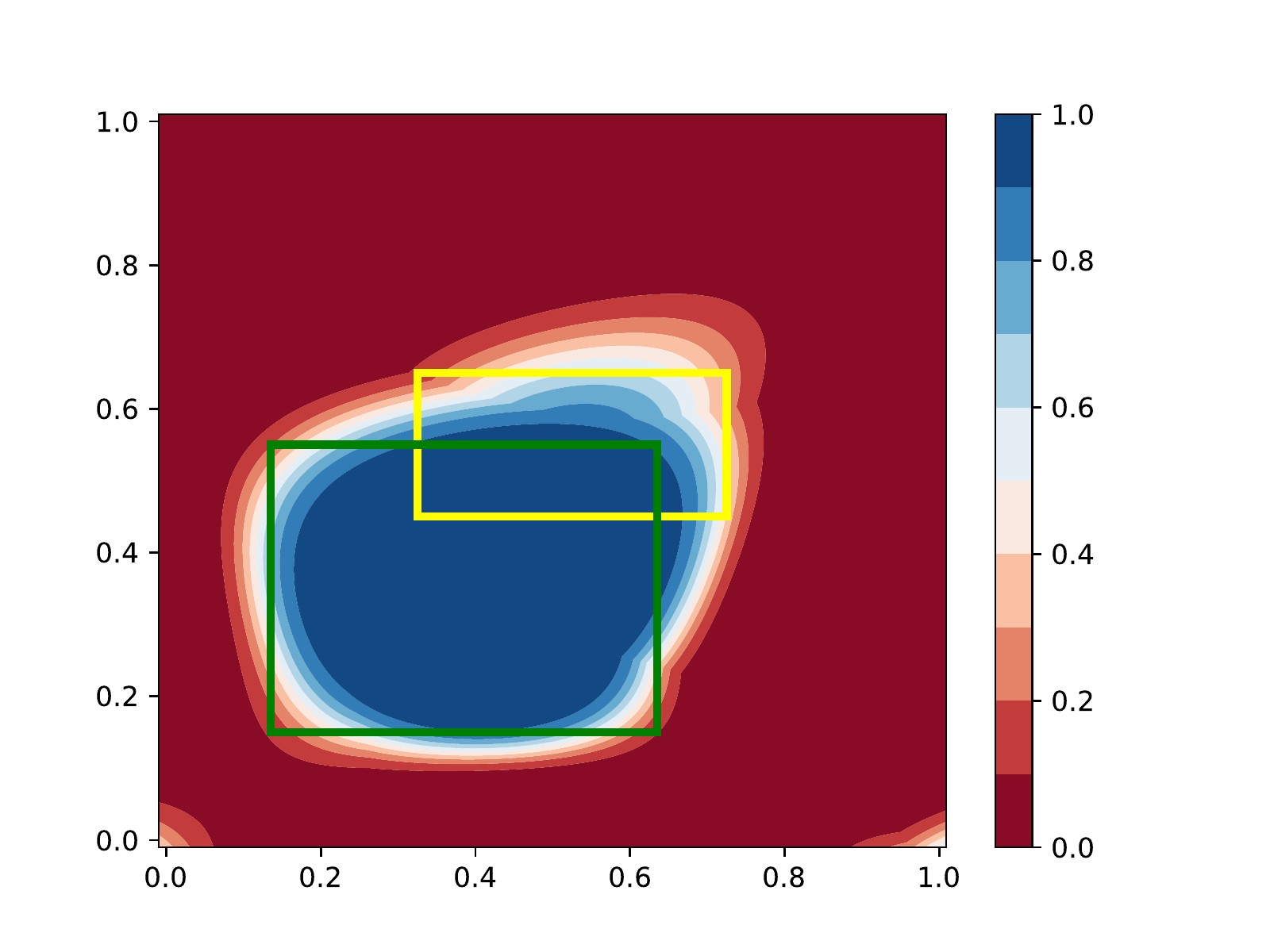}
\end{minipage} &
\begin{minipage}{.135\textwidth}
    \includegraphics[width=\textwidth,trim={1.1cm 0 3.7cm 1cm},clip]{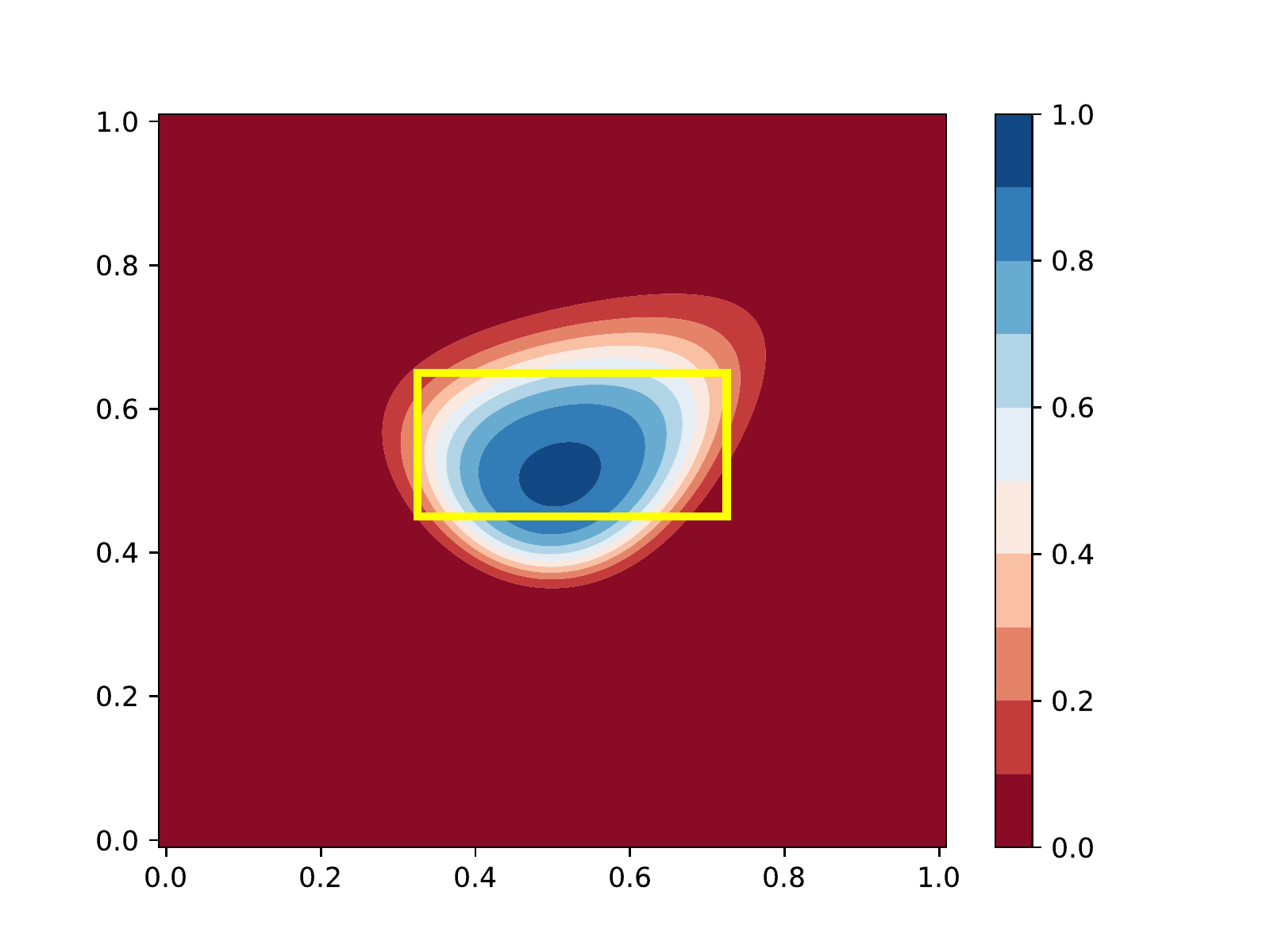}
\end{minipage} &
\begin{minipage}{.135\textwidth}
    \includegraphics[width=\textwidth,trim={1.1cm 0 3.7cm 1cm},clip]{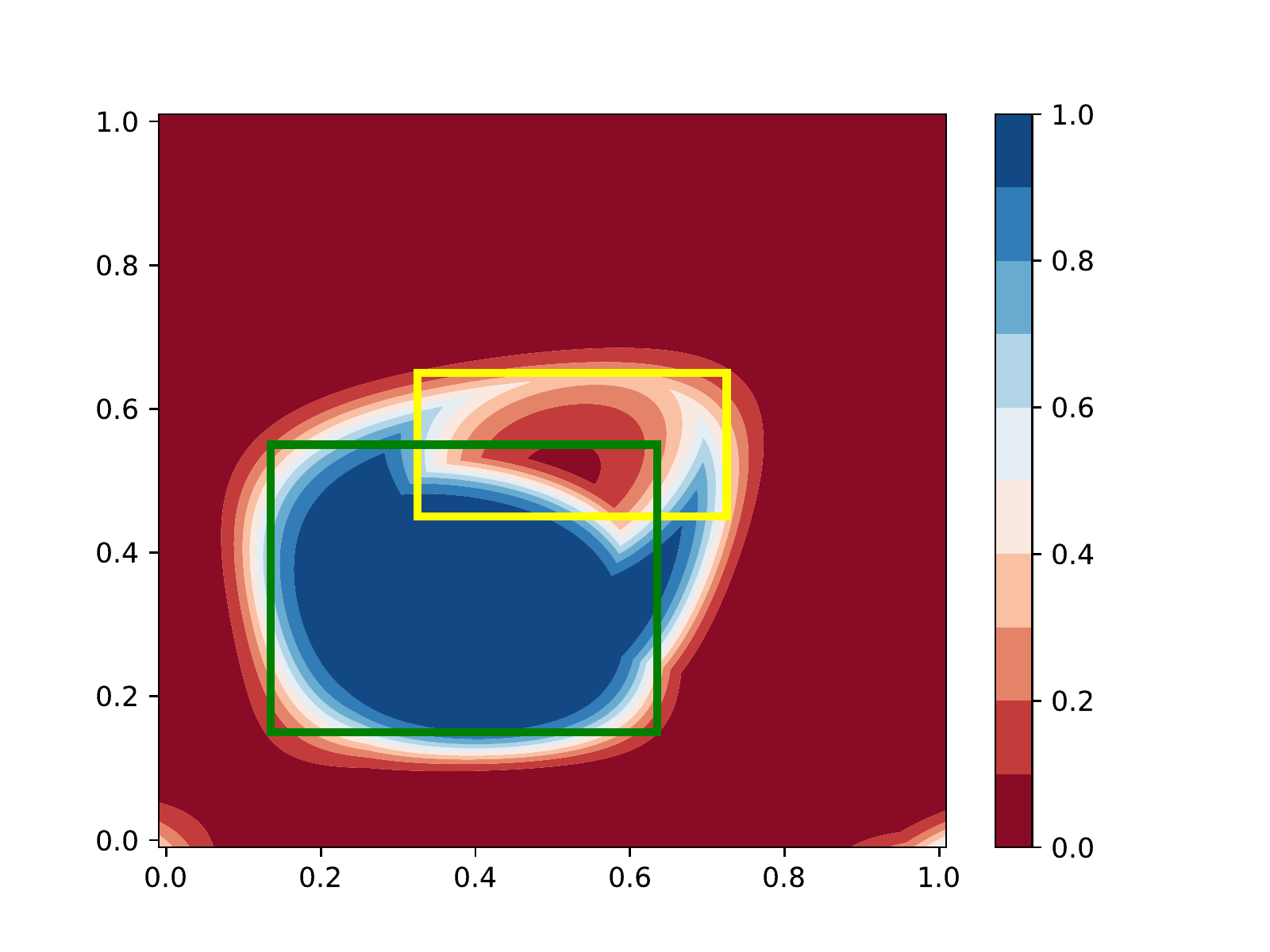}
\end{minipage} &
\begin{minipage}{.135\textwidth}
    \includegraphics[width=\textwidth,trim={1.1cm 0 3.7cm 1cm},clip]{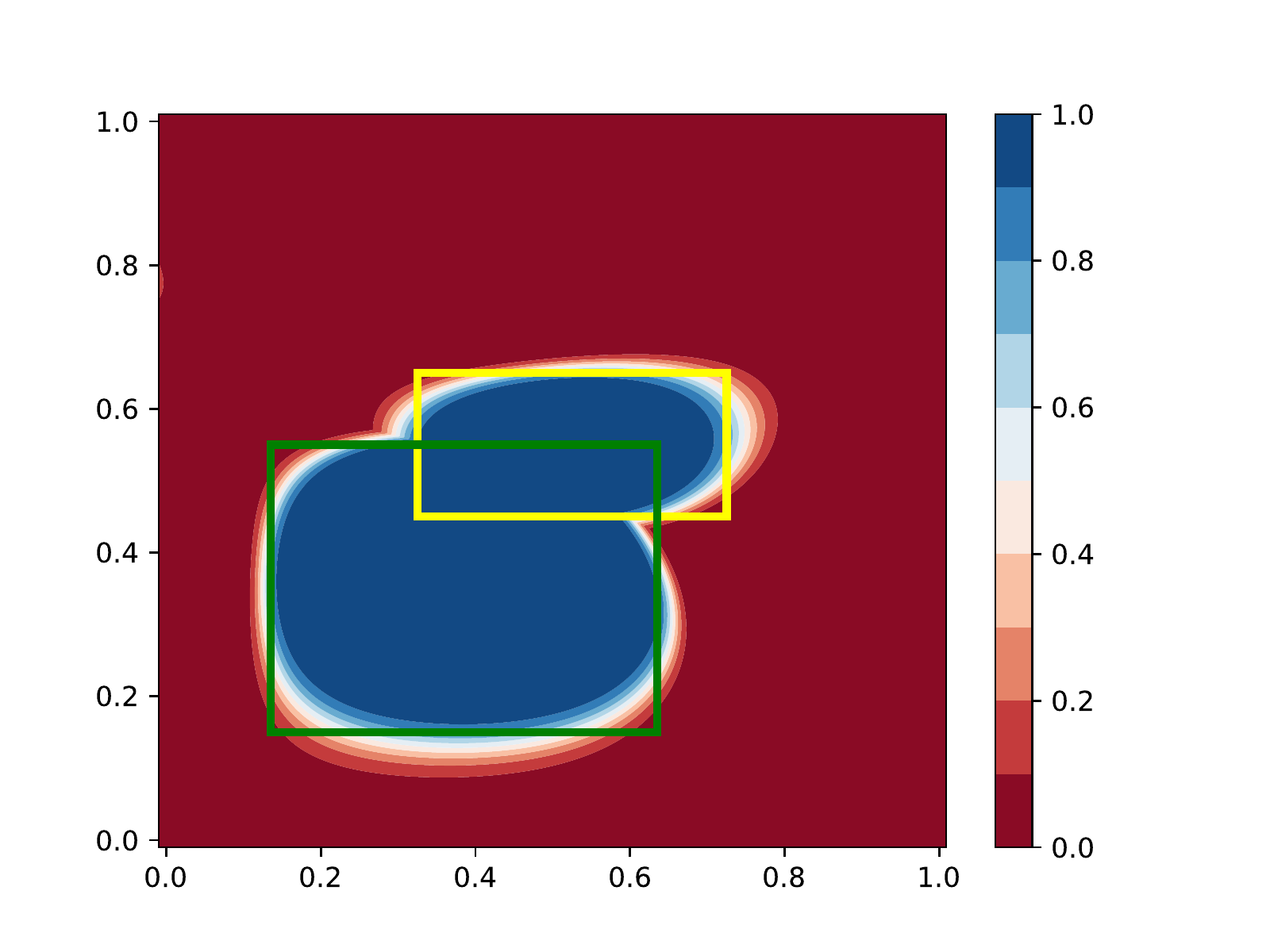}
\end{minipage} &
\begin{minipage}{.135\textwidth}
    \includegraphics[width=\linewidth,trim={1.1cm 0 3.7cm 1cm},clip]{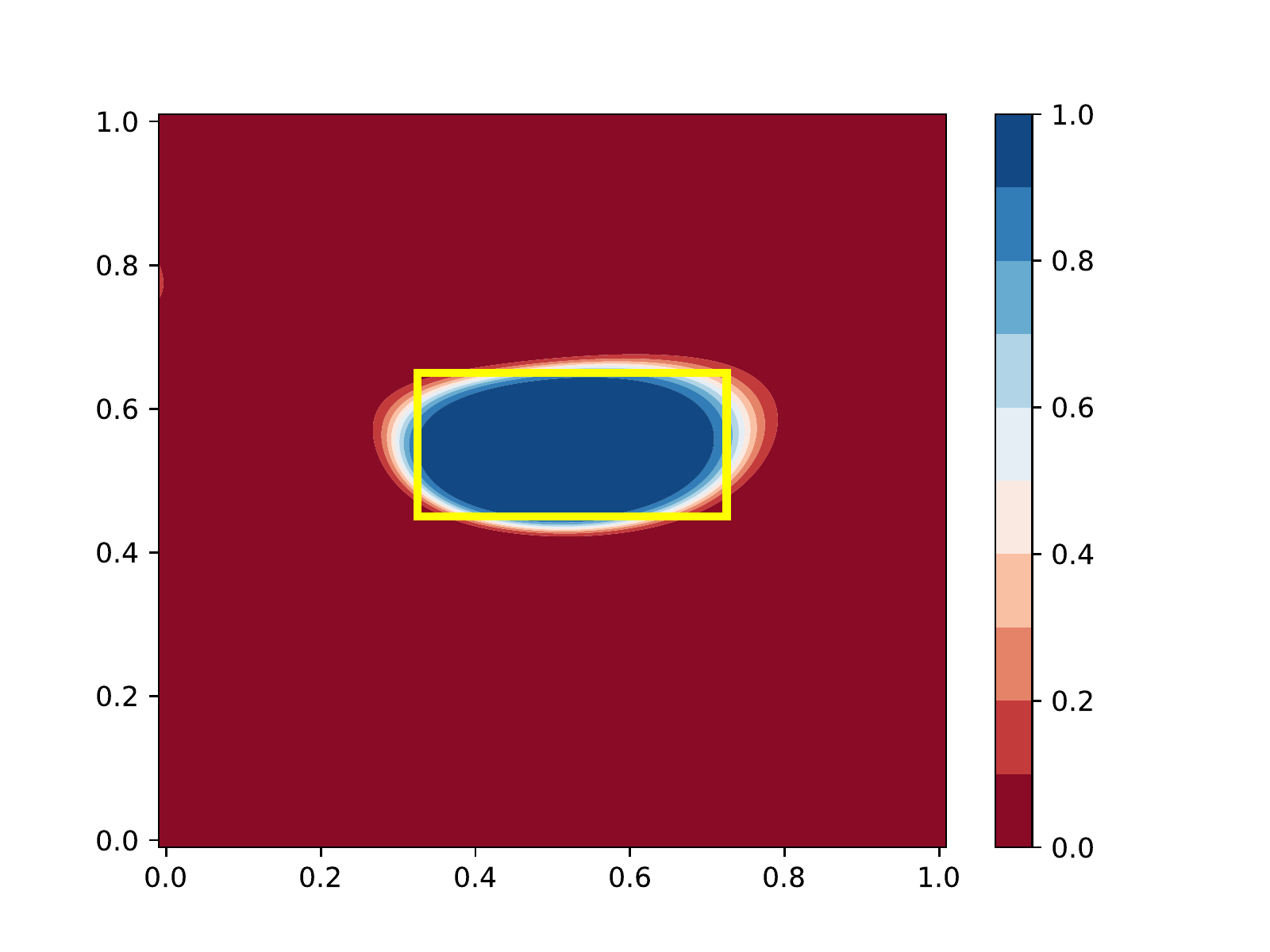}
\end{minipage} &
\begin{minipage}{.17\textwidth}
    \includegraphics[width=\linewidth,trim={1.1cm 0 1cm 1cm},clip]{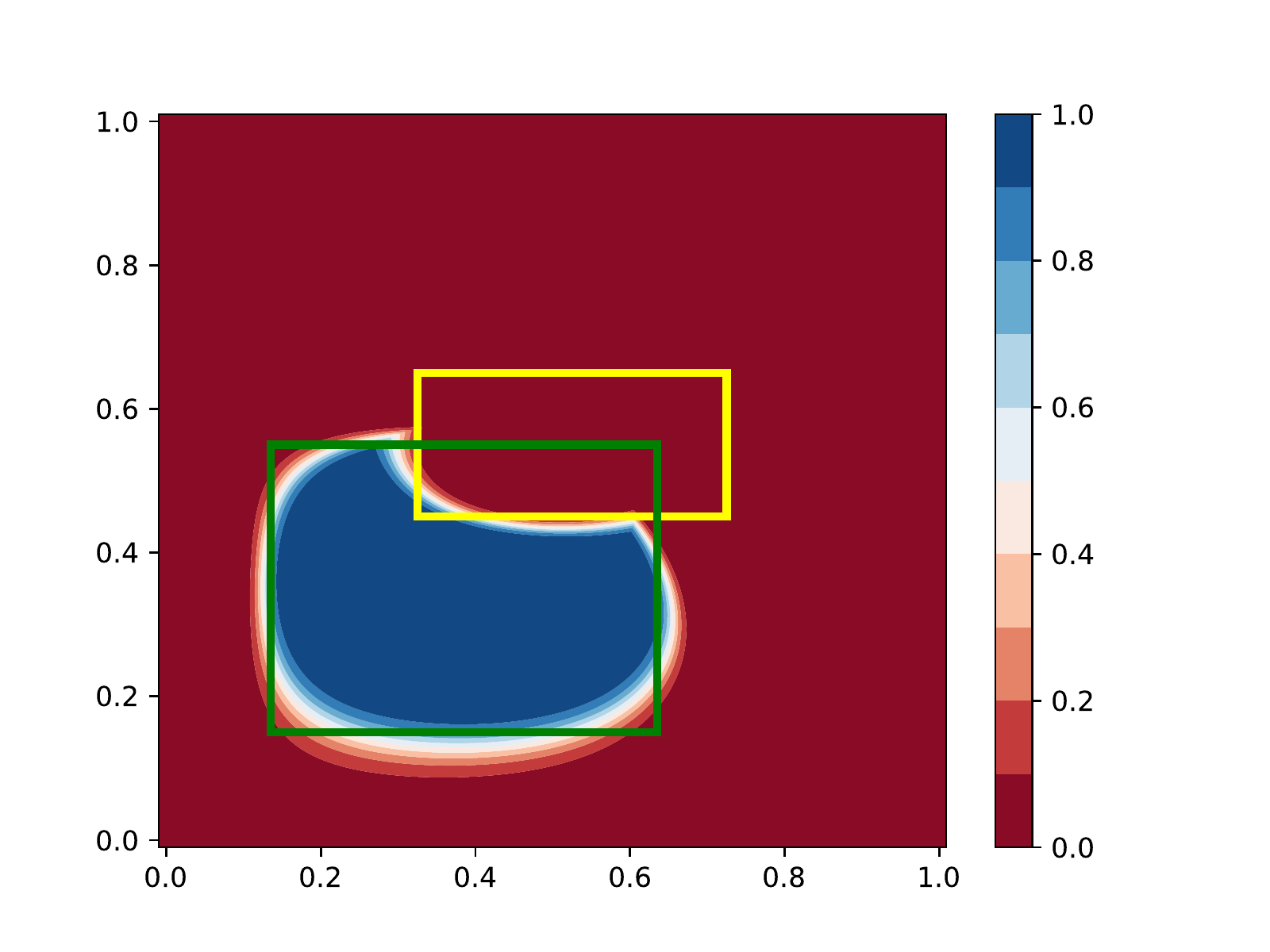}
\end{minipage} \vspace*{-1ex}
\end{tabular}
\caption{
 First three columns: decision boundaries of $f$ for the classes $A$, $A_1$, and $A_2$. Last three columns: decision boundaries of \system{$h$} for the classes $A$, $A_1$, and $A_2$. In each figure, the darker the blue (resp., red), the more confident a model is that the data points in the region is associated (not associated) with the class (see the scale at the end of each row).
}
\label{fig:dec_bound_figs_gen}
\end{figure*}

In the second step, to effectively exploit the constraints during training, \system{$h$} is trained with a new loss function, called {\sl constraint loss} (\loss), 
which has two goals: 
\begin{enumerate}
    \item we want to give each class the correct supervision 
    (e.g., if $y_A \,{=}\, 1$, then we want to teach $h$ to increase $h_A$ and not to decrease it), and 
    \item given a constraint, we want to teach $h$ to rely on the prediction for the classes in the body to make prediction for the class in the head only when the body is satisfied (e.g., for (\ref{eq:constr_a2}), when $y_A \,{=}\, 1$ and $y_{A_1}\,{=}\, 0$).

\end{enumerate} 
To achieve the above goals, $\loss$ is defined as $\loss = \loss_A + \loss_{A_1} + \loss_{A_2}$, where:
$$
    \begin{aligned}
   \loss_A = \! & -y_A\ln(\max(h_A, h_{A_1}y_{A_1},h_{A_2}y_{A_2} )) 
   - \ov{y}_{A} \ln(\ov\module_A), \\ %
   \loss_{A_1} \!\!  = \! & -y_{A_1} \ln(\module_{A_1}) - \ov{y}_{A_1}\ln(\ov{\module}_{A_1}), \\
   \loss_{A_2} \!\! = \! &-y_{A_2}\ln(\max(h_{A_2}, \min(h_{A}y_A, \ov{h}_{A_1}\ov{y}_{A_1}),
   \min(h_{A_1}y_{A_1}, \ov{h}_{A_1}\ov{y}_{A_1}))) \\
   & 
   \!\!\!\!\!\!- \ov{y}_{A_2}\!\ln(1\!\!-\!\max(h_{A_2}, \!\min(h_{A}\ov{y}_A\!\!+\!y_A, \ov{h}_{A_1}y_{A_1}\!\!+\!\ov{y}_{A_1}), 
     \min(h_{A_1}\ov{y}_{A_1}\!\!+\!y_{A_1}, \ov{h}_{A_1}y_{A_1}\!\!+\!\ov{y}_{A_1}))). \\
    \end{aligned}
$$

\begin{figure*}[t]
\centering
\begin{tabular}{c@{\ \ \,}c@{\ \ \ \,}c@{\ \ \ \ }|@{\ \ \,}c@{\ \ \ \,}c@{\ \ \ \,}c@{\ \ \ }}
\multicolumn{3}{c}{Neural Network $f$} & 
\multicolumn{3}{c}{Neural Network $h$} \\
Class $A$ & Class $A_1$ & Class $A_2$ & Class $A$ & Class $A_1$ & Class $A_2$ \\
\begin{minipage}{.135\textwidth}
    \includegraphics[width=\textwidth,trim={1.1cm 0 3.7cm 1cm},clip]{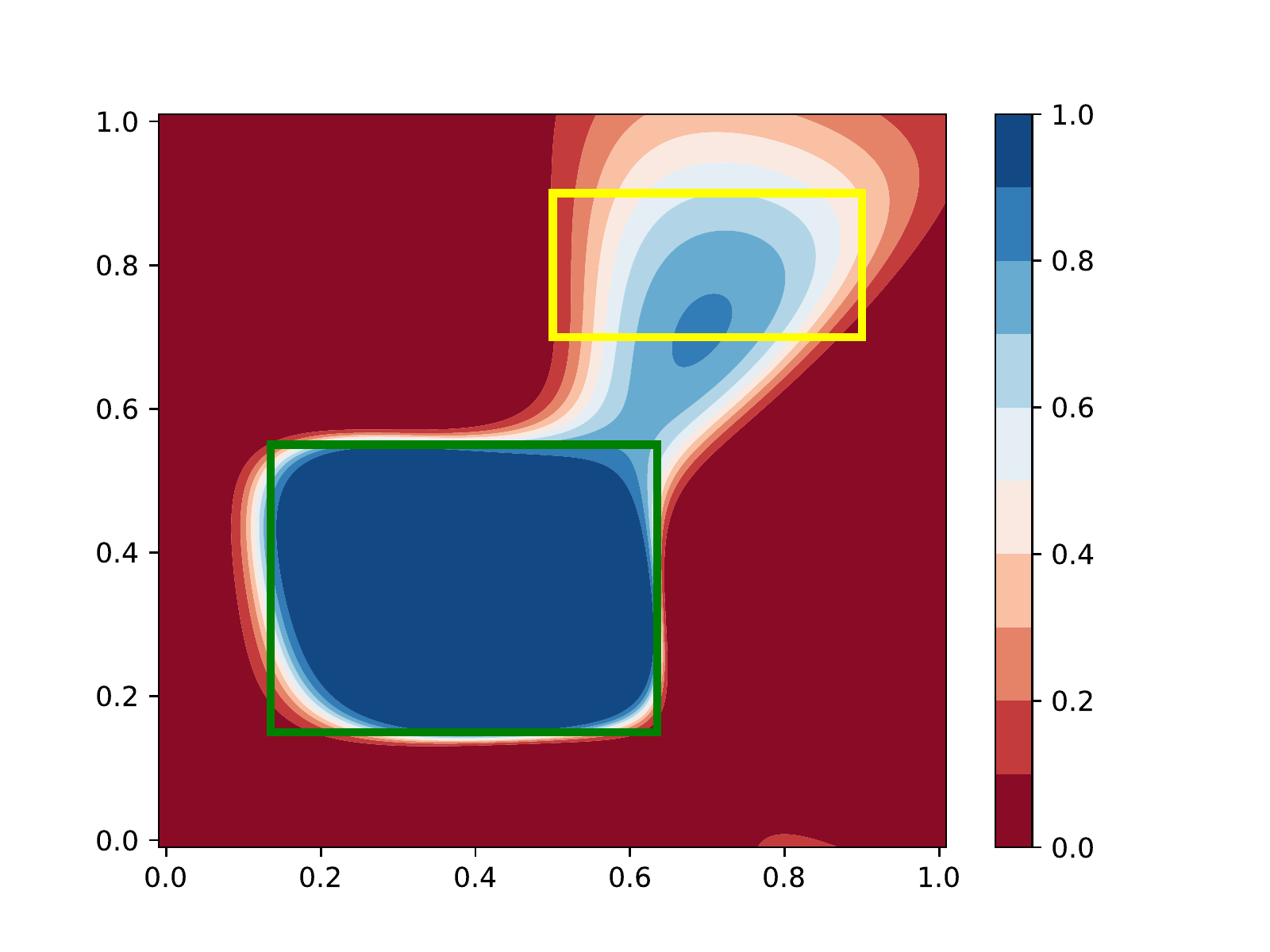}
\end{minipage} &
\begin{minipage}{.135\textwidth}
    \includegraphics[width=\textwidth,trim={1.1cm 0 3.7cm 1cm},clip]{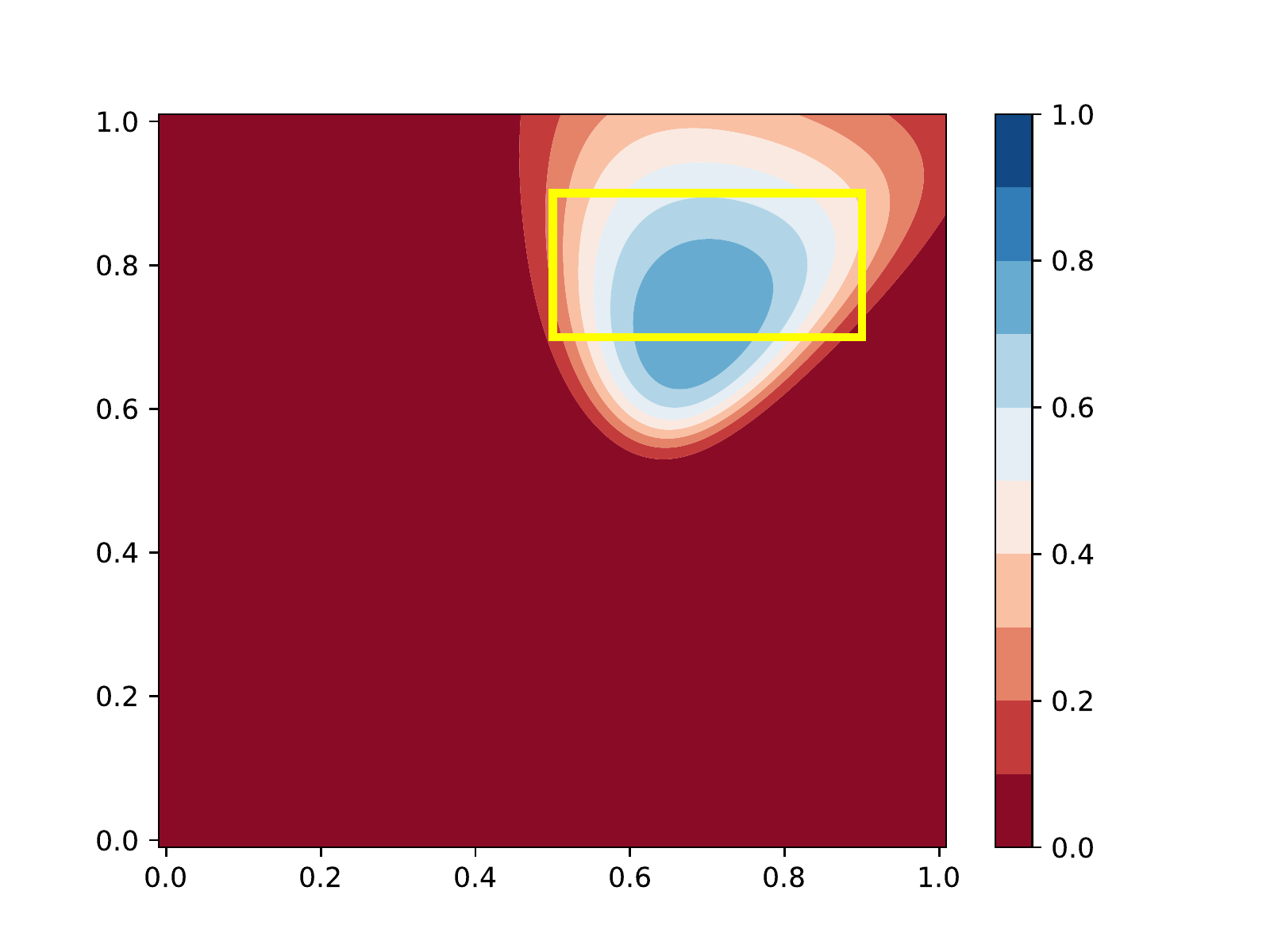}
\end{minipage} &
\begin{minipage}{.135\textwidth}
    \includegraphics[width=\textwidth,trim={1.1cm 0 3.7cm 1cm},clip]{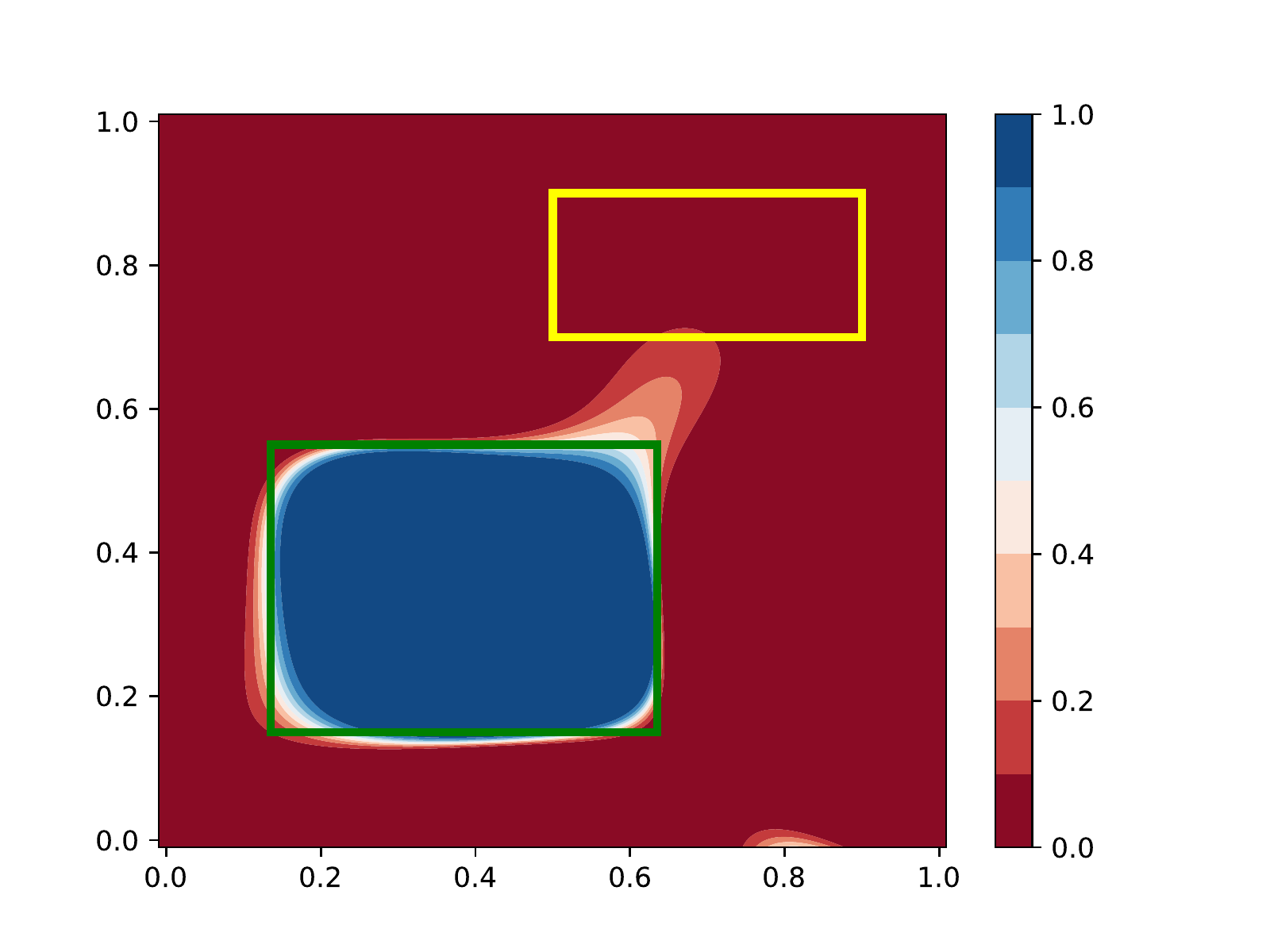}
\end{minipage} &
\begin{minipage}{.135\textwidth}
    \includegraphics[width=\textwidth,trim={1.1cm 0 3.7cm 1cm},clip]{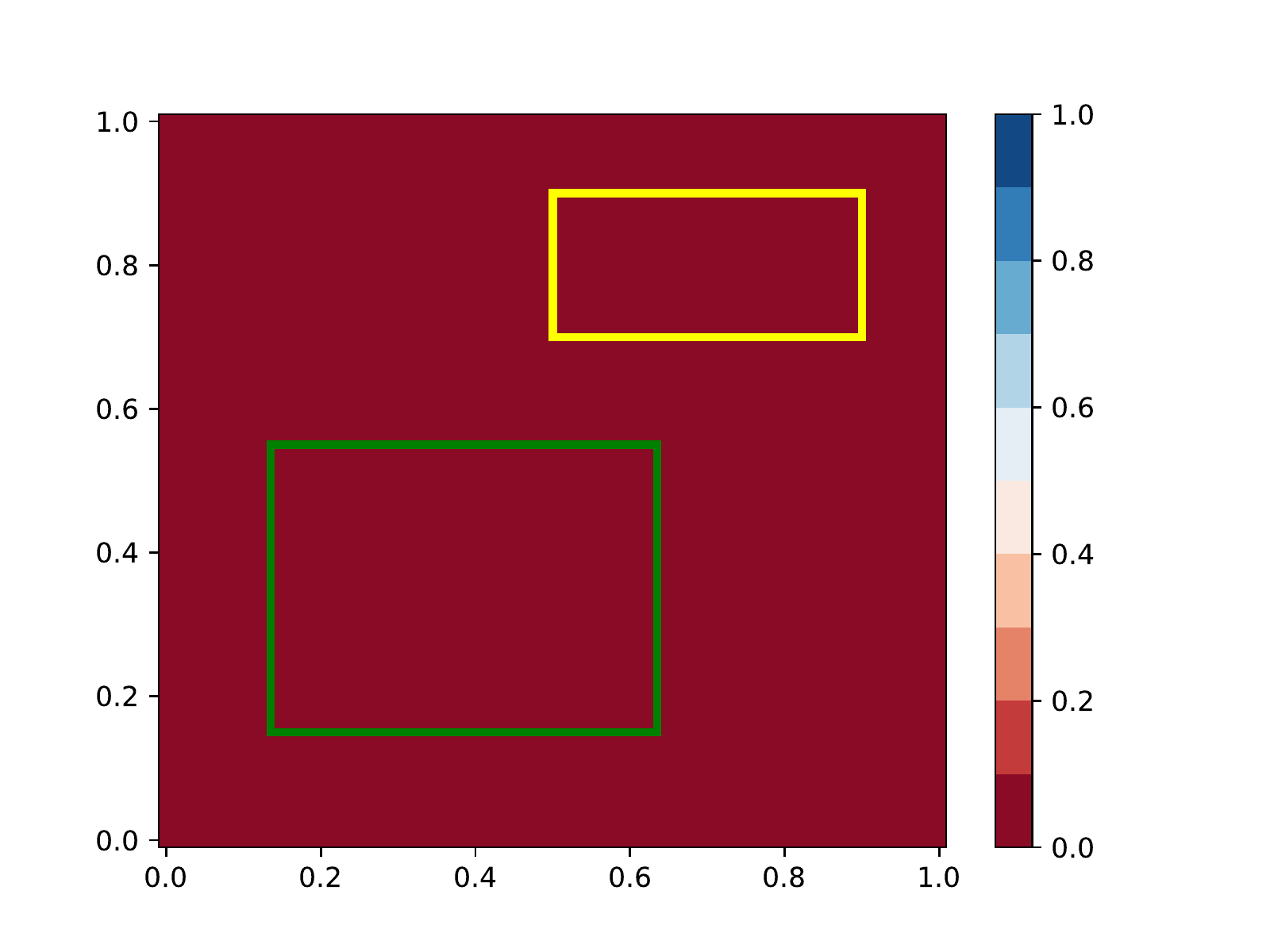}
\end{minipage} &
\begin{minipage}{.135\textwidth}
    \includegraphics[width=\linewidth,trim={1.1cm 0 3.7cm 1cm},clip]{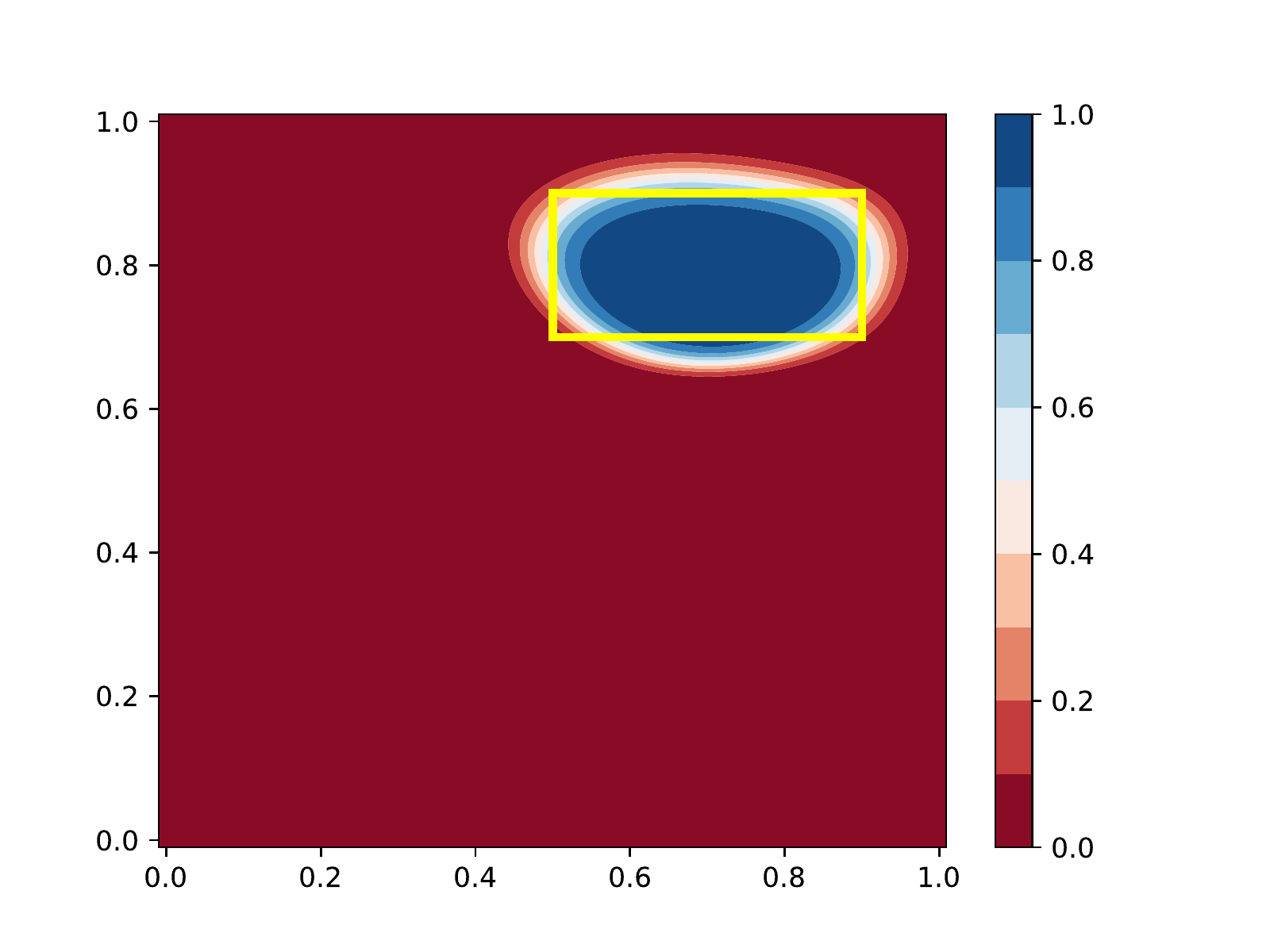}
\end{minipage} &
\begin{minipage}{.165\textwidth}
    \includegraphics[width=\linewidth,trim={1.1cm 0 1cm 1cm},clip]{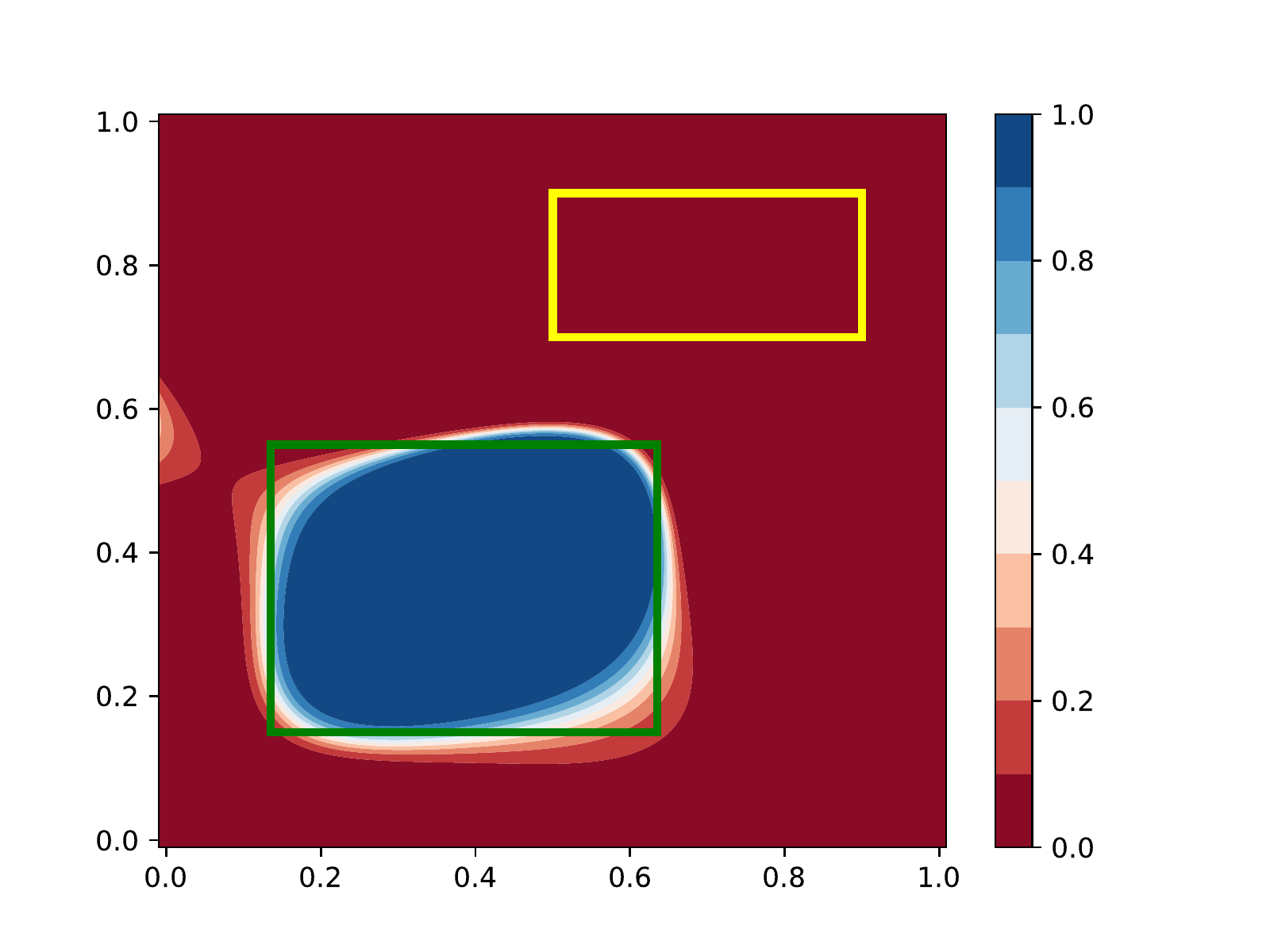}
\end{minipage} \\
\begin{minipage}{.135\textwidth}
    \includegraphics[width=\textwidth,trim={1.1cm 0 3.7cm 1cm},clip]{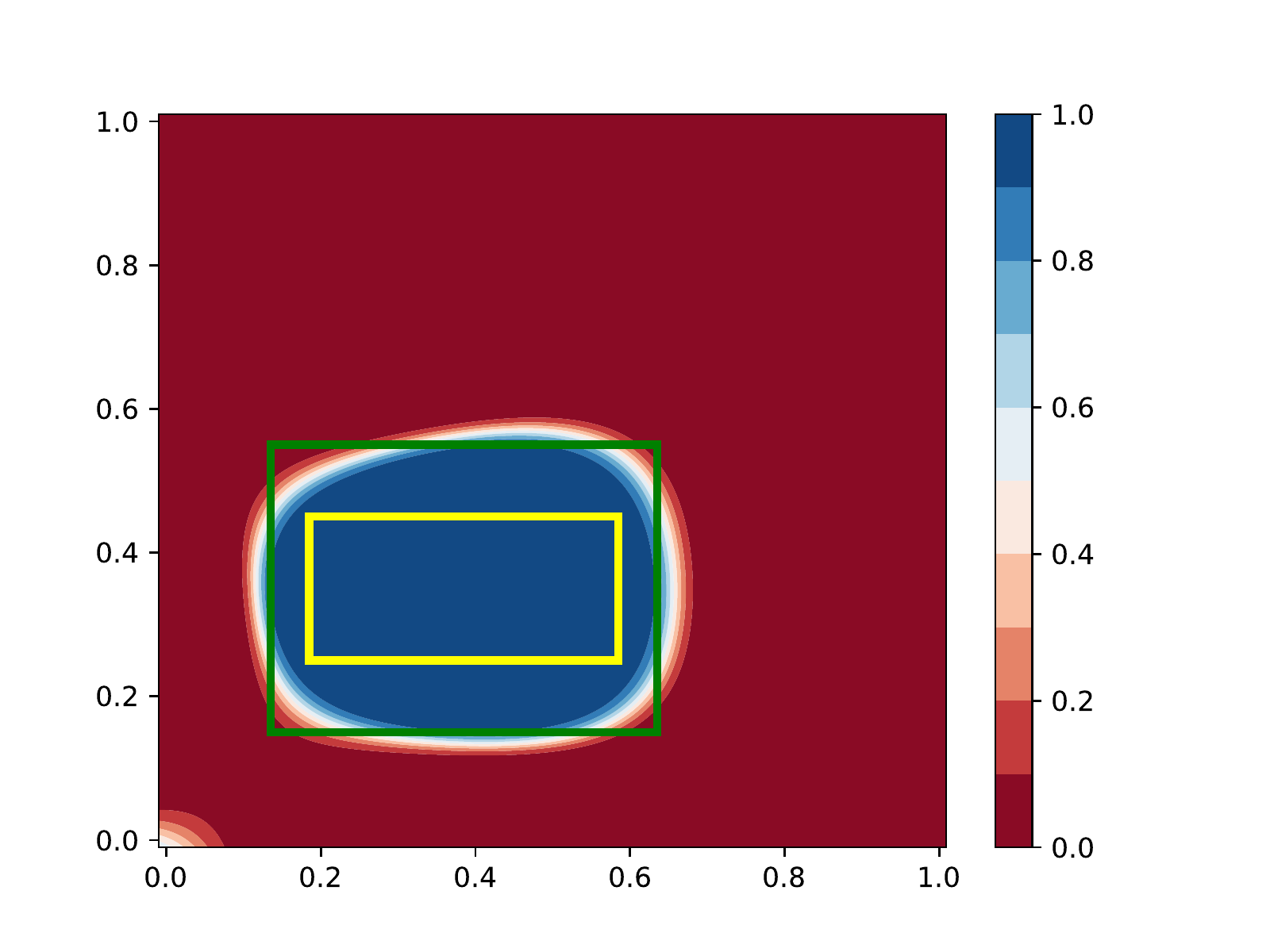}
\end{minipage} &
\begin{minipage}{.135\textwidth}
    \includegraphics[width=\textwidth,trim={1.1cm 0 3.7cm 1cm},clip]{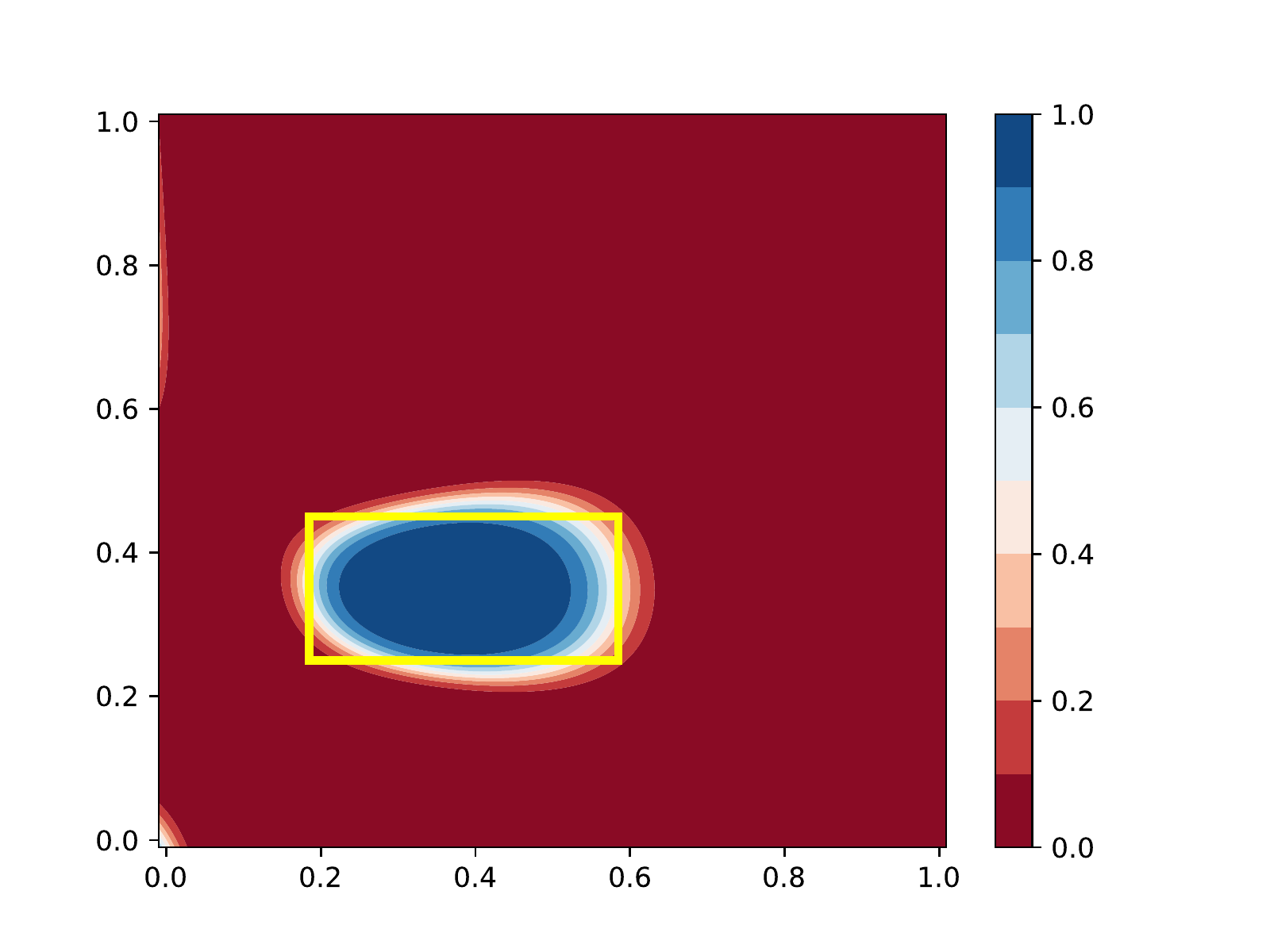} 
\end{minipage} &
\begin{minipage}{.135\textwidth}
    \includegraphics[width=\textwidth,trim={1.1cm 0 3.7cm 1cm},clip]{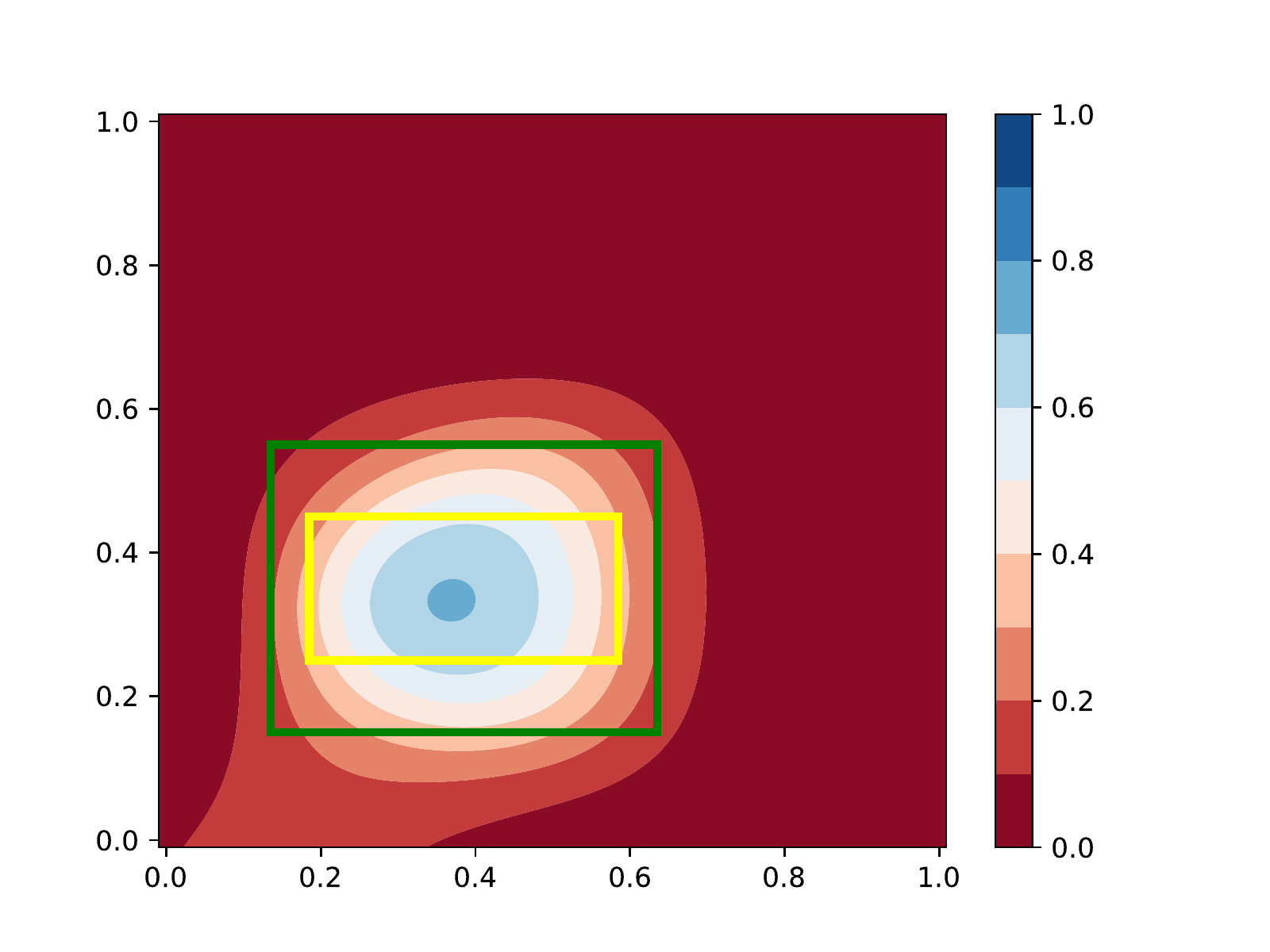} 
\end{minipage} &
\begin{minipage}{.135\textwidth}
    \includegraphics[width=\textwidth,trim={1.1cm 0 3.7cm 1cm},clip]{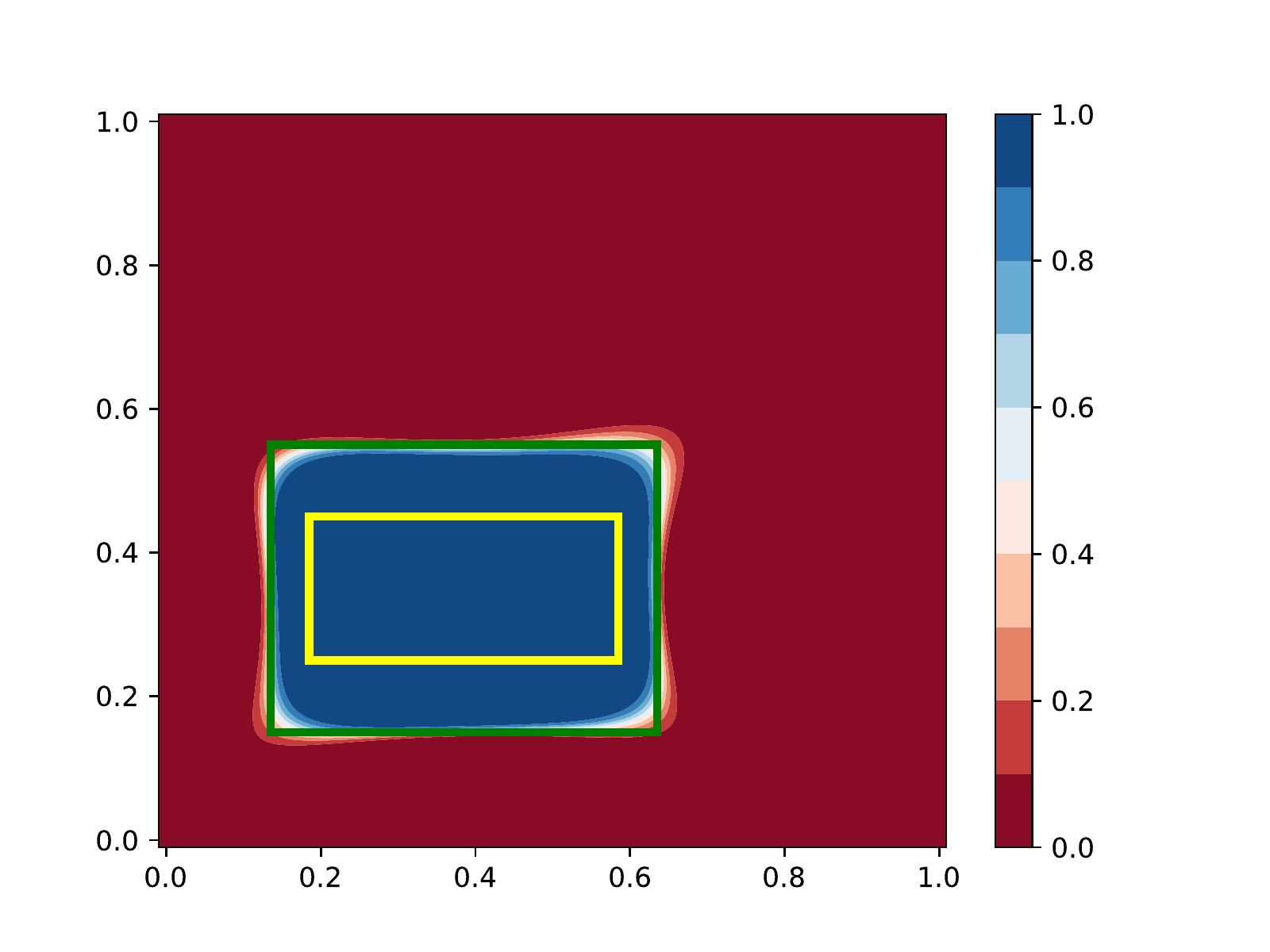} 
    \end{minipage} &
\begin{minipage}{.135\textwidth}
    \includegraphics[width=\linewidth,trim={1.1cm 0 3.7cm 1cm},clip]{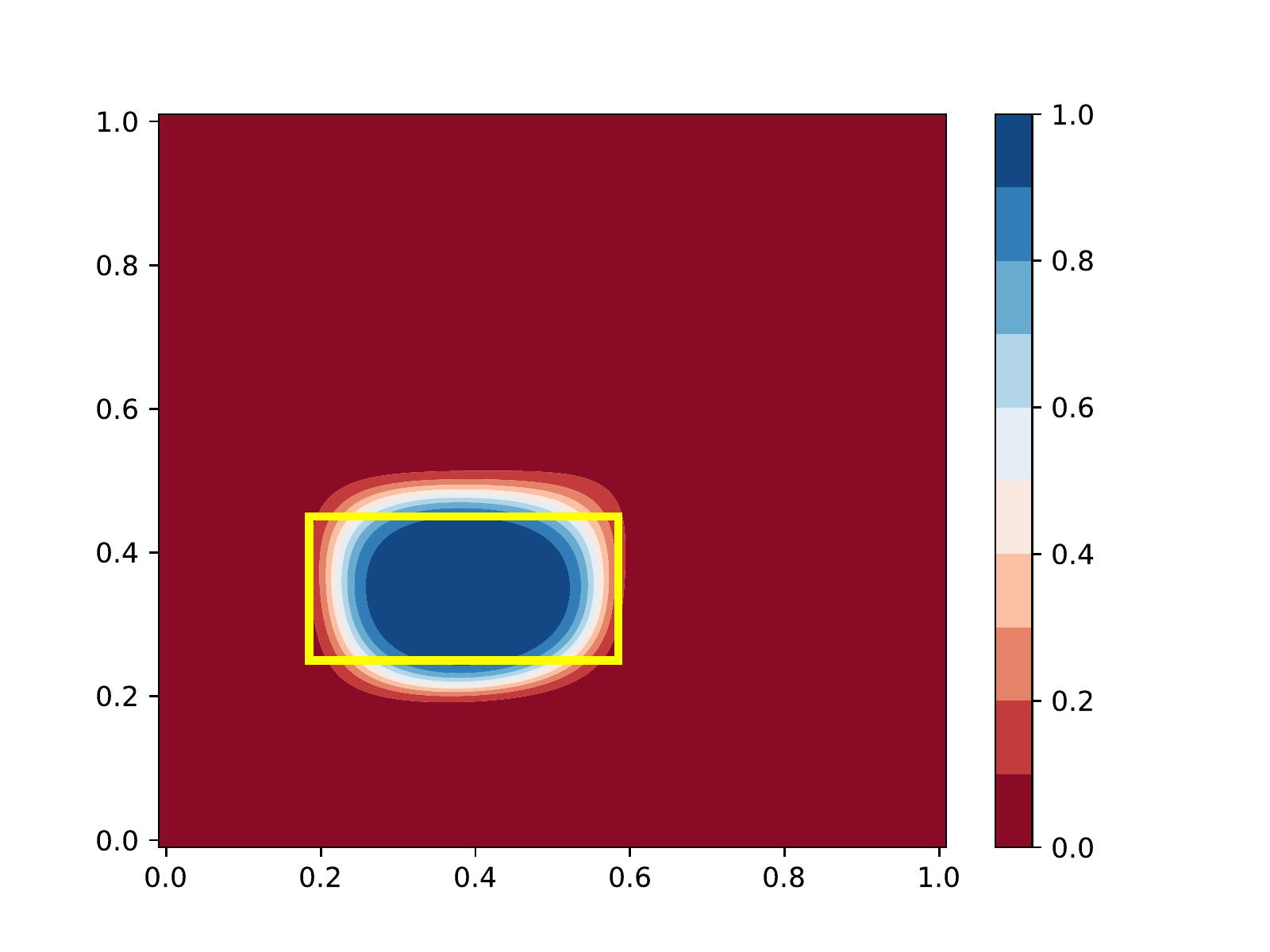}
\end{minipage} &
\begin{minipage}{.165\textwidth} 
    \includegraphics[width=\linewidth,trim={1.1cm 0 1cm 1cm},clip]{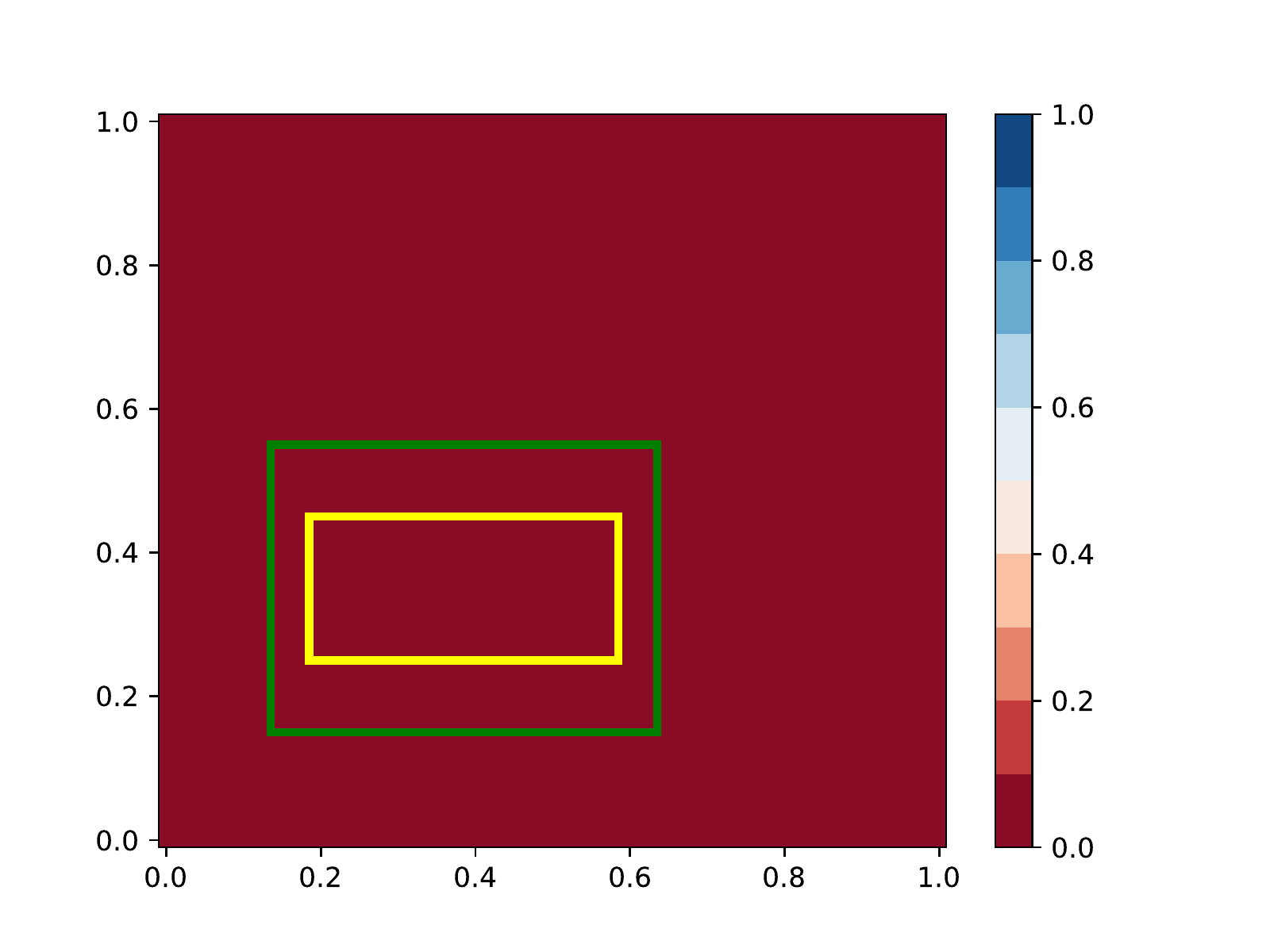}
\end{minipage} \\
\begin{minipage}{.135\textwidth}
    \includegraphics[width=\textwidth,trim={1.1cm 0 3.7cm 1cm},clip]{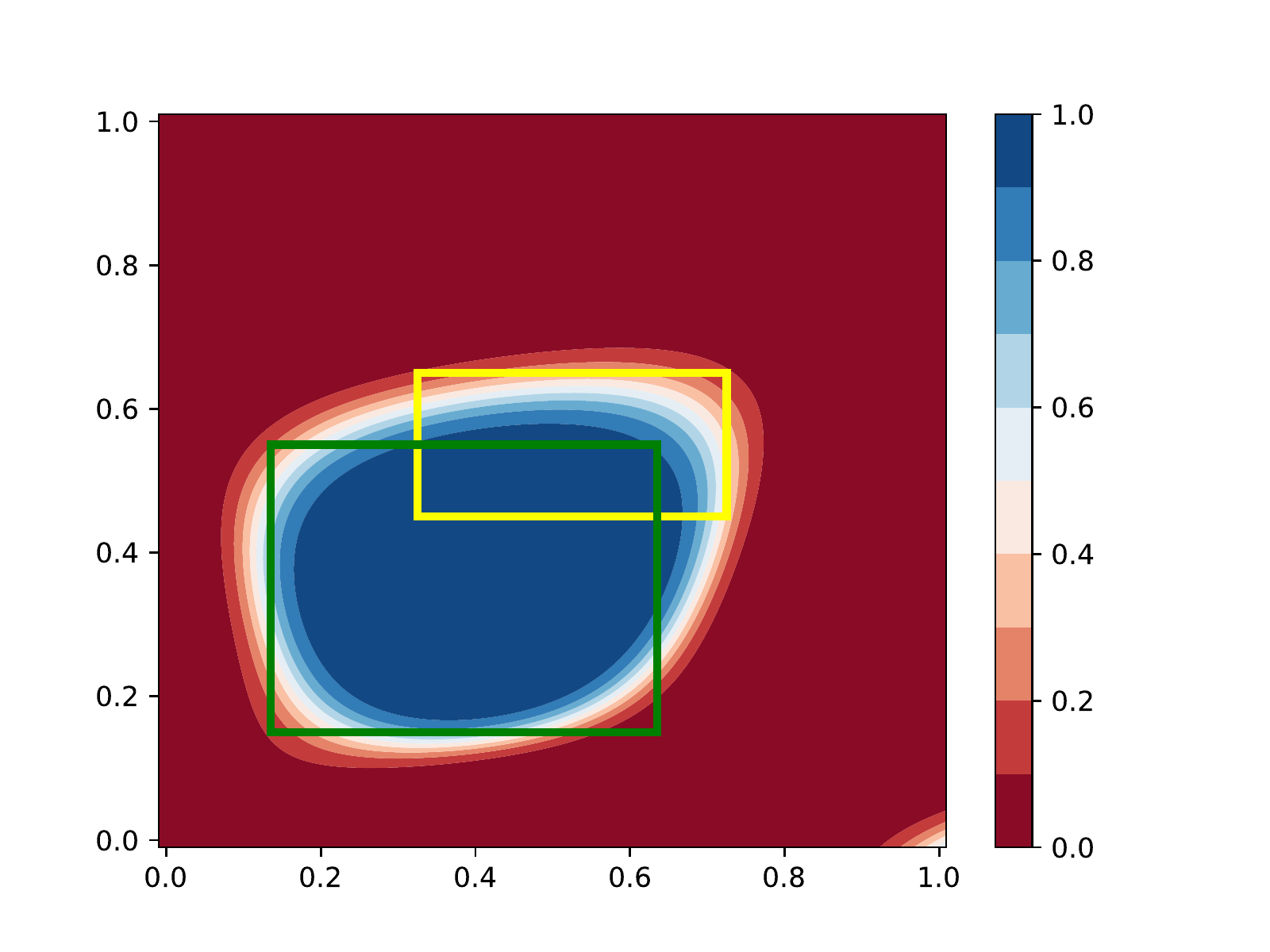}
\end{minipage} &
\begin{minipage}{.135\textwidth}
    \includegraphics[width=\textwidth,trim={1.1cm 0 3.7cm 1cm},clip]{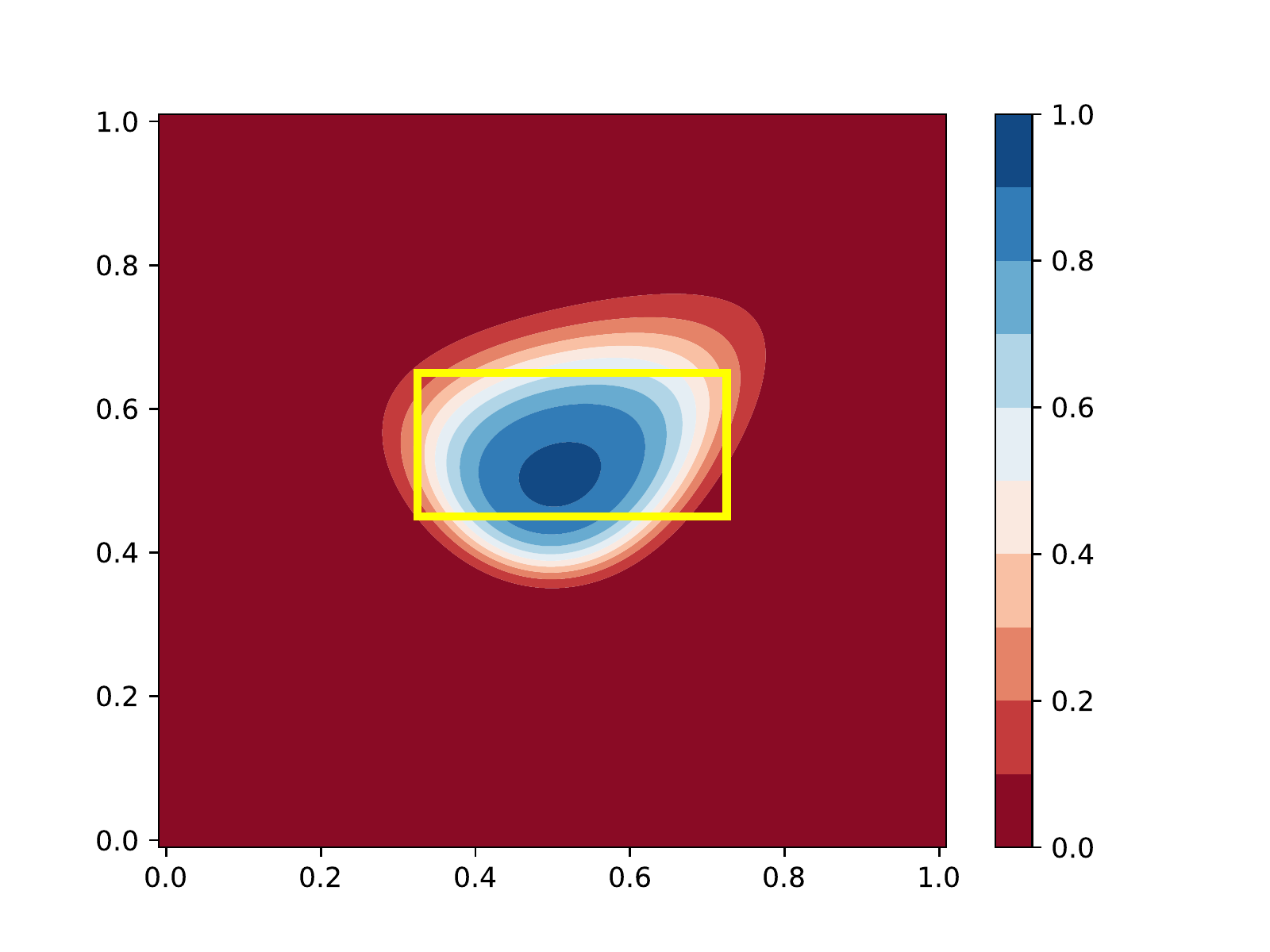}
\end{minipage} &
\begin{minipage}{.135\textwidth}
    \includegraphics[width=\textwidth,trim={1.1cm 0 3.7cm 1cm},clip]{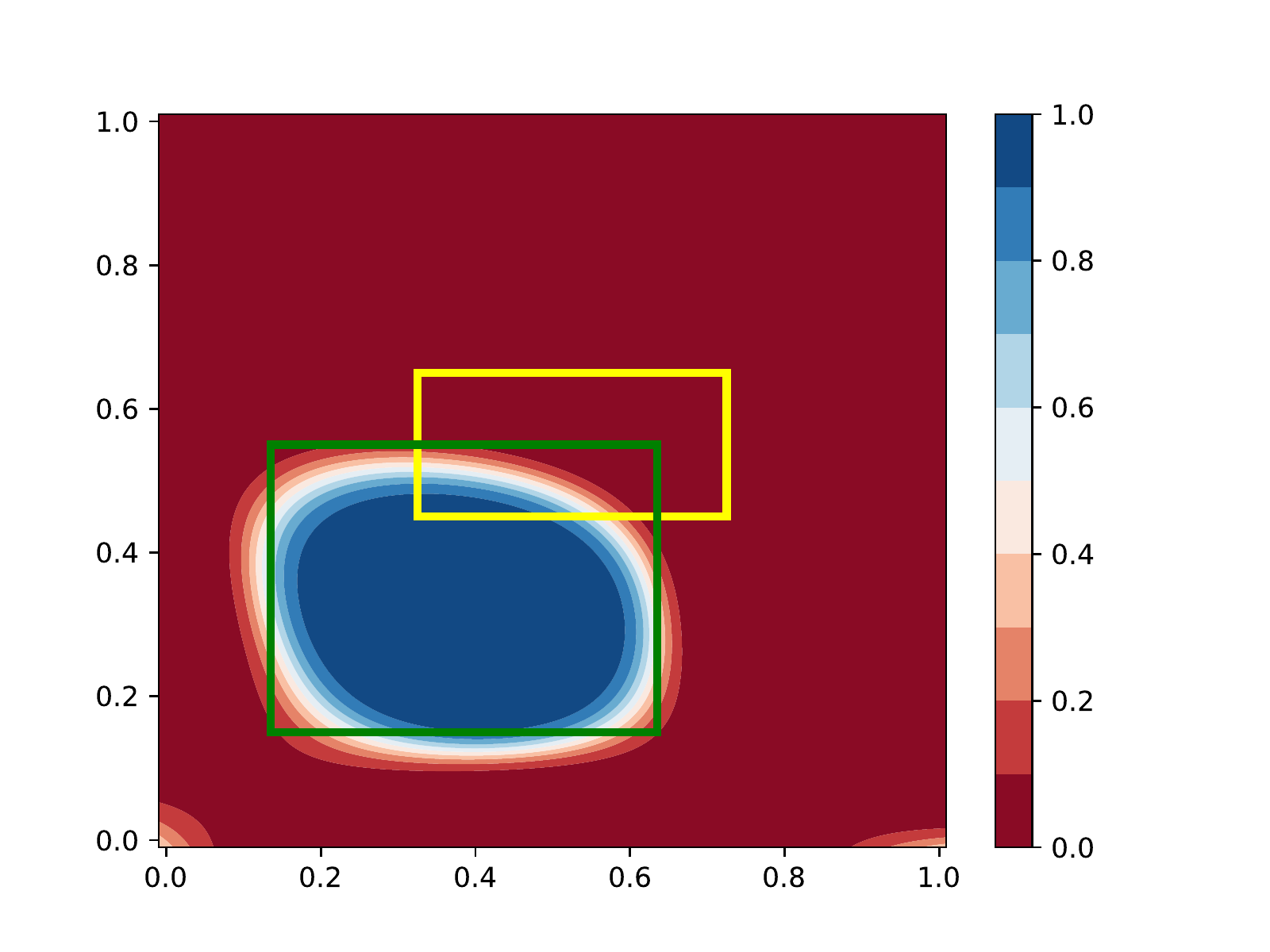}
\end{minipage} &
\begin{minipage}{.135\textwidth}
    \includegraphics[width=\textwidth,trim={1.1cm 0 3.7cm 1cm},clip]{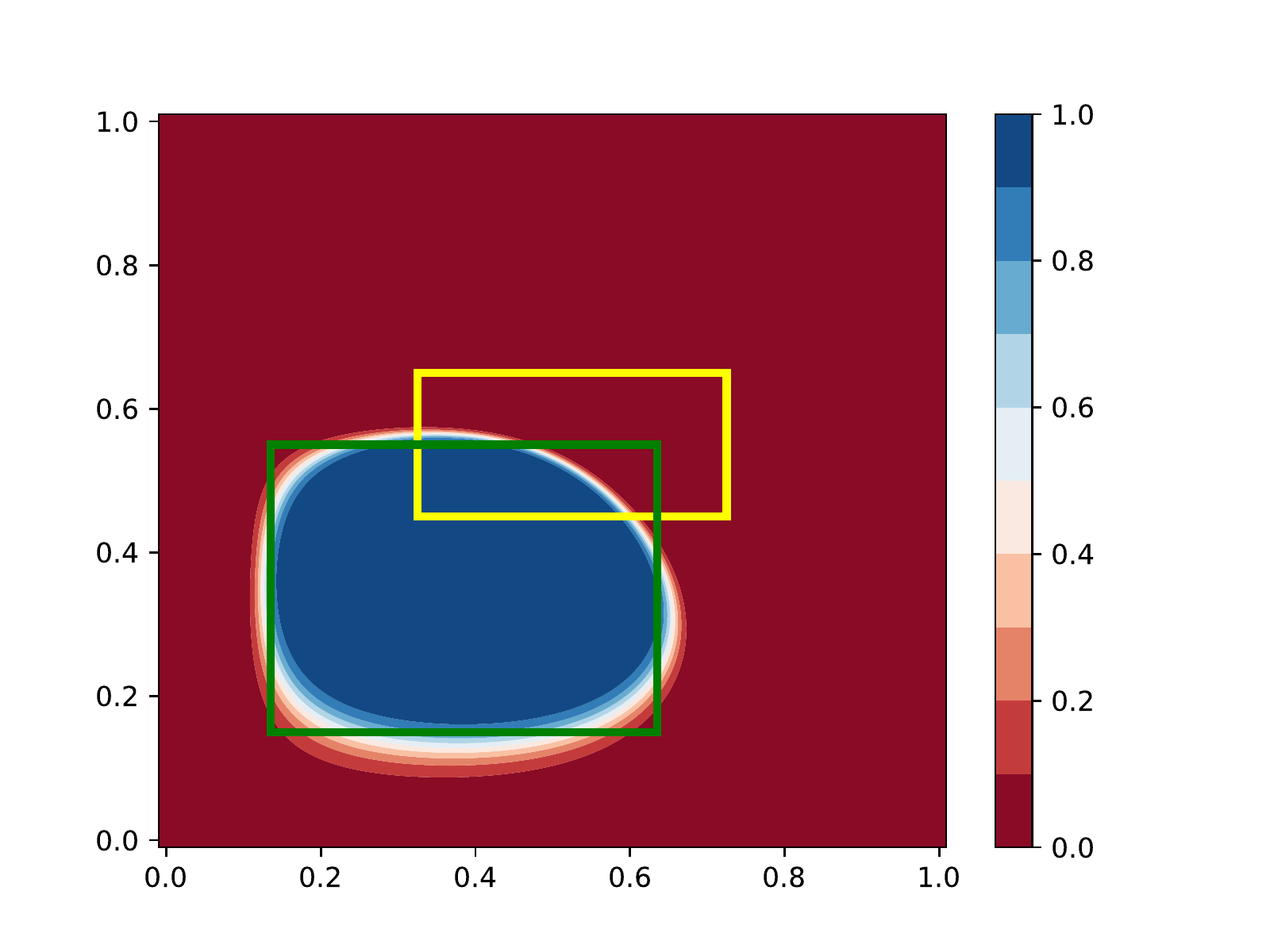}
\end{minipage} &
\begin{minipage}{.135\textwidth}
    \includegraphics[width=\linewidth,trim={1.1cm 0 3.7cm 1cm},clip]{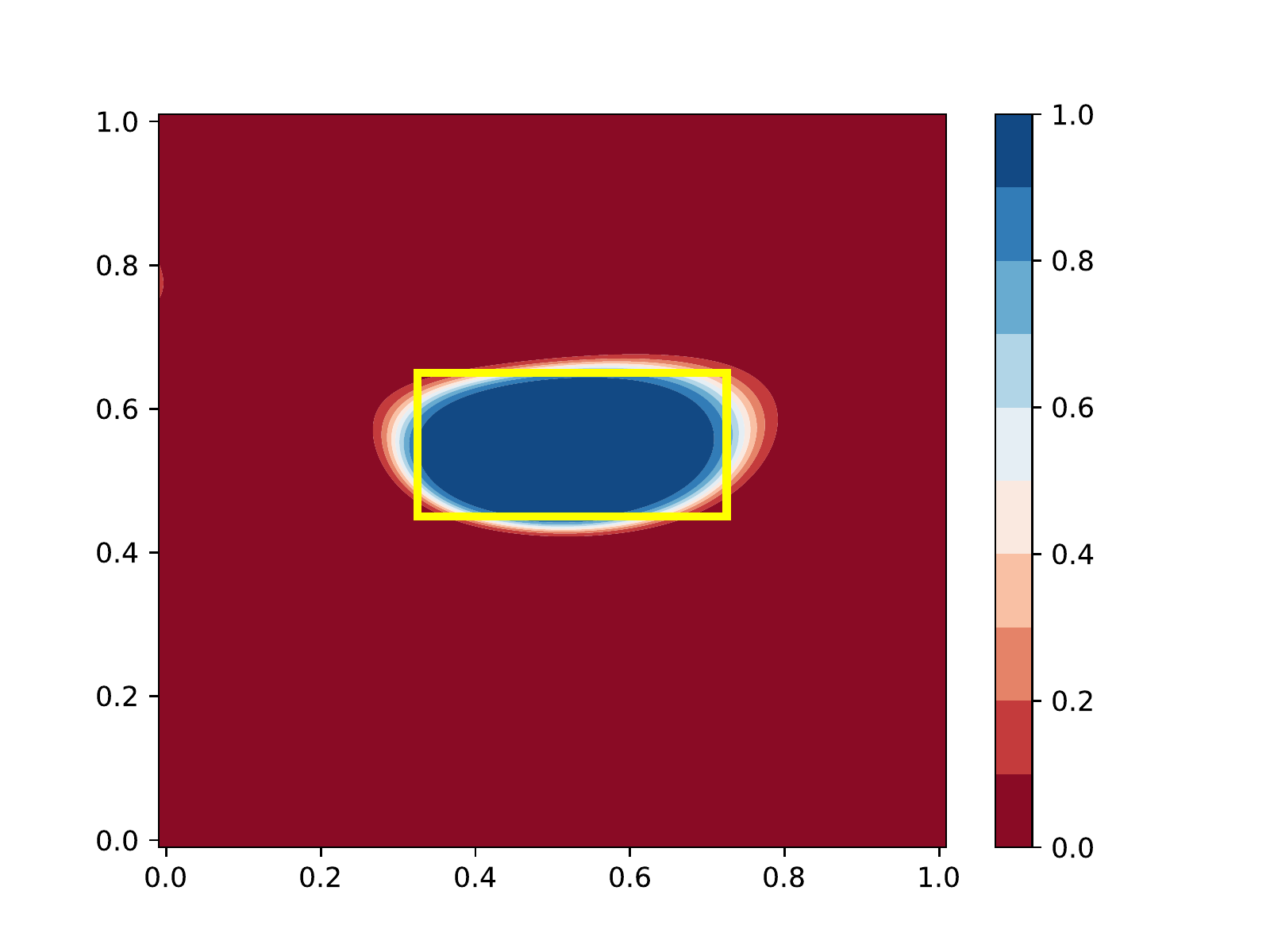}
\end{minipage} &
\begin{minipage}{.17\textwidth}
    \includegraphics[width=\linewidth,trim={1.1cm 0 1cm 1cm},clip]{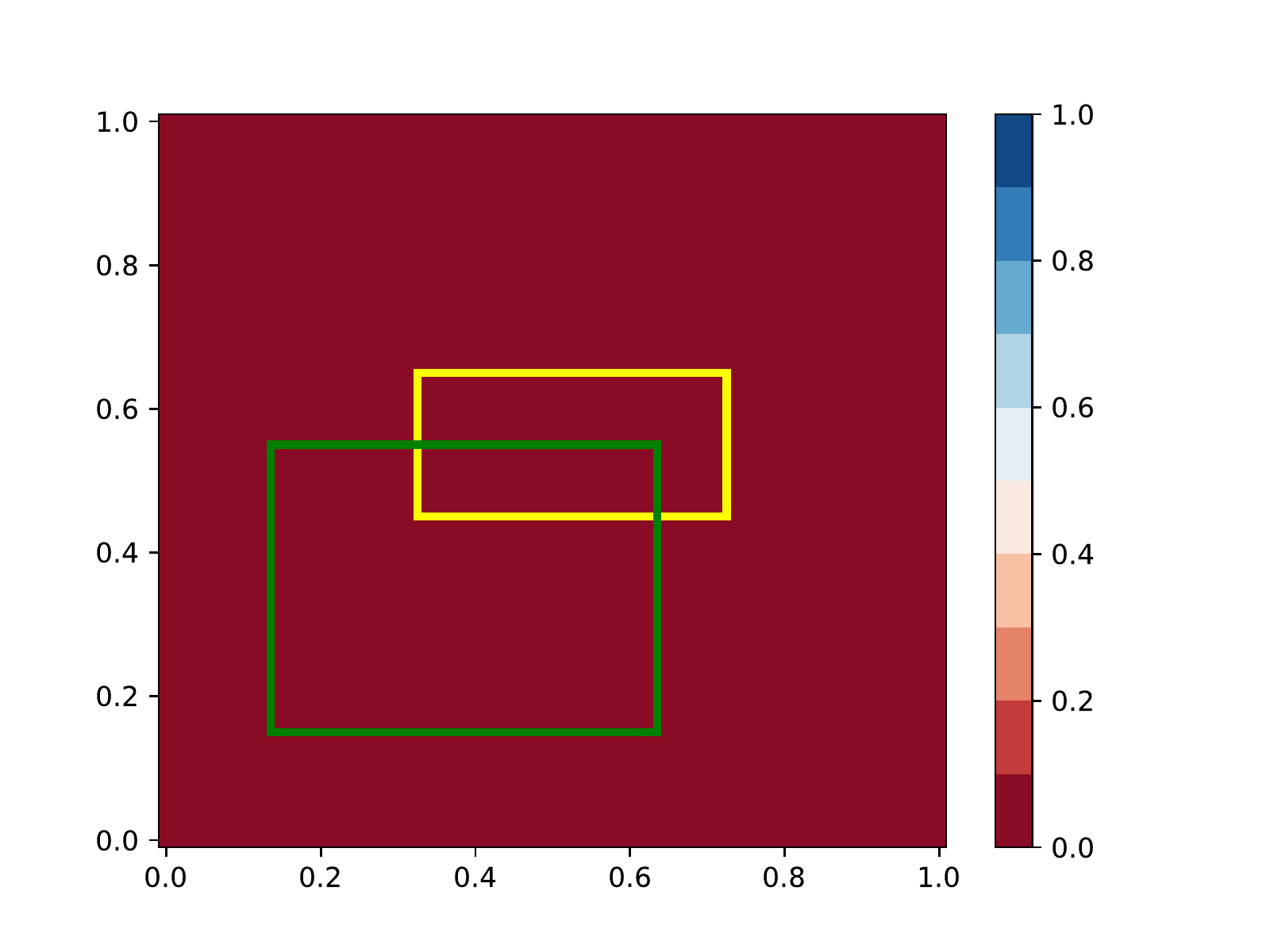}
\end{minipage} \vspace*{-1ex}
\end{tabular}
\caption{
 First three columns: decision boundaries of $f$ for the classes $A$, $A_1$, and $A_2$. Last three columns: decision boundaries of \system{$h$} for the classes $A$, $A_1$, and $A_2$. 
}
\label{fig:dec_bound_figs_gen_before}
\end{figure*}

\loss{} differs from the standard binary cross entropy loss function $\lss$, as highlighted by the following example.

\begin{example}\label{ex:gradients}{\rm 
Assume that $h_A = 0.6$, $h_{A_1} = 0.2$, $h_{A_2} = 0.3$, $y_{A} = y_{A_1} = 1$, and $y_{A_2} = 0$. 
Then,
{
\begin{gather*}
\loss = \loss_A + \loss_{A_1} + \loss_{A_2} = 
-\ln(h_A) -\ln(h_{A_1}) -\ln(h_{A_1}),
\end{gather*}}
and
{
\begin{gather*}
\frac{\partial \loss}{\partial h_A}\!=\!-\frac{1}{h_A} \!\sim\!-1.6, \qquad
\frac{\partial \loss}{\partial h_{A_1}} = -\frac{2}{h_{A_1}} = -10, \qquad
\frac{\partial \loss}{\partial h_{A_2}} = 0,
\end{gather*}}
and \system{$h$} rightly learns to increase both $h_A$ and $h_{A_1}$. 

On the other hand, using the standard binary cross-entropy loss $\lss$ after {\module}, we obtain:
$$
\begin{aligned}
\lss  = - \ln(\module_A) - \ln(\module_{A_1}) - \ln(\ov{\module}_{A_2})  = - \ln(h_{A}) - \ln(h_{A_1}) - \ln(\ov{h}_{A}), 
\end{aligned}
$$
and thus
$$
\begin{aligned}
\frac{\partial \lss}{\partial h_A} = - \frac{1}{h_{A}} +\frac{1}{\ov{h}_{A}} \sim  0.8, \qquad
\qquad \frac{\partial \lss}{\partial h_{A_1}} =- \frac{1}{h_{A_1}}  =  -5, \qquad \qquad \frac{\partial \lss}{\partial h_{A_2}} = 0\,. 
\end{aligned}
$$
Hence, if trained with $\lss$, {\system{$h$}} would learn to decrea\-se~$h_{A}$ while keeping $h_{A_2}$ despite the fact that $y_A = 1$.\hfill$\lhd$ 
}\end{example}

To test the effectiveness of our approach, we consider again
 the yellow ($R_1$) and green ($R_2$) rectangles in Figure~\ref{fig:dec_bound_figs_gen} with
 $A = R_1 \cup R_2$, $A_1 = R_1$, and $A_2 = R_2 \setminus R_1$. 
We implemen\-ted~$f$ and $h$ as feedforward neural networks with one hidden layer with 4 neurons and {\it tanh} nonlinearity. 
We trained $f$ with binary cross-entropy loss, and \system{$h$} using $\loss$. We trained both  networks for 20k epochs using Adam optimization~\cite{kingma2014}, with learning rate $10^{-2}$ ($\beta_1 = 0.9, \beta_2 = 0.999$). The datasets consisted of $5000$ (50/50 train/test split) data points sampled from a uniform distribution over $[0,1]^2$. 
In order to obtain predictions that are compliant with the constraints also for the neural network $f$, we apply an additional post-processing step at inference time,
obtaining $f^+$, whose outputs are defined 
as follows:
\begin{equation}
    \begin{aligned}
    &f^+_A  \,\, =  \max(f_A, f_{A_1}, f_{A_2}), \\
    &f^+_{A_1}  =  f_{A_1},\\
    &f^+_{A_2}  =  \max(f_{A_2}, \min(f_{A}, \ov{f}_{A_1}), \min(f_{A_1},\ov{f}_{A_1})).
    \end{aligned}
\end{equation}

We plot the final decision boundaries of $f^+$ (first three columns) and \system{$h$} (last three columns) for all classes in Figure~\ref{fig:dec_bound_figs_gen}, while the decision boundaries of $f$ (first three columns) and $h$ (last three columns) are plotted in Figure~\ref{fig:dec_bound_figs_gen_before}. 
In these figures, we can see that $f$ struggles, as expected, in learning the decision boundaries for the classes $A$ and $A_2$, and that the application of the constraints as a post-processing step, as it happens in $f^+$, can lead to a decay in performance. On the contrary, we can see that \system{$h$} is able to easily learn the decision boundaries for all the classes through a smart exploitation of the constraints. Indeed, as it can be seen in the last three columns of Figure~\ref{fig:dec_bound_figs_gen_before}, on the ground of the positions of the rectangles $R_1$ and $R_2$, \system{$h$} knows which constraints to exploit:
\begin{itemize}
    \item if  $R_1$ and $R_2$ are arranged as in the first row, then \system{$h$} exploits the constraints $A_1 \to A$ and $A_2 \to A$. Thus, the bottom module $h$ does not learn $A$, which is instead computed from $A_1$ and $A_2$, 
    \item  if  $R_1$ and $R_2$ are arranged as in the second row, then \system{$h$} exploits the constraint $A, \neg A_1 \to A_2$. Thus, the bottom module $h$ does not learn $A_2$, which is instead computed from $A_1$ and $A_2$, and 
    \item if  $R_1$ and $R_2$ are arranged as in the third row, then \system{$h$} exploits the constraints $A, \neg A_1 \to A_2$ and $A_1 \to A$. Thus, the bottom module $h$ does not learn $A_2$, and learns $A$ only partially. Then, $A_2$ is computed from $A$ and $A_1$, while $A$ exploits $A_1$ to make predictions on the points belonging to $R_2$.
\end{itemize}

\subsection{General Case}\label{sec:gen_case}
We now present the general solution. 
 We consider a general {\cmc} problem $(${\problem}$,\Pi)$ and a model $h$ for {\problem}.  We first show how \system{$h$} computes the set of classes associated to every data point (Section~\ref{sec:gen_module}), and why the definition of {\module} requires some care in order to satisfy some desired properties, stated and motivated at the beginning of the same section. 
 We then present the loss function used to train \system{$h$} (Section~\ref{sec:gen_loss}).
 We end stating that \system{$h$} is a generalization of \hmcsys{$h$}, i.e., that \hmcsys{$h$} and \system{$h$} have the same behavior when given an HMC problem (Section~\ref{sec:relation}).

\subsubsection{Constraint Module --- {\module}}\label{sec:gen_module}

The basic idea of \system{$h$} is to
\begin{enumerate}
    \item have an initial set of classes decided by $h$, and
    \item have all the other classes predicted also on the grounds of the constraints in $\Pi$.
\end{enumerate}
In the example in the basic case, $h$ decides $\{A_1\}$:  every data point will have or will 
not have class $A_1$ depending on the value of $h_{A_1}$. 
The decision on the classes $\{A, A_2\}$ takes into account not only $h_A, h_{A_2}$ but also the constraints. 
In particular, \system{$h$} may 
\begin{enumerate}
    \item predict $A$ given the values of $h_{A_1}$ and $h_{A_2}$, or
   \item predict $A_2$ given the values of $h_A$ and $h_{A_1}$.  
\end{enumerate}
The final set of classes $\mathcal{M}$ predicted by \system{$h$} will
\begin{enumerate}
    \item extend the set of classes ${\cal H}$ predicted by $h$ (i.e., ${\cal H} \subseteq {\cal M}$) and be coherent with $\Pi$;
    \item be such that any class in $\mathcal{M} \setminus {\cal H}$ is in the head of a  constraint $r$ in $\Pi$, with the chain of rules used to satisfy  $body(r)$ grounded in $\cal H$;
    \item include only those classes that are either in $\cal H$ or are forced to be in $\cal M$ through the explicit use of chains of rules grounded in $\cal H$;
    \item be unique, i.e., there will be no other set of classes ${\cal M}'$ satisfying the above requirements.
\end{enumerate}
The first  requirement is the obvious one: the constraints in $\Pi$ must be satisfied, and \system{$h$} can only derive more classes in the head of the constraints.
Whereas the second requirement is formalized by the concept of supportedness as defined in~\cite{aptBW88}.
\begin{definition}%
Let $(${\problem}$,\Pi)$ be an {\cmc} problem.  Let $h$ be a model for {\problem}. Let $\mathcal{H}$ be the set of classes predicted by $h$. A set of classes 
${\cal M}$ is {\sl supported} relative to ${\cal H}$ and $\Pi$, if for any class  $A\,{ \in}\, {\cal M}$, $A\,{ \in}\, {\cal H}$, or there exists a constraint $r \in \Pi$ such that $head(r)\,{ = }\, A$, $body^+(r) \subseteq {\cal M}$, and for each  $\neg B\,{ \in}\, body^-(r)$, we have $B\,{ \not\in}\, {\cal M}$.
\end{definition}

The third requirement is a minimality condition.
\begin{definition}%
Let $(${\problem}$,\Pi)$ be an {\cmc} problem.  Let $h$ be a model for {\problem}. Let $\mathcal{H}$ be the set of classes predicted by $h$.
${\cal M}$ is {\sl minimal} relative to ${\cal H}$ and $\Pi$, if there exists no set of classes $\mathcal{M}'$ with $\mathcal{H} \subseteq \mathcal{M}' \subset \mathcal{M}$ that is coherent with $\Pi$.
\end{definition}

The four requirements together ensure that the final predictions made by {\system{$h$}} are coherent with $\Pi$ and that can be uniquely explained on the grounds of the initial predictions made by $h$ and the constraints in $\Pi$.

Intuitively, we could expect that all the above requirements are met if, for each class $A$, we could define
\begin{equation}\label{eq:cm}
    m_A = \max(h_A,  m^{r_1}_A, \ldots, m^{r_p}_A),
\end{equation}
where $r_1, \ldots, r_p$ are all the constraints in $\Pi$ with head $A$ and, for each such constraint $r_i$ of the form~(\ref{eq:rule}), 
$$
 m^{r_i}_A = \min(m_{A_1},\ldots,m_{A_k},\ov{m}_{A_{k+1}},\ldots,\ov{m}_{A_n}). %
$$

However, in general, the above equations may lead to not uniquely defined values and not minimal predictions because of circular definitions.

\begin{example}\label{ex:pcirc} {\rm 
If $\Pi$ is the set of constraints
\begin{equation}\label{eq:circ-p}
A_1 \to A_2; \qquad A_2 \to A_1,
\end{equation}
then, by Equation (\ref{eq:cm}),
$
m_{A_1} \,{=}\, \max(h_{A_1}, m_{A_2})$ and  
$m_{A_2} \,{=}\, \max(h_{A_2}, m_{A_1})$, 
and this allows for infinitely many solutions, unless $h_{A_1}\,{=}\,1$ or $h_{A_2}\,{=}\, 1$. Furthermore, any solution with $m_{A_1} = m_{A_2} > \tv$ when $h_{A_1} < \tv$ and $h_{A_2} < \tv$ leads to a set of predictions that satisfies the constraints but is not minimal.\hfill$\lhd$
}\end{example}

We will show that such problems, due to circularities involving only positive classes (as the one in the example), can be solved if we consider the minimum of the set of tuples of values satisfying the equations (\ref{eq:cm}): in the case of  Example \ref{ex:pcirc}, the minimum is $m_{A_1} = m_{A_2} = \max(h_{A_1}, h_{A_2})$.%
\footnote{Given a set $S$ of $t$-tuples of real numbers, $s_1,\ldots,s_t \in S$ is the {\sl minimum of $S$} if for every $t_1,\ldots,t_t \in S$ and for every $1 \leq i \leq t$, $s_i \leq t_i$. Such a minimum might not exist.} 

More problems arise when we have circularities involving negated classes.

\begin{example} {\rm 
If $\Pi$ is the set of constraints
\begin{equation}\label{eq:circ}
\neg A_1 \to A_2; \qquad \neg A_2 \to A_1,
\end{equation}
then, by Equation (\ref{eq:cm}),
$
m_{A_1} = \max(h_{A_1}, \ov{m}_{A_2})$ and  $m_{A_2} = \max(h_{A_2}, \ov{m}_{A_1}),
$
and,  e.g., for $h_{A_1} \,{ =}\, h_{A_2}  \,{ =}\, 0$, 
it exists no minimum pair of values satisfying the equations.
Furthermore, if we set 
$
m_{A_1}  \,{ =}\, \max(h_{A_1}, \ov{h}_{A_2})$ and
$m_{A_2}   =\max(h_{A_2}, \ov{h}_{A_1})
$,
then if for a data point $x$, we get $h_{A_1} < \tv$ and $h_{A_2} < \tv$, then 
$m_{A_1} > \tv$ and $m_{A_2} > \tv$, i.e.,
even if $h$ predicts that $x$ belongs to neither $A_1$ nor $A_2$, $m$  predicts that $x$ belongs to both  $A_1$ and $A_2$, 
and the set of classes ${\cal M} = \{A_1, A_2\}$ is not supported relative to  ${\cal H} = \emptyset$ and the constraints in (\ref{eq:circ}).\hfill$\lhd$
}\end{example}

To avoid the situation described in the above example,  
whenever we use the negation on a class, we should refer to an already known value for the class itself.
More specifically, first some
classes should be computed 
without the use of negation. Next, some new classes can be computed possibly using the negation of the already computed classes, and this process can be iterated. When this is possible, the set of constraints is stratified~\cite{aptBW88}.

There are several equivalent definitions of stratifiedness. Here, we use the one from \cite{aptBW88}. 
\begin{definition}
A set of constraints $\Pi$ is {\sl stratified} if there is a partition 
$\Pi_1, \Pi_2, \ldots, \Pi_s 
$ of $\Pi$,
with $\Pi_1$ possibly empty,
such that, for every $i\in \{1,\ldots,s\}$,
\begin{enumerate}
    \item for every class $A \in \cup_{r \in \Pi_i} body^+(r)$, all the constraints with head $A$ in $\Pi$ belong to $\cup_{j=1}^i \Pi_j$;
    \item for every $\neg A \in \cup_{r \in \Pi_i} body^-(r)$, all the constraints with head~$A$ in $\Pi$ belong to $\cup_{j=1}^{i-1} \Pi_j$.
\end{enumerate}
$\Pi_1, \Pi_2, \ldots, \Pi_s$ is a {\sl stratification} of $\Pi$, and each $\Pi_i$ is a {\sl stratum}.  
\end{definition}

The check on whether $\Pi$ is stratified and then the computation of a stratification can be done on the dependency graph of $\Pi$~\cite{aptBW88}.

\begin{definition}
Let $(${\problem}$,\Pi)$ be an {\cmc} problem. 
The {\sl dependency graph} $G_\Pi$ of $\Pi$ is the directed graph 
having the set of classes as nodes and with, for each constraint $r \in \Pi$,
\begin{enumerate}
    \item a positive edge from each class in $body^+(r)$ to $head(r)$, %
    \item a negative edge from each class $A$ such that  $\neg A \in body^-(r)$ to $head(r)$.
\end{enumerate}
\end{definition}

The following theorem is from \cite{aptBW88}.

\begin{theorem}
Let $(${\problem}$,\Pi)$ be an {\cmc} problem. $\Pi$ is stratified iff the dependency graph $G_\Pi$ of $\Pi$ contains no cycles with a negative edge.
\end{theorem}

As an easy consequence of the above theorem, every set of constraints containing only definite rules (as, e.g., in the HMC case) is stratified. 
An example with a stratified and with a non-stratified set of constraints, both containing non definite rules, is the following.

\begin{example}%
\label{ex:stratum}{\rm 
If ${\classes} = \{A, A_1, A_2, A_3, A_4\}$, and $\Pi$ is the set of constraints in (\ref{eq:constr_a}), (\ref{eq:constr_a2}),  and $A_3 \to A_4$, i.e., 
$\Pi=\{ A_1 \to A$; $A_2 \to A$; $A, \neg A_1 \to A_2$; $A_3 \to A_4\}$, then $\Pi$ is stratified: e.g., take $\Pi_1= \{A_3\to A_4\}$, and $\Pi_2 = \Pi \setminus \Pi_1$.%
\footnote{This set $\Pi$ of constraints is an example of a semi-positive set of rules \cite{aptBW88}. A set $\Pi$ is {\sl semi-positive} if for every head $A$ of a rule in $\Pi$, there is not a rule $r \in \Pi$ with $\neg A \in body(r)$. Every set of definite rules is also semi-positive, and every semi-positve set of rules is stratified.}
Any set containing the constraints in (\ref{eq:circ})  is not stratified.\hfill$\lhd$ 
}\end{example}

For a stratified set of constraints, there can be many stratifications, 
as shown by the following example.

\begin{example}[Ex. \ref{ex:stratum}, cont'd]\label{ex:stratum1}{\rm 
  $\Pi'_1= \{A_3\to A_4\}$, $\Pi'_2 = \{A_1 \to A\}$, and $\Pi'_3 = \{A_2 \to A$; $A, \neg A_1 \to A_2\}$ is another stratification of the set $\Pi$ of constraints in Example \ref{ex:stratum}.\hfill$\lhd$
}\end{example}

However, 
it is well known in the area of logic programming that all the stratifications  
lead to the same result
\cite{aptBW88}. Given this,
comparing the two stratifications in Examples \ref{ex:stratum} and \ref{ex:stratum1}, the latter has two drawbacks:
\begin{enumerate}
    \item the class $A$ is in the head of constraints belonging to different strata, and
    \item it has one more stratum.
\end{enumerate}
Indeed, for each stratum $\Pi_i$, we 
want to
compute a value for all the classes in the head of the constraints in $\Pi_i$ as a single step on GPUs, and thus
\begin{enumerate}
\item
we would like to have all the constraints with the same head $A$ just in one stratum, so that we 
can compute a value for $A$ just once, and 
\item
we would like to have as few strata as possible,  to minimize the number of steps.
\end{enumerate}
Thus, assuming $\Pi$ is stratified,
\begin{enumerate}
    \item we compute the acyclic component graph \cite{cormen2009} of the dependency graph $G_\Pi$ of $\Pi$, i.e., the DAG obtained by shrinking each strongly connected component in $G_\Pi$ into a single vertex (notice that since $\Pi$ is stratified, negative edges are not involved in any cycle in $G_\Pi$),
    \item we assign to the classes in each node of the DAG the number 1 plus the maximum number of negative edges connecting a root to the node, and
    \item we define: 
    \begin{enumerate}
        \item ${\classes}_i$ as the set of classes
        having the number $i$ assigned at the previous step, and
        \item $\Pi_i$ as the set of constraints in $\Pi$ whose head is in ${\classes}_i$.
    \end{enumerate}
\end{enumerate}
We call the above procedure CompStrata($\Pi$).
Given a stratified set $\Pi$ of constraints, CompStrata($\Pi$) computes a partition ${\classes}_1, \ldots, {\classes}_s$ of the set ${\classes}$ of classes and also the corresponding stratification $\Pi_1,\ldots,\Pi_s$ of $\Pi$ with the smallest possible number of strata.

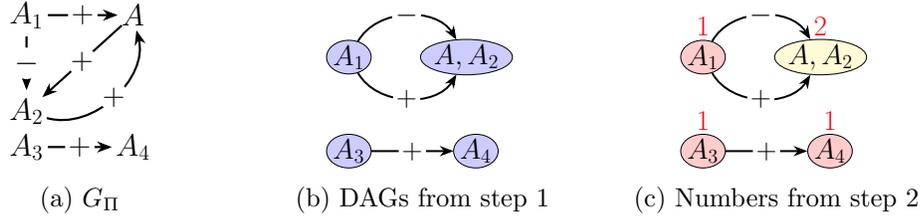
\begin{figure}[]
    \centering
    \begin{subfigure}[b]{0.23\linewidth}
         \centering
        \resizebox{\textwidth}{!}{\hspace{1cm}\usetikzlibrary{arrows.meta}
\begin{tikzpicture}
\begin{scope}[every node/.style={circle,thick,draw=none,inner sep=0, outer sep=0}]
    \node (A2) at (0,-0.4) {\LARGE $A_2$};
    \node (A1) at (0,1.5) {\LARGE $A_1$};
    \node (A) at (2.1,1.5) {\LARGE $A$};
    \node (A3) at (0,-1.1) {\LARGE $A_3$};
    \node (A4) at (2.1,-1.1) {\LARGE $A_4$};
\end{scope}

\begin{scope}[>={Stealth[black]},
              every node/.style={draw=none,fill=white,circle,inner sep=0, outer sep=0},
              every edge/.style={draw=black,very thick}]
    \path [->] (A1) edge node {\LARGE $+$} (A);
    \path [->] (A1) edge node {\LARGE $-$} (A2);
    \path [->] (A2) edge[bend right=60] node {\LARGE $+$} (A); 
    \path [->] (A) edge node {\LARGE $+$} (A2); 
    \path [->] (A3) edge node {\LARGE $+$} (A4); 
\end{scope}
\end{tikzpicture}\hspace{1cm}}
         \caption{$G_\Pi$}
         \label{fig:g_pi}
     \end{subfigure}
     \qquad
     \begin{subfigure}[b]{0.25\linewidth}
         \centering
         \resizebox{\linewidth}{!}{\hspace{1cm}\begin{tikzpicture}
\begin{scope}[every node/.style={ellipse,thin,draw=black, fill=blue!20, inner sep=0, outer sep=0}]
    \node (A1) at (0,0) {\LARGE $A_1$};
    \node (AA2) at (2.5,0) {\LARGE $A, A_2$};
    \node (A3) at (0,-2) {\LARGE $A_3$};
    \node (A4) at (2.7,-2) {\LARGE $A_4$};
\end{scope}

\begin{scope}[>={Stealth[black]},
              every node/.style={draw=none,fill=white,circle,inner sep=0, outer sep=0},
              every edge/.style={draw=black,very thick}]
    \path [->] (A1) edge[bend right=60] node {\LARGE $+$} (AA2);
    \path [->] (A1) edge[bend right=-60] node {\LARGE $-$} (AA2);
    \path [->] (A3) edge node {\LARGE $+$} (A4);
\end{scope}
\end{tikzpicture}\hspace{1cm}} 
         \caption{DAGs from step 1}
         \label{fig:dags1}
     \end{subfigure} 
     \qquad
     \begin{subfigure}[b]{0.25\linewidth}
         \centering
         \resizebox{\linewidth}{!}{\hspace{1cm}\begin{tikzpicture}
\begin{scope}[every node/.style={ellipse,thin,draw=black, fill=red!20, inner sep=0, outer sep=0.2}]
    \node[label={\LARGE \color{red}{1}}] (A1) at (0,0) {\LARGE $A_1$};
    \node[label={\LARGE \color{red}{2}}, fill=yellow!20] (AA2) at (2.5,0) {\LARGE $A, A_2$};
    \node[label={\LARGE \color{red}{1}}] (A3) at (0,-2) {\LARGE $A_3$};
    \node[label={\LARGE \color{red}{1}}] (A4) at (2.7,-2) {\LARGE $A_4$};
\end{scope}

\begin{scope}[>={Stealth[black]},
              every node/.style={draw=none,fill=white,circle,inner sep=0, outer sep=0},
              every edge/.style={draw=black,very thick}]
    \path [->] (A1) edge[bend right=60] node {\LARGE $+$} (AA2);
    \path [->] (A1) edge[bend right=-60] node {\LARGE $-$} (AA2);
    \path [->] (A3) edge node {\LARGE $+$} (A4);
\end{scope}
\end{tikzpicture}\hspace{1cm}}
         \caption{Numbers from step 2}
         \label{fig:dags2}
    \end{subfigure}
    \caption{Given $\Pi$ as in Example \ref{ex:strata-comp}, visual representation of (a) $G_{\Pi}$, (b) the acyclic~com\-po\-nent graph of $G_\Pi$, (c) the number assigned to each class: 1 to $A_1, A_3, A_4$, and 2 to~$A, A_2$.}
    \label{fig:alg_strata}
\end{figure}

\begin{example}[Ex.~\ref{ex:stratum} %
cont'd]\label{ex:strata-comp}{\rm 
If $\Pi=\{ A_1 \to A$; $A_2 \to A$; $A, \neg A_1 \to A_2$; $A_3 \to A_4\}$, then {\rm CompStrata}$(\Pi)$ computes
 ${\classes}_1 \,{=} \,\{A_1, A_3, A_4\}$, ${\classes}_2 \,{=}\, \{A, A_2\}$,
 $\Pi_1= \{A_3\to A_4\}$, and $\Pi_2 = \Pi \setminus \Pi_1$, as shown in Figure~\ref{fig:alg_strata}.\hfill$\lhd$
}\end{example}

We now prove that CompStrata($\Pi$) indeed computes the stratification of $\Pi$ having the smallest number of strata. 

\begin{theorem}
Let $(${\problem}$,\Pi)$ be an {\cmc} problem with stratified $\Pi$. Let $\Pi_1, \ldots, \Pi_s$ be the partition of $\Pi$ computed by {\rm CompStrata}$(\Pi)$. Then,
\begin{enumerate}
    \item $\Pi_1, \ldots, \Pi_s$ is a stratification of $\Pi$, and
    \item there exists no stratification of $\Pi$ with a smaller number of strata.
\end{enumerate}
\end{theorem}

\begin{proof}
By induction on the number $s$ of strata.

Assume $s=1$. Then, the constraints in $\Pi$ are
definite rules, $\Pi_1 = \Pi$, and the statement follows. 

Assume $s=n+1 > 1$, and 
let ${\classes}_1, \ldots, {\classes}_{n+1}$ be the partition 
of the set ${\classes}$ of classes computed by CompStrata($\Pi$).
By the induction hypothesis,
$\Pi_1, \ldots, \Pi_n$ is a  stratification of $\cup_{i=1}^n \Pi_i$ with the smallest number of strata. 
By construction, for each class $A$ in ${\classes}_{n+1}$ 
the path from a root to $A$ in $G_{\Pi}$ containing the maximum number of negative edges contains $n$ negative edges. 
Thus, 
for each 
class $A \in {\classes}_{n+1}$,  
there exists a constraint $r \in \Pi_{n+1}$ such that:
\begin{enumerate}
\item $head(r) = A$, 
\item there exists a class $B \in {\classes}_n$ such that $\neg B \in body^-(r)$.
\end{enumerate}
Furthermore, $\Pi = \cup_{i=1}^{n+1} \Pi_i$ is stratified, since $\cup_{i=1}^{n} \Pi_i$ is stratified by the induction hypothesis and, for every constraint $r \in \Pi_{n+1}$,
\begin{enumerate}
    \item each class in $body^+(r)$ belongs to $\cup_{i=1}^{n+1} {\classes}_i$, and
    \item each class in $body^-(r)$ belongs to $\cup_{i=1}^n{\classes}_i$.
\end{enumerate}
 Since $\Pi_1, \ldots, \Pi_n$ is a stratification of $\cup_{i=1}^{n} \Pi_i$ with the smallest number of strata, then  $\Pi_1, \ldots, \Pi_{n+1}$ is a stratification of $\Pi$ with the smallest number of strata. 
\end{proof}

In the sequel, assume $\Pi$ is stratified, and that ${\classes}_1,{\classes}_2,\ldots,{\classes}_s$ and  $\Pi_1,\Pi_2,\ldots,\Pi_s$ are the partition of ${\classes}$ and the  stratification of $\Pi$  computed by CompStrata($\Pi$), respectively.

Consider the $i$th stratum ($1 \leq i \leq s$).

If there is more than one stratum ($s > 1$), then we 
iteratively compute the values for the classes in 
${\classes}_i$ 
${\classes}_1 \cup \ldots \cup {\classes}_{i-1}$.
However, inside each single stratum (even the first one), there can be a chain of rules affecting the values of the classes in the stratum. Consider, e.g., the set of constraints $\Pi = \{A_1, A_2 \to A_3; A_3 \to A\}$. In this case,  
\begin{enumerate}
    \item we do not want to first compute the final value for $A_3$ as $\max(h_{A_3}, \min(h_{A_1},h_{A_2}))$ and then, in a second step, use it to compute the final value for $A$ (which 
    could
    be problematic if we also have, e.g., $A \to A_3$), instead
    \item we want to directly compute the final value for $A$ as $\max(h_A, h_{A_3}, \min(h_{A_1},h_{A_2}))$ as a single operation.
\end{enumerate}
We therefore compute and then use the transitive closure of all the constraints in the same stratum. The additional constraints in the closure are conceptually redundant, but they allow for an improvement in performance, as in \cite{dengeccv2014}.

Define $\Pi_i^*$ to be the set of constraints
\begin{enumerate}
    \item initially equal to $\Pi_i$, and then 
    \item obtained by recursively adding the constraints obtained from a constraint $r$ already in $\Pi_i^*$ by substituting a class $A \in body^+(r) \cap {\classes}_i$ with $body(r')$ for any rule $r'$ with $head(r') = A$ (hence, $r' \in\Pi_i$), and finally
    \item eliminating the constraints $r$ such that
    $head(r)\in body^+(r)$, or for which there exists another constraint $r' \in \Pi_i$ with $head(r)=head(r')$ and $body(r') \subset body(r)$.
\end{enumerate}
$\Pi_i^*$ is guaranteed to be finite, since the set of classes ${\classes}$ is finite, and we do not allow for repetitions in the body of constraints. The constraints being eliminated in the third step are redundant.

Then, for each class $A \in {\classes}_i$, we define the output $\module_A$ of the constraint module $\module$~via
\begin{equation}\label{eq:cms}
    \module_A = \max(h_A,  h^{r_1}_A, \ldots, h^{r_p}_A),
\end{equation}
where
\begin{enumerate}
    \item $r_1, \ldots, r_p$ are all the constraints in $\Pi^*_i$ with head $A$, and
    \item assuming $r \in \Pi^*_i$ has the form (\ref{eq:rule}),
$$
 h^{r}_A = \min(v_{A_1},\ldots,v_{A_k},\ov{\module}_{A_{k+1}},\ldots,\ov{\module}_{A_n}), %
$$

with $v_{A_1} = h_{A_1}$ if $A_1 \in {\classes}_i$, and $v_{A_1} = \module_{A_1}$ if $A_1 \in \cup_{j = 1}^{i-1}{\classes}_j$. 
Analogously for $v_{A_2}, \ldots, v_{A_k}$.
\end{enumerate}
The above definition is well-founded:
\begin{enumerate}
    \item $\Pi^*_1$ does not contain negated classes, and thus the definition of $\module_A$ when $A \in {\classes}_1$ relies only on the outputs of the bottom module $h$, and
    \item  the definition of $\module_A$ when $A \in {\classes}_i$ ($i > 1$) uses only outputs of $h$ or of already defined outputs of $\module$.
\end{enumerate}

\begin{example}[Ex.~\ref{ex:strata-comp},  
cont'd]\label{ex:stratum7}{\rm 
$\Pi_1 \,{=}\, \Pi^*_1 \,{=}\, \{A_3 \to A_4\}$, while $\Pi^*_2  \,{=}\,  \Pi_2 \cup \{A_1, \neg A_1 \to A_2\}$. Thus,
$$
\begin{array}{c}
\module_{A_1} = h_{A_1},\quad
\module_{A_3} = h_{A_3},\quad
\module_{A_4} = \max(h_{A_3},h_{A_4}),\\
\module_{A} = \max(h_A, \module_{A_1}, h_{A_2}), \qquad
\module_{A_2} = \max(h_{A_2},\min(h_{A},\ov{\module}_{A_1}),\min(\module_{A_1},\ov{\module}_{A_1})).
\end{array}
$$
If $h_{A_1} = 0.2$, $h_{A_2} = 0.3$, $h_{A} = 0.6$ (as in Example \ref{ex:gradients}), then
$\module_{A_1} = 0.2$, $\module_{A_2} = 0.6$, $\module_{A} = 0.6$.

The constraint $A_1, \neg A_1 \to A_2 \in \Pi^*_2$, which leads to the inclusion of $\min(\module_{A_1},\ov{\module}_{A_1}) = \min(h_{A_1},\ov{h}_{A_1})$ in the definition of $\module_{A_2}$, is necessary in order to guarantee that $\module$ never violates (\ref{eq:constr_a2}), i.e., that it always holds
\begin{equation}\label{eq:condA2}
\min(\module_{A}, \ov{\module}_{A_1}) \leq \module_{A_2}.
\end{equation}
Indeed assume, $h_A \,{=}\, h_{A_2} \,{=}\, 0.3, h_{A_1} \,{=}\, 0.6$. Then, %
$\module_A = \module_{A_1} \,{=}\, 0.6$, $\module_{A_2} \,{=}\, 0.4$, and (\ref{eq:condA2}) is satisfied. If we would have defined 
$\module_{A_2} \,{=}\, \max(h_{A_2},\min(h_{A},\ov{\module}_{A_1}))$ 
(omitting $\min(\module_{A_1},\ov{\module}_{A_1})$), 
then we would have obtained $\module_{A_2} \,{=}\, 0.3$, and
(\ref{eq:condA2}) would have been no longer satisfied.\hfill$\lhd$
}\end{example}

Given a stratified set $\Pi$ of constraints, {\system{$h$}} is guaranteed to always satisfy $\Pi$.

\begin{theorem}\label{thm:8}
Let $(${\problem}$,\Pi)$ be an {\cmc} problem. Assume   $\Pi$ is stratified.  Let $h$ be a model for~{\problem}. Then, 
\system{$h$} satisfies Equation~(\ref{eq:cm}) and thus commits no constraint violations.
\end{theorem}
\begin{proof}
Let ${\classes}_1, \ldots, {\classes}_s$ be the partition of ${\classes}$ computed by {\rm CompStrata}$(\Pi)$.
We recall that, for a class $A$, (\ref{eq:cm}) is
$$
m_A = \max(h_A,  m^{r_1}_A, \ldots, m^{r_p}_A),
$$
where $r_1, \ldots, r_p$ are all the constraints in $\Pi$ with head $A$ and, for each such constraint $r_j$ of form
(\ref{eq:rule}), 
$$
 m^{r_j}_A = \min(m_{A_1},\ldots,m_{A_k},\ov{m}_{A_{k+1}},\ldots,\ov{m}_{A_n}). %
$$

Consider a generic class $A \in {\classes}_i$. We show that the definition of $\module_A$ is equal to the expression 
resulting from the substitution of
each $m_{A_l}$ in $m^{r_j}_A$ ($1 \leq j \leq p$) with 
\begin{enumerate}
    \item %
    $\module_{A_l}$ 
    if $A_l\! \in\! \cup_{j=1}^{i-1}{\classes}_j$,
    \item %
    the right-hand side of Equation (\ref{eq:cms}) if $A_l \in {\classes}_i$. 
\end{enumerate}
Consider the result of such a substitution in $m^{r_j}_A$. Applying the distributivity of the minimum operation over the maximum operation, we get
$$
m^{r_j}_A = \max_{r \in \Pi^*_i(r_j)} (\min(v_{body(r)})), 
$$
where 
\begin{enumerate}
    \item $\Pi^*_i(r_j)$ is the set of constraints initially equal to $\{r_j\}$ and then obtained by recursively adding the constraints obtained from a constraint $r \in \Pi^*_i(r_j)$  by substituting a class $B \in body^+(r) \cap {\classes}_i$ with $body(r')$ for any constraint $r'$ with $head(r')=B$, and
    \item $v_{body(r)}$ is the set
    $$
    \begin{array}{cc}
    \{\module_B : B \in body^+(r), B \in \cup_{j=1}^{i-1}{\classes}_j\} \cup \\
        \{h_B : B \in body^+(r), B \in {\classes}_i\} \cup\\
                \{\ov{\module}_B : \neg B \in body^-(r)\}.
                \end{array}
                $$
\end{enumerate}

Since the set of constraints in $\Pi^*_i$ with head $A$ is equal to $\cup_{j=1}^p \Pi^*_i(r_j)$, the statement follows.
\end{proof}

Given the values for the classes in $\cup_{j=1}^{i-1} {\classes}_j$, the 
values of $\system{h}$ for the classes in ${\classes}_i$
correspond to the minimum of the set of tuples satisfying (\ref{eq:cm}).

\begin{theorem}\label{thm:smallest}
Let $(${\problem}$,\Pi)$  be an {\cmc} problem. Assume   $\Pi$ is stratified.  Let $h$ be a model for~{\problem}.
Let ${\classes}_1, \ldots, {\classes}_s$ be the partition of ${\classes}$ computed by {\rm CompStrata}$(\Pi)$. For
$1 \leq i \leq s$, let $m$ be a model for {\problem} satisfying Equation~(\ref{eq:cm}) and such that for every class $B \in \cup_{j=1}^{i-1} {\classes}_j$, $m_B =$ \system{$h$}$_B$. For every class $A \in {\classes}_i$, $m_A \geq$ \system{$h$}$_A$. 
\end{theorem}

\begin{proof}
We recall that, for each class $A$, \system{$h$}$_A = \module_A$, and we use $\module_A$, since shorter.

Consider the partition ${\classes}_1, \ldots, {\classes}_s$ of the set of classes ${\classes}$ and the corresponding stratification $\Pi_1, \ldots, \Pi_s$ of $\Pi$.
We prove that the outputs of \module{} for the classes in ${\classes}_i$ are the smallest values satisfying (\ref{eq:cm}), given the values $\module_{B}$ for $B \in \cup_{j=1}^{i-1} {\classes}_j$.

If a model satisfies (\ref{eq:cm}), then it also satisfies the inequalities (\ref{eq:safer}) associated with  each constraint (\ref{eq:rule}) in $\Pi_i$.
Thus, we first prove that for any model $m$ satisfying (\ref{eq:safer}) for each constraint (\ref{eq:rule}) in $\Pi_i$, given the values $m_{B}$ for $B \in \cup_{j=1}^{i-1} {\classes}_j$, $m_A \geq \module_A$. This allows us to conclude that any model $m$ satisfying (\ref{eq:cm}) has values bigger than or equal to those of \module.

Let $I(\Pi_i)$ be the set of inequalities (\ref{eq:safer}), one for each constraint of the form (\ref{eq:rule}) in $\Pi_i$ union, for each class $A \in {\classes}_i$,
\begin{equation}\label{eq:mh}
  h_A \leq m_A.
\end{equation}
We  represent $I(\Pi_i)$ as the set of pairs $\langle \mathcal{S},m_{A}\rangle$, where $\cal{S}$ is the set  $\{m_{A_1},\ldots,\ov{m}_{A_n}\}$ (resp., $\{h_A\}$) in the case of (\ref{eq:safer}) (resp., (\ref{eq:mh})).

Define $I^*(\Pi_i)$ to be the set of inequalities obtained from $I(\Pi_i)$ by recursively adding the pairs $\langle \mathcal{S} \cup \mathcal{S}' \setminus A, m_{B} \rangle$
such that
$\langle \mathcal{S},m_A\rangle \in I^*(\Pi_i)$, $\langle \mathcal{S}',m_{B}\rangle \in I^*(\Pi_i)$ and $A \in \mathcal{S}'$. $I^*(\Pi_i)$ is finite, and each inequality in 
$I^*(\Pi_i)$ is entailed by the inequalities in $I(\Pi_i)$. Thus, $I^*(\Pi_i)$ and $I(\Pi_i)$ have the same set of solutions.
For each constraint $r$ in $\Pi^*_i$ of the form (\ref{eq:rule})  with head $A \in {\classes}_i$, the inequality
\begin{equation}\label{eq:cmva}
    \langle \{v_{A_1},\ldots,v_{A_k},\ov{m}_{A_{k+1}},\ldots,\ov{m}_{A_n}\},A\rangle
\end{equation}
belongs to $I^*(\Pi_i)$, where
$v_{A_1} = h_{A_1}$ if $A_1 \in {\classes}_i$, and $v_{A_1} = m_{A_1}$ if $A_1 \in \cup_{j = 1}^{i-1} {\classes}_j$. 
Analogously for $v_{A_2}, \ldots, v_{A_k}$.

By definition, $\module_A$ is the smallest value satisfying the inequalities (\ref{eq:cmva}) in $I^*(\Pi_i)$. Thus, for any model $m$ satisfying $I^*(\Pi_i)$, $m_A \geq \module_A$.
\end{proof}

We can now state that \system{$h$} has the desired properties mentioned at the beginning of the section.

\begin{theorem}\label{thm:properties}
Let $(${\problem}$,\Pi)$  be an {\cmc} problem. Assume   $\Pi$ is stratified.  Let $h$ be a model for {\problem}. 
Let $\mathcal{H}$ be the set of classes predicted by $h$.
Let $\mathcal{M}$ be the set of classes predicted by \system{$h$}.
Then, $\mathcal{M}$
\begin{enumerate}
    \item extends $\mathcal{H}$ and is coherent with $\Pi$, 
    \item is supported relative to $\mathcal{H}$ and $\Pi$, 
    \item is minimal relative to $\mathcal{H}$ and $\Pi$, and 
    \item is the unique set satisfying the previous properties.
\end{enumerate}
\end{theorem}

\begin{proof}
We prove each statement of the theorem separately. By hypothesis, $\mathcal{H}$ is the set of classes $A$ such that $h_A > \tv$. $\mathcal{M}$ is the set of classes $A$ such that $\module_A > \tv$.
${\classes}_1, \ldots, {\classes}_s$ and $\Pi_1, \ldots, \Pi_s$ is the partition of ${\classes}$ and $\Pi$ computed by CompStrata($\Pi$), respectively.

\begin{enumerate}
\item {\system{$h$} extends the set of classes associated with $x$ by $h$}. We have to prove that $\mathcal{H} \subseteq \mathcal{M}$.
$A \in \mathcal{H}$ iff  $h_A > \tv$. Since $\module_A \geq h_A$, if 
$A \in \mathcal{H}$, then $A \in \mathcal{M}$.

$\mathcal{M}$ is coherent with $\Pi$.
We have to prove that for each constraint $r$ of the form (\ref{eq:rule}), if $\min(\module_{A_1},\ldots,\module_{A_k},$ $\ov{\module}_{A_{k+1}},\ldots,\ov{\module}_{A_n}) > \tv$, then $\module_A > \tv$, which is an easy consequence of the fact that the outputs of $\module$ satisfy (\ref{eq:cm}) and thus also (\ref{eq:safer}).

\item{$\mathcal{M}$ is supported relative to  $\mathcal{H}$ and $\Pi$}. Assume it is not. Then, consider a class $A \in \mathcal{M} \setminus \mathcal{H}$
such that for each constraint $r \in \Pi$ of the form (\ref{eq:rule}), there exists an index~$j\in \{1, \ldots, n\}$ such that
\begin{enumerate}
    \item both  $\module_{A_j} \leq \tv$ and $j = 1,\ldots,k$, or 
    \item both  $\ov{\module}_{A_j} \leq \tv$ and $j = k+1,\ldots, n$. 
\end{enumerate}
By Theorem \ref{thm:8}, $\module_A$ is the smallest value satisfying Equation~(\ref{eq:cm}) given the values $m_{B}$ for $B \in \cup_{j=1}^{i-1} {\classes}_j$. Thus, $\module_A > \tv$ is not possible, and this contradicts $A \in \mathcal{M}$.

\item{$\mathcal{M}$ is minimal relative to  $\mathcal{H}$ and $\Pi$}. 
Let $\mathcal{M}_0 = \mathcal{H}$. Let $\mathcal{M}_{i+1}$ be the closure of $\mathcal{M}_i$ under $\Pi^*_{i+1}$ ($1 \leq i < s$). The statement is a consequence of the
minimality of each $\mathcal{M}_{i+1}$ and the 
fact that $\mathcal{M} = \mathcal{M}_s$.

\item
$\mathcal{M}$ is the unique set satisfying all the previous properties. 
As above, let $\mathcal{M}_0 = \mathcal{H}$, and let $\mathcal{M}_{i+1}$ be the closure of $\mathcal{M}_i$ under $\Pi^*_{i+1}$ ($1 \leq i < s$). The statement is a consequence of the 
uniqueness of each $\mathcal{M}_{i+1}$ and the 
fact that $\mathcal{M} = \mathcal{M}_s$.
\end{enumerate}
\end{proof}

If we interpret the given set of constraints as a stratified normal logic program, 
we can establish  a relation with
the canonical model semantics of stratified normal logic programs \cite{aptBW88}, which coincides with the stable model semantics
 \cite{gelfond88}.

\begin{definition}
Let $\Pi$ be a finite set of normal rules. Let ${\cal M}$ be a set of classes.
The {\sl reduct} of $\Pi$ relative to ${\cal M}$ is the set $\Pi^{\cal M}$ of definite rules obtained by: 
    \begin{enumerate}
        \item dropping the rules $r$ in $\Pi$ such that for a class $A \in {\cal M}$, $\neg A \in body(r)$, and then
        \item dropping $body^-(r)$ from the remaining rules $r$. 
    \end{enumerate}
${\cal M}$ is a {\sl stable model} of $\Pi$ iff 
${\cal M}$ is the smallest set closed under $\Pi^{\cal M}$.
\end{definition}

\begin{example}{\rm 
Let $\Pi=\{ A_1 \to A$; $A_2 \to A$; $A, \neg A_1 \to A_2$; $A_3 \to A_4\}$.
Then,
\begin{enumerate}
    \item the stable model of $\Pi$ is the empty set,
    \item the stable model of $\Pi \cup \{\to A\}$ is $\{A_2, A\}$, and 
    \item the stable model of $\Pi \cup \{\to A; \to A_4\}$ is $\{A_2, A_4, A\}$.\hfill$\lhd$
\end{enumerate}}
\end{example}

\begin{theorem}\label{thm:model}
Let $(${\problem}$,\Pi)$  be an {\cmc} problem with stratified $\Pi$.  Let $h$ be a model for {\problem}. 
Let $\mathcal{H}$ be the set of classes predicted by $h$.
Let $\mathcal{M}$ be the set of classes predicted by {\rm \system{$h$}}.
Then, $\mathcal{M}$ is the stable model of the set of constraints
$\Pi \cup \{ \to A : A \in {\cal H}\}$.
\end{theorem}
\begin{proof}
First, observe that if $\Pi$ is stratified, then also $\Pi \cup \{ \to A : A \in {\cal H}\}$ is stratified. The theorem is an easy consequence of the fact that in the case of a stratified set of constraints, the stable and the canonical model semantics coincide (see, e.g., \cite{gelfond88}), and the latter, for $\Pi \cup \{ \to A : A \in {\cal H}\}$, is defined as follows.

Consider the partition ${\classes}_1, \ldots, {\classes}_s$ of the set of classes ${\classes}$ and the corresponding stratification $\Pi_1, \ldots, \Pi_s$ resulting from CompStrata($\Pi$).
Define $\mathcal{M}_0 = \mathcal{H}$, and, for each $0 \leq i < s$,
$$
\mathcal{M}_{i+1} = T(\mathcal{M}_i,{\Pi_{i+1}^{\mathcal{M}_i}}),
$$
where $T(\mathcal{M}_i,{\Pi_{i+1}^{\mathcal{M}_i}})$ is the smallest superset of $\mathcal{M}_i$  closed under the reduct $\Pi_{i+1}^{\mathcal{M}_i}$ of $\Pi_{i+1}$ relative to $\mathcal{M}_i$.
$\mathcal{M}_{i+1}$ is unique.

Then, the canonical model  $\mathcal{M}$ of $\Pi \cup \{ \to A : A \in {\cal H}\}$ 
is $\mathcal{M} = \mathcal{M}_s$. 
\end{proof}

Finally,  the definition of the output $\module_A$ of \system{$h$} for class $A \in {\classes}_i$ is based on the computation of $\Pi^*_i$, which, as we have seen in Example \ref{ex:stratum7}, may contain constraints with logically contradictory bodies (i.e., with $B$ and $\neg B$ in the body, for some class $B$). Indeed, if in the construction of $\Pi^*_i$ we would have not included such constraints with contradictory bodies, the resulting system may exhibit
\begin{enumerate}
    \item  constraint violations (as seen in Example \ref{ex:stratum7}), but
    \item still no logical violations, since the set of predicted classes does not change.
\end{enumerate}

\subsubsection{Constraint Loss --- {\loss}}\label{sec:gen_loss}

In the general case, for every data point, the value of the loss function {\loss} used to train  \system{$h$} is defined as:
$$
\loss = \sum_{A \in {\mathcal A}} \loss_{A},
$$
$\loss_{A}$ being the value of the loss for class $A$, defined as:
 $$
 \loss_A = -y_A \ln(\module_A^+) - \ov{y}_A \ln(\ov{\module}_A^-), 
 $$
 where:
 \begin{itemize}
     \item $y_A$ is the ground truth for class $A$, 
     \item $\module_A^+$ is the value to optimize when $y_A = 1$, and
     \item $\module_A^-$ is the value to optimize when $y_A = 0$. 
 \end{itemize}
$\module_A^+$ and $\module_A^-$ differ from the output value $\module_A$ of \system{$h$} for class $A$, i.e., from \system{$h$}$_A$. Indeed, as it has been the case in the HMC setting and already discussed in the basic case, for $\module_A^+$ and $\module_A^-$ we have to take into account also the ground-truth. 

Similarly to what has been done for computing
$\module_{A}$, for a stratified set of constraints,
the computation of $\module_A^+$ and $\module_A^-$ will be done stratum after stratum, starting from the first. We thus assume: 
\begin{enumerate}
    \item 
that $\Pi$ is stratified, 
\item
that 
${\classes}_1, {\classes}_2, \ldots, {\classes}_s$ and
$\Pi_1, \Pi_2, \ldots, \Pi_s$ ($s \geq 1$) are the partitions of ${\classes}$ and $\Pi$  computed by CompStrata($\Pi$), and 
\item
that
$\Pi^*_1, \Pi^*_2, \ldots, \Pi^*_s$ are the  sets of constraints corresponding to $\Pi_1, \Pi_2, \ldots, \Pi_s$ and defined as in the previous subsection.
\end{enumerate}

The values $\module_A^+$ and $\module_A^-$ associated to a class $A$ in the $i$th stratum will depend on the set $\Pi^*_A$ of constraints with head $A$ in $\Pi^*_i$, and thus on:
\begin{enumerate}
    \item the values computed by model $h$ for $A$ and for the classes in the $i$th stratum and in the body of a constraint in $\Pi^*_A$,
    \item the values of the already computed loss for the classes in the lower strata and in the body of a constraint in $\Pi^*_A$,
    \item the ground truth for the classes in the body of the constraints in $\Pi^*_A$.
\end{enumerate}

Consider a class $A$ in the $i$th stratum (i.e., $A \in {\classes}_i$). To each constraint $r$ with head $A$ in $\Pi^*_i$ we associate two values
\begin{enumerate}
    \item $h^{+,r}_A$ to be used with $y_A = 1$, and
    \item $h^{-,r}_A$ to be used with $y_A = 0$.
\end{enumerate}

Assume $y_A = 1$. Consider a constraint $r$ in $\Pi^*_i$ with head $A$ of the form (\ref{eq:rule}).
Then, we want to teach \system{$h$} to possibly exploit the constraint $r$ for predicting $A$ if 
$y_{A_1} = \ldots = y_{A_k} =1$ 
and 
$y_{A_{k+1}} = \ldots = y_{A_n} = 0$. We thus define,
$$
h_A^{+,r} = \min(v_{A_1}y_{A_1}, \ldots, v_{A_k}y_{A_k}, \ov{v}_{A_{k+1}}\ov{y}_{A_{k+1}}, \ldots, \ov{v}_{A_n}\ov{y}_{A_{n}}),
$$
where $v_{A_l}$ is
\begin{enumerate}
    \item $h_{A_l}$ if $A_l \in {\classes}_i$ (and thus $1 \leq l \leq k$),
    \item $\module^+_{A_l}$ if $A_l \in \cup_{j=1}^{i-1} {\classes}_j$ and $1 \leq l \leq k$, 
    \item $\module^-_{A_l}$ if $A_l \in \cup_{j=1}^{i-1} {\classes}_j$ and $k+1 \leq l \leq n$.
\end{enumerate}
The value $\module_A^+$ associated to class $A$ when $y_A = 1$ is
$$
\module_A^+ = \max(h_A, h_A^{+,r_1}, \ldots, h_A^{+,r_p}), 
$$
where $r_1, \ldots, r_p$ are all the constraints in $\Pi^*_i$ with head $A$.

Assume $y_A = 0$. Consider a constraint $r$ in $\Pi^*_i$ with head $A$ of the form (\ref{eq:rule}). Then, for some class $A_l \in body^+(r)$ $y_{A_l} = 0$, or for some class $A_l \in body^-(r)$ $y_{A_l} = 1$, and we want to teach \system{$h$} to not fire the constraint $r$. We thus define
$$
 h_A^{-,r} = \min(v_{A_1}\ov{y}_{A_1} + y_{A_1}, \ldots, v_{A_k} \ov{y}_{A_k} + y_{A_k},  \ov{v}_{A_{k+1}}{y}_{A_{k+1}}+\ov{y}_{A_{k+1}}, \ldots, \ov{v}_{A_n}{y}_{A_{n}} + \ov{y}_{A_{n}}),
$$
where $v_{A_l}$ now is
\begin{enumerate}
    \item $h_{A_l}$ if $A_l \in {\classes}_i$ (and thus $1 \leq l \leq k$),
    \item $\module^-_{A_l}$ if $A_l \in \cup_{j=1}^{i-1} {\classes}_j$ and $1 \leq l \leq k$, 
    \item $\module^+_{A_l}$ if $A_l \in \cup_{j=1}^{i-1} {\classes}_j$ and $k+1 \leq l \leq n$.
\end{enumerate}
The value $\module_A^-$ associated with the class $A$ when $y_A = 0$ is
$$
\module_A^- = \max(h_A, h_A^{-,r_1}, \ldots, h_A^{-,r_p}), 
$$
where $r_1, \ldots, r_p$ are all the constraints in $\Pi^*_i$ with head $A$.

\begin{example}\label{ex:cmh}{\rm 
Consider the simpler version of Example \ref{ex:stratum7} with ${\classes}= \{A_1,A_2,A\}$, $\Pi = \{A_1 \to A; A_2 \to A; A, \neg A_1 \to A_2\}$, $h_{A_1} = 0.2$, $h_{A_2} = 0.3$, $h_A = 0.6$, as in Example \ref{ex:gradients}. 
Then, ${\classes}_1 = \{A_1\}$, ${\classes}_2 = \{A, A_2\}$, $\Pi_1^* = \emptyset$, $\Pi_2^* = \Pi \cup \{A_1, \neg A_1 \to A_2\}$ (see also Example \ref{ex:stratum7}). Assume $y_{A_1} = y_A = 1$ and $y_{A_2} = 0$. 

If $r_1, \ldots, r_4$ are the constraints listed as above, then
\begin{enumerate}
    \item $h_A^{+,r_1} = h_{A_1} = 0.2$,
$h_A^{+,r_2} = 0$, 
    \item
$\module_{A_1}^+ = h_{A_1}=0.2$, $\module_A^+ = h_{A} = 0.6$,
    \item $h_{A_2}^{-,r_3} = 1-h_{A_1} = 0.8$,
$h_{A_2}^{-,r_4} = 1-h_{A_1} = 0.8$, and
    \item $\module_{A_2}^- = 1- h_{A_1} = 0.8$.
\end{enumerate}
Thus,
$$
\loss = -\ln(h_{A_1}) - \ln(1-(1-h_{A_1})) -\ln(h_{A}) = - 2 \ln(h_{A_1}) -\ln(h_{A}), 
$$
as already calculated in Example \ref{ex:gradients}.\hfill$\lhd$
}\end{example}

As in the hierarchical case, $\loss$ has the fundamental property that the negative gradient descent algorithm behaves as expected, i.e., that for each class, it moves in the ``right'' direction as given by the ground truth.

\begin{theorem}
Let $(${\problem}$,\Pi)$ be an {\cmc} problem. For any model $h$ for {\problem} and class $A$, let $\frac{\partial \text{\rm\loss}}{\partial h_A}$ be the partial derivative of  $\text{\rm\loss}$ with respect to $h_A$. For each data point,
if $y_A = 0$, then $\frac{\partial \text{\rm\loss}}{\partial h_A} \geq 0$, and 
if $y_A = 1$, then $\frac{\partial \text{\rm\loss}}{\partial h_A} \leq 0$.
\end{theorem}

\begin{proof}
Consider a data point and a class $A$. 
$$
\frac{\partial \loss}{\partial h_A} = \sum_{B \in {\classes}} \frac{\partial \loss_{B}}{\partial h_A}.
$$

We consider only the case $y_A = 1$ (the case $y_A = 0$ is analogous). 
Consider a class $B$.
\begin{enumerate}
    \item If $y_{B} = 1$,
$$
\loss_{B} = - \ln (\module_B^+), \qquad \qquad
\frac{\partial \loss_{B}}{\partial h_A} =
- \frac{1}{\module_B^+}\frac{\partial \module_B^+}{\partial h_A} \le 0, 
$$
because:
\begin{itemize}
    \item either $\module_B^+ = h_A$ (which is possible only if 
    $A=B$, or 
    there exists a constraint $r$ with head $B$ such that 
    $h^{+,r}_B = h_A$) and then $\frac{\partial \loss_{B}}{\partial h_A} = -\frac{1}{h_A} \le 0$,
    \item or $\module_B^+$ is a value not dependent on $h_A$ and then
    $\frac{\partial \loss_{B}}{\partial h_A} = 0$ (since $y_A = 1$, it cannot be the case that there exists a constraint $r$ with head $B$ such that $h_B^{+,r}=\ov{h}_A$).
\end{itemize}
\item if $y_{B} = 0$,
$$
\loss_{B} = -\ln{(\ov{\module}_{B}^-)}, \qquad \qquad \frac{\partial \loss_{B}}{\partial h_A} = \frac{1}{\ov{\module}_{B}^-}\frac{\partial \module_{B}^-}{\partial h_A} \le 0, 
$$
because:
\begin{itemize}
    \item either $\module_B^- = \ov{h}_A$ (which is possible only if 
    $A=B$ or 
    there exists a constraint $r$ with head $B$ such that 
    $h^{-,r}_B = \ov{h}_A$) and then $\frac{\partial \loss_{B}}{\partial h_A} = -\frac{1}{h_A} \le 0$,
    \item or $\module_B^-$ is a value not dependent on $h_A$ and then
    $\frac{\partial \loss_{B}}{\partial h_A} = 0$ (since $y_A = 1$, it cannot be the case that there exists a constraint $r$ with head $B$ such that $h_B^{-,r}=h_A$).
\end{itemize}
\end{enumerate}
Since $\frac{\partial \loss}{\partial h_A}$ is the sum of quantities that are at most zero, then $\frac{\partial \loss}{\partial h_A} \le 0$.
\end{proof}

\subsubsection{Relation between \hmcsys{$h$} and \system{$h$}}\label{sec:relation}

From the definitions of \module$_A$ and \loss$_A$, it is clear that they generalize the corresponding definitions given in the hierarchical case. Thus, \hmcsys{$h$} and \system{$h$} have the same behavior when considering HMC problems.

\begin{theorem}\label{thm:eq}
Let $(${\problem}$,\Pi)$ be an HMC problem. Let $h$ be a model for {\problem}. For any class $A$, \hmcsys{$h$}$_A$ = \system{$h$}$_A$.
\end{theorem}

\section{GPU Implementation}\label{sec:gpu}

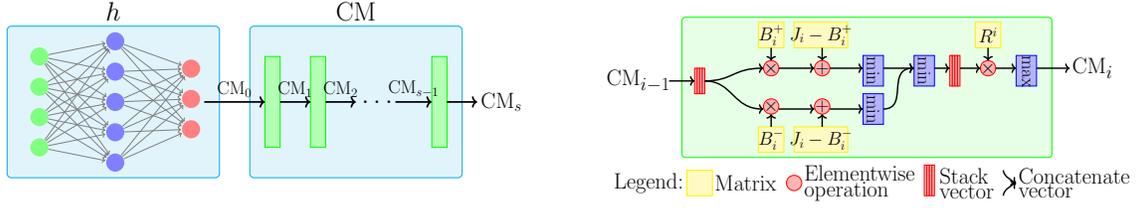
\begin{figure}[t]
\centering
\resizebox{7cm}{!}{\begin{tikzpicture}
  
  \draw [fill=cyan!10,draw=cyan,thick,rounded corners,] (0,0) rectangle (7,5) node[pos=.5] (h) {\def\layersep{2.5cm}

\begin{tikzpicture}[shorten >=1pt,->,draw=black!50, node distance=\layersep]
    \tikzstyle{every pin edge}=[<-,shorten <=1pt]
    \tikzstyle{neuron}=[circle,fill=black!25,minimum size=17pt,inner sep=0pt]
    \tikzstyle{input neuron}=[neuron, fill=green!50];
    \tikzstyle{output neuron}=[neuron, fill=red!50];
    \tikzstyle{hidden neuron}=[neuron, fill=blue!50];
    \foreach \name / \y in {1,...,4}
        \node[input neuron] (I-\name) at (0,-\y) {};
    \foreach \name / \y in {1,...,5}
        \path[yshift=0.5cm]
            node[hidden neuron] (H-\name) at (\layersep,-\y cm) {};
     \foreach \name / \y in {1,...,3}
        \path[yshift=-0.4cm]
            node[output neuron] (O-\name) at (2*\layersep,-\y cm) {};

    \foreach \source in {1,...,4}
        \foreach \dest in {1,...,5}
            \path (I-\source) edge (H-\dest);
    \foreach \source in {1,...,5}
        \foreach \dest in {1,...,3}
            \path (H-\source) edge (O-\dest);
\end{tikzpicture}} node at (3.5, 5.5) (h) {\Huge $h$};
  \draw  [fill=cyan!10,draw=cyan,thick,rounded corners,] (8,0) rectangle (15,5) node at (11.5, 5.5) (module) {\Huge \module};
  \draw  [fill=green!30,draw=green,thick] (8.5,1) rectangle (9,4) node (s1) {};
  \draw  [fill=green!30,draw=green,thick] (10,1) rectangle (10.5,4) node (s2) {};
  \draw  [fill=green!30,draw=green,thick] (14,1) rectangle (14.5,4) node (sN) {};
  \draw [circle,fill=white,draw=white] (16.3,2.5) node {\huge $\module_s$};
\draw [circle,fill=none,draw=none] (12.2,2.5) node {\Huge $\ldots$};
  \draw[->, ultra thick] (6.5,2.5) -- node [midway,above] {\LARGE $ \module_0$} (8.5,2.5);
  \draw[->, ultra thick] (9,2.5) --  node [midway,above] {\LARGE $\module_1$} (10,2.5);
  \draw[->, ultra thick] (10.5,2.5) -- node [midway,above] {\LARGE $\module_2$}(11.5,2.5);
  \draw[->, ultra thick] (12.8,2.5) --  node [midway,above] {\LARGE $\module_{s-1}$} (14,2.5);
  \draw[->, ultra thick] (14.5,2.5) --  node [right] {} (15.5,2.5);
   \draw [fill=none,draw=none] (0,-1.5) rectangle (10,0);
\end{tikzpicture}}
\qquad 
\resizebox{7cm}{!}{\begin{tikzpicture}
  \draw [fill=green!10,draw=green,thick,rounded corners,] (0,-0.5) rectangle (14.5,5) node (h) {};
   \draw [circle,fill=none,draw=none] (-1.7,2.5) node {\Huge $\module_{i-1}$};
   \draw [circle,fill=none,draw=none] (16.1,3) node {\Huge $\module_{i}$};
   \draw  [fill=red!30,draw=red,thick] (0.5,2) rectangle (0.6,3) node {};
   \draw  [fill=red!30,draw=red,thick] (0.6,2) rectangle (0.7,3) node {};
   \draw  [fill=red!30,draw=red,thick] (0.7,2) rectangle (0.8,3) node {};
   \draw  [fill=red!30,draw=red,thick] (0.8,2) rectangle (0.9,3) node {};
   \draw  [fill=yellow!30,draw=yellow,thick] (3,3.8) rectangle (4,4.8) node[pos=.5]  (iplus) {\huge $B^{+}_i$};
   \draw  [fill=red!30,draw=red,thick] (3.5,3.0) circle (0.3) node  (iplusx) {\huge $\times$};
   \draw  [fill=yellow!30,draw=yellow,thick] (3,-0.3) rectangle (4,0.7) node[pos=.5]  (iminus) {\huge $B^{-}_i$};
   \draw  [fill=red!30,draw=red,thick] (3.5,1.5) circle (0.3) node  (iminusx) {\huge $\times$};
   \draw  [fill=red!30,draw=red,thick] (5.5,3.0) circle (0.3) node  (plus) {\huge +};
   \draw  [fill=red!30,draw=red,thick] (5.5,1.5) circle (0.3) node  (plus) {\huge +};
   \draw  [fill=yellow!30,draw=yellow,thick] (4.4,3.8) rectangle (6.5,4.8) node[pos=.5]  (l) {\huge $J_i-B^+_i$};
   \draw  [fill=yellow!30,draw=yellow,thick] (4.4,-0.3) rectangle (6.5,0.7) node[pos=.5]  (l) {\huge $J_i-B^-_i$};
   \draw  [fill=blue!30,draw=blue,thick] (7.1,2.3) rectangle (7.9,3.5) node[pos=.5]  (l) {\rotatebox{-90}{\huge $\min$}};
   \draw  [fill=blue!30,draw=blue,thick] (7.1,0.8) rectangle (7.9,2.0) node[pos=.5]  (l) {\rotatebox{-90}{\huge $\min$}};
   \draw  [fill=blue!30,draw=blue,thick] (9.1,2.3) rectangle (9.9,3.5) node[pos=.5]  (l) {\rotatebox{-90}{\huge $\min$}};
   \draw  [fill=red!30,draw=red,thick] (10.5,2.3) rectangle (10.6,3.5) node {};
   \draw  [fill=red!30,draw=red,thick] (10.6,2.3) rectangle (10.7,3.5) node {};
   \draw  [fill=red!30,draw=red,thick] (10.7,2.3) rectangle (10.8,3.5) node {};
   \draw  [fill=red!30,draw=red,thick] (10.8,2.3) rectangle (10.9,3.5) node {};
   \draw  [fill=yellow!30,draw=yellow,thick] (11.5,3.8) rectangle (12.5,4.8) node[pos=.5]  (l) {\huge $R^i$};
   \draw  [fill=red!30,draw=red,thick] (12,3.0) circle (0.3) node  (Rx) {\huge $\times$};
   \draw  [fill=blue!30,draw=blue,thick] (13.1,2.3) rectangle (13.9,3.5) node[pos=.5]  (l) {\rotatebox{-90}{\huge $\max$}};
   
   \draw [->, ultra thick] (0.9,2.5) .. controls (2.0,2.5) and (2.0,3) .. (3.2,3);
   \draw [->, ultra thick] (0.9,2.5) .. controls (2.0,2.5) and (2.0,1.4) .. (3.2,1.4);
   \draw [->, ultra thick] (3.8,3) -- (5.2, 3);
   \draw [->, ultra thick] (3.8,1.4) -- (5.2, 1.4);
   \draw [->, ultra thick] (5.8,3) -- (7.1, 3);
   \draw [->, ultra thick] (5.8,1.4) -- (7.1, 1.4);
   \draw [->, ultra thick] (7.9,3) -- (9.1, 3);
   \draw [->, ultra thick] (9.9,3) -- (10.5, 3);
   \draw [->, ultra thick] (10.9,3) -- (11.7, 3);
   \draw [->, ultra thick] (12.3,3) -- (13.1, 3);
   \draw [->, ultra thick] (13.9,3) -- (15.2, 3);
   \draw [->, ultra thick] (-0.5,2.5) -- (0.5,2.5);
   
   \draw [->, ultra thick] (3.5,3.8) -- (3.5,3.3);
   \draw [->, ultra thick] (3.5,0.7) -- (3.5,1.2);
   \draw [->, ultra thick] (5.5,3.8) -- (5.5,3.3);
   \draw [->, ultra thick] (5.5,0.7) -- (5.5,1.2);
   \draw [->, ultra thick] (12,3.8) -- (12,3.3);
   
   \draw [->, ultra thick] (7.9,1.4) .. controls (8.9,1.4) and (8.0,3) .. (9.1,3);

   \draw[circle,fill=none,draw=none] (-1.3, -1.5) node {\Huge Legend:};
   \draw  [fill=yellow!30,draw=yellow,thick] (0.2,-2) rectangle (1.2,-1) node at (2.5,-1.5)   {\Huge Matrix};
   \draw  [fill=red!30,draw=red,thick] (4.4,-1.5) circle (0.3) node[text width=4cm] at (6.8,-1.5) {\Huge Elementwise \\ operation};
   \draw  [fill=red!30,draw=red,thick] (9.5,-2) rectangle (9.6,-0.8) node {};
   \draw  [fill=red!30,draw=red,thick] (9.6,-2) rectangle (9.7,-0.8) node {};
   \draw  [fill=red!30,draw=red,thick] (9.7,-2) rectangle (9.8,-0.8) node {};
   \draw  [fill=red!30,draw=red,thick] (9.8,-2) rectangle (9.9,-0.8) node[text width=4cm] at (12.1,-1.5) {\Huge Stack \\ vector};
   \draw[->,ultra thick,rotate=270] (1,12.6) parabola (1.5,13.1);
   \draw[->,ultra thick,rotate=270] (2,12.6) parabola (1.5,13.1) node[text width=4cm] at (1.5, 15.2) {\Huge Concatenate \\ vector};

\end{tikzpicture}}
\caption{Visual representation of \system{$h$} (left), and details of the operations associated with each stratum (right).}\label{fig:system}\vspace*{-2ex}
\end{figure}

In the previous section, both $\module_A$ and $\loss_A$ have been defined for a specific class $A$. However, it is possible to compute both $\module_A$ and ${\loss}_A$ for 
all classes in the same stratum in parallel, leveraging GPU architectures. In this section, we show how to compute 
the values first of the constraint module 
and then of the loss function on a GPU. %
At the end, we show how this computation can be simplified in the HMC case. %

Consider an {\cmc} problem $(${\problem}$,\Pi)$ with  stratified $\Pi$.
We assume to have $l$ classes (i.e., $|{\classes}| = l$),  that $\Pi_1, \Pi_2, \ldots, \Pi_s$ is the stratification computed by CompStrata($\Pi$), and that
$\Pi^*_1, \Pi^*_2, \ldots, \Pi^*_s$ are the corresponding sets as defined in the previous section.

\subsection{Constraint Module}\label{sec:cm_gpu}

The basic idea, starting from the vector $\module_0$ (which contains the $l$ values resulting from the bottom module $h$) is to iteratively compute the vector of values $\module_i$ corresponding to the outputs of $\module$ if given the set of constraints $\cup_{j=1}^i \Pi^*_j$. The final output of $\module$ will correspond to $\module_s$. 
Figure~\ref{fig:system} shows a visual representation of the process and of {\system{$h$}}.

For each $i \in\{ 1, \ldots, s\}$, $p_i$ is the number of constraints in $\Pi^*_i$ ($p_i = | \Pi^*_i |$), $r_{ij}$ denotes the $j$th constraint in $\Pi^*_i$, and
\begin{itemize}
    \item $B^+_i$ is the $p_i \times l$ matrix whose $j,k$ element is $1$ if $A_k \in body^+(r_{ij})$, and $0$ otherwise; 
    \item $B^-_i$ is the $p_i \times l$ matrix whose $j,k$ element is $1$ if $\neg A_k \in body^-(r_{ij})$, and $0$ otherwise; 
    \item $C_{i-1}$ is the $p_i \times l$ matrix obtained by stacking $p_i$ times $\module_{i-1}$.
\end{itemize}
{\sl Stacking} $p$ times a vector $v$ of size $q$ returns the 
$p \times q$ matrix $1_{p}^T \times v$ whose $j,k$ element is~$v[k]$.
    
Then, the $j$th value $v_{i}[j]$ of the vector $v_{i}$ associated with the body of the constraint $r_{ij}$ is
$$
\begin{array}{cc}
v^+_i =\min(B^+_i \odot C_{i-1} + (J_{p_i,l}-B^{+}_i), \text{dim}=1), \\
v^-_i =\min(B^-_i \odot (J_{p_i,l} - C_{i-1}) + (J_{p_i,l}-B^{-}_i), \text{dim}=1),
\\
v_{i}[j] = \min(v^+_i[j],v^-_i[j]), 
\end{array}
$$
where 
\begin{enumerate}
    \item $\odot$ represents the Hadamard product,
    \item $J_{p,q}$ is the $p \times q$ matrix of ones, and
    \item given an arbitrary $p \times q$ matrix $Q$, $\min(Q, \text{dim}\,{=}\,1)$ (resp., $\max(Q,$ $\text{dim}\,{=}\,1)$) returns a vector of length $p$ whose $i$th element is equal to $\min(Q_{i1}, \ldots, Q_{iq})$ (resp., $\max(Q_{i1}, \ldots, Q_{iq})$). 
\end{enumerate} 
Then, the output of the $i$th layer associated with the $i$th stratum is given by:
$$
\module_i = \max(I\!H_i \odot V_i, \text{dim}=1),
$$
where
\begin{itemize}
   \item $H_i$ is the $l\times p_i$ matrix whose $j,k$ element is $1$ if $A_j = head(r_{ik})$, and $0$ otherwise,
    \item $I\!H_i$ is the $l \times (l+p_i)$ matrix obtained by stacking the $l\times l$ identity matrix and $H_i$,
    \item $V_i$ is the $l \times (l+p_i)$ matrix obtained by stacking $l$ times the concatenation of $\module_{i-1}$ and $v_i$. 
\end{itemize}

\begin{example}\label{ex:cm}{\rm 
Let ${\classes}= \{A_1,A_2,A\}$, $\Pi = \{A_1 \to A; A_2 \to A; A, \neg A_1 \to A_2\}$, $h_{A_1} = 0.2$, $h_{A_2} = 0.3$, $h_A = 0.6$, as in Example \ref{ex:stratum7}. Then, ${\classes}_1 = \{A_1\}$, ${\classes}_2 = \{A, A_2\}$, $\Pi_1^* = \emptyset$, $\Pi_2^* = \Pi \cup \{A_1, \neg A_1 \to A_2\}$. 

Then,
$
\module_0 = 
\begin{bmatrix}
0.2 & 0.3 & 0.6
\end{bmatrix}
$,
$B^+_1$, $B^-_1$, $C_0$, $H_1$ are the empty matrices, $I\!H_1$ is the $3\times 3$ identity matrix, while
$$
\begin{array}{ccccc}
B^+_2= 
\begin{bmatrix}
1 & 0 & 0 \\
0 & 1 & 0 \\
0 & 0 & 1 \\
1 & 0 & 0 
\end{bmatrix}
&
B^-_2= 
\begin{bmatrix}
0 & 0 & 0 \\
0 & 0 & 0 \\
1 & 0 & 0 \\
1 & 0 & 0 
\end{bmatrix}
&
C_1 = 1_4^T \times 
\begin{bmatrix}
0.2 & 0.3 & 0.6 
\end{bmatrix}
=
\begin{bmatrix}
0.2 & 0.3 & 0.6 \\
0.2 & 0.3 & 0.6 \\
0.2 & 0.3 & 0.6 \\
0.2 & 0.3 & 0.6 
\end{bmatrix}
\end{array}
$$
$$
\begin{array}{cc}
H_2= 
\begin{bmatrix}
0 & 0 & 0 & 0 \\
0 & 0 & 1 & 1 \\
1 & 1 & 0 & 0  
\end{bmatrix}
&
I\!H_2= 
\begin{bmatrix}
1 & 0 & 0 & 0 & 0 & 0 & 0 \\
0 & 1 & 0 & 0 & 0 & 1 & 1 \\
0 & 0 & 1 & 1 & 1 & 0 & 0  
\end{bmatrix}
\end{array}
$$

$$
\begin{array}{cc}
V_1= 1_3^T \times
\begin{bmatrix}
0.2 & 0.3 & 0.6
\end{bmatrix}
&
\module_1= 
\begin{bmatrix}
0.2 & 0.3 & 0.6 
\end{bmatrix}
\end{array}
$$

$$
\begin{array}{ccc}
v_2^+ =
\begin{bmatrix}
0.2 & 0.3 & 0.6 & 0.2 
\end{bmatrix}
&
v_2^- =
\begin{bmatrix}
1 & 1 & 0.8 & 0.8 
\end{bmatrix}
&
v_2 = v_2^+
\end{array}
$$

$$
\begin{array}{cc}
V_2= 1_3^T \times
\begin{bmatrix}
0.2 & 0.3 & 0.6 & 0.2 & 0.3 & 0.6 & 0.2  
\end{bmatrix}
&
\module_2= 
\begin{bmatrix}
0.2 & 0.6 & 0.6
\end{bmatrix}
\end{array}
$$
and thus $\module_{A_1}=h_{A_1}=0.2$, $\module_{A_2}=h_{A} = 0.6$, and $\module_A = h_{A} = 0.6$, as expected (see Example \ref{ex:stratum7}).\hfill$\lhd$
}\end{example}

\subsection{Constraint Loss}\label{sec:closs_gpu}

We now show how to compute $\loss_A$ for all classes in parallel, leveraging GPU architectures. Here, we define $Y_i$ the $p_i \times l$ matrix obtained by stacking $p_i$ times the ground-truth vector $y$, and $v^+_i[j]$ to be the value associated with the  body of the $j$th constraint $r_{ij} \in \Pi_i^*$ when the ground-truth class $y_{head(r_{ij})}=1$:
$$
\begin{array}{cc}
v^{+\bck+}_i =\min(B^+_i \odot C_{i-1} \odot Y_i + (J_{p_i,l}-B^{+}_i), \text{dim}=1), \\
\begin{aligned}
v^{+\bck-}_i =\min(&B^-_i \odot (J_{p_i,l} - C_{i-1}) \odot (J_{p_i,l}-Y_i) + (J_{p_i,l}-B^{-}_i), \text{dim}=1),
\end{aligned} \\
v^+_{i}[j] = \min(v^{+\bck+}_i[j],v^{+\bck-}_i[j]).
\end{array}
$$
Analogously, $v^-_i[j]$ is the value  associated with the body of the  $j$th constraint $r_{ij} \in \Pi_i^*$ when $y_{head(r_{ij})}=0$:
$$
\begin{array}{cc}
\begin{aligned}
v^{-\bck+}_i =\min(& B^+_i \odot C_{i-1} \odot (J_{p_i,l}-Y_i) + (J_{p_i,l}-B^{+}_i)  + B_i^+ \odot Y_i, \text{dim}=1), \\
v^{-\bck-}_i =\min(& B^-_i \odot (J_{p_i,l} - C_{i-1}) \odot Y_i + (J_{p_i,l}-B^{-}_i) 
 + B_i^- \odot (J_{p_i,l}-Y_i) ,\text{dim}=1),
\end{aligned}
\\
v^-_{i}[j] = \min(v^{-\bck+}_i[j],v^{-\bck-}_i[j]).
\end{array}
$$
Then, the output of the $i$th layer associated with the $i$th stratum is given by:
\begin{align*}
\module^{+\bck-}_i = \max\big(&(I\!H_i \odot V_i^+)\odot IY_i +  (I\!H_i \odot V_i^-)\odot(J_{l,l+p_i}-IY_i), \text{dim}=1\big), 
\end{align*}
where:
\begin{itemize}
    \item $\module^{+\bck-}_0 = \module_0$,
    \item $IY_i$ is the $l \times (l+p_i)$ matrix obtained by stacking the $l\times l$ identity matrix and $Y_i^T$, the transposed of $Y_i$, 
    \item $V_i^+$ is the $l \times (l+p_i)$ matrix obtained by stacking $l$ times the concatenation of $\module^{+\bck-}_{i-1}$ and $v_i^+$, and 
    \item $V_i^-$ is the $l \times (l+p_i)$ matrix obtained by stacking $l$ times the concatenation of $\module^{+\bck-}_{i-1}$ and $v_i^-$.
\end{itemize}

Once we have computed $\module^{+\bck-}_s$, we can compute $\loss$ by using the standard binary cross-entropy loss (BCELoss), as given in any standard library (e.g., PyTorch):
$$
\loss = \text{BCELoss}(y, \module^{+\bck-}_s).
$$

\begin{example}[Ex.~\ref{ex:cm}, cont'd]{\rm 
Assume that $y_{A} = y_{A_1} = 1$, and $y_{A_2} = 0$.
Then, $y = 
\begin{bmatrix}
1 & 0 & 1 
\end{bmatrix}
$,
$Y_1$ is the empty matrix, 
$
V_1^+=V_1^-=V_1,
$
$IY_1$ is the identity matrix,
$$
\module^{+\bck-}_1 =
\begin{bmatrix}
0.2 & 0.3 & 0.6
\end{bmatrix}
$$

$$
B^+_2 \odot C_{1} \odot Y_2 + (J_{4,3}-B^{+}_2) =
\begin{bmatrix}
0.2 & 1 & 1 \\
1 & 0 & 1 \\
1 & 1 & 0.6 \\
0.2 & 1 & 1 
\end{bmatrix}
$$

$$
B^-_2 \odot (J_{4,3} - C_{1})  \odot (J_{4,3} - Y_2)  + (J_{4,3}-B^{-}_2) =
\begin{bmatrix}
1 & 1 & 1 \\
1 & 1 & 1 \\
0 & 1 & 1 \\
0 & 1 & 1 
\end{bmatrix}
$$

$$
v^+_2=
\begin{bmatrix}
0.2 & 0 & 0 & 0 
\end{bmatrix}
$$
The only value which is not 0 is the first one, corresponding to the first constraint, being the only constraint whose body is satisfied by the ground truth.

$$
B^+_2 \odot C_{1} \odot (J_{4,3}-Y_2) + (J_{4,3}-B^{+}_2) + B^+_2 \odot Y_2 =
\begin{bmatrix}
1 & 1 & 1 \\
1 & 0.3 & 1 \\
1 & 1 & 1 \\
1 & 1 & 1 
\end{bmatrix}
$$

$$
B^-_2 \odot (J_{4,3} - C_{1})  \odot Y_2  + (J_{4,3}-B^{-}_2) + B^-_2 \odot (J_{4,3}-Y_2)=
\begin{bmatrix}
1 & 1 & 1 \\
1 & 1 & 1 \\
0.8 & 1 & 1 \\
0.8 & 1 & 1 
\end{bmatrix}
$$

$$
v^-_2=
\begin{bmatrix}
1 & 0.3 & 0.8 & 0.8 
\end{bmatrix}
$$
The values which are not 1 correspond to the last four constraints, being the constraints whose body is not satisfied by the ground truth.

$$
Y_2= 1_4^T \times
\begin{bmatrix}
1 & 0 & 1 
\end{bmatrix}
$$

$$
V_2^+= 1_3^T \times
\begin{bmatrix}
0.2 & 0.3 & 0.6 & 0.2 & 0 & 0 & 0 
\end{bmatrix}
$$

$$
V_2^-= 1_3^T \times
\begin{bmatrix}
0.2 & 0.3 & 0.6 & 1 & 0.3 & 0.8 & 0.8 
\end{bmatrix}
$$

$$
\begin{array}{ccc}
\begin{array}{c}
(I\!H_2 \odot V_2^+)\odot IY_2 \\ + \\
(I\!H_2 \odot V_2^-)\odot(J_{3,7}-IY_2)
\end{array}
&
= 
&
\begin{bmatrix}
0.2 & 0 & 0 & 0 & 0 & 0 & 0 \\
0 & 0.3 & 0 & 0 & 0 & 0.8 & 0.8 \\
0 & 0 & 0.6 & 0.2 & 0 & 0 & 0  
\end{bmatrix}

\end{array}
$$

$$
\module^{+\bck-}_2 = 
\begin{bmatrix}
0.2 & 0.8 & 0.6
\end{bmatrix}
$$
The three values of $\module^{+\bck-}_2$ correspond to
the expected values of $\module_{A_1}^+, \module_{A_2}^-, \module_{A}^+$ as computed in  Example  \ref{ex:cmh}, and equal to
 $h_{A_1}, 1-h_{A_1}, h_{A}$, respectively.\hfill$\lhd$
}\end{example}

\subsection{Hierarchical Multi-Label Classification}\label{sec:hmc_gpu}

When dealing with HMC problems, the above implementation can be simplified. Indeed, in the hierarchical case, we have that all constraints have only one class in the body,
all constraints are definite and thus $s=1$.

Let $H$ be an $l \times l$ matrix obtained by stacking $l$ times {$\module_0$}. Let $M$ be an $l \times l$ matrix such that, for $i,j \in\{ 1, \ldots, n\}$, $M_{ij} = 1$ if $A_j$ is a subclass of $A_i$, and $M_{ij} = 0$, otherwise.
The constraint module can be simply computed as: 
$$
\module_s = \max(H \odot M, \text{dim}=1)\,,
$$

For \loss{}, we can use the same mask $M$ to modify the standard BCELoss. In detail, let $y$ be the ground-truth vector, and $H'$ be the $l \times l$ matrix obtained by stacking $l$ times the vector $\module_s \odot y$. Then, 
$$
\loss = \text{BCELoss}\big(((1-y) \odot \module_s)+(y \odot \max(M \odot H', \text{dim}=1)),y\big).
$$

\section{Experimental Analysis}\label{sec:exp}

In this section, we present the results of the experimental analysis. We first present our results focusing on HMC problems, and then we present the results obtained in the more general framework of {\cmc} problems. 
In the presentation, we always speak about \system{$h$} since,
given
Theorem \ref{thm:eq}, 
in the case of HMC problems there is no difference between \system{$h$} and \hmcsys{$h$}. 

\begin{figure}[t]\centering
    \includegraphics[width=0.6\textwidth]{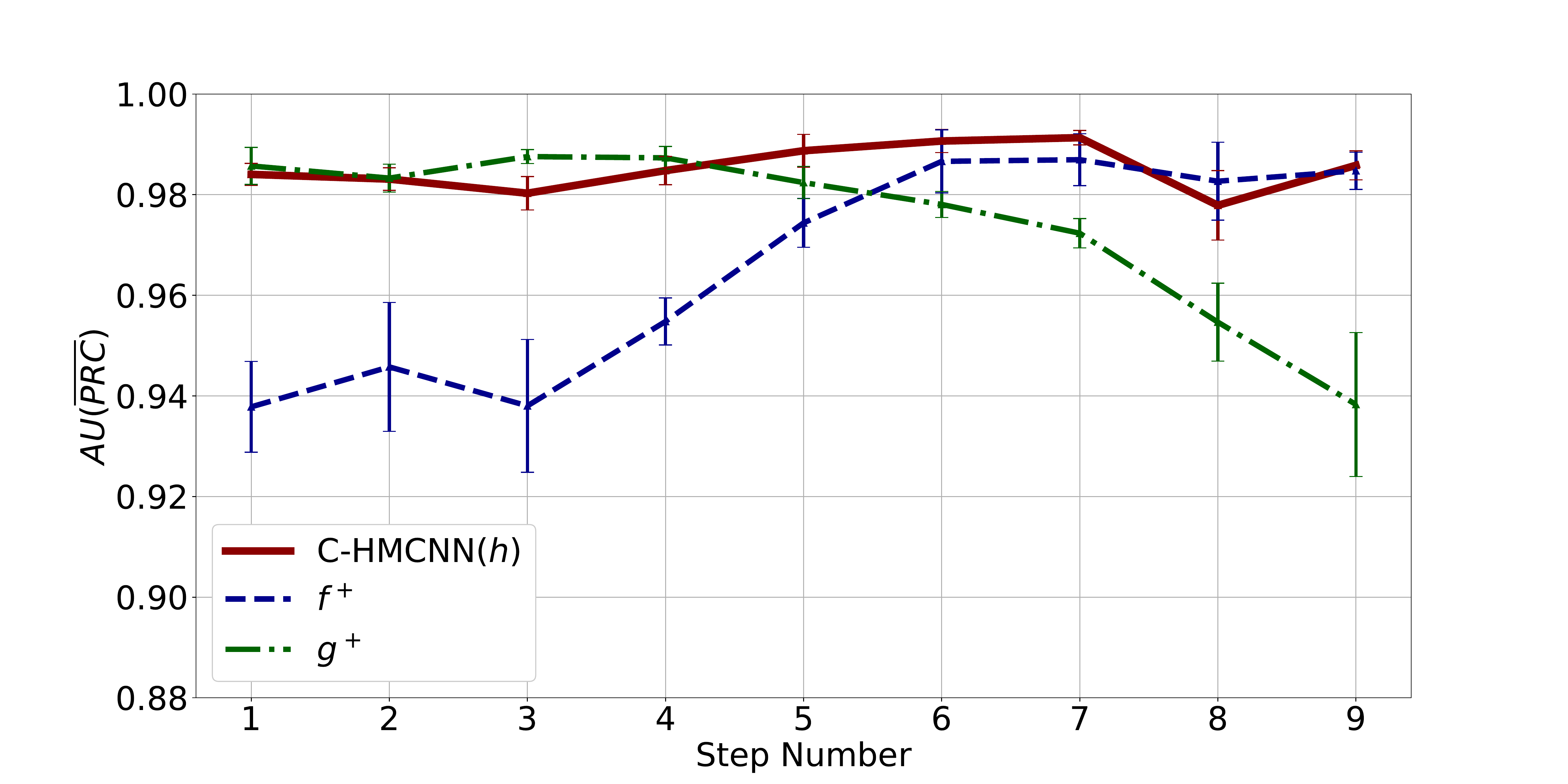}
    \caption{Mean {\auprc} with standard deviation of \system{$h$}, $f^+$, and $g^+$ for each step.\label{fig:moving_rect}}
\end{figure}

\subsection{Hierarchical Multi-Label Classification}

In this section, we present the experimental results of \system{$h$}, first considering two synthetic experiments, and then on 20 real-world datasets for which we compare with current state-of-the-art models for HMC problems. Finally, ablation studies highlight the positive impact of both \module{} and \loss{} on \system{$h$}'s performance.\footnote{Link:  {https://github.com/EGiunchiglia/C-HMCNN/}}

For evaluating performance, we consider the area under the average precision and recall curve \auprc. 
\auprc{}  has the advantage of being independent from the threshold used to predict when a data point belongs to a particular class (which is often heavily application-dependent) and is the most used in the HMC literature \cite{kwok2011,vens2008,cerri2018}.

\subsubsection{Synthetic experiment 1}
Consider the generalization of the experiment  in Section~\ref{sec:hmc_basic} in which we started with $R_1$ outside $R_2$ (as in the second row of Figure~\ref{fig:dec_bound_figs}), and then moved $R_1$ towards the centre of $R_2$ (as in the first row of  Figure~\ref{fig:dec_bound_figs}) in 9 uniform steps.
The last row of  Figure~\ref{fig:dec_bound_figs} corresponds to the
fifth step, i.e., $R_1$ was~halfway.
This experiment is meant to show how the performance of \system{$h$},  $f^+$, and $g^+$ defined as in Section~\ref{sec:hmc_basic}
vary depending on the relative positions of $R_1$ and $R_2$. Here, 
$f$, $g$, and $h$ were implemented and trained as 
in Section~\ref{sec:hmc_basic}. For each step, we run the experiment 10 times,\footnote{All subfigures in Figure~\ref{fig:dec_bound_figs} correspond to the decision boundaries of $f$, $g$, and $h$ in the first of the 10 runs. } and we plot the mean \auprc{} together with the standard deviation for \system{$h$}, $f^+$, and $g^+$ in Figure~\ref{fig:moving_rect}.

As expected, Figure~\ref{fig:moving_rect} shows that $f^+$ performed poorly in the first three steps when $R_1 \cap R_2 = \emptyset$, it then started to perform better at step 4 when $R_1 \cap R_2 \not \in \{R_1, \emptyset\}$, and it performed well from step 6 when $R_1$ overlaps significantly with $R_2$ (at least 65\% of its area). Conversely, $g^+$ performed well on the first five steps, and its performance started decaying from step 6. \system{$h$} performed well at all steps, as expected, showing robustness with respect to the relative positions of $R_1$ and $R_2$. 
Further, \system{$h$} exhibits much more stable performances than $f^+$ and $g^+$ as highlighted by the visibly much smaller standard deviations of \system{$h$}.

\begin{figure*}[t]
\centering
\begin{tabular}{c|}
\begin{minipage}{.2\textwidth}
    \resizebox{2.7cm}{!}{\definecolor{orange}{rgb}{1.0, 0.6, 0.2}
\definecolor{darkblue}{rgb}{0.0, 0.0, 1.0}
\definecolor{darkgreen}{rgb}{0.0, 0.42, 0.24}
\definecolor{darkred}{rgb}{0.8, 0.0, 0.0}

   \begin{tikzpicture}[scale=0.195]
        \draw [fill=red,opacity=0.1,very thick] (0,0) rectangle (10,10);
        \draw [fill=darkblue,opacity=0.4] (1,1) rectangle (4,4)
        node[pos=.5,text=black, opacity=1] { \scriptsize $7$};
        \draw [fill=blue,opacity=0.4] (1,6) rectangle (4,9)
        node[pos=.5,text=black, opacity=1] { \scriptsize $1$};
        \draw [fill=blue,opacity=0.4] (6,1) rectangle (9,4)
        node[pos=.5,text=black, opacity=1] { \scriptsize $9$};
        \draw [fill=blue,opacity=0.4] (6,6) rectangle (9,9)
        node[pos=.5,text=black, opacity=1] { \scriptsize $3$};
        \draw [fill=blue,opacity=0.4] (4.5,4.5) rectangle (5.5,5.5)
        node[pos=.5,text=black, opacity=1] { \scriptsize $5$};
        \draw [fill=blue,opacity=0.4] (1,4.5) rectangle (4,5.5)
        node[pos=.5,text=black, opacity=1] { \scriptsize $4$};
        \draw [fill=blue,opacity=0.4] (4.5,1) rectangle (5.5,4)
        node[pos=.5,text=black, opacity=1] { \scriptsize $8$};
        \draw [fill=blue,opacity=0.4] (4.5,6) rectangle (5.5,9)
        node[pos=.5,text=black, opacity=1] { \scriptsize $2$};
        \draw [fill=blue,opacity=0.4] (6,4.5) rectangle (9,5.5)
        node[pos=.5,text=black, opacity=1] { \scriptsize $6$};
\end{tikzpicture}}
\end{minipage}  \\
\\
\begin{minipage}{.2\textwidth}
    \resizebox{3.1cm}{!}{\usetikzlibrary{arrows.meta}
\begin{tikzpicture}
\begin{scope}[every node/.style={circle,thick,draw=none,inner sep=0, outer sep=0}]
    \node (A5) at (3,2) {\LARGE $A_5$};
     \node (A1) at (0,0) {\LARGE $A_1$};
     \node (A2) at (1,0) {\LARGE $A_2$};
     \node (A4) at (2,0) {\LARGE $A_4$};
     \node (A6) at (3,0) {\LARGE $A_6$};
     \node (A7) at (4,0) {\LARGE $A_7$};
    \node (A8) at (5,0) {\LARGE $A_8$};
    \node (A9) at (6,0) {\LARGE $A_9$};
    \node (A3) at (3,-2) {\LARGE $A_3$};
\end{scope}

\begin{scope}[>={Stealth[black]},
              every node/.style={draw=none,fill=white,circle,inner sep=0, outer sep=0},
              every edge/.style={draw=black,very thick}]
    \path [->] (A1) edge (A5);
    \path [->] (A2) edge (A5);
    \path [->] (A4) edge (A5);
    \path [->] (A6) edge (A5);
    \path [->] (A7) edge (A5);
    \path [->] (A8) edge (A5);
    \path [->] (A9) edge (A5);
    
    \path [->] (A3) edge (A1);
     \path [->] (A3) edge (A2);
      \path [->] (A3) edge (A4);
       \path [->] (A3) edge (A6);
        \path [->] (A3) edge (A7);
         \path [->] (A3) edge (A8);
          \path [->] (A3) edge (A9);

\end{scope}
\end{tikzpicture}}
\end{minipage}  \vspace{-4mm}\\
\end{tabular}
\resizebox{0.7\textwidth}{!}{%
\begin{tabular} {c c c | c c c}
\multicolumn{3}{c |}{\system{$h$}+$\lss$} & 
\multicolumn{3}{c}{\system{$h$}} \\
\begin{minipage}{.135\textwidth}
    \includegraphics[width=\textwidth,trim={1.1cm 0 3.7cm 1cm},clip]{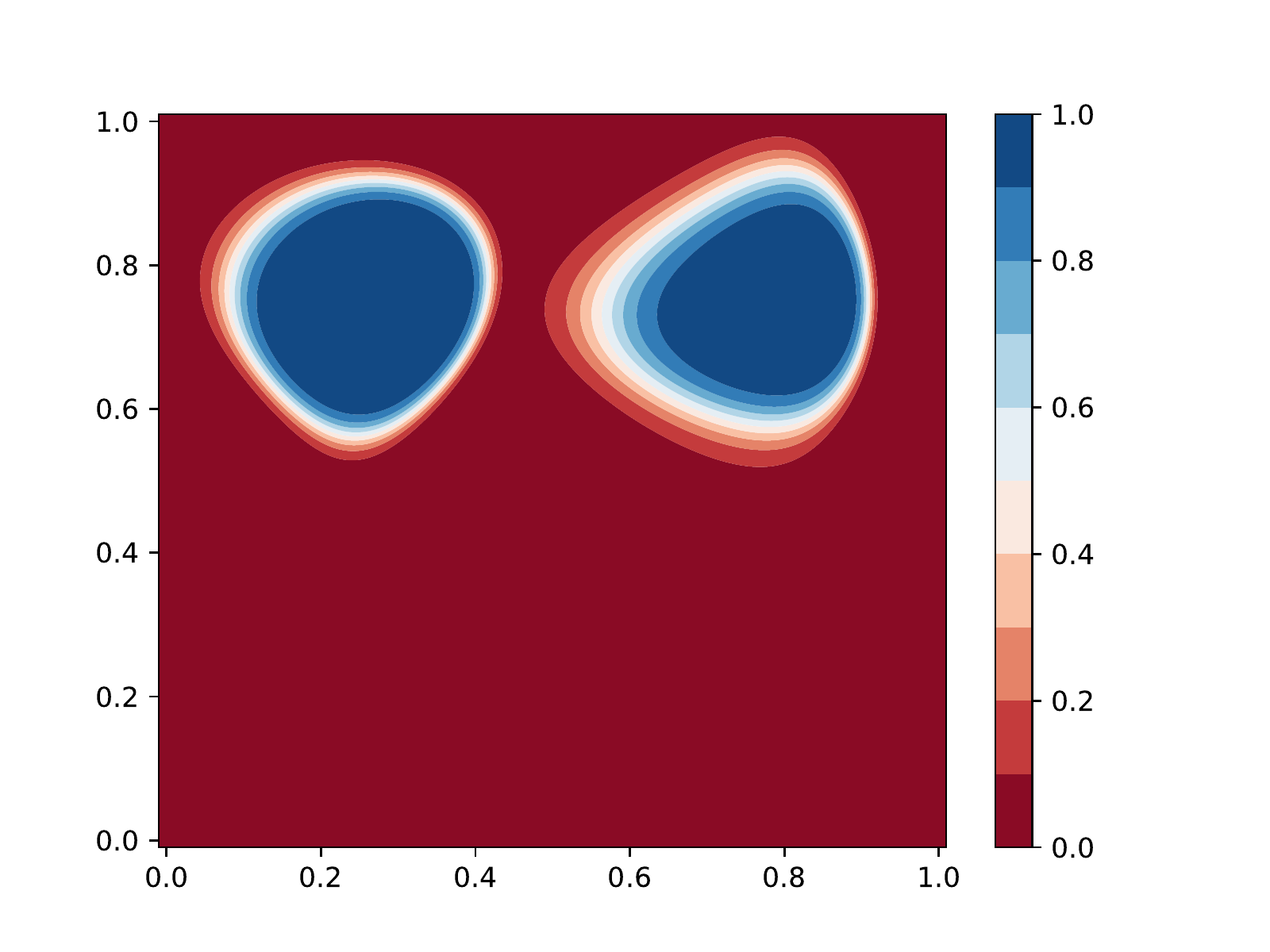}
\end{minipage} &
\begin{minipage}{.135\textwidth}
    \includegraphics[width=\textwidth,trim={1.1cm 0 3.7cm 1cm},clip]{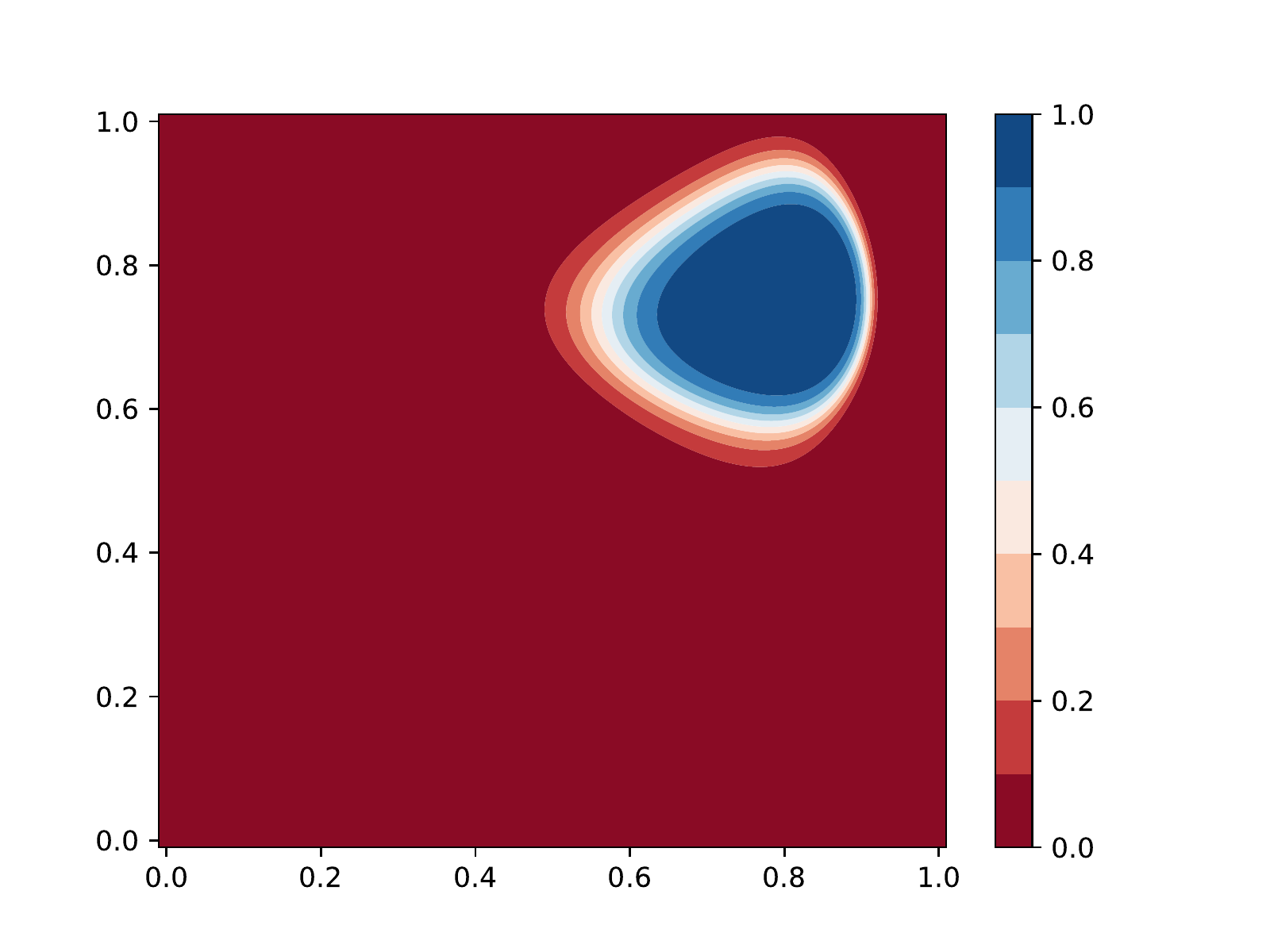}
\end{minipage} &
\begin{minipage}{.135\textwidth}
    \includegraphics[width=\textwidth,trim={1.1cm 0 3.7cm 1cm},clip]{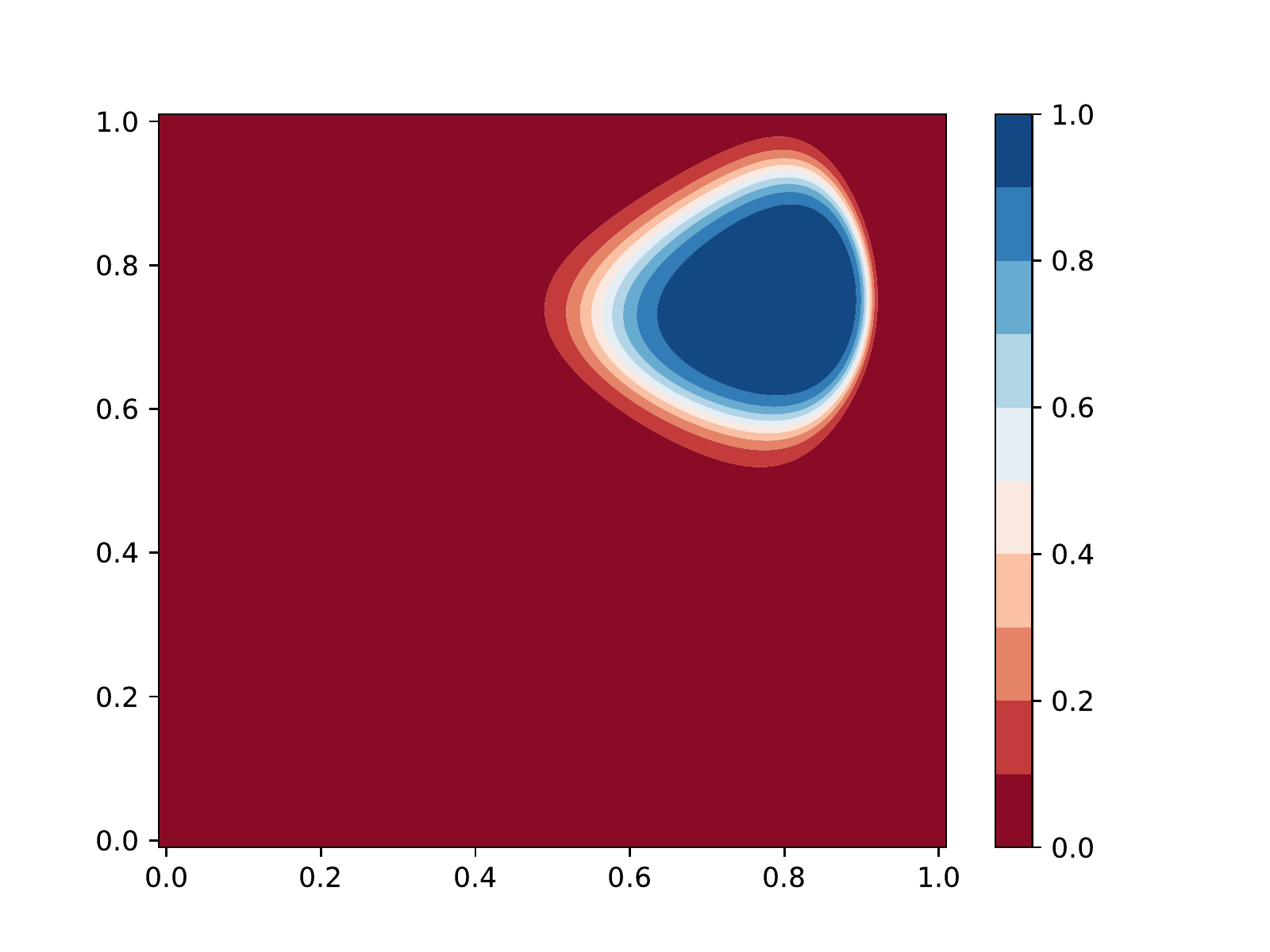}
\end{minipage} &
\begin{minipage}{.135\textwidth}
    \includegraphics[width=\textwidth,trim={1.1cm 0 3.7cm 1cm},clip]{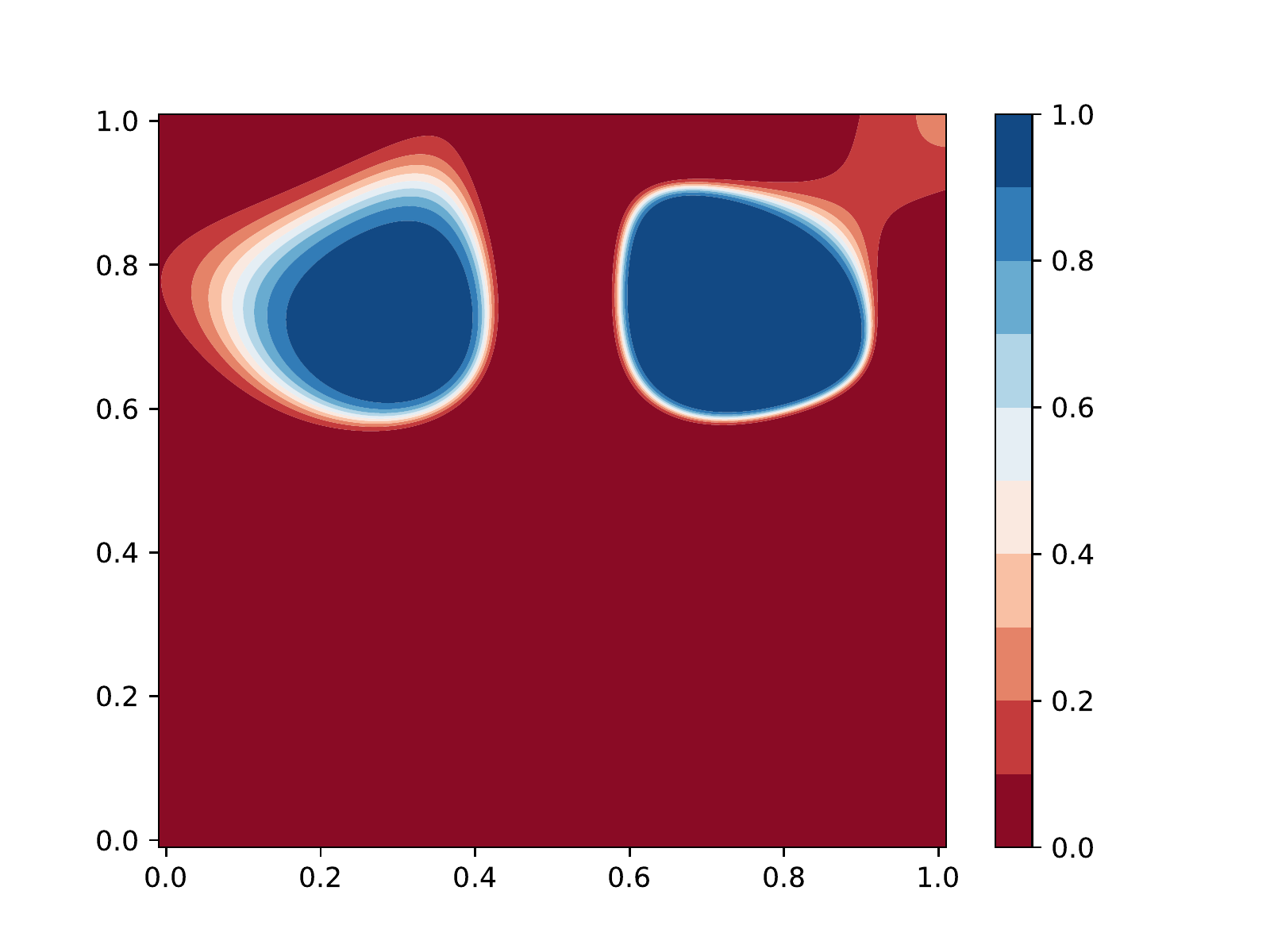}
\end{minipage} &
\begin{minipage}{.135\textwidth}
    \includegraphics[width=\linewidth,trim={1.1cm 0 3.7cm 1cm},clip]{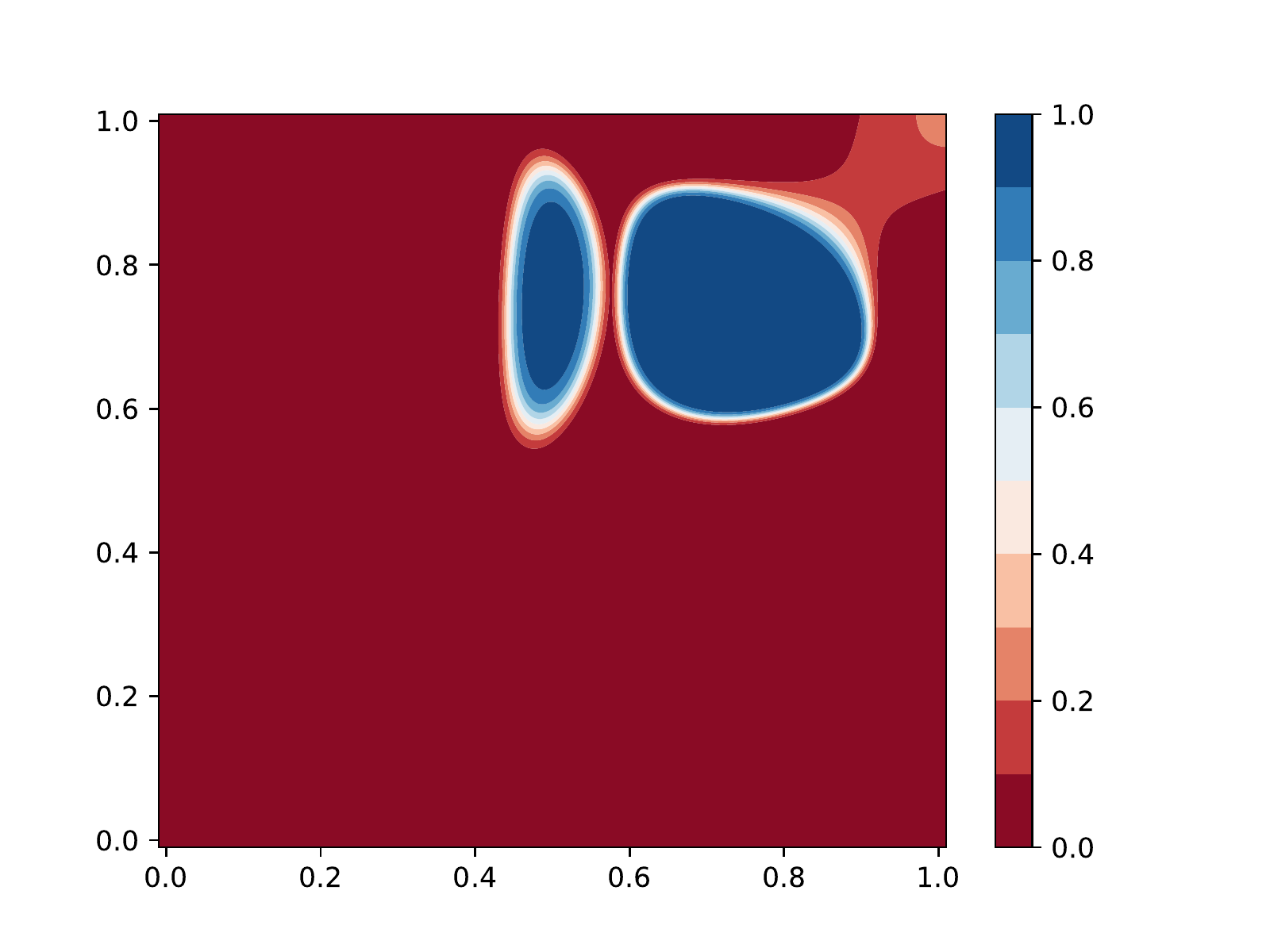}
\end{minipage} &
\begin{minipage}{.17\textwidth}
    \includegraphics[width=\linewidth,trim={1.1cm 0 1cm 1cm},clip]{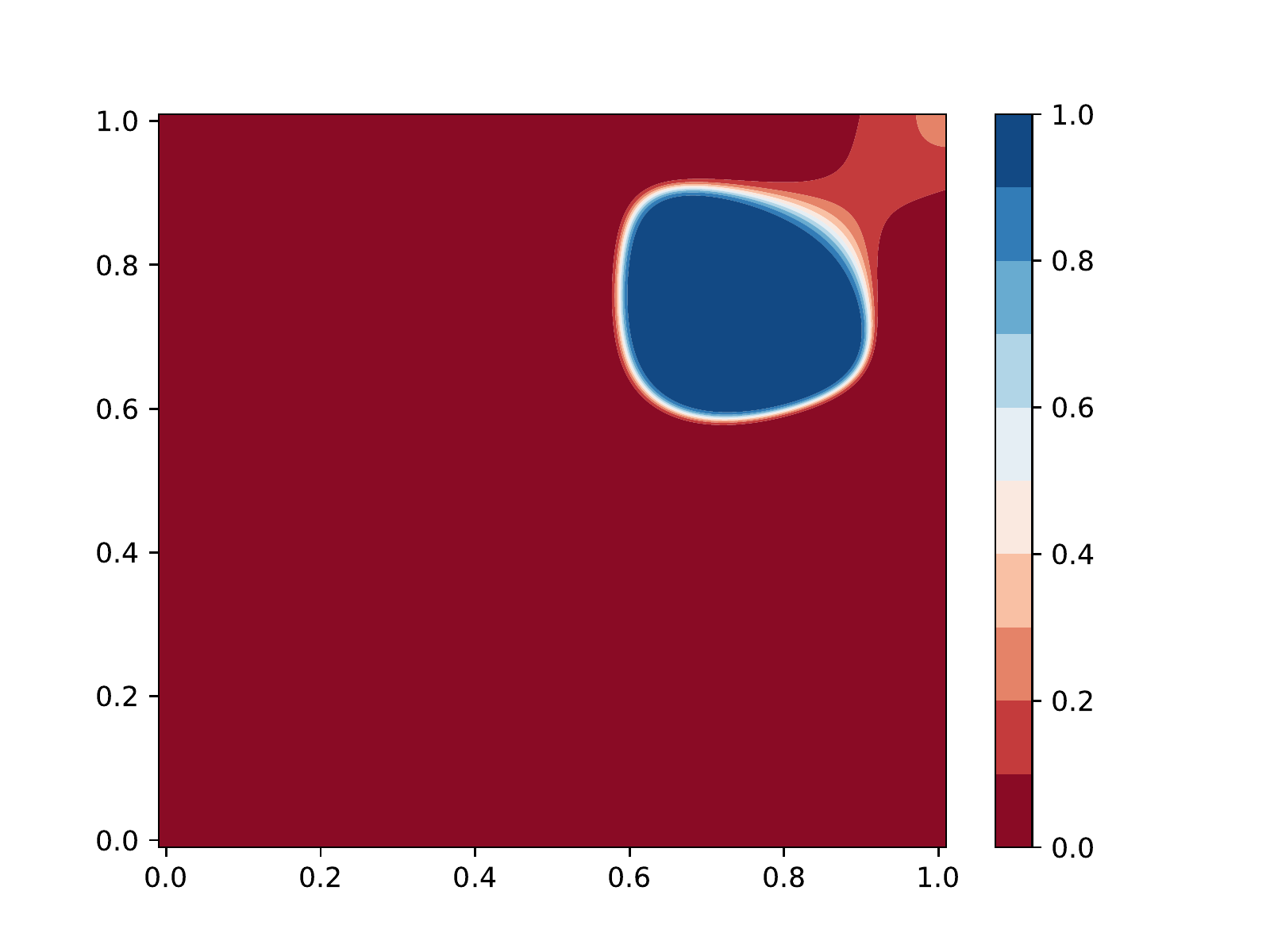}
\end{minipage} \\
\begin{minipage}{.135\textwidth}
    \includegraphics[width=\textwidth,trim={1.1cm 0 3.7cm 1cm},clip]{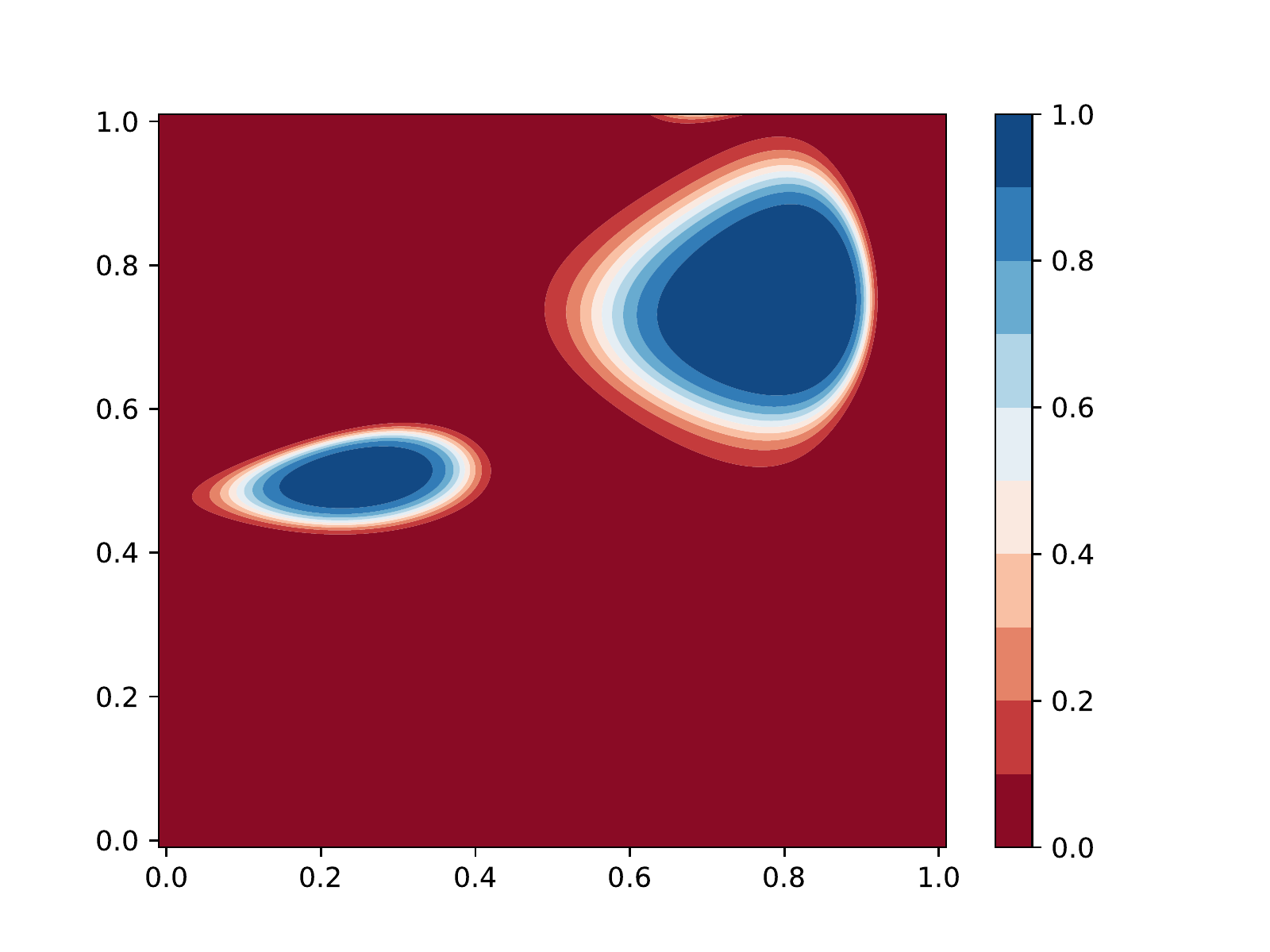}
\end{minipage} &
\begin{minipage}{.135\textwidth}
    \includegraphics[width=\textwidth,trim={1.1cm 0 3.7cm 1cm},clip]{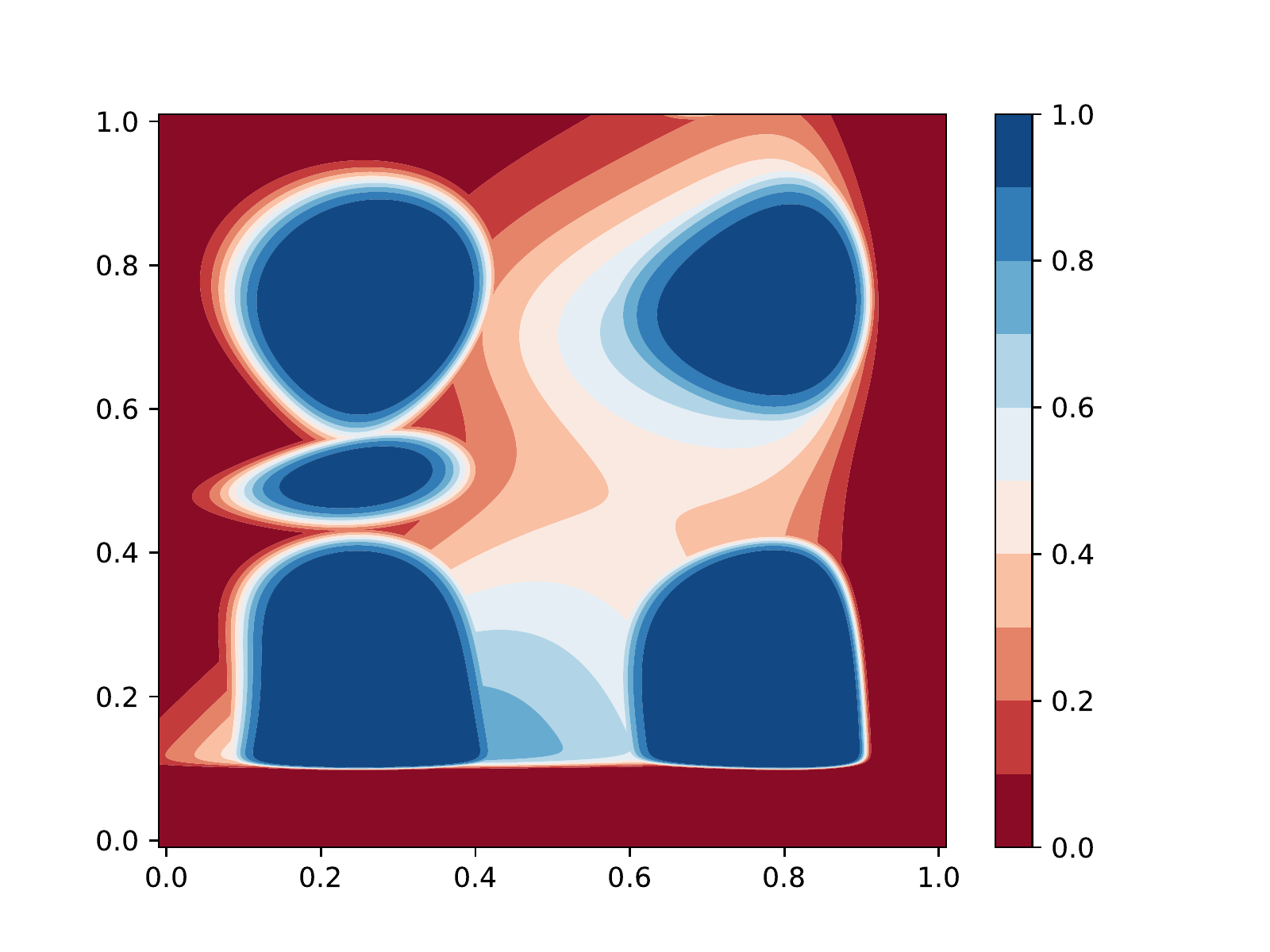} 
\end{minipage} &
\begin{minipage}{.135\textwidth}
    \includegraphics[width=\textwidth,trim={1.1cm 0 3.7cm 1cm},clip]{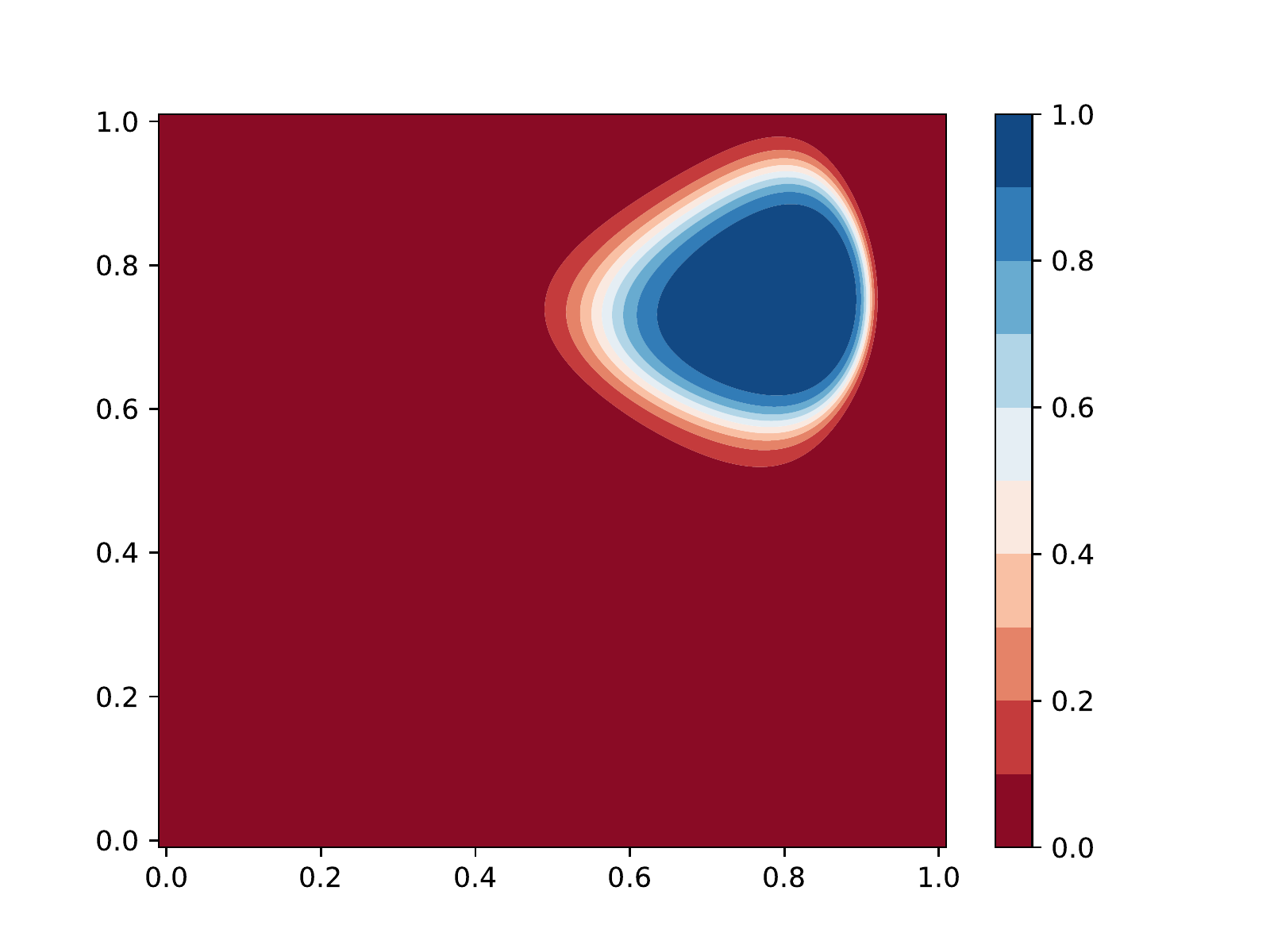} 
\end{minipage} &
\begin{minipage}{.135\textwidth}
    \includegraphics[width=\textwidth,trim={1.1cm 0 3.7cm 1cm},clip]{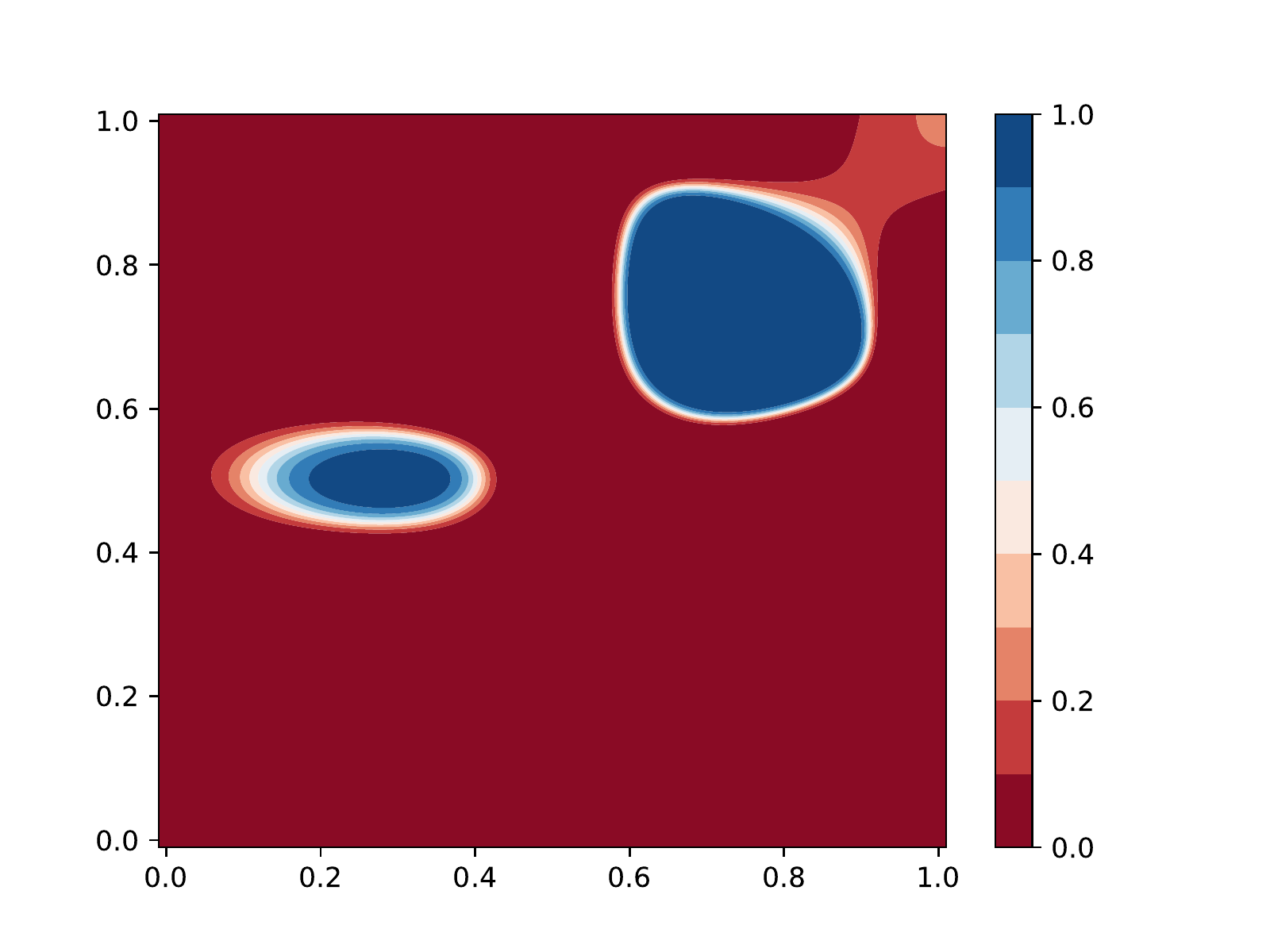} 
    \end{minipage} &
\begin{minipage}{.135\textwidth}
    \includegraphics[width=\linewidth,trim={1.1cm 0 3.7cm 1cm},clip]{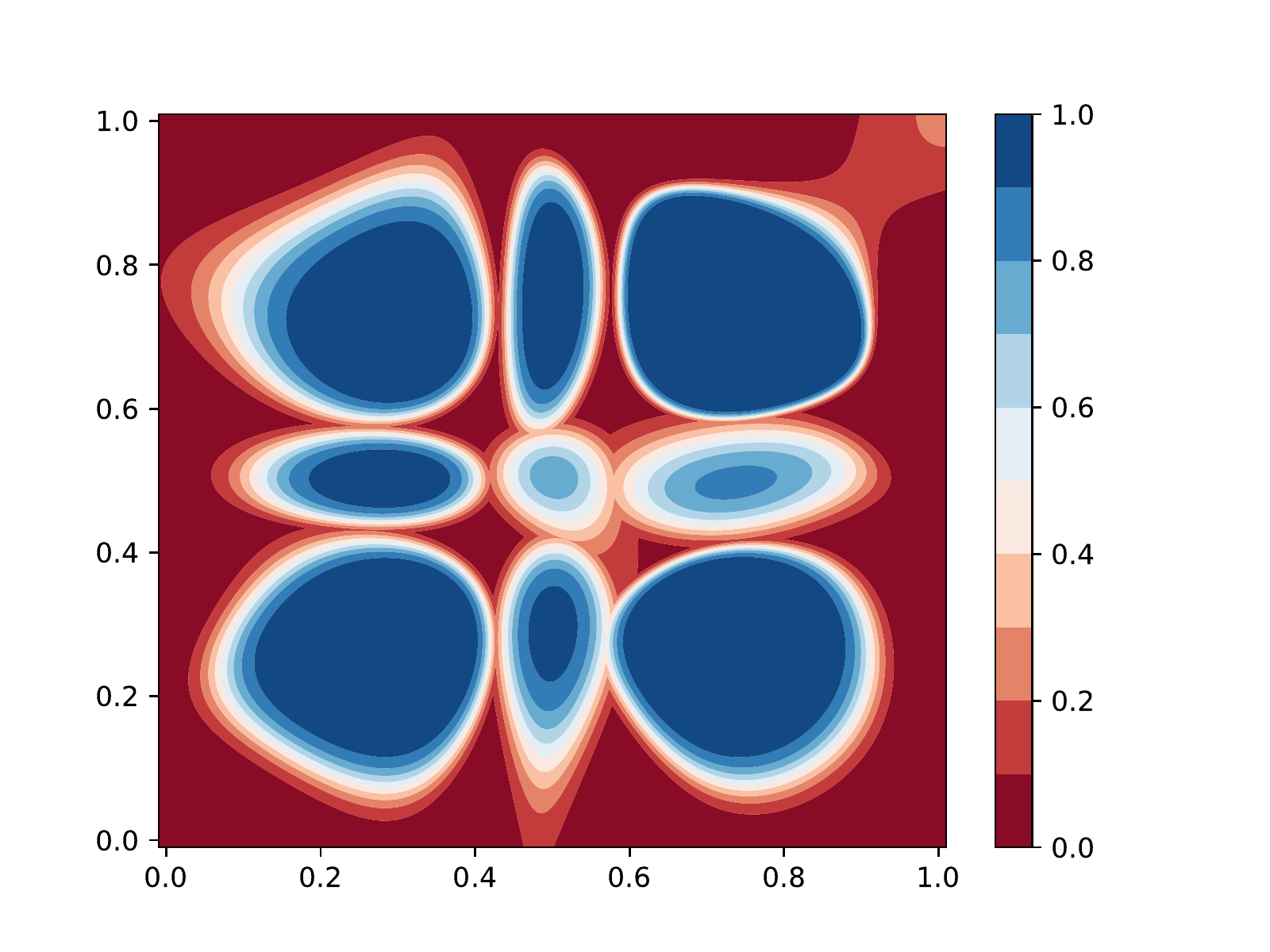}
\end{minipage} &
\begin{minipage}{.17\textwidth} 
    \includegraphics[width=\linewidth,trim={1.1cm 0 1cm 1cm},clip]{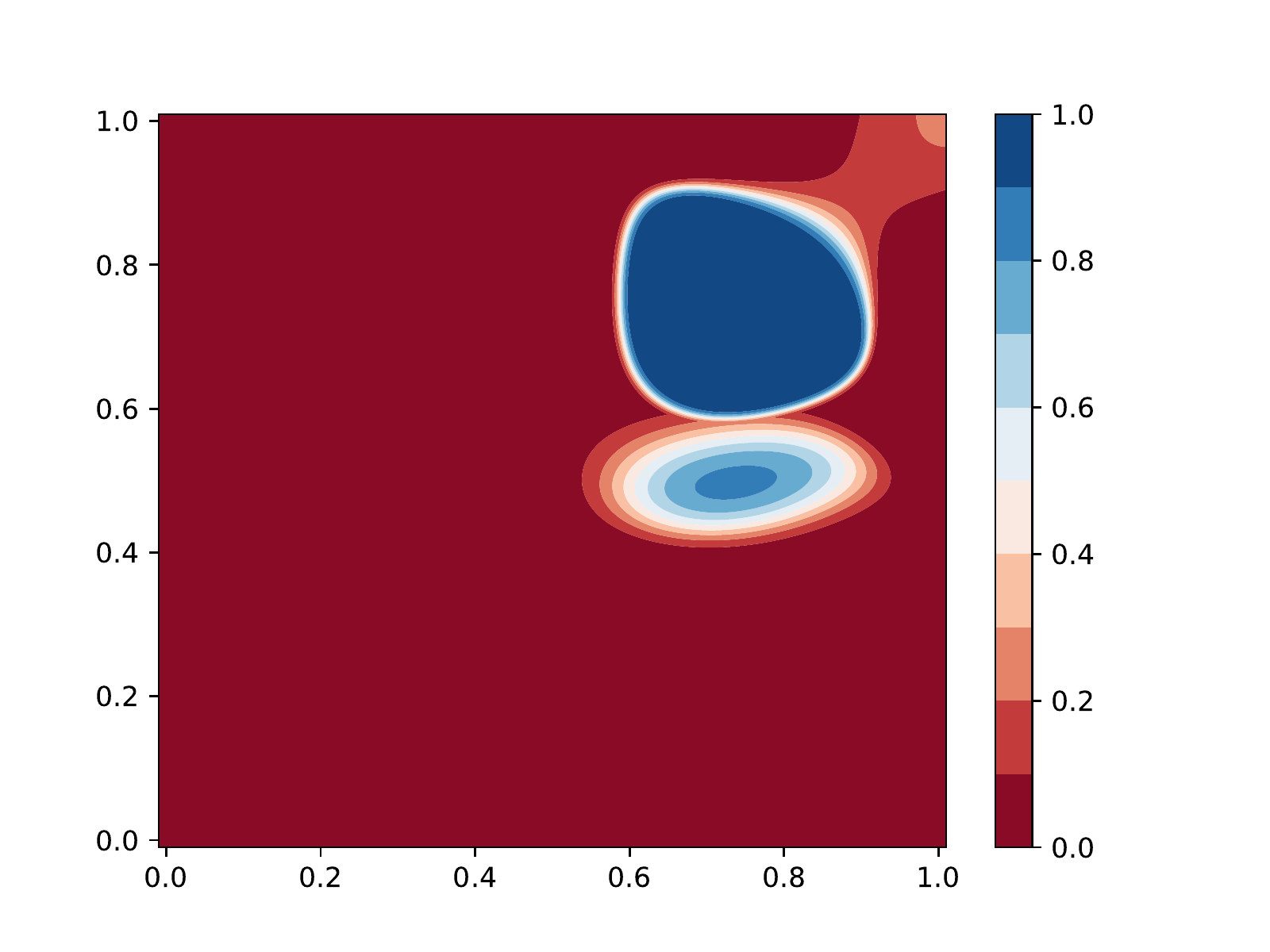}
\end{minipage} \\
\begin{minipage}{.135\textwidth}
    \includegraphics[width=\textwidth,trim={1.1cm 0 3.7cm 1cm},clip]{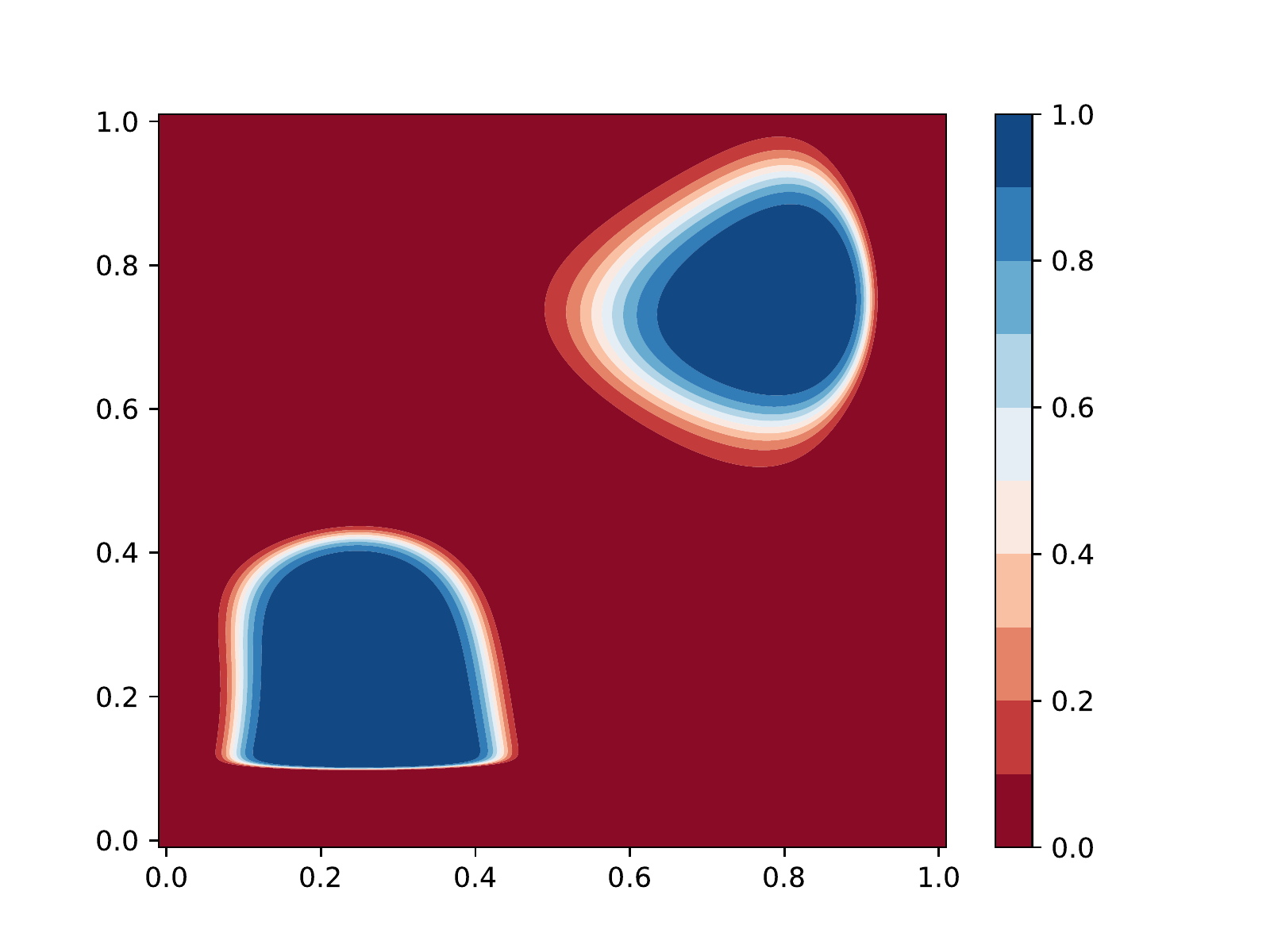}
\end{minipage} &
\begin{minipage}{.135\textwidth}
    \includegraphics[width=\textwidth,trim={1.1cm 0 3.7cm 1cm},clip]{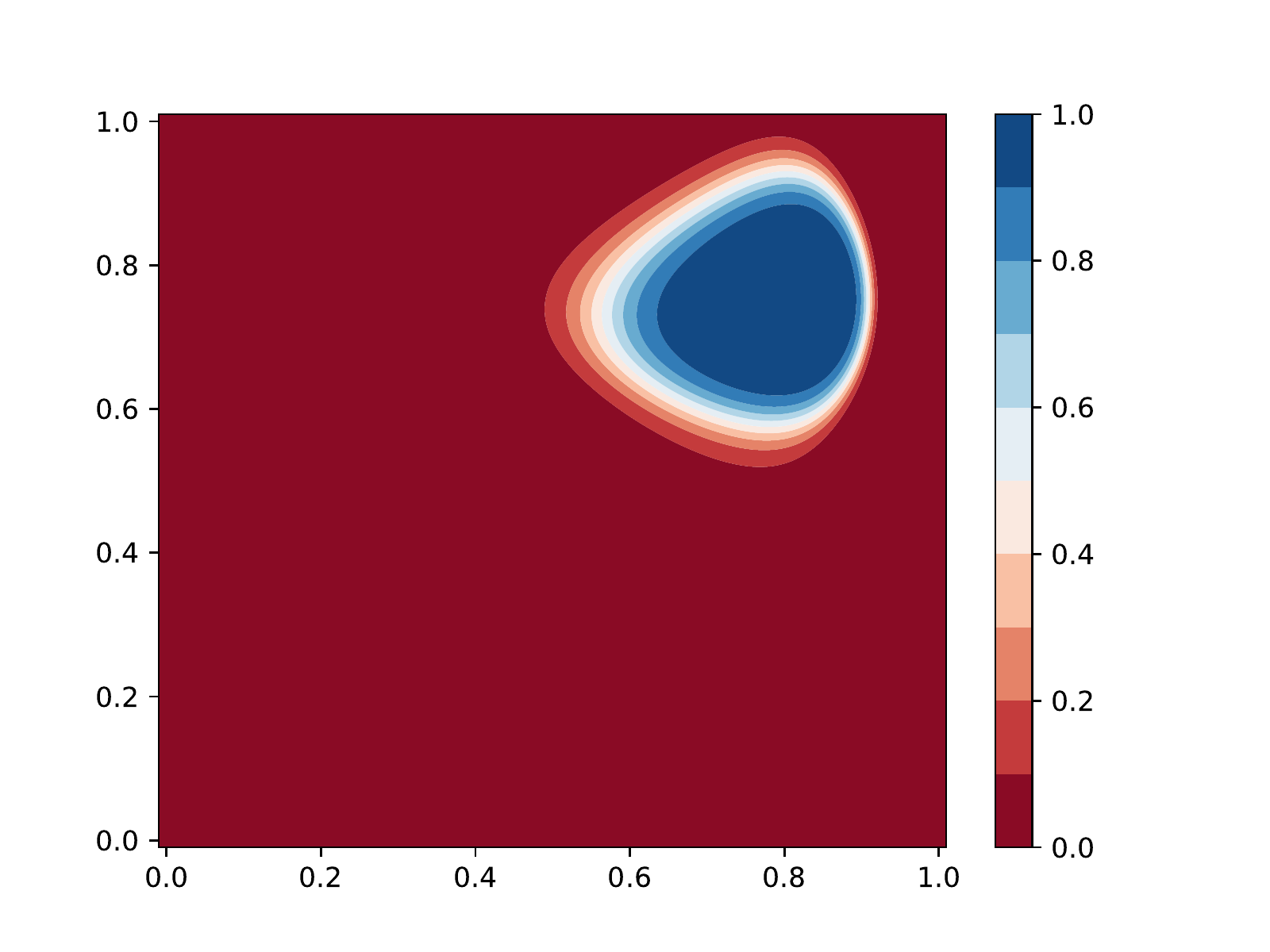}
\end{minipage} &
\begin{minipage}{.135\textwidth}
    \includegraphics[width=\textwidth,trim={1.1cm 0 3.7cm 1cm},clip]{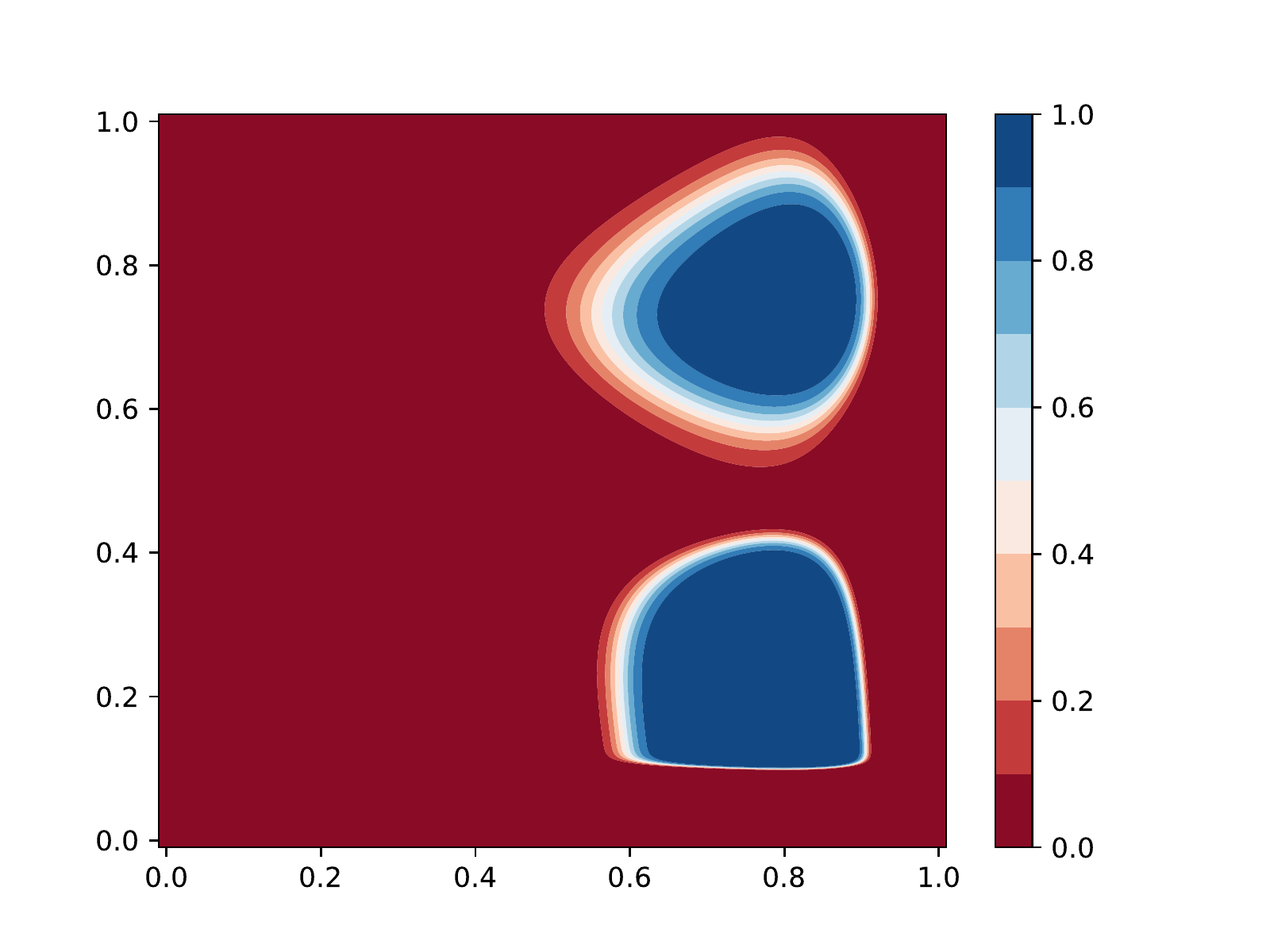}
\end{minipage} &
\begin{minipage}{.135\textwidth}
    \includegraphics[width=\textwidth,trim={1.1cm 0 3.7cm 1cm},clip]{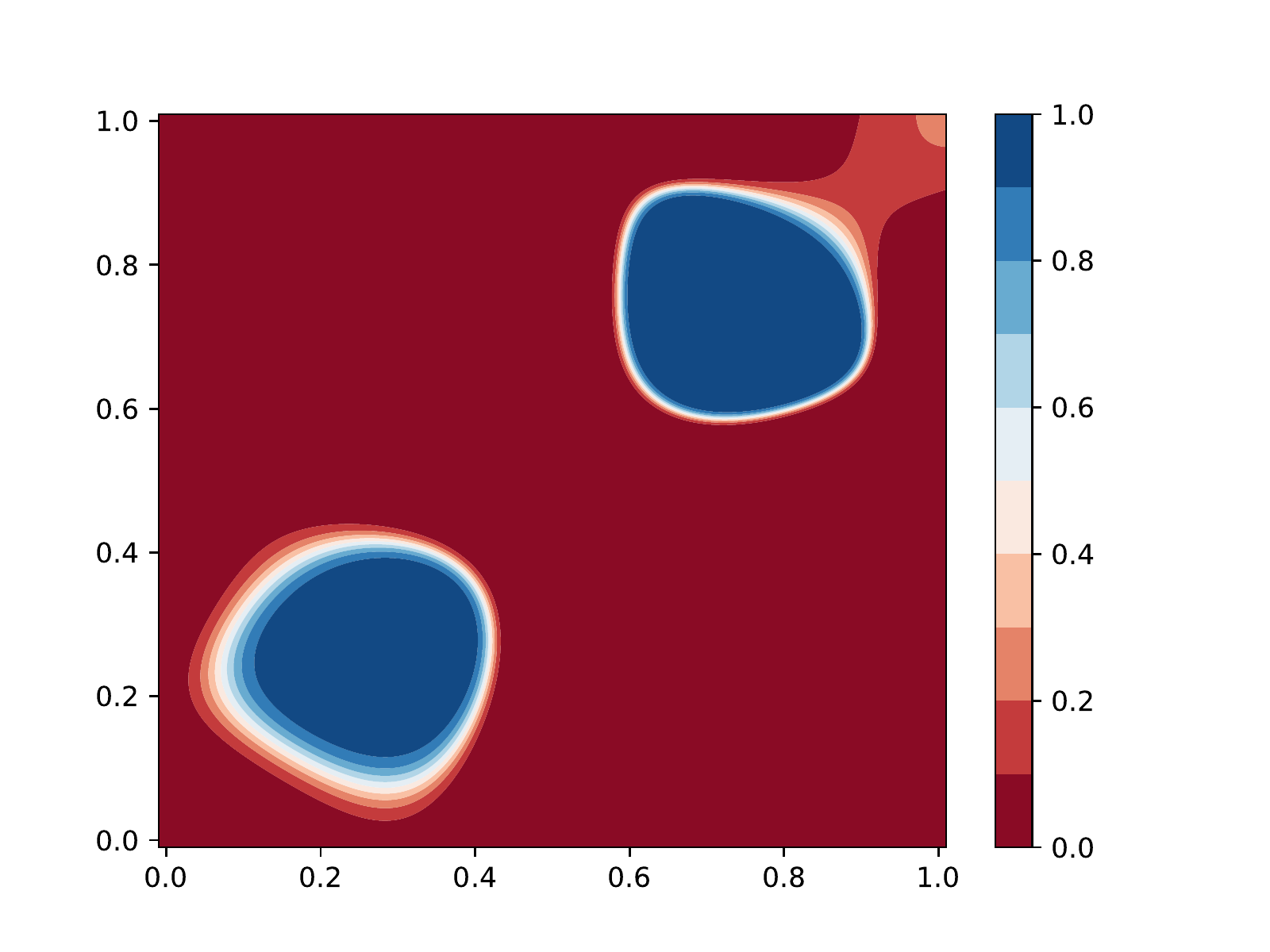}
\end{minipage} &
\begin{minipage}{.135\textwidth}
    \includegraphics[width=\linewidth,trim={1.1cm 0 3.7cm 1cm},clip]{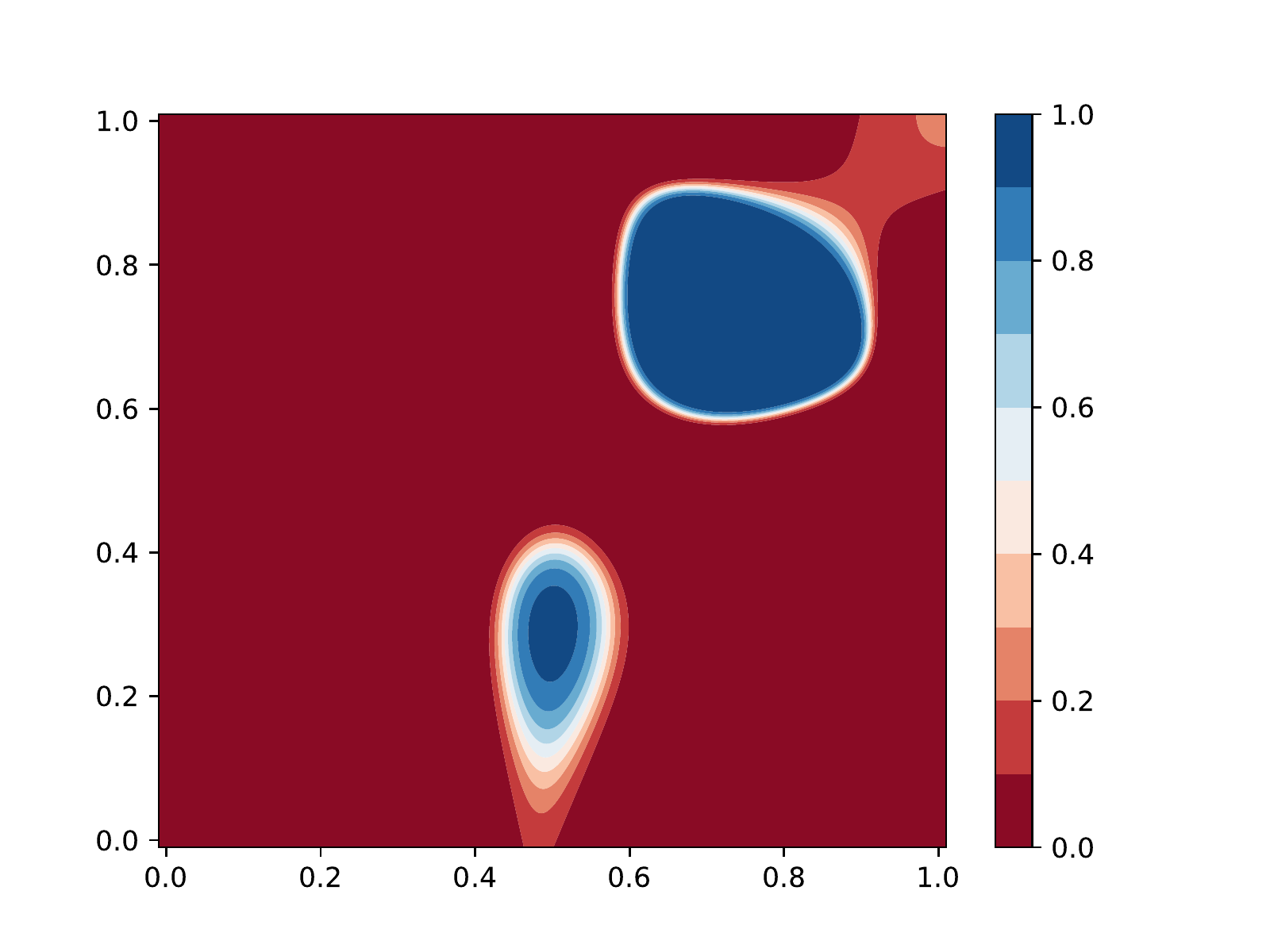}
\end{minipage} &
\begin{minipage}{.17\textwidth}
    \includegraphics[width=\linewidth,trim={1.1cm 0 1cm 1cm},clip]{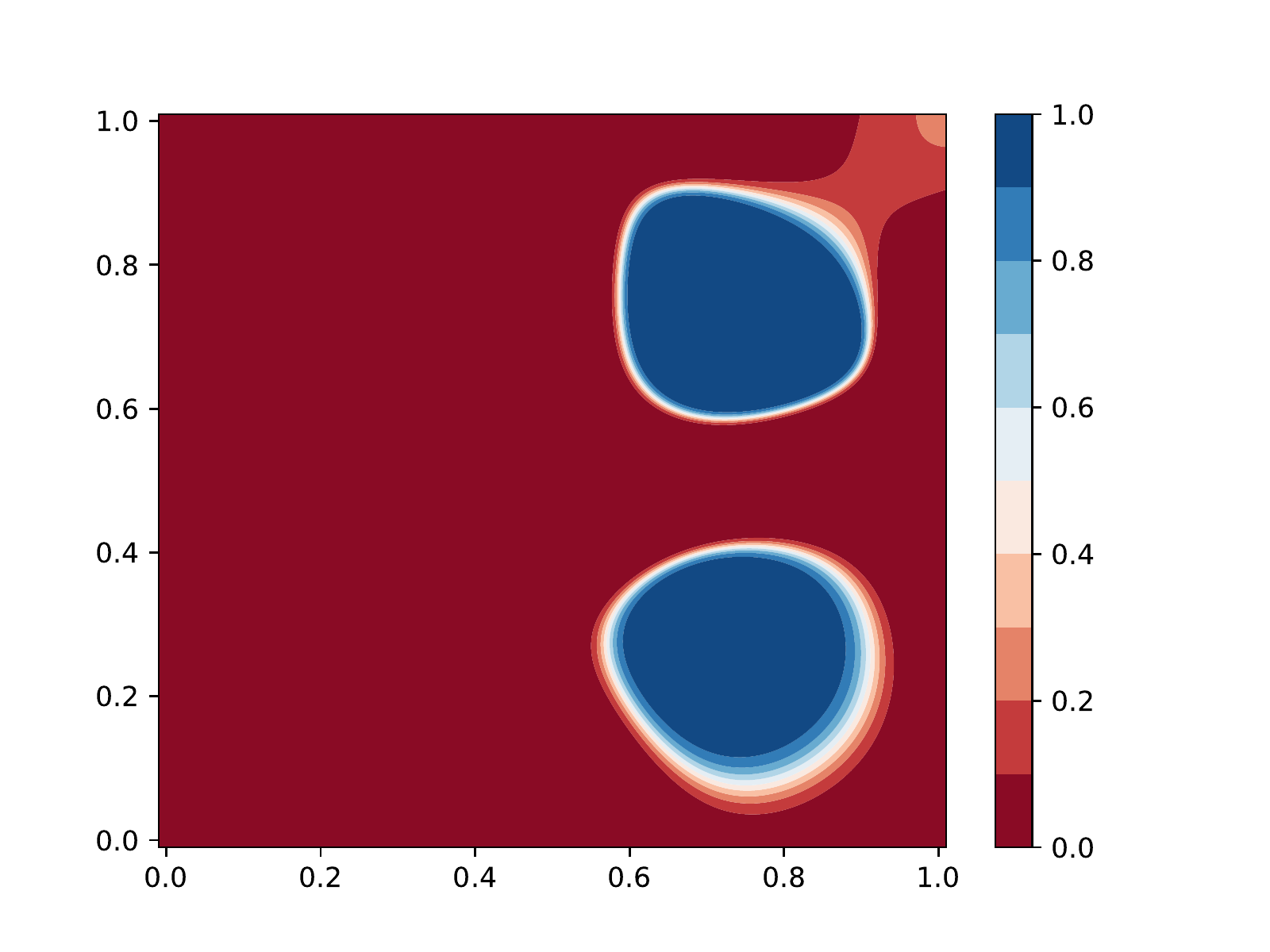}
\end{minipage} \vspace*{-1ex}
\end{tabular}
}
\caption{
 First column: rectangles disposition above and hierarchy representation below. Second column: decision boundaries of \system{$h$}+$\lss$ for each class. Last  column: decision boundaries of \system{$h$}  for each class. In each figure showing decision boundaries, the darker the blue (resp., red), the more confident a model is that the data points in the region belong (resp., do not belong) to the class (see the scale at the end of each row).
}
\label{fig:dec_bound_9classes}
\end{figure*}
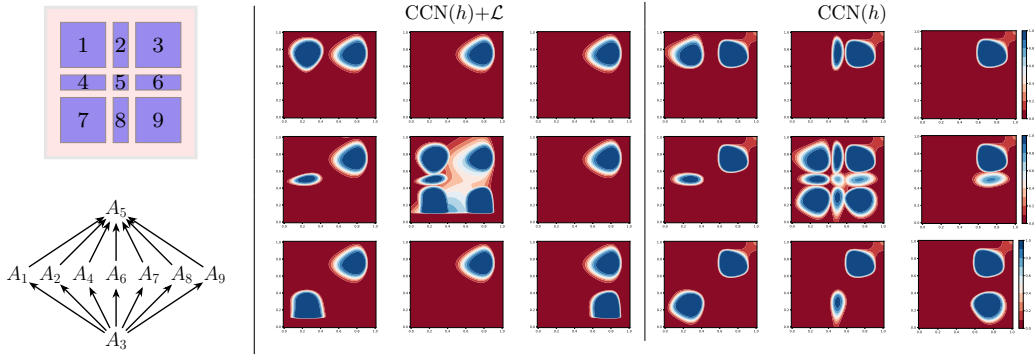

\subsubsection{Synthetic experiment 2}
In order to prove the importance of using \loss{} instead of the standard binary cross entropy loss $\lss$, in this experiment, we compare two models: (i) our model \system{$h$}, and (ii) $h+\module$, i.e., $h$ with $\module$ built on top and trained with $\lss$.
Consider 
the nine rectangles arranged as showed on the top left of Figure~\ref{fig:dec_bound_9classes}
named $R_1, \ldots, R_9$. Assume that 
\begin{enumerate}
    \item we have classes $A_1 \ldots A_9$, 
    \item a data point belongs to $A_i$ if it belongs to the $i$th rectangle, and 
    \item $A_5$ (resp., $A_3$) is an ancestor (resp., descendant) of every class, as shown in the hierarchy on the bottom left of Figure~\ref{fig:dec_bound_9classes}. 
\end{enumerate}
Thus, all points in $R_3$ belong to all classes, and if a data point belongs to a rectangle, then it also belongs to class $A_5$. The datasets consisted of 5000 (50/50 train/test split) data points sampled from a uniform distribution over $[0,1]^2$.

Let $h$ be a feedforward neural network with a single hidden layer with 7 neurons.  We train both $h+\module$ and \system{$h$} for 20k epochs using Adam optimization with learning rate $10^{-2}$ ($\beta_1= 0.9, \beta_2= 0.999$). 
As expected, the average {\auprc} (and standard deviation) over 10 runs for $h+\module$ trained with $\mathcal{L}$ is 0.938 (0.038), while $h+\module$ trained with {\loss} (\system{$h$}) is 0.974 (0.007). Notice that not only $h+\module$ performs worse, but also, due to the convergence to bad local optima, the standard deviation obtained with $h+\module$ is 5 times higher than the one of \system{$h$}: the (min, median, max) {\auprc} for $h+\module$ are $(0.871, 0.945,
0.990)$, while for \system{$h$} are $(0.964, 0.975, 0.990)$.
The difference between \system{$h$} and \system{$h$}+$\lss$ in performance is further highlighted in Figure~\ref{fig:dec_bound_9classes}, which shows the decision boundaries of the 6th best performing networks.\footnote{We picked the 6th best performing networks due to the high variance of the results \system{$h$}+$\lss$.} The figure points out how 
mistakes given by wrong supervisions in lower levels of the hierarchy (see decision boundaries for  $A_2, A_6$, and $A_8$) might 
have dramatic consequences in upper levels of the hierarchy (see decision boundaries for $A_5$).

\begin{table}[t]
    \centering
    \begin{tabular}{l l c c c c c c}
    \toprule
         {\sc Taxonomy} & {\sc Dataset} & $D$ & $n$ & {\sc Train} & {\sc Val} & {\sc Test}  \\
    \midrule
         {\sc FunCat (FUN)} & {\sc Cellcycle} & 77 & 499 & 1625 & 848 & 1281 \\
         {\sc FunCat (FUN)} & {\sc Derisi} & 63 & 499 &  1605 & 842 & 1272 \\
         {\sc FunCat (FUN)} & {\sc Eisen} & 79 & 461 &  1055 & 529 & 835 \\
         {\sc FunCat (FUN)} & {\sc Expr} & 551 & 499 &  1636 & 849 & 1288 \\
         {\sc FunCat (FUN)} & {\sc Gasch1} & 173 & 499 &  1631 & 846 & 1281 \\
         {\sc FunCat (FUN)} & {\sc Gasch2} & 52 & 499 &  1636 & 849 & 1288 \\
         {\sc FunCat (FUN)} & {\sc Seq} & 478 & 499 & 1692 & 876 & 1332 \\
         {\sc FunCat(FUN)} & {\sc Spo} & 80 & 499 & 1597 & 837 & 1263 \\
    \midrule   
    {\sc Gene Ontology (GO)} & {\sc Cellcycle} & 77 & 4122 & 1625 & 848 & 1281 \\
     {\sc Gene Ontology (GO)} & {\sc Derisi} & 63 & 4116 & 1605 & 842 & 1272 \\
     {\sc Gene Ontology (GO)} & {\sc Eisen} & 79 & 3570 & 1055 & 528 & 835 \\
     {\sc Gene Ontology (GO)} & {\sc Expr} & 551 & 4128 & 1636 & 849 & 1288 \\
     {\sc Gene Ontology (GO)} & {\sc Gasch1} & 173 & 4122 & 1631 & 846 & 1281 \\
     {\sc Gene Ontology (GO)} & {\sc Gasch2} & 52 & 4128 & 1636 & 849 & 1288 \\
     {\sc Gene Ontology (GO)} & {\sc Seq} & 478 & 4130 & 1692 & 876 & 1332 \\
     {\sc Gene Ontology (GO)} & {\sc Spo} & 80 & 4166 & 1597 & 837 & 1263 \\
     \midrule
     {\sc Tree} & {\sc Diatoms} & 371 & 398 & 1085 & 464 & 1054 \\
    {\sc Tree} & {\sc Enron} & 1000 & 56 & 692 & 296 & 660 \\
     {\sc Tree} & {\sc Imclef07a} & 80 & 96 & 7000 & 3000 & 1006 \\
     {\sc Tree} & {\sc Imclef07d} & 80 & 46 & 7000 & 3000 & 1006 \\
    \bottomrule
    \end{tabular}
    \caption{Summary of the 20 real-world datasets. Number of features ($D$), number of classes ($n$), and number of data points for each dataset split.}
    \label{tab:datasets_hmc}
\end{table}

\begin{table}[t]
    \centering
    \begin{tabular}{l c c c | c c }
    \toprule
         Dataset & \system{$h$}  & \sc{\lmlp} & \sc{\ens} & \sc{\hmcnr} & \sc{\hmcnf} \\
         \midrule
         \sc{Cellcycle FUN} & {\textbf{0.255}} & 0.207 & 0.227 &0.247& 0.252  \\ 
         \sc{Derisi FUN} & {\textbf{0.195}}  & 0.182 & 0.187 &0.189 & 0.193 \\
         \sc{Eisen FUN} & {\textbf{0.306}}  & 0.245 & 0.286 &0.298 & 0.298\\
         \sc{Expr FUN} & {\textbf{0.302}} & 0.242 & 0.271 &0.300 & 0.301 \\
         \sc{Gasch1 FUN} & {\textbf{0.286}} & 0.235 & 0.267 & 0.283 & 0.284 \\
         \sc{Gasch2 FUN} & {\textbf{0.258}}  & 0.211 & 0.231 & 0.249 & 0.254\\
         \sc{Seq FUN} & {\textbf{0.292}} & 0.236 & 0.284 & 0.290 & 0.291 \\ 
         \sc{Spo FUN} & {\textbf{0.215}} & 0.186 & 0.211 &0.210 & 0.211 \\
         \midrule
         \sc{Cellcycle GO} &{\textbf{0.413}} & 0.361 & 0.387 &0.395& 0.400  \\ 
         \sc{Derisi GO} & {\textbf{0.370}}  & 0.343 & 0.361 &0.368& 0.369 \\
         \sc{Eisen GO} & {\textbf{0.455}} & 0.406 & 0.433 &0.435& 0.440\\
         \sc{Expr GO} & 0.447 & 0.373 & 0.422 &0.450& \textbf{0.452} \\
         \sc{Gasch1} GO & {\textbf{0.436}} & 0.380 & 0.415 &0.416& 0.428 \\
         \sc{Gasch2} GO & 0.414 & 0.371 & 0.395 &0.463& \textbf{0.465}\\
         \sc{Seq GO} & 0.446  & 0.370 & 0.438 &0.443& \textbf{0.447} \\ 
         \sc{Spo GO} & {\textbf{0.382}}  & 0.342  & 0.371 &0.375& 0.376 \\
         \midrule
         \sc{Diatoms} & {\textbf{0.758}} & - & 0.501 & 0.514 & 0.530 \\
         \sc{Enron} & {\textbf{0.756}} & - & 0.696 & 0.710 & 0.724 \\
         \sc{Imclef07a} &{\textbf{0.956}} & - & 0.803 & 0.904 & 0.950 \\
         \sc{Imclef07d} & {\textbf{0.927}} & - & 0.881 & 0.897 & 0.920 \\
         \midrule
         \sc{Average Ranking} & 1.25 & 5.00 & 3.93 & 2.93 & 1.90 \\
    \bottomrule
    \end{tabular}
        \caption{Comparison of \system{$h$} with the other state-of-the-art models. The performance of each system is measured as the \auprc{} obtained on the test set. The best results are in bold.    \label{tab:comparison}}
\end{table}

\subsubsection{Comparison with the state of the art}

We tested \system{$h$} on 20 real-world datasets commonly used to compare HMC systems (see, e.g.,~\cite{kwok2011,nakano2019,vens2008,cerri2018}): 16 are functional genomics datasets \cite{clare2003}, 2 contain medical images \cite{dimitrovski2008}, 1 contains images of microalgae \cite{dimitrovski2011}, and 1 is a text categorization dataset  \cite{klimt2004}.\footnote{Links: https://dtai.cs.kuleuven.be/clus/hmcdatasets and http://kt.ijs.si/DragiKocev/PhD/resources}
The characteristics of these datasets are summarized in Table~\ref{tab:datasets_hmc}. These datasets are particularly challenging,  because their number of training samples is rather limited, and they have a large variation, both in the number of features (from 52 to 1000) and in the number of classes (from 56 to 4130). We applied the same preprocessing  to all the datasets. All the categorical features were transformed using one-hot encoding. The missing values were replaced by their mean in the case of numeric features and by a vector of all zeros in the case of categorical ones. All the features were standardized. 

We built $h$ as a feedforward neural network with two 
hidden layers and ReLU non-linearity.
 To prove the robustness of \system{$h$},
 we kept all the hyperparameters fixed except the hidden dimension and the learning rate used for each dataset, which are given in Appendix~\ref{app:hidden_dim} and were optimized over the validation sets. In all experiments, the loss was minimized using Adam optimizer with weight decay $10^{-5}$, and patience 20 ($\beta_1 = 0.9$, $\beta_2 = 0.999$). The dropout rate was set to 70\% and the batch size to 4. 
As in \cite{cerri2018}, we retrained \system{$h$} on both training and validation data for the same number of epochs, as the early stopping procedure determined was optimal in the first pass.

For each dataset, we run \system{$h$}, {\ens} \cite{schietgat2010}, and {\lmlp} \cite{cerri2016} 10 times, and the average \auprc~is reported in Table~\ref{tab:comparison}. For simplicity, we omit the standard deviations, which for \system{$h$} are in the range $[0.5 \times 10^{-3}, 2.6 \times 10^{-3}]$,  proving that it is  a very stable model. As reported in \cite{nakano2019}, {\ens} and {\lmlp} are the current state-of-the-art models with publicly available code. These models were run with the suggested configuration settings on each dataset.%
\footnote{We also ran the code from \cite{masera2018}. However, we obtained very different results from the ones reported in the paper. 
Similar negative results are also reported in \cite{nakano2019}.}
The results are shown in Table \ref{tab:comparison}, left side. On the right side, we show the results of {\hmcnr} and {\hmcnf} directly taken from \cite{cerri2018}, since the code is not publicly available. We report the results of both systems, because, while {\hmcnr} has worse results than {\hmcnf}, the amount of parameters of the latter grows with the number of hierarchical levels. As a consequence,   {\hmcnr} is much lighter in terms of total amount of parameters, and the authors advise that for very large hierarchies, {\hmcnr} is probably a better choice than {\hmcnf} considering the trade-off performance vs. computational cost~\cite{cerri2018}. Note that the number of parameters of \system{$h$} is independent from the number of hierarchical levels.

  As reported in Table~\ref{tab:comparison}, \system{$h$} has the greatest number of wins (it has the best performance on all datasets but 3) and best average ranking (1.25). We also verified the statistical significance of the
  \begin{figure}
    \centering
    \includegraphics[trim=30 0 30 0,clip, width=0.6\textwidth]{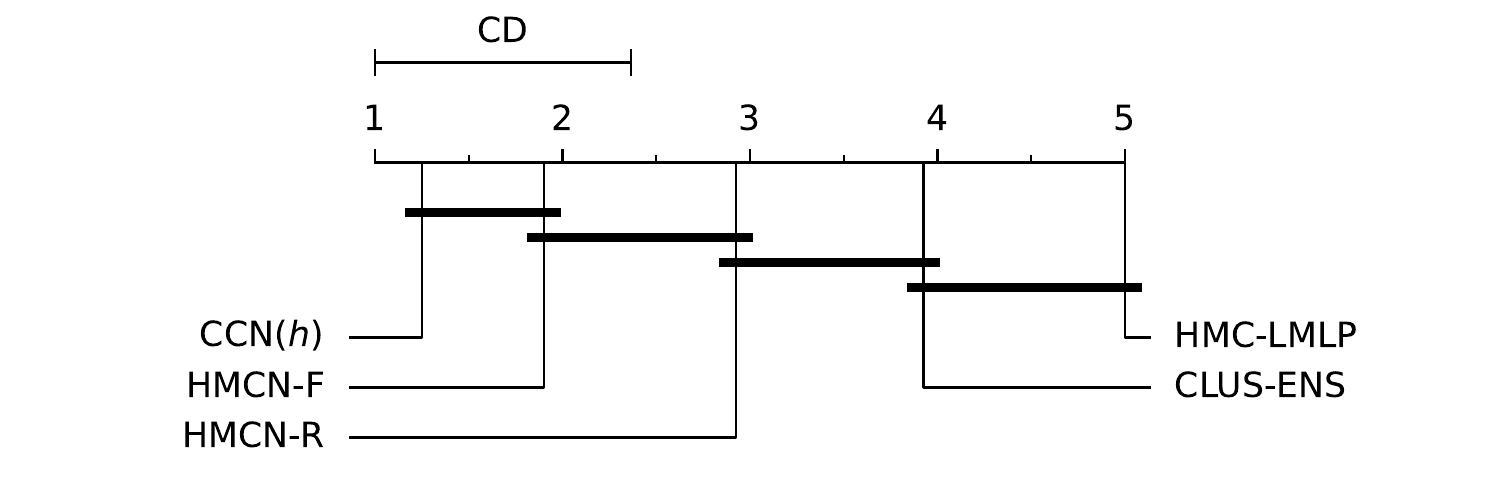}
\caption{Critical diagram for the Nemenyi's statistical test.\label{fig:nemenyi}}
\end{figure}
   results following %
  \cite{demsar2006}. We first executed the Friedman test, obtaining p-value $4.26 \times 10^{-15}$.
  We then performed the post-hoc Nemenyi
   test, and the resulting critical diagram is shown in Figure~\ref{fig:nemenyi}, where the group of methods that do not differ significantly (significance level 0.05) are connected through a horizontal line. The Nemenyi test is powerful enough to conclude that there is a statistical significant difference between the performance of \system{$h$} and all the other models but \hmcnf. Hence, %
   following \cite{demsar2006,benavoli2016}, we compared \system{$h$} and {\hmcnf} using the Wilcoxon test. This test, contrarily to the Friedman test and the Nemenyi test, takes into account not only the ranking, but also the differences in performance of the two algorithms.  
 The Wilcoxon test allows us to conclude that there is a statistical significant difference between the  performance of \system{$h$} and {\hmcnf} with p-value of $6.01 \times 10^{-3}$.

\begin{table}[t]
    \centering
    \begin{tabular}{l c c c c c c}
    \toprule
           & \multicolumn{2}{c}{$h^+$} & \multicolumn{2}{c}{$h\!+\!\module$} & \multicolumn{2}{c}{\system{$h$}} \\
         Dataset & \auprc & Epochs & \auprc & Epochs & \auprc & Epochs \\
         \midrule
         {\sc Cellcycle} & 0.240 & 107 & 0.238 & 108 & {\textbf{0.255}} & 106 \\
         {\sc Derisi} & 0.190 & 64 & 0.188 & 66  &  {\textbf{0.195}} & 67\\
         {\sc Eisen} & 0.290 & 112 & 0.286 & 107 & {\textbf{0.306}} & 110 \\
         {\sc Expr} & 0.272 & 39 &  0.267 & 19 & {\textbf{0.302}} & 20 \\
         {\sc Gasch1} & 0.265 & 41 & 0.262 & 42 & {\textbf{0.286}} & 38\\
         {\sc Gasch2} & 0.244 & 128 & 0.242 & 132 & {\textbf{0.258}} & 131 \\
         {\sc Seq} & 0.249 & 13 & 0.252 & 13 & {\textbf{0.292}} & 13 \\
         {\sc Spo} & 0.201 & 108 & 0.202 & 117 & {\textbf{0.215}} & 115 \\
         \midrule
         {\sc Average Ranking} & 2.94 & & 2.06 & & 1.00 & \\
    \bottomrule
    \end{tabular}
        \caption{Impact of $\module$ and $\module$+$\loss$ on the performance measured as {\auprc} and on the total number of epochs for the validation set of the Funcat datasets.}
    \label{tab:ablation}
\end{table}

\subsubsection{Ablation studies}

To analyze the impact of both $\module$ and $\loss$, we compared the performance of \system{$h$} on the FunCat datasets against the performance of  $h^+$, i.e., $h$ with $\module$ applied as post-processing at inference time and $h\!+\!\module$, i.e., $h$ with $\module$ built on top. Both these models were trained using the standard binary cross-entropy loss. 
As it can be seen in Table~\ref{tab:ablation}, \system{$h$}, by exploiting both \module{} and \loss, always outperforms $h^+$ and $h\!+\!\module$ on all datasets. In Table~\ref{tab:ablation}, we also report after how many epochs the algorithm stopped training in average. As it can be seen, 
\system{$h$}, $h\!+\!\module$ and $h^+$ always require approximately the same number of epochs. 

\subsection{Multi-Label Classification with Logical Hard Constraints}

As for the hierarchical case,
 we first consider a generalization of the synthetic experiment proposed in the basic case. Then
we test \system{$h$}  on 16 real-world datasets with general constraints, and finally we present the ablation studies.\footnote{Link: https://github.com/EGiunchiglia/CCN/} 

About the metrics, the analysis of 64 papers on MC problems conducted by \cite{spolaor2013} and reported in~\cite{pereira2018metrics}, shows that already in 2013 as many as 19 different metrics have been used to evaluate MC models, and still today different papers use different subsets of such metrics. However, as suggested by~\cite{pereira2018metrics}, not all subsets can be used, as the experimental results may appear to favor a specific behavior depending on the subset of measures chosen, thus possibly leading to misleading conclusions. To avoid such undesired results, the authors conducted a correlation analysis of the metrics 
that led to the individuation of clusters of correlated measures and thus to the proposal of various subsets of metrics, chosen according to the following criteria:
\begin{enumerate}
\item first, Hamming loss is highly recommended for inclusion in each subset: it is not correlated with others, and is the most
used metric in the literature (55 papers out of the 64 surveyed),
\item 
next it could be considered employing other measures not correlated with any others like coverage error and
ranking loss, and
\item 
finally, a suitable selection should include at least one metric from each cluster of correlated measures. Among them, multi-label accuracy is a good choice because it is among the ones with the highest correlations to other measures.
\end{enumerate}
Other criteria they suggest and use for the selection of the proposed subsets are: popularity in the literature, the choice to include or not AUC-based metrics, and the size of the resulting set of metrics.

Following the above criteria, we used the following six metrics, each taking value in the interval $[0,1]$ and each annotated with either   $\uparrow$ or $\downarrow$ to mean that larger values for that metric stand for better (resp. worse) performance: 
\begin{enumerate}
\begin{multicols}{2}
    \item average precision ($\uparrow$),
    \item coverage error ($\downarrow$),
    \item Hamming loss ($\downarrow$),
    \item multi-label accuracy ($\uparrow$),
    \item one-error ($\downarrow$), and
    \item ranking loss ($\downarrow$). 
\end{multicols}
\end{enumerate}
The above six metrics are exactly those belonging to the first two subsets of metrics proposed in~\cite{pereira2018metrics}.%
\footnote{In particular, the first of the proposed subsets includes coverage error, Hamming loss, multi-label accuracy  and ranking loss, while the second includes
average precision, coverage error, Hamming loss, one-error and ranking loss; see~\cite{pereira2018metrics} for more details.} Notice that, the above list does not include {\auprc}: already in 2013 and still today, contrarily to the specialized HMC literature, it is generally not used in the MC literature~\cite{spolaor2013,pereira2018metrics,feng2019}.

\begin{figure*}[t]
\centering
\begin{tabular}{c@{\ \ \,}c@{\ \ \ \,}c@{\ \ \ \ }c@{\ \ \ \,}}
\begin{minipage}{.32\textwidth}
    \includegraphics[width=\textwidth]{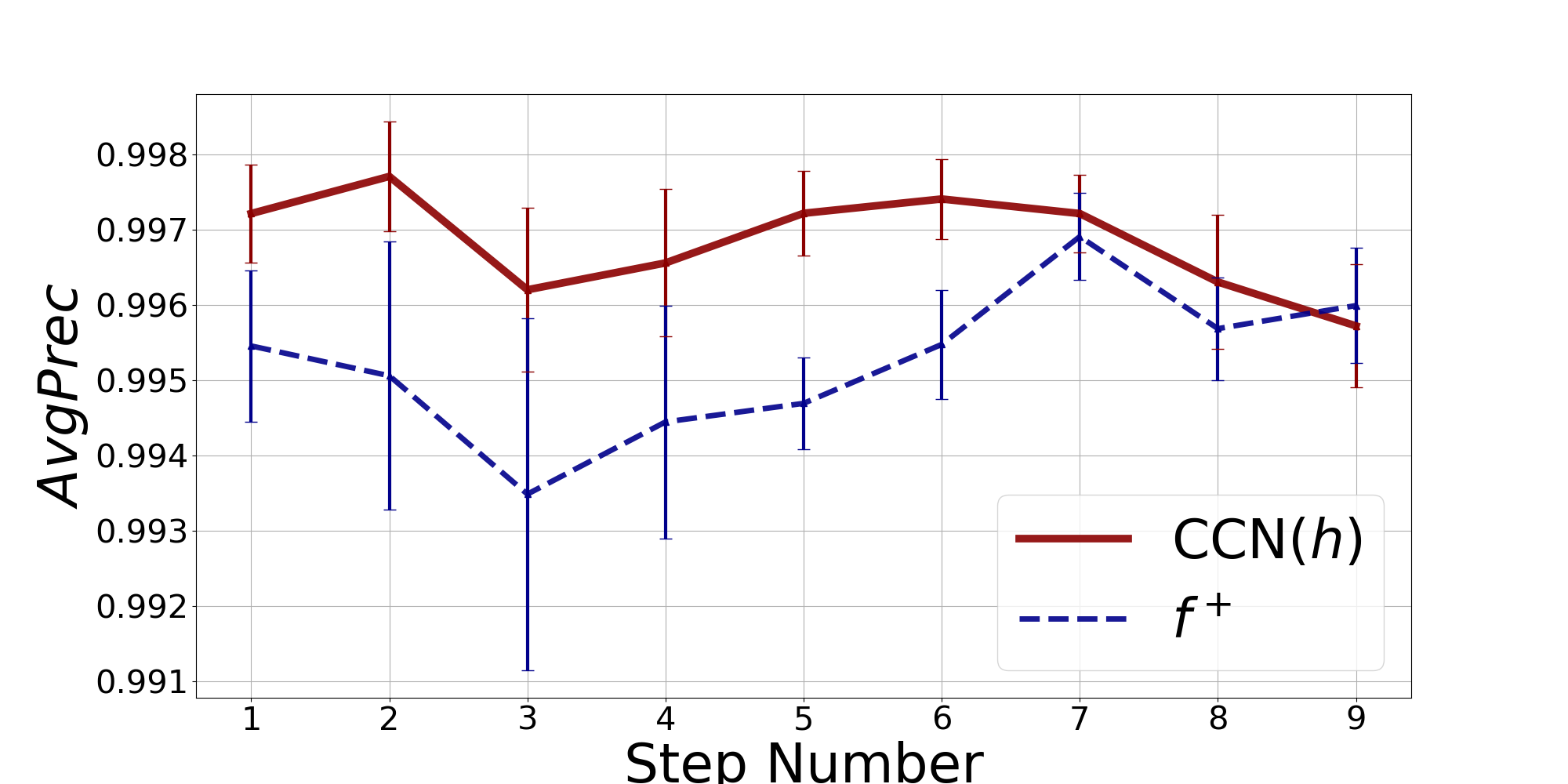}
\end{minipage} &
\begin{minipage}{.32\textwidth}
    \includegraphics[width=\textwidth]{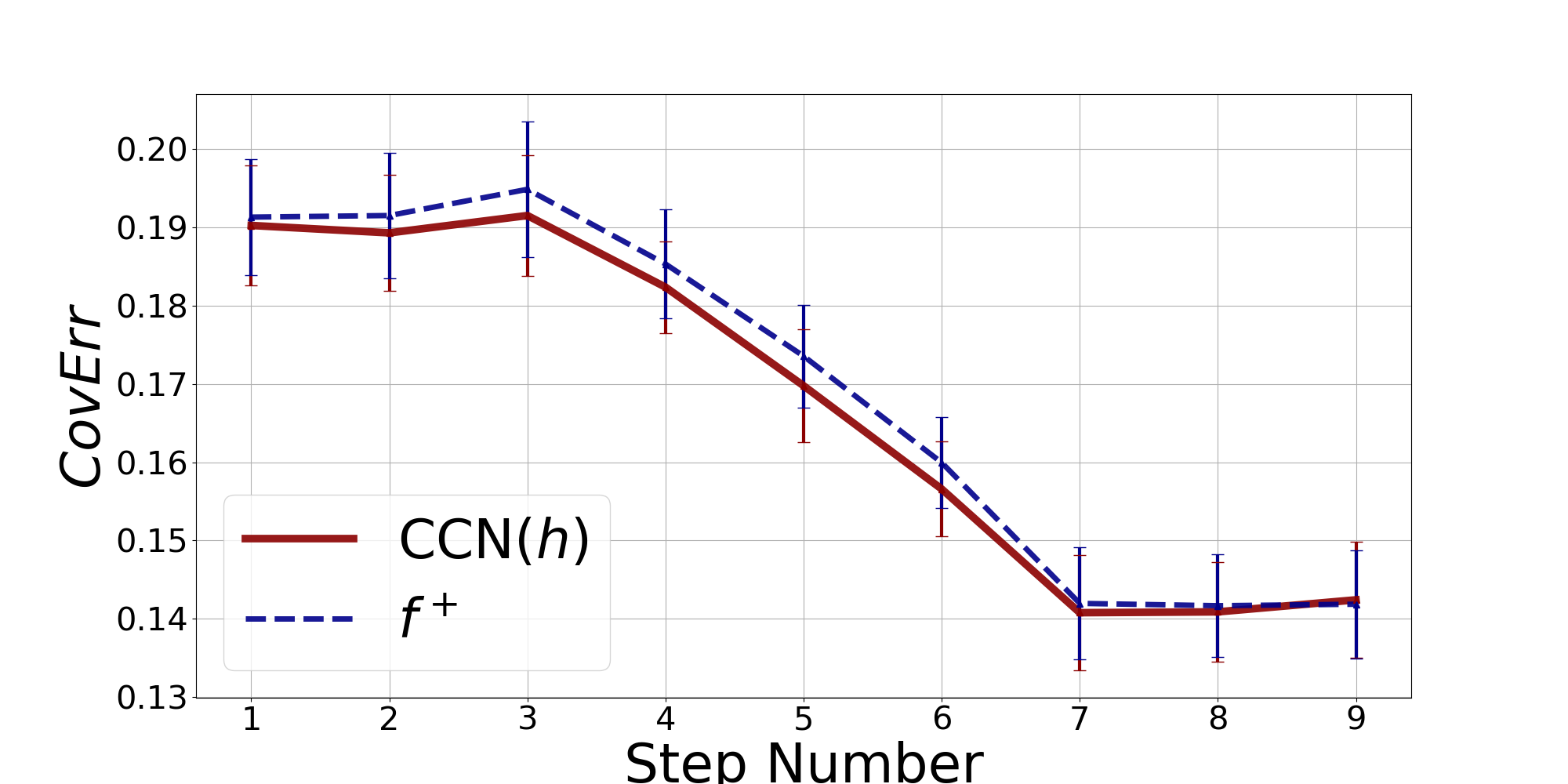}
\end{minipage} &
\begin{minipage}{.32\textwidth}
    \includegraphics[width=\textwidth]{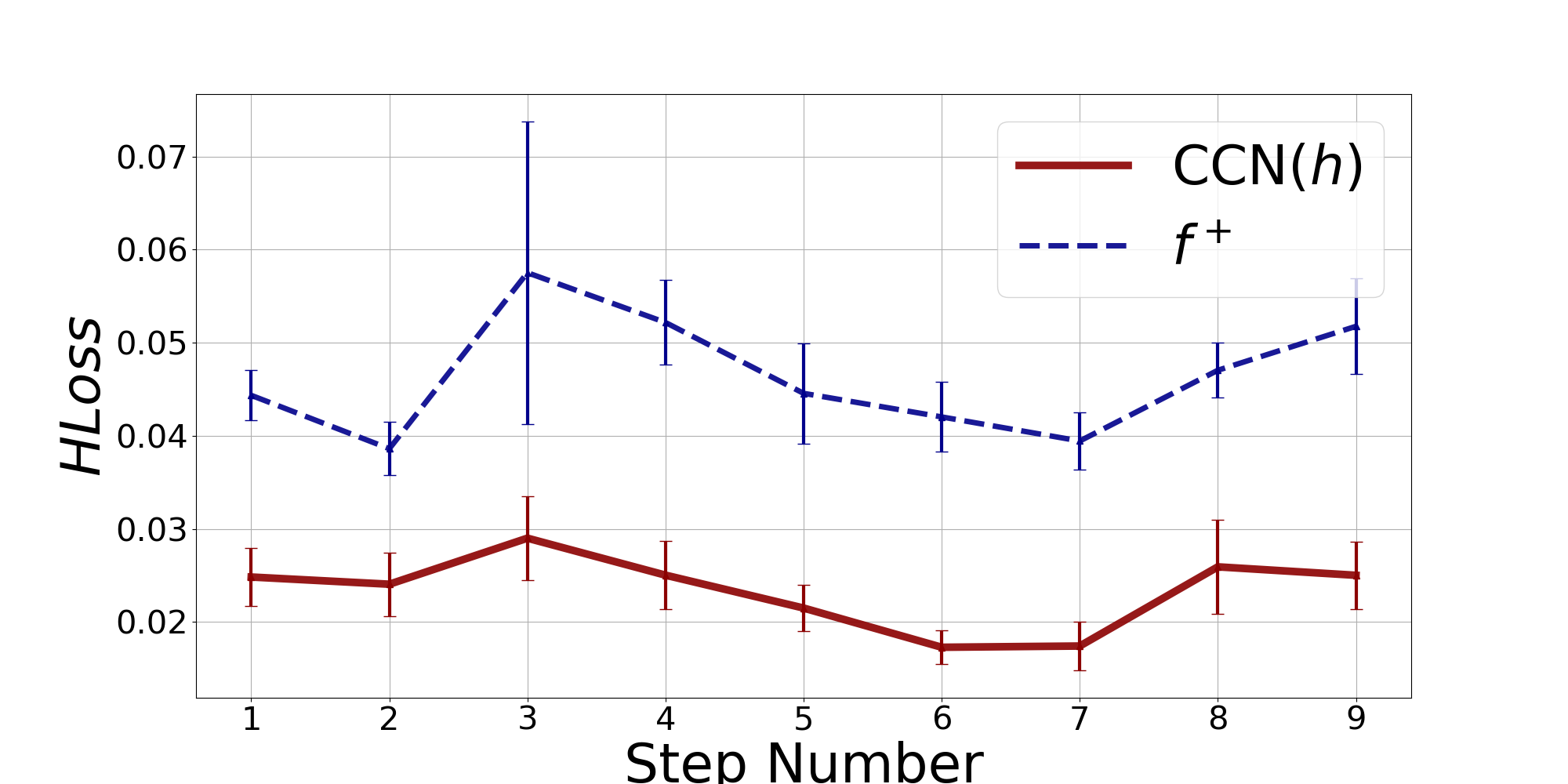}
\end{minipage} \\
\begin{minipage}{.32\textwidth}
    \includegraphics[width=\textwidth]{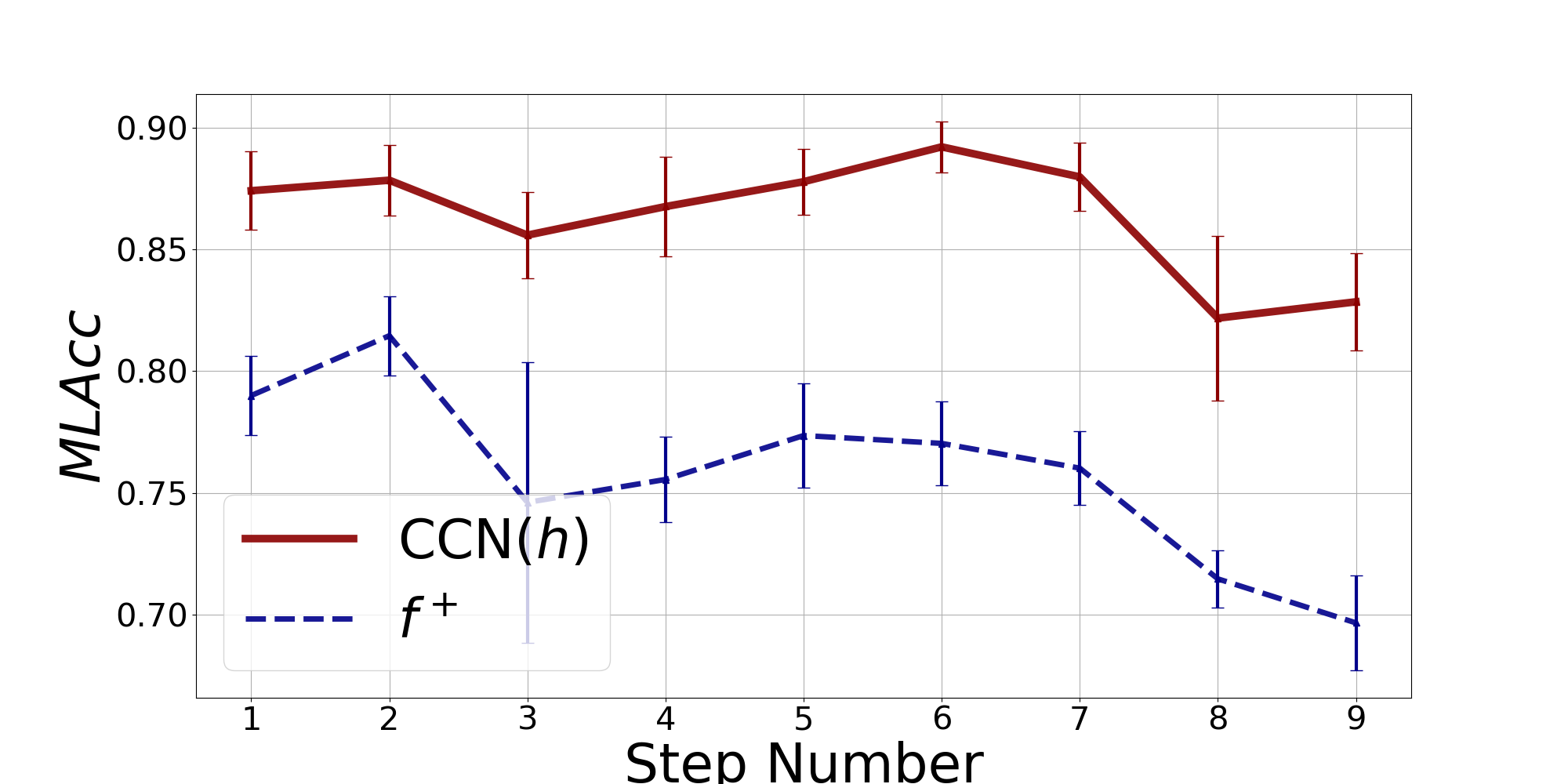}
\end{minipage} &
\begin{minipage}{.32\textwidth}
    \includegraphics[width=\textwidth]{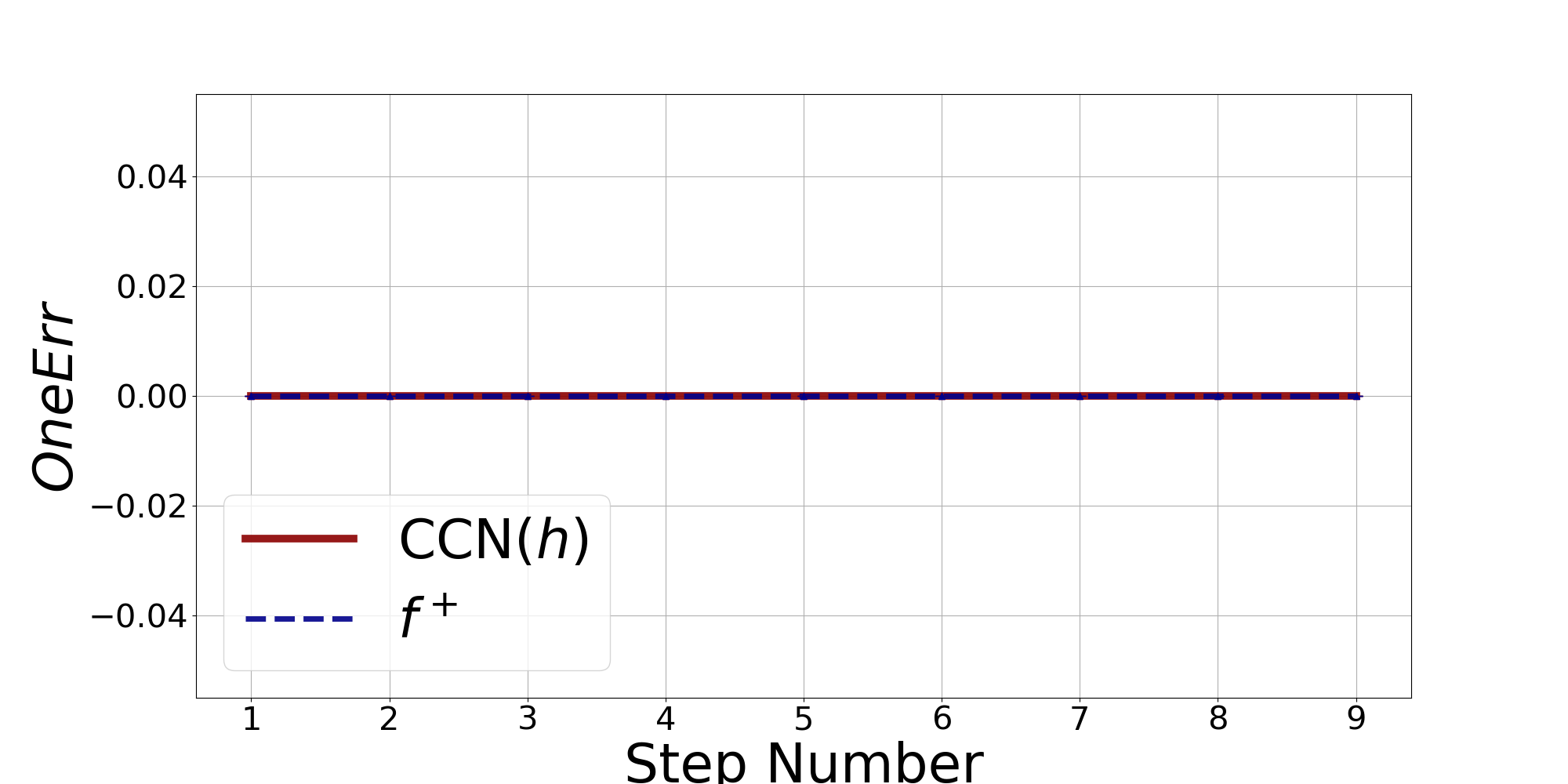}
\end{minipage} &
\begin{minipage}{.32\textwidth}
    \includegraphics[width=\textwidth]{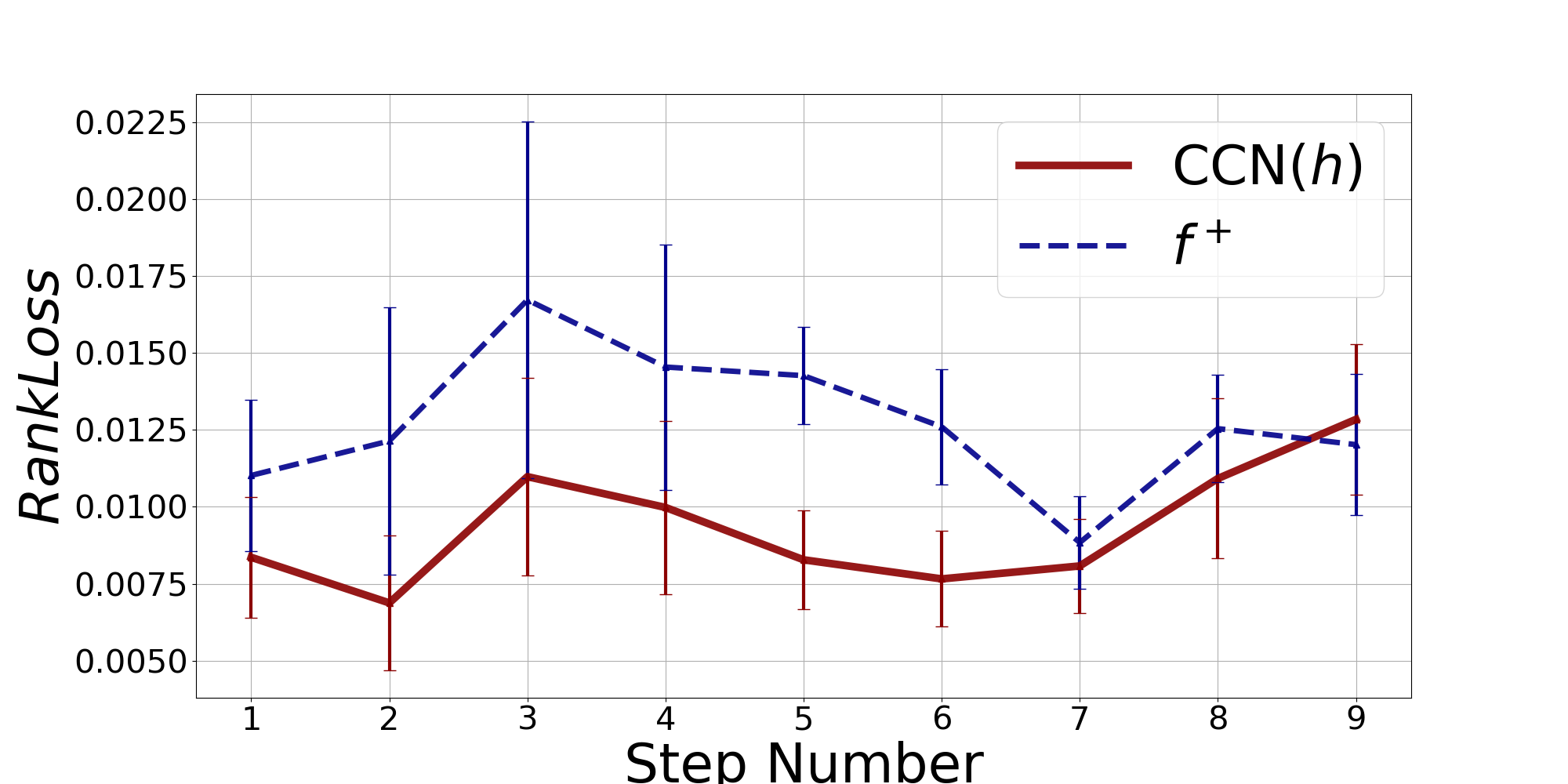}
\end{minipage} \\
\\
\end{tabular}
\caption{Mean average precision, coverage error, hamming loss, multi-label accuracy, one-error, and ranking loss with standard deviation of \system{$h$} and $f^+$ for each step.
}
\label{img:synth_exp}
\end{figure*}

\subsubsection{Synthetic Experiment}
Consider the generalization of the experiment presented as basic case, in which we started with $R_1$ outside $R_2$ (as in the first row of Fig.~\ref{fig:dec_bound_figs_gen}), and then moved $R_1$ towards the centre of $R_2$ (as in the second row of  Fig.~\ref{fig:dec_bound_figs_gen}) in 9 uniform steps. 
As for HMC problems, this experiment is meant to show how the performance of \system{$h$} and $f^+$ defined as in Section~\ref{sec:basic_case} vary depending on the relative positions of $R_1$ and $R_2$. As expected, Figure~\ref{img:synth_exp} shows that \system{$h$} performs better or equally to $f^+$ at all steps and for all metrics. In particular: 
\begin{itemize}
    \item \system{$h$} performs consistently better than $f^+$ at all steps in terms of average precision, coverage error, Hamming loss, multi-label accuracy and ranking loss. Further, as in the HMC case, \system{$h$} exhibits much more stable performances than $f^+$ as highlighted by the visibly much smaller standard deviation bar of \system{$h$}.
    \item \system{$h$} and $f^+$ perform identically in terms of one-error. This is due to the fact that one-error measures the fraction of instances whose most confident class is irrelevant, since neither \system{$h$} nor $f^+$ ever make this mistake, they both have one-error equal to zero at all steps.
\end{itemize}

\begin{table}[t]
    \centering
    \begin{tabular}{l c c c c c c c c c c c}
    \toprule
         {\sc Dataset} & $D$ & $L$ & {\sc Train} & {\sc Val} & {\sc Test} & $C$ & $H$ & $B$ & $B^+$ & $B^-$ & $H/L$\\
    \midrule
{\sc Arts}        & 462  & 26  & 2975        & 525       & 1500       & 344 & 11  & 7.1 & 5.6   & 1.5   & 42.3\% \\
{\sc Business}    & 438  & 30  & 2975        & 525       & 1500       & 77  & 7   & 2.5 & 2.3   & 0.2   & 23.3\% \\
{\sc Cal500}      & 68   & 174 & 298         & 53        & 151        & 39  & 7   & 2   & 1.1   & 0.9   & 4.0\%  \\
{\sc Emotions}    & 72   & 6   & 332         & 59        & 202        & 1   & 1   & 5   & 0     & 5     & 16.7\% \\
{\sc Enron}       & 1001 & 53  & 954         & 169       & 579        & 7   & 5   & 1   & 1     & 0     & 9.4\%  \\
{\sc Genbase}     & 1186 & 27  & 393         & 70        & 199        & 88  & 13  & 2.2 & 2.0   & 0.2   & 48.1\% \\
{\sc Image}       & 294  & 5   & 1190        & 210       & 600        & 1   & 1   & 4   & 0     & 4     & 20.0\% \\
{\sc Medical}     & 1449 & 45  & 283         & 50        & 645        & 17  & 9   & 1.2 & 1.2   & 0     & 20.0\% \\
{\sc Rcv1subset1} & 944  & 101 & 3570        & 630       & 1800       & 247 & 16  & 3.7 & 2.9   & 0.8   & 15.8\% \\
{\sc Rcv1subset2} & 944  & 101 & 3570        & 630       & 1800       & 81  & 15  & 2.4 & 1.7   & 0.7   & 14.9\% \\
{\sc Rcv1subset3} & 944  & 101 & 3570        & 630       & 1800       & 72  & 16  & 2.3 & 1.7   & 0.6   & 15.8\% \\
{\sc Rcv1subset4} & 944  & 101 & 3570        & 630       & 1800       & 63  & 14  & 2.1 & 1.7   & 0.4   & 13.9\% \\
{\sc Rcv1subset5} & 944  & 101 & 3570        & 630       & 1800       & 73  & 11  & 2.5 & 2     & 0.5   & 10.9\% \\
{\sc Science}     & 743  & 40  & 2975        & 525       & 1500       & 37  & 11  & 2.1 & 1.7   & 0.4   & 27.5\% \\
{\sc Scene}       & 294  & 6   & 1029        & 182       & 1196       & 1   & 1   & 5   & 0     & 5     & 16.7\% \\
{\sc Yeast}       & 103  & 14  & 1275        & 225       & 917        & 34  & 11  & 2.3 & 1.9   & 0.4   & 78.6\% \\
    \bottomrule
    \end{tabular}
        \caption{Summary of the real-world MC datasets. For each dataset, we report from left to right: (i) name, (ii) number of features ($D$), (iii) number of classes~($L$), (iv-vi) number of data points for each split, (vii) number of constraints~($C$), (iix) number of different classes that appear at least once as head of a constraint~($H$), (ix) average number of classes appearing in the body~($B$), (x-xi) average number of classes appearing positively (resp.,  negatively) in the body ($B^+$), (resp., ($B^-$)), and (xii) percentage of classes appearing at least once as head of a constraint.}
    \label{tab:datasets}
\end{table}

\newcommand{\fs}[1]{\tiny{\textbf{#1}}}
\begin{table}[hp]
    \centering
        \setlength{\tabcolsep}{3.5pt}
    \begin{tabular}{l c c c c c c c c c c c c c c c c}
    \toprule
         Model & {\sc Arts}  & {\sc Business} & {\sc Cal500} & {\sc Emotions} & {\sc Enron} & {\sc Genbase} & {\sc Image} & {\sc Medical} \\
         \midrule
         & \multicolumn{8}{c}{Average precision ($\uparrow$)}\\
         \midrule
         \system{$h$} & 0.623 & \textbf{0.904} & \textbf{0.520} & \textbf{0.800} & 0.704 & 0.996 & \textbf{0.807} & \textbf{0.866} \\
\camel       & \textbf{0.625} & 0.899 & 0.513 & 0.756 & \textbf{0.708} & 0.990 & 0.793 & 0.807 \\
\ecc         & 0.544 & 0.867 & 0.401 & 0.772 & 0.643 & \textbf{1.000} & 0.738 & 0.823 \\
\br          & 0.546 & 0.863 & 0.441 & 0.793 & 0.643 & \textbf{1.000} & 0.726 & 0.823 \\
\rakel       & 0.530 & 0.856 & 0.433 & 0.798 & 0.636 & \textbf{1.000} & 0.721 & 0.811 \\
          \midrule
          & \multicolumn{8}{c}{Coverage error $(\downarrow)$} \\
          \midrule
       \system{$h$} & \textbf{0.172} & \textbf{0.065} & \textbf{0.734} & \textbf{0.315} & \textbf{0.217} & 0.016 & \textbf{0.187} & \textbf{0.035} \\
\camel       & 0.202 & 0.083 & 0.791 & 0.372 & 0.256 & 0.010 & 0.201 & 0.036 \\
\ecc         & 0.223 & 0.089 & 0.853 & 0.338 & 0.285 & \textbf{0.009} & 0.242 & 0.045 \\
\br          & 0.217 & 0.086 & 0.789 & 0.324 & 0.288 & \textbf{0.009} & 0.245 & 0.045 \\
\rakel       & 0.221 & 0.085 & 0.791 & 0.317 & 0.294 & \textbf{0.009} & 0.250 & 0.049 \\
\midrule
          & \multicolumn{8}{c}{Hamming loss ($\downarrow$)} \\
          \midrule
\system{$h$} & \textbf{0.054} & \textbf{0.023} & \textbf{0.136} & \textbf{0.197} & \textbf{0.046} & \textbf{0.001} & \textbf{0.172} & \textbf{0.013} \\
\camel       & 0.055 & \textbf{0.023} & 0.138 & 0.265 & 0.047 & 0.003 & 0.174 & 0.024 \\
\ecc         & 0.081 & 0.031 & 0.172 & 0.245 & 0.055 & \textbf{0.001} & 0.218 & 0.019 \\
\br          & 0.079 & 0.032 & 0.162 & 0.229 & 0.054 & \textbf{0.001} & 0.232 & 0.019 \\
\rakel       & 0.082 & 0.034 & 0.165 & 0.223 & 0.055 & \textbf{0.001} & 0.225 & 0.019 \\
\midrule
          & \multicolumn{8}{c}{Multi-label accuracy ($\uparrow$)} \\
          \midrule
\system{$h$} & \textbf{0.238} & 0.601 & 0.203 & \textbf{0.534} & \textbf{0.395} & 0.986 & \textbf{0.488} & \textbf{0.589} \\
\camel       & 0.218 & \textbf{0.609} & 0.210 & 0.354 & 0.381 & 0.943 & 0.456 & 0.284 \\
\ecc         & 0.217 & 0.548 & 0.220 & 0.446 & 0.361 & \textbf{0.992} & 0.387 & 0.481 \\
\br          & 0.217 & 0.538 & 0.221 & 0.465 & 0.365 & \textbf{0.992} & 0.369 & 0.477 \\
\rakel       & 0.215 & 0.527 & \textbf{0.222} & 0.485 & 0.361 & \textbf{0.992} & 0.376 & 0.481 \\
\midrule
          & \multicolumn{8}{c}{One-error ($\downarrow$)} \\
         \midrule
\system{$h$} & 0.475 & 0.093 & \textbf{0.113} & \textbf{0.273} & 0.235 & \textbf{0.000} & \textbf{0.296} & \textbf{0.181} \\
\camel       & \textbf{0.460} & \textbf{0.090} & 0.133 & 0.381 & \textbf{0.223} & 0.020 & 0.310 & 0.285 \\
\ecc         & 0.568 & 0.137 & 0.378 & 0.332 & 0.309 & \textbf{0.000} & 0.392 & 0.251 \\
\br          & 0.567 & 0.147 & 0.232 & 0.292 & 0.299 & \textbf{0.000} & 0.425 & 0.251 \\
\rakel       & 0.586 & 0.159 & 0.232 & 0.292 & 0.304 & \textbf{0.000} & 0.430 & 0.266 \\
\midrule
          & \multicolumn{8}{c}{Ranking loss ($\downarrow$)}\\
         \midrule
\system{$h$} & \textbf{0.115} & \textbf{0.030} & \textbf{0.173} & \textbf{0.161} & \textbf{0.076} & 0.003 & \textbf{0.159} & \textbf{0.024} \\
\camel       & 0.136 & 0.040 & 0.189 & 0.237 & 0.086 & \textbf{0.001} & 0.177 & 0.026 \\
\ecc         & 0.158 & 0.046 & 0.257 & 0.193 & 0.107 & \textbf{0.001} & 0.231 & 0.033 \\
\br          & 0.155 & 0.045 & 0.218 & 0.177 & 0.108 & \textbf{0.001} & 0.234 & 0.032 \\
\rakel       & 0.159 & 0.044 & 0.220 & 0.169 & 0.112 & \textbf{0.001} & 0.242 & 0.037 \\
    \bottomrule
    \end{tabular}
        \caption{Comparison of \system{$h$} with the other state-of-the-art models. The best results are in bold.     \label{tab:comparison_gen1}}
\end{table}

\begin{table}[hp]
    \centering
    \setlength{\tabcolsep}{3.5pt}
    \begin{tabular}{l c c c c c c c c}
    \toprule
         Model & {\sc Rcv1S1}   & {\sc Rcv1S2} & {\sc Rcv1S3} & {\sc Rcv1S4} & {\sc Rcv1S5} & {\sc Science} & {\sc Scene} & {\sc Yeast}\\
         \midrule
         & \multicolumn{8}{c}{Average precision ($\uparrow$)} \\
         \midrule
\system{$h$} & \textbf{0.642} & \textbf{0.666} & \textbf{0.647} & \textbf{0.675} & 0.560 & 0.603 & \textbf{0.868} & \textbf{0.768} \\
\camel       & 0.622 & 0.647 & 0.636 & 0.654 & \textbf{0.564} & \textbf{0.614} & 0.824 & 0.766 \\
\ecc         & 0.549 & 0.575 & 0.585 & 0.609 & 0.529 & 0.502 & 0.794 & 0.724 \\
\br          & 0.536 & 0.563 & 0.572 & 0.600 & 0.524 & 0.500 & 0.781 & 0.743 \\
\rakel       & 0.532 & 0.556 & 0.562 & 0.589 & 0.508 & 0.493 & 0.794 & 0.732 \\
 \midrule
         & \multicolumn{8}{c}{Coverage error ($\downarrow$)} \\
         \midrule
\system{$h$} & \textbf{0.092} & \textbf{0.089} & \textbf{0.103} & \textbf{0.080}  & \textbf{0.107} & \textbf{0.131} & \textbf{0.077} & \textbf{0.452} \\
\camel       & 0.131 & 0.115 & 0.123 & 0.103 & 0.130 & 0.162 & 0.106 & 0.457 \\
\ecc         & 0.185 & 0.166 & 0.167 & 0.169 & 0.196 & 0.225 & 0.127 & 0.495 \\
\br          & 0.194 & 0.181 & 0.178 & 0.184 & 0.210 & 0.227 & 0.128 & 0.476 \\
\rakel       & 0.201 & 0.180 & 0.185 & 0.195 & 0.209 & 0.225 & 0.123 & 0.481 \\
 \midrule
         & \multicolumn{8}{c}{Hamming loss ($\downarrow$)} \\
         \midrule
\system{$h$} & \textbf{0.026} & \textbf{0.022} & \textbf{0.024} & \textbf{0.019} & \textbf{0.025} & \textbf{0.031} & \textbf{0.092} & \textbf{0.196} \\
\camel       & 0.027 & \textbf{0.022} & \textbf{0.024} & 0.021 & \textbf{0.025} & \textbf{0.031} & 0.109 & \textbf{0.196} \\
\ecc         & 0.031 & 0.027 & 0.028 & 0.026 & 0.030 & 0.049 & 0.131 & 0.221 \\
\br          & 0.032 & 0.028 & 0.029 & 0.027 & 0.031 & 0.051 & 0.151 & 0.214 \\
\rakel       & 0.033 & 0.029 & 0.030 & 0.027 & 0.031 & 0.051 & 0.130 & 0.225 \\
 \midrule
         & \multicolumn{8}{c}{Multi-label accuracy ($\uparrow$)} \\
         \midrule
\system{$h$} & \textbf{0.296} & \textbf{0.310}  & \textbf{0.303} & \textbf{0.324} & \textbf{0.275} & \textbf{0.255} & \textbf{0.607} & \textbf{0.480}  \\
\camel       & 0.204 & 0.222 & 0.210 & 0.257 & 0.223 & 0.217 & 0.528 & \textbf{0.480} \\
\ecc         & 0.264 & 0.277 & 0.273 & 0.297 & 0.269 & 0.209 & 0.478 & 0.443 \\
\br          & 0.263 & 0.279 & 0.275 & 0.289 & 0.263 & 0.200 & 0.438 & 0.456 \\
\rakel       & 0.263 & 0.272 & 0.268 & 0.290 & 0.258 & 0.201 & 0.481 & 0.445 \\
 \midrule
         & \multicolumn{8}{c}{One-error ($\downarrow$)} \\
         \midrule
\system{$h$} & \textbf{0.413} & \textbf{0.389} & \textbf{0.405} & \textbf{0.379} & \textbf{0.402} & 0.494 & \textbf{0.224} & 0.234 \\
\camel       & \textbf{0.413} & 0.397 & 0.413 & 0.399 & 0.414 & \textbf{0.472} & 0.287 & \textbf{0.231} \\
\ecc         & 0.477 & 0.462 & 0.453 & 0.436 & 0.451 & 0.603 & 0.319 & 0.300 \\
\br          & 0.492 & 0.474 & 0.466 & 0.435 & 0.466 & 0.605 & 0.358 & 0.259 \\
\rakel       & 0.488 & 0.481 & 0.471 & 0.439 & 0.474 & 0.606 & 0.329 & 0.270 \\
 \midrule
         & \multicolumn{8}{c}{Ranking loss ($\downarrow$)} \\
         \midrule
\system{$h$} & \textbf{0.036} & \textbf{0.035} & \textbf{0.046} & \textbf{0.035} & \textbf{0.046} & \textbf{0.094} & \textbf{0.073} & \textbf{0.172} \\
\camel       & 0.051 & 0.048 & 0.050 & 0.046 & 0.054 & 0.117 & 0.101 & 0.173 \\
\ecc         & 0.086 & 0.078 & 0.078 & 0.086 & 0.093 & 0.176 & 0.103 & 0.208 \\
\br          & 0.091 & 0.088 & 0.085 & 0.097 & 0.101 & 0.177 & 0.131 & 0.190 \\
\rakel       & 0.093 & 0.088 & 0.089 & 0.103 & 0.102 & 0.180 & 0.127 & 0.200 \\
    \bottomrule
    \end{tabular}
        \caption{Comparison of \system{$h$} with the other state-of-the-art models. The best results are in bold.     \label{tab:comparison_gen2}}
\end{table}

\begin{table}[h]
    \centering
    \footnotesize
    \setlength{\tabcolsep}{4pt}
    \begin{tabular}{l c c c c c c c c c c}
    \toprule
    \multirow{2}{*}{Metric} & \multicolumn{5}{c}{Average ranking} & Friedman test  & Wilcoxon test \\
    & \system{$h$} & CAMEL & ECC & BR & RAKEL & p-value & \system{$h$} vs. \camel \\
    \midrule
    Average precision & \textbf{1.43} & 2.31 & 3.34 & 3.56 & 4.34 & $\checkmark \checkmark$ & $\checkmark \checkmark$\\
    Coverage Error & \textbf{1.25} & 2.41 & 3.60 & 3.78 & 3.94 & $\checkmark \checkmark$ & $\checkmark \checkmark$\\
    Hamming loss & \textbf{1.28} & 2.38 & 3.38 & 3.69 & 4.09 & $\checkmark \checkmark$ & $\checkmark \checkmark$ \\
    Multi-label accuracy & \textbf{1.53} & 3.53 & 3.09 & 3.38 & 3.47 & $\checkmark \checkmark$ & $\checkmark \checkmark$\\
    One-error & \textbf{1.44} & 2.22 & 3.50 & 3.56 & 4.28 & $\checkmark \checkmark$ \\
    Ranking loss & \textbf{1.25} & 2.22 & 3.44 & 3.69 & 4.31 & $\checkmark \checkmark$ & $\checkmark \checkmark$\\
    \bottomrule
    \end{tabular}
    \caption{Average ranking for each metric and model, results of the Friedman and Wilcoxon test (the latter deployed to compare the performance of \system{$h$} and \camel). We use $\checkmark$ (resp., $\checkmark \checkmark$) to indicate that the test returned p-value $<0.05$ (resp., $< 0.01$). \label{tab:rankings}}
\end{table}

\subsection{Comparison with the State of the Art}

In order to prove the superiority of our model, we adopted the same methodology presented in~\cite{feng2019} and consider the three well-established MC models and the state-of-the-art MC model tested in ~\cite{feng2019}, which can be characterized by the order of classes correlations they exploit.
Thus, \system{$h$} is compared with 
\begin{enumerate}
    \item BR~\cite{boutell2004}, a first order model which considers each class separately, ignoring class correlations, and
    \item ECC~\cite{read2009},  RAKEL~\cite{tsoumakas2009CorrelationBasedPO} and CAMEL~\cite{feng2019}, which exploit correlations among two or more classes.
\end{enumerate}
BR, ECC, and RAKEL are the well-established MC models, and CAMEL is the current state-of-the-art MC model~\cite{feng2019}.
Since these models are not guaranteed to output predictions that are  coherent with the constraints, we applied {\module} as additional post-processing steps.

Being the first paper on {\cmc} problems, we created 16 real-world {\cmc} datasets, each obtained by enriching a popular and publicly available MC dataset with constraints extracted using the apriori algorithm~\cite{agrawal1994}. The list of the datasets together with a summary of their characteristics is reported in Table \ref{tab:datasets}. The various datasets come from different application domains, in particular:
\begin{enumerate}
    \item {\sc Cal500} and {\sc Emotions} are 2 music classification datasets \cite{turnbull2008CAL500,tsoumakas2008}, 
    \item {\sc Genbase} and {\sc Yeast} are 2 functional genomics datasets~\cite{diplaris2005genbase,eliseeff2001yeast}, 
    \item {\sc Image} and {\sc Scene} are 2 image classification datasets \cite{zhang2007image,boutell2004}, and 
    \item the remaining 10 are text classification datasets \cite{pestian2007,read2008,Srivastava2005,tsoumakas2007rakel}.%
    \footnote{
Link to datasets with constraints: https://github.com/EGiunchiglia/CCN/tree/master/data/
}  
\end{enumerate}
 Furthermore, as it can be seen from Table~\ref{tab:datasets}, they differ significantly both in the number of data points/classes (columns $D$ and $L$) and in the characteristics of the associated sets of constraints. Indeed, we have
 datasets having just a few (one)/many constraints (column $C$), involving a few/many classes in the head (column $H/L$) and in the body, either positively or negatively (columns $B$, $B^+$ and $B^-$). 

 As in the HMC experiments, we applied the same preprocessing to all the datasets. All the categorical features were transformed using one-hot encoding. The missing values were replaced by their mean in the case of numeric features and by a vector of all zeros in the case of categorical ones. All the features were standardized. As in the HMC experiments, we  built $h$ as  a  feedforward  neural  network  with  two  hidden  layers  and  ReLU  non-linearity.  
 To prove the robustness of \system{$h$}, we kept all the hyperparameters fixed except the hidden dimension used for each dataset, which is given in Appendix~\ref{app:hidden_dim_gen}. Such hidden dimensions  were  optimized  over the  validation  sets.  In  all  experiments,  the  loss  was minimized  using  Adam  optimizer  with  batch size equal to 4,  learning rate equal to $10^{-4}$, and  patience  $20$  ($\beta_1 =  0.9, \beta_2 = 0.999$). Since some datasets have very few data points, we set the dropout rate equal to $80\%$ and the weight decay equal to $10^{-4}$.
 As for the HMC case, we retrained \system{$h$} on both training and validation data for the same number of epochs, as the early stopping procedure determined was optimal in the first pass. For each dataset, we run the models 10 times, and the average for each of the metrics is reported in Tables~\ref{tab:comparison_gen1} and~\ref{tab:comparison_gen2}. For ease of presentation, we omit the standard deviations, which for \system{$h$} are in the range $[2.9 \times 10^{-17}, 8.0 \times 10^{-3}]$, proving that it is a very stable model. ECC, BR, and RAKEL were implemented using scikit-multilearn~\cite{scikit-multilearn2017} 
by deploying the logistic regression model as the base classifier. Regarding CAMEL, we used the publicly available authors' implementation,\footnote{Link: https://github.com/hinanmu/CAMEL} with the hyperparameters suggested by the authors. For all the models, we got results comparable to the ones reported in~\cite{feng2019}. 

As it can be seen in Tables~\ref{tab:comparison_gen1} and~\ref{tab:comparison_gen2}, \system{$h$} has the highest number of wins in all metrics. Indeed, as it can be seen in Table~\ref{tab:rankings}, \system{$h$} has the best average ranking in all metrics. We also verified the statistical significance of the results following~\cite{demsar2006} by performing the Friedman test for each metric. 
As reported in Table~\ref{tab:rankings}, the Friedman test returned p-value $<0.01$ for each metric; we could thus proceed with the post-hoc Nemenyi tests, and the resulting critical diagrams are reported in Figure~\ref{fig:nemenyi_gen}. In each diagram, the groups of methods that do not differ significantly (significance level 0.05) are connected through an horizontal line. According to the Nemenyi test, {\system{$h$}}'s performance differs significantly to the performance of all the other models but {\camel} in terms of average precision, coverage error, Hamming loss, one-error and ranking loss, while it differs significantly to all the models (including \camel) in terms of multi-label accuracy. As in the HMC case, we then followed the guidelines given in~\cite{demsar2006,benavoli2016} and performed the Wilcoxon test to compare the difference between {\camel} and \system{$h$}. As reported in the last column of Table~\ref{tab:rankings}, the performances of {\system{$h$}} and {\camel} differ significantly for all metrics but one-error, thus further confirming the superiority of our model.

\begin{figure}
\hfill
\begin{subfigure}{.45\linewidth}
  \includegraphics[width=\linewidth]{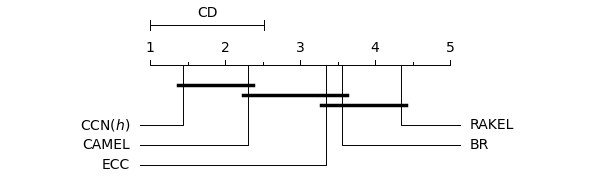}
  \caption{Average precision \label{fig:nemeneyi_avg_prec}}
\end{subfigure}
\hfill
\begin{subfigure}{.45\linewidth}
  \includegraphics[width=\linewidth]{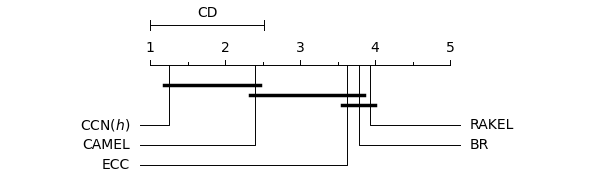}
  \caption{Coverage error\label{fig:nemeneyi_cov_err}}
\end{subfigure}
\hfill

\hfill
\begin{subfigure}{.45\linewidth}
  \includegraphics[width=\linewidth]{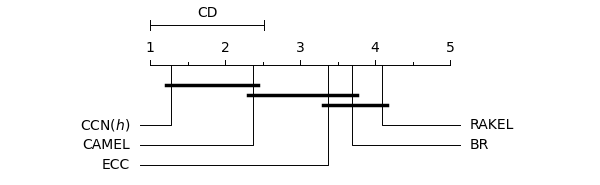}
  \caption{Hamming loss\label{fig:nemeneyi_hloss}}
\end{subfigure}
\hfill
\begin{subfigure}{.45\linewidth}
  \includegraphics[width=\linewidth]{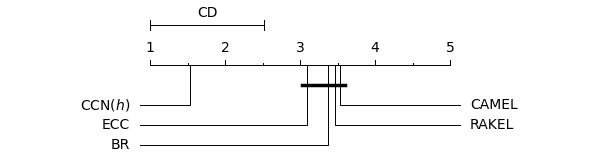}
  \caption{Multi-label accuracy \label{fig:nemeneyi_mlacc}}
\end{subfigure}
\hfill

\hfill
\begin{subfigure}{.45\linewidth}
  \includegraphics[width=\linewidth]{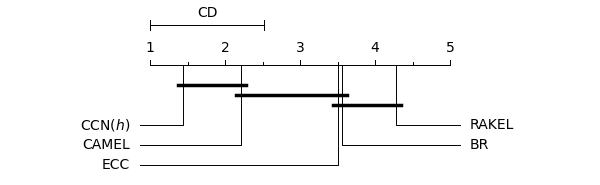}
  \caption{One-error\label{fig:nemeneyi_one_err}}
\end{subfigure}
\hfill
\begin{subfigure}{.45\linewidth}
  \includegraphics[width=\linewidth]{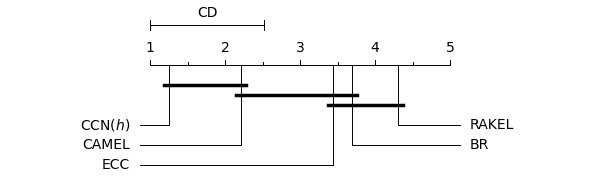}
  \caption{Ranking loss\label{fig:nemeneyi_rnk_loss}}
\end{subfigure}
\hfill
\caption{Critical diagram for each metric obtained with the post-hoc Nemenyi test. \label{fig:nemenyi_gen}}
\end{figure}

\begin{table}[t]
    \centering
    \setlength{\tabcolsep}{3.5pt}
    \begin{tabular}{l c c c c c c c c}
    \toprule
         Model & {\sc Arts}  & {\sc Business} & {\sc Cal500} & {\sc Emotions} & {\sc Enron} & {\sc Genbase} & {\sc Image} & {\sc Medical} \\
         \midrule
         & \multicolumn{8}{c}{Average precision ($\uparrow$)} \\
         \midrule
\system{$h$} & \textbf{0.623} & \textbf{0.904} & \textbf{0.520} & \textbf{0.800} & \textbf{0.704} & \textbf{0.996} & \textbf{0.807} & 0.866 \\
$h^+$        & \textbf{0.623} & 0.903 & 0.519 & 0.796 & \textbf{0.704} & \textbf{0.996} & 0.800 & \textbf{0.868} \\
$h$+\module  & \textbf{0.623} & 0.902 & 0.519  & \textbf{0.800} & \textbf{0.704} & 0.978 & \textbf{0.807} & 0.864 \\
 \midrule
 & \multicolumn{8}{c}{Coverage error ($\downarrow$)} \\
         \midrule
\system{$h$} & \textbf{0.172} & \textbf{0.065} & \textbf{0.734} & 0.315 & \textbf{0.217} & \textbf{0.016} & \textbf{0.187} & \textbf{0.035} \\
$h^+$        & \textbf{0.172} & 0.073 & \textbf{0.734} & 0.319 & \textbf{0.217} & 0.017 & 0.194 & 0.037 \\
$h$+\module  & 0.173 & 0.072 & \textbf{0.734} & \textbf{0.311} & \textbf{0.217} & 0.020  & \textbf{0.187} & 0.037 \\
 \midrule
 & \multicolumn{8}{c}{Hamming loss ($\downarrow$)} \\
\midrule
\system{$h$} & \textbf{0.054} & \textbf{0.023} & \textbf{0.136} & \textbf{0.197} & \textbf{0.046} & \textbf{0.001} & 0.172 & \textbf{0.013} \\
$h^+$        & \textbf{0.054} & 0.026 & \textbf{0.136} & 0.201 & \textbf{0.046} & \textbf{0.001} & 0.172 & 0.015 \\
$h$+\module  & \textbf{0.054} & 0.026 & \textbf{0.136} & 0.198 & \textbf{0.046} & 0.003 & \textbf{0.169} & 0.015 \\
\midrule
 & \multicolumn{8}{c}{Multi-label accuracy ($\uparrow$)} \\
\midrule
\system{$h$} & \textbf{0.238} & \textbf{0.601} & \textbf{0.203} & \textbf{0.534} & \textbf{0.395} & \textbf{0.986} & \textbf{0.488} & 0.589 \\
$h^+$        & \textbf{0.238} & \textbf{0.601} & \textbf{0.203} & 0.512 & \textbf{0.395} & \textbf{0.986} & 0.475 & \textbf{0.591} \\
$h$+\module  & 0.237 & 0.599 & \textbf{0.203} & 0.519 & \textbf{0.395} & 0.935 & 0.487 & 0.587 \\
\midrule
 & \multicolumn{8}{c}{One-error ($\downarrow$)} \\
\midrule
\system{$h$} & \textbf{0.475} & \textbf{0.093} & \textbf{0.113} & \textbf{0.273} & \textbf{0.235} & \textbf{0.000} & \textbf{0.296} & \textbf{0.181} \\
$h^+$        & 0.476 & \textbf{0.093} & \textbf{0.113} & 0.283 & 0.238 & \textbf{0.000} & 0.307 & \textbf{0.181}  \\
$h$+\module  & \textbf{0.475} & \textbf{0.093} & \textbf{0.113} & 0.276 & 0.237 & 0.004 & 0.297 & 0.185 \\
\midrule
 & \multicolumn{8}{c}{Ranking loss ($\downarrow$)} \\
\midrule
\system{$h$} & \textbf{0.115} & \textbf{0.030}  & \textbf{0.173} & 0.161 & \textbf{0.076} & \textbf{0.003} & \textbf{0.159} & \textbf{0.024} \\
$h^+$        & 0.116 & 0.034 & \textbf{0.173} & 0.167 & \textbf{0.076} & \textbf{0.003} & 0.169 & 0.026 \\
$h$+\module  & 0.116 & 0.034 & \textbf{0.173} & \textbf{0.159} & \textbf{0.076} & 0.006 & 0.160  & 0.025 \\
\bottomrule
 \end{tabular}
        \caption{Results of the ablations studies. The best results are in bold.     \label{tab:comparison_abl1}}
 \end{table}
 
 \begin{table}[t]
    \centering
    \setlength{\tabcolsep}{3.5pt}
    \begin{tabular}{l c c c c c c c c}
    \toprule
    Model & {\sc Rcv1S1}   & {\sc Rcv1S2} & {\sc Rcv1S3} & {\sc Rcv1S4} & {\sc Rcv1S5} & {\sc Science} & {\sc Scene} & {\sc Yeast}\\
         \midrule
         & \multicolumn{8}{c}{Average precision ($\uparrow$)} \\
         \midrule
         \system{$h$} & \textbf{0.642} & \textbf{0.666} & \textbf{0.647} & 0.675 & \textbf{0.560}  & \textbf{0.603} & \textbf{0.868} & \textbf{0.768} \\
        $h^+$        & 0.639 & 0.662 & 0.645 & 0.663 & 0.548 & \textbf{0.603} & 0.864 & 0.765 \\
        $h$+\module  & 0.630  & 0.663 & 0.640  & \textbf{0.682} & \textbf{0.560} & 0.602 & 0.867 & 0.767 \\
         \midrule
         & \multicolumn{8}{c}{Coverage error ($\downarrow$)} \\
         \midrule
         \system{$h$} & \textbf{0.092} & \textbf{0.089} & 0.103 & \textbf{0.080}  & \textbf{0.107} & \textbf{0.131} & \textbf{0.077} & \textbf{0.452} \\
$h^+$        & 0.093 & 0.090  & \textbf{0.102} & 0.089 & 0.112 & 0.135 & 0.081 & 0.455 \\
$h$+\module  & \textbf{0.092} & 0.090  & 0.103 & 0.089 & 0.115 & 0.135 & 0.078 & 0.453 \\
\midrule
         & \multicolumn{8}{c}{Hamming loss ($\downarrow$)} \\
         \midrule
         \system{$h$} & \textbf{0.026} & \textbf{0.022} & \textbf{0.024} & \textbf{0.019} & \textbf{0.025} & \textbf{0.031} & \textbf{0.092} & \textbf{0.196} \\
    $h^+$        & 0.027 & 0.023 & 0.025 & 0.022 & 0.027 & 0.032 & \textbf{0.092} & \textbf{0.196} \\
    $h$+\module  & 0.027 & 0.023 & 0.025 & 0.020  & \textbf{0.025} & 0.032 & \textbf{0.092} & \textbf{0.196} \\
\midrule
         & \multicolumn{8}{c}{Multi-label accuracy ($\uparrow$)} \\
         \midrule
        \system{$h$} & \textbf{0.296} & \textbf{0.310}  & 0.303 & \textbf{0.324} & 0.275 & 0.255 & \textbf{0.607} & 0.480  \\
$h^+$        & \textbf{0.296} & 0.309  & \textbf{0.305} & 0.323 & \textbf{0.277} & \textbf{0.257} & 0.603 & \textbf{0.487} \\
$h$+\module  & 0.283 & 0.306 & 0.301 & 0.321 & 0.257 & 0.250  & 0.604 & 0.482 \\
\midrule
         & \multicolumn{8}{c}{One-error ($\downarrow$)} \\
         \midrule
\system{$h$} & \textbf{0.413} & \textbf{0.389} & \textbf{0.405} & 0.379 & \textbf{0.402} & \textbf{0.494} & \textbf{0.224} & 0.234 \\
$h^+$        & \textbf{0.413} & 0.396 & \textbf{0.405} & 0.394 & 0.420  & \textbf{0.494} & 0.227 & 0.234 \\
$h$+\module  & 0.434 & 0.391 & 0.407 & \textbf{0.357} & 0.407 & 0.496 & 0.226 & \textbf{0.231}   \\      
\midrule
         & \multicolumn{8}{c}{Ranking loss ($\downarrow$)} \\
         \midrule
\system{$h$} & \textbf{0.036} & \textbf{0.035} & \textbf{0.046} & \textbf{0.035} & \textbf{0.046} & \textbf{0.094} & \textbf{0.073} & \textbf{0.172} \\
$h^+$        & \textbf{0.036} & 0.036 & \textbf{0.046} & 0.039 & 0.049 & 0.097 & 0.077 & \textbf{0.172} \\
$h$+\module  & 0.037 & 0.037 & 0.047 & 0.039 & 0.049 & 0.096 & 0.074 & 0.174 \\
\bottomrule
 \end{tabular}
        \caption{Results of the ablation studies. The best results are in bold.     \label{tab:comparison_abl2}}
 \end{table}

\begin{table}[h]
    \centering
    \setlength{\tabcolsep}{6pt}
    \begin{tabular}{l c c c c c}
    \toprule
    \multirow{2}{*}{Metric} & \multicolumn{3}{c}{Average ranking} & \multicolumn{2}{c}{Wilcoxon test} \\
    & \system{$h$} & $h^+$ & $h$+\module & \system{$h$} vs. $h^+$ & \system{$h$} vs. $h$+\module \\
    \midrule
    Average precision & \textbf{1.41} & 2.34 & 2.25 & $\checkmark \checkmark$ & $\checkmark$\\
     Coverage error & \textbf{1.38} & 2.41 & 2.22 & $\checkmark \checkmark$ & $\checkmark$ \\
    Hamming loss & \textbf{1.47} & 2.38 & 2.16 & $\checkmark \checkmark$ & $\checkmark$  \\
    Multi-label accuracy & \textbf{1.63} & 1.75 & 2.63 & & $\checkmark \checkmark$ \\
    One-error & \textbf{1.47} & 2.38 & 2.16 & $\checkmark$ & \\
    Ranking loss & \textbf{1.31} & 2.31 & 2.38 & $\checkmark \checkmark$  & $\checkmark \checkmark$  \\
    \bottomrule
    \end{tabular}
    \caption{Average ranking for each metric and model, and results of the Wilcoxon test: $\checkmark$ (resp., $\checkmark \checkmark$) is used to indicate that the Wilcoxon test returned p-value $< 0.05$ (resp., $< 0.01$). \label{tab:rankings_ablwilc}}
\end{table}

\subsection{Ablation Studies}

As in the HMC case, in order to analyze the impact of both {\module} and \loss, we compared the performance of \system{$h$} against the performance of $h^+$, i.e., $h$ with the enforcement of the constraints done as a post-processing step, and $h+\module$, i.e., $h$ with {\module} built on top. Both these models were trained using the standard binary  cross-entropy loss. The results of the ablation studies for each metric and each dataset are given in Tables~\ref{tab:comparison_abl1} and~\ref{tab:comparison_abl2}, while the average ranking for each metric can be found in Table~\ref{tab:rankings_ablwilc}. As it can be seen in Table~\ref{tab:rankings_ablwilc}, {\system{$h$}} has the highest average ranking for all metrics. Furthermore, we check the statistical significance of our results through the Wilcoxon test, whose results are reported in the last two columns of Table~\ref{tab:rankings_ablwilc}. As it can be seen, \system{$h$} performs significantly better than $h^+$ for all metrics but multi-label accuracy. On the other hand, \system{$h$} performs significantly better than $h+\module$ for all metrics but one-error.

\section{Related Work}\label{sec:rel_work}

In this paper, we introduced {\cmc} problems, i.e., MC problems in which every prediction must satisfy a given set of hard constraints expressed as normal rules. HMC problems are special cases of {\cmc} in which the body of each constraint is a single class and with the additional restriction that the constraint graph associated to the set of constraints is acyclic.

We divide this section on the related work into two parts:
in the first part, we focus on the literature in the HMC field, while in the second part, we present and discuss the previous works that have already dealt with the problem of imposing more expressive constraints on MC problems.

\subsection{Hierarchical Multi-Label Classification}

In the literature, HMC methods are traditionally divided into local and global approaches \cite{silla2011}. 

Local approaches decompose the problem into smaller classification ones, and then the solutions are combined to solve the main task. Local approaches  can be further divided based on %
the strategy that they deploy to decompose the main task. If a method trains a different classifier for each level of the hierarchy, then we have a {\sl local classifier per level}, as in \cite{cerri2011,cerri2014,cerri2016,li2018,zou2019}. The works \cite{cerri2011,cerri2014,cerri2016} are extended by \citeauthor{cerri2018}~\citeyear{cerri2018}, where \hmcnr{} and \hmcnf{} are presented. Since  \hmcnr{} and \hmcnf{} are trained with both a local loss and a global loss, they are considered hybrid local-global approaches. 
If a method trains a classifier for each node of the hierarchy, then we have a {\sl local classifier per node}. In \cite{cesabianchi2006}, a linear classifier is trained for each node with a loss function that captures the hierarchy structure. On the other hand, in~\cite{feng2018}, one multi-layer perceptron for each node is deployed. A different approach is proposed in~\cite{kwok2011}, where kernel dependency estimation is employed to project each class to a low-dimensional vector. To preserve the hierarchy structure, a generalized condensing sort and select algorithm is developed, and each vector is then learned singularly using ridge regression. Finally, if a method trains a different classifier per parent node in the hierarchy, then we have a {\sl local classifier per parent node}. For example, \citeauthor{kulmanov2018}~\citeyear{kulmanov2018} propose to train a model for each sub-ontology of the Gene Ontology, combining features automatically learned from the sequences and features based on protein interactions. \citeauthor{xu2019aaai}~\citeyear{xu2019aaai}, instead, try to solve the overfitting problem typical of local models by representing the correlation among the classes by the class distribution, and then training each local model to map data points to class distributions.

Global methods consist of single models able to map objects with their corresponding classes in the hierarchy as a whole. A well-known global method is {\sc Clus-HMC} \cite{vens2008}, consisting of a single predictive clustering tree for the entire hierarchy. This work is extended by \citeauthor{schietgat2010}~\citeyear{schietgat2010}, who propose \ens{}: an ensemble of {\sc Clus-HMC}. In \cite{masera2018}, a neural network incorporating the structure of the hierarchy in its architecture is proposed. While this network makes predictions that are coherent with the hierarchy, it also makes the assumption that each parent class is the union of the children. \citeauthor{borges2012}~\citeyear{borges2012} propose {\sl competitive neural networks}, whose architecture replicates the hierarchy. The name of these networks comes from the fact that the neurons in the output layer compete with each other to be activated.

If we move the focus on how the satisfaction of the constraints is guaranteed, HMC models can be divided in: 
\begin{enumerate}
    \item approaches that satisfy the constraints by construction, and 
 \item approaches that require a post-processing step to enforce the constraints. 
 \end{enumerate}

 In the first category, we can find methods such as {\clus}~\cite{vens2008}, in which each constraint $A_1 \to A$ is satisfied by properly setting the thresholds for both $A$ and $A_1$.
 {\ens}, the ensemble of {\clus}, computes the average of all class vectors predicted by the trees in the ensemble and then applies the threshold mechanism of {\clus}. In this way, its predictions always satisfy the constraints. The model proposed in~\cite{kwok2011}, in order to preserve the hierarchy structure, develops a generalized condensing sort and select algorithm, while \citeauthor{cesabianchi2006}~\citeyear{cesabianchi2006} evaluate the node classifiers in a top-down fashion, thus not making a prediction at all for the descendants of a node that has not been predicted. 

  Many of the neural-network-based models belong to the second category. A common policy to enforce the satisfaction of $A_1 \to A$ is to force the output for class $A_1$ to be smaller than or equal to the output for $A$ (see, e.g.,  \cite{cerri2011,cerri2014,cerri2016,cerri2018}). However, other solutions are possible. For example, \citeauthor{borges2012}~\citeyear{borges2012} associate to each data point an initial set of classes which is then extended to include all their ancestors in the hierarchy, while \citeauthor{feng2018}~\citeyear{feng2018} apply a post processing method based on  Bayesian networks in order to guarantee that the results are coherent with the hierarchy constraints.
For a detailed overview on the many different policies that can be used to impose the hierarchy constraints as a post-processing step, see \cite{obozinski2008}.

\subsection{More expressive constraints }

When dealing with more expressive constraints, researchers have mostly focused on the problem of exploiting them to improve their models and/or to deal with data scarcity, 
curiously neglecting the problem of guaranteeing their satisfaction.

If we focus on constraints expressed as logic rules, then we can find works such as  \cite{hu2016harnessing}, in which the authors introduce an iterative method to embed structured information expressed by first order logic (FOL) formulas into the weights of different kinds of deep neural networks. 
At each step, they consider a {\sl teacher network}  based on the set of FOL rules to train a {\sl student network} to fit both supervisions and logic rules.  
The work has been later extended to jointly learn the structure of the rules and their weights \cite{hu2016deep}.
A different approach is considered in \cite{li2019augmenting}, where a new framework is introduced to augment a neural network assigning semantics via logical rules to its neurons. Indeed, some neurons are associated with logical predicates, and then their activation is modified depending on the activation of the neurons corresponding to predicates that co-occur in the same rules. 
Many works embed logical constraints into penalty functions to formulate a learning problem, see e.g., \cite{diligenti2017semantic,donadello2017logic,xu2018semantic}. These works generally consider a fuzzy relaxation of FOL  formulas to get a differentiable loss function that may be %
minimized by gradient descent. However, as the loss function is minimized, there is no guarantee that the constraints are fully satisfied. 

If we further broaden our view to  more general constraints, i.e., to any constraint that  enforces some correlation on the outputs of the model, then we can find a very wide literature, in which many works have shown that exploiting the background knowledge expressed by the constraints can bring some benefits.
For example, \citeauthor{ermon2017}~\citeyear{ermon2017} exploit the background knowledge coming from known laws of physics to train neural networks without any labeled example. \citeauthor{chen2015learning}~\citeyear{chen2015learning} combine Markov random fields with deep neural networks to express some context correlation in an image segmentation task. Similarly, \citeauthor{huang2015bidirectional}~\citeyear{huang2015bidirectional} consider conditional random fields on top of a bidirectional LSTM~\cite{hochreiter1997LSTM} to enforce some statistical co-occurrencies (n-grams) among the network predictions, achieving state-of-the-art performance on POS and chunking tasks. 
Some of these works have also tried to impose hard constraints on learning algorithms, however, dealing effectively with hard constraints commonly requires a specific optimization schema, e.g.,  by considering a sequence of problems where the hard constraints are replaced by soft ones associated with larger and larger
penalties, (see \cite{bertsekas2014constrained,fletcher2013practical} for a detailed dissertation on optimization methods), and convergence is only guaranteed under suitable conditions; see \cite{luenberger1997optimization}. \citeauthor{gnecco2014learning}~\citeyear{gnecco2014learning} consider different learning problems where both soft and hard constraints are taken into account at the same time. Soft constraints are added to the penalty function, while the optimal solution (if it exists) is required to satisfy a system of (in)equalities corresponding to the hard constraints. A newly devised approach to asymptotically satisfy hard constraints is considered in \cite{farina2019asynchronous}, where the authors use the distributed asynchronous method of multipliers to solve the optimization problem.

\section{Conclusion} \label{sec:concl}

In this paper, we introduced and dealt with {\cmc} problems, i.e., on MC problems with hard constraints expressed as normal logic rules. 
We first focused on the special case represented by HMC problems, and proposed \hmcsys{$h$}, a novel model based on neural networks that (i) is able to leverage the hierarchical information to learn when to delegate the prediction on a superclass to one of its subclasses, (ii) produces predictions coherent by construction, and (iii) outperforms current state-of-the-art models on 20 commonly used real-world HMC benchmarks, and (iv) it can be easily implemented on GPUs using standard libraries. 
We then considered the general case and proposed \system{$h$}, which has four distinguishing features: (i) its predictions are always coherent with the given constraints,
(ii) it can be implemented on GPUs using standard libraries,
(iii) it extends {\hmcsys{$h$}}, and thus outperforms the state-of-the-art HMC models on HMC problems, and (iv) it outperforms the state-of-the-art MC models on 16 {\cmc} problems obtained adding automatically extracted constraints to well-known MC problems.

This work opens the path to many different lines of research.
In general, the basic idea behind this paper is to (i) incorporate the constraints in neural networks models  guaranteeing their coherency, and (ii) exploit the constraints to improve performance. In this paper, we focussed on MC problems with hard logical  constraints on the output, but indeed the same idea can be applied also to other problems, and 
it could be interesting to investigate (i) whether it is possible to impose even more expressive constraints (e.g., involving also the input), (ii) how to incorporate both soft and hard constraints, and (iii) how to impose hard constraints on regression, binary classification, and multi-class classification problems.

\subsection*{Acknowledgments} This paper is a revised and substantially extended version of the paper \cite{giunchiglia2020}. We  thank Lei Feng, Francesco Giannini and Marco Gori for useful discussions. 
Eleonora Giunchiglia is supported by
the EPSRC under the grant EP/N509711/1 and by an Oxford-DeepMind Graduate Scholarship. This
work was also supported by the Alan Turing Institute under the EPSRC grant EP/N510129/1 and
by the AXA Research Fund. We also acknowledge the use of the EPSRC-funded Tier 2 facility
JADE (EP/P020275/1) and GPU computing support by Scan Computers International Ltd.

\bibliographystyle{theapa}
\bibliography{bibliography}

\clearpage

\appendix

\section{HMC - Experimental Analysis Details}\label{app:hidden_dim}
\begin{table}[ht]
    \centering
    \begin{tabular}{l c c c}
        \toprule
         {\sc Dataset} & Hidden Dimension &  Learning Rate & Time per batch (ms) \\
         \midrule
         {\sc Cellcycle FUN} & 500 & $10^{-4}$ & 2.0\\
         {\sc Derisi FUN} & 500 & $10^{-4}$ & 2.0\\
         {\sc Eisen FUN} & 500 & $10^{-4}$ & 1.7\\
         {\sc Expr FUN} & 1000 & $10^{-4}$ & 1.9\\
         {\sc Gasch1 FUN} & 1000 & $10^{-4}$ & 2.0\\
         {\sc Gasch2 FUN} & 500 & $10^{-4}$ & 2.8\\
         {\sc Seq FUN} & 2000 & $10^{-4}$ & 2.0\\
         {\sc Spo FUN} & 250& $10^{-4}$ & 1.6\\
         \midrule
         {\sc Cellcycle GO} & 1000 & $10^{-4}$ & 2.4 \\
         {\sc Derisi GO} & 500 & $10^{-4}$ & 2.5\\
         {\sc Eisen GO} & 500 & $10^{-4}$ &3.4\\
         {\sc Expr GO} & 4000 & $10^{-5}$  & 3.9\\
         {\sc Gasch1 GO} & 500 & $10^{-4}$ & 2.5\\
         {\sc Gasch2 GO} & 500 & $10^{-4}$ & 2.8\\
         {\sc Seq GO} & 9000 & $10^{-5}$ & 2.6 \\
         {\sc Spo GO} & 500 & $10^{-4}$ & 3.3\\
         \midrule
         {\sc Diatoms} & 2000 & $10^{-5}$ & 2.0\\
         {\sc Enron} & 1000 & $10^{-5}$ & 3.6\\
         {\sc Imclef07a} & 1000 & $10^{-5}$ & 3.4\\
         {\sc Imclef07d} & 1000 & $10^{-5}$ & 2.9\\
         \bottomrule
    \end{tabular}
     \caption{Hidden dimension used for each dataset, learning rate used for each dataset, and average inference time per batch in milliseconds (ms). Average computed over 500 batches for each dataset.}
    \label{tab:hidden_dim}
\end{table}
In this section, we provide more details about the conducted experimental analysis for HMC problems. 
As stated in the paper, across the different experiments, we kept all hyperparameters fixed with the exception of the hidden dimension and the learning rate, which are reported in the first two columns of Table~\ref{tab:hidden_dim}. The other hyperparameters were determined by searching the best hyperparameters configuration on the Funcat datasets; we then took the configuration that led to the best results on the highest number of datasets. The hyperparameter values taken in consideration were: 
(i) learning rate: $[10^{-3}, 10^{-4}, 10^{-5}]$, (ii) batch size: $[4, 64, 256]$, (iii) dropout: $[0.6, 0.7]$, and (iv) weight decay: $[10^{-3}, 10^{-5}]$. Concerning the hidden dimension, we took into account all possible dimensions from 250 to 2000 with step equal to 250, and from 2000 to 10000 with step 1000. 
The last column of Table~\ref{tab:hidden_dim} shows the average inference time per batch in milliseconds. The average is computed over 500 batches for each dataset.
All experiments were run on an Nvidia Titan Xp with 12 GB  memory. 

\section{\cmc{} - Experimental Analysis Details}\label{app:hidden_dim_gen}

\begin{table}[ht]
    \centering
    \begin{tabular}{l c}
    \toprule
         Dataset &  Hidden Dimension \\
    \midrule
         {\sc Arts} & 4000 \\
         {\sc Business} & 4000 \\
         {\sc Cal500} & 100 \\
         {\sc Emotions} & 100\\
         {\sc Enron} & 2500\\
         {\sc Genbase} & 5000\\
         {\sc Image} & 1000 \\
         {\sc Medical} & 800 \\
         {\sc Rcv1Subset1} & 600 \\
         {\sc Rcv1Subset2} & 1500 \\
         {\sc Rcv1Subset3} & 4000\\
         {\sc Rcv1Subset4} & 1000\\
         {\sc Rcv1Subset5} & 4000 \\
         {\sc Science} & 2000 \\
         {\sc Scene} & 1500 \\
         {\sc Yeast} & 4000\\
    \bottomrule
    \end{tabular}
 \caption{Hidden dimension used for each dataset.}
 \label{tab:hidden_dim_gen}
\end{table}

In this section, we provide more details about the conducted experimental analysis for \cmc{} problems. 
As stated in the paper, across the different experiments, we kept all hyperparameters fixed with the exception of the hidden dimension, which are reported in Table~\ref{tab:hidden_dim_gen}. The other hyperparameters were determined by searching the best hyperparameters configuration on the validation sets; we then took the configuration that led to the best results on the highest number of datasets. The hyperparameter values taken in consideration were: 
(i) learning rate: $[10^{-3}, 10^{-4}, 10^{-5}]$, (ii) batch size: $[4, 64, 256]$, (iii) dropout: $[0.6, 0.7, 0.8]$, and (iv) weight decay: $[10^{-3}, 10^{-4}, 10^{-5}]$. Concerning the hidden dimension, we took into account all possible dimensions from 100 to 1000 with step equal to 100, and from 1000 to 5000 with step 500. Again, all experiments were run on an Nvidia Titan Xp with 12 GB  memory. 

\end{document}